\documentclass[11pt]{article}
\usepackage[margin=1.1in,a4paper]{geometry}

\usepackage[utf8]{inputenc}

\usepackage{authblk}

\usepackage{graphicx}
\usepackage{subfig}

\usepackage{booktabs}
\usepackage{diagbox}

\usepackage{amsmath}
\usepackage{amsthm}
\usepackage{amsfonts}
\usepackage{gensymb}
\usepackage{dsfont}
\usepackage{amssymb}
\usepackage{centernot}
\usepackage{mathtools}
\usepackage{stmaryrd }
\usepackage{relsize}

\usepackage{hyperref}
\usepackage{xurl}

\usepackage{textcomp}
\usepackage{breqn}
\usepackage{enumerate}
\usepackage[dvipsnames]{xcolor}

\usepackage{nomencl}
\usepackage[round]{natbib}

\usepackage[toc,page]{appendix}

\newtheorem{theorem}{Theorem}[section]
\newtheorem{corollary}{Corollary}[theorem]
\newtheorem{lemma}[theorem]{Lemma}
\theoremstyle{definition}
\newtheorem{definition}{Definition}[section]
\newtheorem{remark}{Remark}[section]
\newtheorem{assumption}{Assumption}
\newtheorem{proposition}{Proposition}[section]

\newcommand\ie{\textit{i.e.}, }
\newcommand\wrt{\textit{w.r.t }}

\newcommand\eg{\textit{e.g.}, }

\newcommand{\quoting}[1]{\textit{``#1''}}
\newcommand{\iid}{\text{i.i.d.~}}

\newcommand{\transpose}[1]{#1^\intercal}

\newcommand\NonLin{\texttt{NonLin}}
\newcommand\MatMul{\texttt{MatMul}}
\newcommand\Moment{\texttt{Moment}}

\newcommand\ZNonLin{\texttt{ZNonLin}}
\newcommand\ZInit{\texttt{ZInit}}
\newcommand\ZDot{\texttt{ZDot}}
\newcommand\ZHat{\texttt{ZHat}}
\newcommand\ZMatMul{\texttt{ZMatMul}}
\newcommand\ZMoment{\texttt{ZMoment}}

\newcommand\hatW{\widehat{W}}
\newcommand\hatZ{\widehat{Z}}

\newcommand\tildeh{\Tilde{h}}
\newcommand\tildex{\Tilde{x}}
\newcommand\tildef{\Tilde{f}}

\newcommand\smallcirc{\mathrel{\mathsmaller{\mathsmaller{\circ}}}}

\newcommand{\scalarlim}[1]{\overset{\smallcirc}{#1}}
\newcommand\dotZ{\overset{\boldsymbol{.}}{Z}\,}

\newcommand\muP{\text{$\mu$P}}
\newcommand\relu{\text{ReLU}}

\newcommand\IP{\text{IP}}

\newcommand\HP{\text{HP}}
\newcommand\HPZ{\text{HPZ}}

\hypersetup{
    colorlinks=true,
    citecolor=blue,
    linkcolor=blue,
    filecolor=magenta,      
    urlcolor=magenta,
}

\numberwithin{equation}{section}

\begin{document}

\title{Training Integrable Parameterizations of Deep Neural Networks in the Infinite-Width Limit}

\author[,1]{Karl Hajjar\thanks{Corresponding author: \texttt{hajjarkarl@gmail.com}}}
\author[2]{Lénaïc Chizat}
\author[1]{Christophe Giraud}
\affil[1]{Laboratoire de Mathématiques d'Orsay\\ Université Paris-Saclay\\ 91405 Orsay, France}
\affil[2]{Institut de Mathématiques\\
       École Polytechnique Fédérale de Lausanne\\
       Lausanne, Switzerland}

\maketitle

\begin{abstract}%
    To theoretically understand the behavior of trained deep neural networks, it is necessary to study the dynamics induced by gradient methods from a random initialization. However, the nonlinear and compositional structure of these models make these dynamics difficult to analyze. To overcome these challenges, large-width asymptotics have recently emerged as a fruitful viewpoint and led to practical insights on real-world deep networks. For two-layer neural networks, it has been understood via these asymptotics that the nature of the trained model radically changes depending on the scale of the initial random weights, ranging from a kernel regime (for large initial variance) to a feature learning regime (for small initial variance). For deeper networks more regimes are possible, and in this paper we study in detail a specific choice of ``small'' initialization corresponding to \quoting{mean-field} limits of neural networks, which we call \emph{integrable parameterizations} (IPs).
    
    First, we show that under standard i.i.d.~zero-mean initialization, integrable parameterizations of neural networks with more than four layers start at a stationary point in the infinite-width limit and no learning occurs. We then propose various methods to avoid this trivial behavior and analyze in detail the resulting dynamics. In particular, one of these methods consists in using large initial learning rates, and we show that it is equivalent to a modification of the recently proposed maximal update parameterization \muP. We confirm our results with numerical experiments on image classification tasks, which additionally show a strong difference in behavior between various choices of activation functions that is not yet captured by theory. 
\end{abstract}

\section{Introduction}
While artificial neural networks routinely achieve state-of-the art performance in various real-world machine learning tasks, it is still a theoretical challenge to understand why and under which conditions they perform so well. The training algorithm---typically a variant of stochastic gradient descent (SGD) with random initialization---plays a central role in this performance but is difficult to analyze for general neural network architectures, because of their highly non-linear and compositional structure. Large-width asymptotics, which have previously been considered for other purposes~\citep{neal1995bayesian, bengio2005convexNN}, have recently been proposed to overcome some of these difficulties and have brought numerous insights on the training behavior  of neural networks~\citep{nitanda2017stochastic, mei2018mean, jacot2018ntk, rotskoff2018parameters, chizat2018global, sirignano2020mean}.

One of these insights is that the magnitude of the random weights at initialization has a dramatic impact on the learning behavior of neural networks~\citep{chizat2019lazyTraining}. For two-layer networks and with suitable learning rates, initializing the output layer weights with a standard deviation of $1/m$, where $m$ is the width of the network, leads to \emph{feature learning} when $m$ is large, while the same network initialized with a standard deviation of $1/\sqrt{m}$ leads to the Neural Tangent Kernel (NTK) regime, a.k.a.~\emph{lazy} regime, where the network simply learns a linear predictor on top of fixed features. This observation suggests that \emph{parameterizations}---that is, the choice of the scaling factors, with the width $m$, of the initial magnitude and of the learning rates of each layer of a neural network---are of fundamental importance in the theory of neural networks. While standard deep learning  packages offer various choices of scale at initialization \citep{glorot2010, he2015}, those have been designed with the sole criterion in mind to have a non-vanishing first forward and backward passes for arbitrary depths. Theory now offers the tools to explore a larger space of parameterizations and study their dynamics beyond the first forward and backward passes in the infinite-width limit. 

With more than two layers, the categorization of parameterizations is more subtle and there are disparate lines of work. On the one hand, some parameterizations still lead to the kernel regime, which is subject to an intense research activity (\eg \citealp{jacot2018ntk, jacot2019orderNTK, allen2019learning,du2019gradient, arora2019fine, geiger2020scalingNTK, geiger2020disentanglingNTK, yang2020tp2}). Since this regime reduces to learning a linear predictor on top of fixed features in the large width limit, this parameterization is of limited relevance to understand representation learning in networks used in practice (although it should be noted that non-asymptotic analyses reveal interesting effects, \eg \citealp{hanin2019finite}). On the other hand, there is a growing literature around parameterizations where weights are initialized with a standard deviation of $1/m$ (except for the first layer). These are often called \quoting{mean-field} models but we prefer to call them \emph{integrable parameterizations} (IPs) in this work\footnote{For deep neural networks, it is somewhat arbitrary to associate the term \emph{mean-field} with a specific choice of scaling so we believe that this term lacks precision when it comes to discussing various parameterizations.}, in reference to the fact that sums of $m$ terms with standard deviation of order of  $1/m$ are absolutely convergent. There already exists mathematical tools to describe the evolution of the parameters of IPs in the infinite-width limit but they are not fully satisfactory to understand the properties of the learned function in the standard setting used in practice (see review in Section~\ref{sec:related-work}).

Going beyond the dichotomy between the scales $1/m$ and $1/\sqrt{m}$, \citet{yang2020featureLearning} have exhibited, using a technique called the \emph{Tensor Program}~\citep{yang2019tp1, yang2020tp2, yang2020tp3}, a general categorization of parameterizations, in particular between those which allow feature learning and those which do not. As a result from their analysis, they singled out a \emph{maximal update} parameterization \muP~where, as for the NTK parameterization, the intermediate layers' weights are initialized with a standard deviation of $1/\sqrt{m}$, but the last layer weights are initialized with a standard deviation of $1/m$: they show that with appropriate learning rates, this leads to maximal feature learning (in a certain sense). This parameterization had been previously considered in~\citep{geiger2020disentangling} where the authors study empirically the effect of the \emph{scale}~\citep{chizat2019lazyTraining} on learning. 

In \citep{yang2020featureLearning}, IPs have been excluded from the analysis on the basis that they are \emph{trivial}: if one follows the usual training procedure---which we refer to as \emph{Naive-IP}---the network starts on a stationary point in the infinite-width limit and the learned function remains at its initial value. 

\subsection{Contributions}\label{sec:contrib}
Our goal is to draw connections between the various lines of research discussed above, and to improve our understanding of integrable parameterizations: when and why are they trivial? How can we avoid triviality and actually learn features? What are the salient properties of the resulting networks in the infinite-width limit? To answer these questions rigorously, we leverage the Tensor Program technique developed in~\citep{yang2019tp1, yang2020tp2, yang2020tp3, yang2020featureLearning}. Specifically, our contributions are the following: 
\begin{itemize}
    \item We first show in \autoref{th:informal-no-constant-lr} that with learning rates constant in time, the functions learned using SGD for integrable parameterizations of neural networks with four layers or more either remain at their value at initialization or explode in the infinite-width limit when the weights are initialized using the standard zero-mean i.i.d.~schemes used in practice. 
    
    \item We show in \autoref{th:non-trivial-ipllr} that using large learning rates, which grow as a power of $m$, for the first gradient step---and that step only---allows SGD to escape the initial stationary point for integrable parameterizations and to initiate a non-trivial learning phase. In fact, we prove in \autoref{th:hpz-ipllr-equivalence} that the resulting dynamic is equivalent to a modification of the dynamic of \muP\, where, after the first gradient step, one subtracts the initial weights from the learned weights of the intermediate layers.
    
    \item We study two alternative ways to escape the initial stationary point for integrable parameterizations and analyze the corresponding dynamics. Removing the scale factor in $1/m$ on the bias terms allows to escape the initial stationary point when using moderately large initial learning rates. A drawback of the resulting dynamics is that its updates only depend weakly on the input data (see \autoref{th:ip-bias-no-input-dependence-informal}). On the other hand, using a non-centered law also allows to escape the initial stationary point for \iid initializations without having to use large learning rates, but the dynamics become degenerate as the updates of the entries of the weight matrix in a given layer are all equal to the same fixed quantity in the infinite-width limit (see \autoref{th:ip-non-centered-deterministic-informal}). We investigate numerically the performance of those two models and show that the aforementioned behaviors are detrimental to learning. 
\end{itemize}

The code to reproduce the results of the numerical experiments can be found at: \\ \url{https://github.com/karl-hajjar/wide-networks}.

\subsection{Related Work}\label{sec:related-work}
While the study of infinitely wide neural networks has a long history~\citep{barron1993universal, neal1995bayesian, neal1996priors, kurkova2001bounds, mhaskar2004tractability, bengio2005convexNN, bach2017breaking}, it is only recently that their training dynamics have been investigated. Two-layer neural networks with IP enjoy some global convergence properties~\citep{chizat2018global} and favorable guarantees in terms of generalization~\citep{bach2017breaking, chizat2020implicitBias}. Going beyond two layers, \citet{nguyen2020rigorous} and~\citet{pham2020global} study the infinite-width limit of IPs and also prove global convergence results for networks with three layers or more. However, those results hold for standard zero-mean i.i.d.~initialization schemes only for networks with two or three layers (which is consistent with the results of Section~\ref{sec:trivial-deep-ip}): for deeper networks they require non-standard (correlated) initializations.

Several other works describe the infinite-width limit of multi-layer IPs: \citet{araujo2019mfLimit} characterize the infinite-width dynamics via a model of McKean-Vlasov type, for which they prove existence and uniqueness of solutions, and ~\citet{sirignano2021mfAnalysis} prove a global convergence result for three-layer networks. They take the number of units in each layer to infinity sequentially and describe the dynamics of the limit as a system of differential equations over the weights/parameters. On the other hand,~\citet{fang2020modeling} take the infinite-width limit for all layers at once (as in~\citealp{araujo2019mfLimit, nguyen2020rigorous, pham2020global}) and describe the resulting dynamics as an ODE over functions of the features (pre-activations) of the network. 
It is interesting to note that~\citet{araujo2019mfLimit, sirignano2021mfAnalysis, pham2020global} all discuss the difficulties associated with describing the dynamics of the infinite-width of IPs with more than three layers. As noted in~\citep{araujo2019mfLimit}, and appropriately addressed by~\citet{nguyen2020rigorous, fang2020modeling, sirignano2021mfAnalysis}, there is a separation of time scales as soon as there are two hidden layers or more, where the gradients of the intermediate layers appear to scale as $m^{-2}$ whereas the gradients of the input and output layers appear to scale as $m^{-1}$, requiring separate learning rate values which can make the analysis of the infinite-width limit more difficult. 

In a separate line of work,~\citet{yang2020featureLearning} provide with the Tensor Program a theoretical tool to describe the infinite-width limit of different parameterizations of neural networks and categorize them between feature learning and kernel-like behavior. However, IPs with three layers or more are left out of this categorization. Using the same tools, we show that IPs with more than four layers are indeed trivial at any time step if the initial learning rates are not appropriately scaled with $m$ under standard zero-mean i.i.d.~initializations. This closes the gap with~\citep{nguyen2020rigorous} which proves global convergence results for IPs with two or three layers initialized using those standard schemes. We also demonstrate in Section~\ref{sec:llr-learning} how scaling the initial learning rates appropriately allows to properly train an IP---inducing a feature learning regime as defined in~\citep{yang2020featureLearning}---and connect the resulting model with a version of the maximal update parameterization \muP~\citep{yang2020featureLearning} where the initial weights of the intermediate layers are replaced by zero in the first update.

The setting where non-centered i.i.d.~initialization laws are used is covered in~\citep{nguyen2020rigorous}, where it is shown that a certain collapse phenomenon occurs, namely that the updates of the entries of the weight matrix in a given layer are all equal to the same deterministic quantity in the large-width limit. We obtain a similar result in Section~\ref{sec:ip-non-centered} using different theoretical tools.

\paragraph{Tensor Program \textit{vs.}~other formalisms.} In contrast to prior literature on IPs, we do not use the description of the infinite-width limit as a composition of integral transforms. With the standard (centered i.i.d.)~initializations considered in this paper, that description does not offer much insight about the limit beyond the fact that it starts on a stationary point. In order to escape this initial stationary point, we propose in this paper to amplify the random fluctuations around the limit using large initial learning rates. The strength of the Tensor Program formalism~\citep{yang2019tp1, yang2020tp2, yang2020tp3, yang2020featureLearning} is precisely that it is able to describe rigorously the magnitudes of these fluctuations and allows us to analyze the functions learned with various choices of learning rates. This formalism relies on techniques initiated in the statistical physics literature~\citep{bayati2011dynamics,bolthausen2014iterative} that use the Gaussian conditioning technique to describe the behavior of algorithms (such as message passing) involving random matrices and nonlinearities.

\subsection{Organisation of the Paper and Notations}\label{sec:setup-notations}

We define and analyze integrable parameterizations in Section~\ref{sec:ip-trivial} and show that they are trivial for common choices of learning rates. In Section~\ref{sec:llr-learning}, we describe how a specific scaling of the learning rates allows to escape the initial stationary point, and further investigate the connection between IPs with large initial learning rates and \muP. In Section~\ref{sec:alternative-ip}, we present two alternative modifications of IPs to escape the initial stationary point and discuss the impact of each on the learning dynamics. 

We defer all the rigorous proofs of our theoretical results to the Appendix, so as to make the core message of our work stand out more clearly, and keep the flow of the results structured and easy to follow. Among other things, this prevents us from diving too deep into the Tensor Program formalism and calculations (which can be somewhat tedious and abstruse) in the main part of our work. Most proofs require heavy inductions on the time step $t$, and proving the induction step itself often involves inductions on $l$ in the forward pass (from $l=1$ to $l=L$) and in the backward pass (from $l=L$ to $l=1$). Breaking down all these steps makes for a lengthy Appendix, but the ideas of the proof are relatively straightforward, only their proper formal writing is tedious.

Throughout the paper, for two integers $p, q$, we denote by $[p]$  the set $\{1, \ldots, p\}$ and  by $[p, q]$ the set $\{p, \ldots, q\}$. We  write  $u \odot v$ for the Hadamard (\ie element-wise) product of  two vectors $u$ and $v$. 
We use Landau notations for  comparing two real sequences $(u_m)$ and $(v_m)$: we write $u_m=O(v_m)$ when there exists a constant $C>0$ such that $|u_m|\leq C |v_m|$ for large enough $m$, and $u_m=\Theta(v_m)$ when we both have $u_m=O(v_m)$ and $u_m=O(v_m)$. We similarly use the $O$ (respectively $\Theta$) notation for two sequences of real-valued random variables $(u_m)$ and $(v_m)$ when, almost surely, $u_m = O(v_m)$ (respectively $u_m = \Theta(v_m)$).

\section{General Setting}

In this section, we introduce the general setting we consider for this work, as well as the corresponding notations. We also define precisely the notion of parameterization of a neural network and discuss examples of parameterizations commonly found in the literature. 

\subsection{Network and Data}

\paragraph{Training data.} We consider a training dataset $\left\{(\xi^{(i)}, y^{(i)}) \right\}_{i \in [n]}$ containing $n$ (input, output) pairs with $\xi^{(i)}\in \mathbb{R}^d$ and $y^{(i)}\in \mathbb{R}$. We will use $\xi^{(i)}$ or $y^{(i)}$ when we refer to the $i$-th sample in the training dataset, but use $\xi_t$ and $y_t$ to denote the sample(s) fed to train the network at time step $t$, that is for the $(t+1)$-th step of optimization.

\paragraph{Width and depth.} Throughout this work,
we consider a feed-forward  fully connected neural network, with $L$ hidden layers and a common width $m$.  The total number of layers, \ie weight matrices and bias vectors will thus be $L+1$, and most of our results are concerned with four or more layers, that is $L \geq 3$, and in the limit $m \rightarrow \infty$. The integer $l \in [L+1]$ will always be used to index the layers of a network, and we call the \textbf{intermediate layers} of a network the layers indexed by $l \in [2, L]$ (\ie excluding input and output layers). 

\paragraph{Activation function.} 
We assume that all the neurons in the network share the same activation function $\sigma:\mathbb{R}\to \mathbb{R}$. The activation is always taken entry-wise and for any vector $h \in \mathbb{R}^m$, we denote by $\sigma(h)$ the vector $(\sigma(h_p))_{p \in [m]}\in \mathbb{R}^m$. 

\paragraph{Weights and forward pass.} We denote by $W^l(t)$ and $B^l(t)$ respectively the weight matrix and bias vector of layer $l$ at time step $t$ (\ie after $t$ steps of SGD), and thus have $W^1(t) \in \mathbb{R}^{m \times d}$, $W^l(t) \in \mathbb{R}^{m \times m}$ for $l \in [2, L]$ and $W^{L+1}(t) \in \mathbb{R}^{m}$. At any time step $t$ we denote by $h^l_t(\xi)$ and $x^l_t(\xi)$ the pre-activations and activations respectively coming out of the $l$-th layer when feeding input $\xi$ to the network (with the convention that $x^{0}_t(\xi) = \xi$). That  is
\begin{align}\label{eq:model:inner}
    &h^l_t(\xi) := W^l(t) x^{l-1}_t(\xi) + B^l(t), &\textrm{and}\quad \ &x^l_t(\xi) := \sigma(h^l_t(\xi)), &&\textrm{for}\  l \in [1, L].
\end{align}

\paragraph{Output.} We denote the output of the network by 
\begin{equation}\label{eq:model:out}
   f_t(\xi) = f(\theta(t); \xi) := \transpose{{(W^{L+1}(t))}} x^L_t(\xi) + B^{L+1}(t),
\end{equation}
 where $\theta(t)$ denotes the set of all network parameters at time $t$. We often drop the dependency of the forward pass on the input $\xi$ for brevity and simply use $h^l_t, x^l_t$ instead of $h^l_t(\xi), x^l_t(\xi)$ as it should always be clear from the context which input is being fed to the network. Note that the weights and biases as well as all the (pre-)activations depend on the width $m$ of the network (through their dimensions) but we omit this dependency for clarity. 

\paragraph{Loss.} We denote by $\ell$ the loss function used to train the network, which is a function from $\mathbb{R}^2$ to $\mathbb{R}$. The fit of a prediction $\widehat{y}$ is thus measured by $\ell(y, \widehat{y})$ where $y$ is the desired output. In all this work, we make the following assumption on the loss function $\ell$, which is met by most common loss functions:

\begin{assumption}[Smooth loss \wrt second argument]\label{ass:loss}
The loss $\ell$ is differentiable with respect to its second argument and $\partial_2 \ell(y, \cdot)$ is a continuous function for any $y \in \mathbb{R}$.
\end{assumption}
Assumption~\ref{ass:loss} is essentially here to guarantee that if the sequence $(\widehat{y}^{(m)})_{m \in \mathbb{N}^*}$ converges almost surely to some $\widehat{y}^{(\infty)}$, then $\partial_2 \ell(y, \widehat{y}^{(m)})$ also converges almost surely to $\partial_2 \ell(y, \widehat{y}^{(\infty)})$.

\subsection{Parameterizations of Neural Networks}\label{sec:parameterizations-nn}
The fact that the magnitude of the initialization of the weights and of the scale pre-factor for the weights are key quantities that determine the learning regime achieved by neural networks---and more generally by differentiable models---was pointed out in~\citep{chizat2019lazyTraining}.  
In this paper, we are interested in the behavior of neural networks when their width $m$ goes to infinity, and we refer to as a \textit{parameterization} of a neural network the choice of how (a) the pre-factor of the weights, (b) the variance at initialization and (c) the learning rates, evolve as a function of $m$. This concept was called an abc-parameterization by ~\citet{yang2020featureLearning}, because these dependencies are given by $m^{-a}$, $m^{-b}$ and $m^{-c}$.

As explained by these authors, one of those three choices is actually redundant, and one can do with only the choice of two among those three scales. We take the point view considering a parameterization as a choice of scale for the pre-factor of the weights (a) and a choice of scale for the learning rates (c) while the random weights are always initialized (b) with standard \iid Gaussians $\mathcal{N}(0,1)$. We make this (arbitrary) choice as typically in the literature, different models of the infinite-width limit correspond to different choices of scales for the weights' pre-factors, \eg NTK corresponds to a pre-factor in $1/\sqrt{m}$ while ``mean-field'' models correspond to a choice of pre-factor in $1/m$ for the weights. We thus define below ac-parameterizations which are a slight variation of the abc-parameterizations introduced in~\citep{yang2020featureLearning}. 

\begin{definition}{(ac-parameterization).}\label{def:ac-param}
An ac-parameterization of an $L$-hidden layer fully-connected neural network is a choice of scalar exponents $(a_1, \ldots, a_{L+1})$, and $(c_1, \ldots, c_{L+1})$ such that for any layer $l \in [L+1]$,
\begin{enumerate}[(i)]
    \item the \textbf{learnable weights} (\ie those over which we optimize) are initialized with independent standard Gaussian random variables $w^l_{jq}(0) \sim \mathcal{N}(0, 1)$, \iid over $(l,j,q)$, \ie $w^l(0) = U^l$ with $(U^l)_{l \in [L+1]}$ independent random matrices with \iid standard Gaussian entries,
    
    \item the \textbf{learnable biases}  are initialized independently of the weights, with $b^l_{j}(0) \sim \mathcal{N}(0, 1)$, \iid over $(l,j)$, \ie $b^l(0) = v^l$ with $(v^l)_{l \in [L+1]}$ independent standard Gaussian random vectors, independent  of $U^l$,
    
    \item the \textbf{effective weights} $W^l(t)$ used to compute the pre-activations at time $t$ are $W^l(t) = m^{-a_l}w^l(t)$, and the \textbf{effective biases} are $B^l(t) = m^{-a_l} b^l(t)$, so that the pre-activations are
    \begin{align*}
        h^l_t = W^l(t) x^{l-1}_t + B^l(t) = m^{-a_l} \left( w^l(t) \sigma(h^{l-1}_t) + b^l(t) \right),\quad l \in [1, L],
    \end{align*}
    and the output is
    $$f(\theta(t); \xi)=m^{-a_{L+1}} \left(w^{L+1}(t)^T \sigma(h^L_t(\xi)) + b^{L+1}(t)\right),$$
    
    \item the $(t+1)$-th update of learnable weights and biases is given by the update rules
    \begin{align*}
        \Delta w^l(t+1) &:= w^l(t+1) - w^l(t) = - \eta m^{-c_l} \nabla_{w^l} \ell \left(y_t, f(\theta(t);\xi_t) \right), \\
        \Delta b^l(t+1) &:= b^l(t+1) - b^l(t) = - \eta m^{-c_l} \nabla_{b^l} \ell \left(y_t, f(\theta(t);\xi_t) \right),
    \end{align*}
    where $\theta(t) = \left\{(w^1(t), b^1(t)), \ldots, (w^{L+1}(t), b^{L+1}(t)) \right\}$ is the full set of all network parameters, $(\xi_t, y_t)$ represent the input(s) and target(s) to the network at step $t$ and $\eta \in \mathbb{R}_{+}^{*}$ is the scalar part of the learning rate which does not depend on $m$ and which we call the \textbf{base learning rate}. We denote by $\eta_l := \eta m^{-c_l}$ the full learning rate for layer $l$.
\end{enumerate}
\end{definition}

\begin{remark}\label{remark:ac-param}
\
\begin{enumerate}[1.]
    \item Compared to the definition of~\citep{yang2020featureLearning}, we allow for different values of $c_l$ at different layers and remove the redundant initialization scale (the b in abc-parameterizations). Any abc-parameterization with constant $c$ for all layers (as presented in~\citealp{yang2020featureLearning}) can be recovered (same effective weights and biases at any time step) with an ac-parameterization with individual learning rates at each layer via the re-parameterization $a_l \leftarrow a_l +b_l$, $b_l \leftarrow 0$, $c_l := c - 2b_l$.
    
    \item As we study the infinite-width limit $m \rightarrow \infty$, we need to consider an infinite number of random weights at initialization. To this end, we consider for any $l \in [2, L]$, two infinite lists of i.i.d.~standard Gaussian variables, independent of each other: $(U^l_{jq})_{j,q \in \mathbb{N}^*}$ and $(v^l_j)_{p \in \mathbb{N}^*}$, and often simply call, by an abuse of notations, $U^l = (U^l_{jq})_{1 \leq j,q \leq m}$ for the corresponding matrix at width $m$ and $v^l = (v^l_j)_{1 \leq j \leq m}$ the corresponding bias vector at width $m$. We proceed similarly at initialization for the input weights $U^1$ and the output vector $U^{L+1}$.
    
    \item The $(t+1)$-th update of the effective weights is given by $\Delta W^l(t+1) := W^l(t+1) - W^l(t) = - \eta m^{-(2a_l + c_l)} \nabla_{w^l} \ell(y_t, f(\theta(t); \xi_t))$, and the update of the effective biases by $\Delta B^l(t+1) := B^l(t+1) - B^l(t) = - \eta m^{-(2a_l + c_l)} \nabla_{b^l} \ell(y_t, f(\theta(t); \xi_t))$
\end{enumerate}
\end{remark}

\paragraph{Examples of ac-parameterizations:}

\paragraph{NTK parameterization.} For the NTK parametrization~\citep{jacot2018ntk} the scaling is 
$a_1 = 0$ for the input layer, and $a_l = 1/2$ for all the other layers $l \in [2, L+1]$. The scaling of the learning rates is  $c_l = 0$ for all layers. Neural networks in the NTK parametrization have been shown to behave as kernel methods in the infinite-width limit~\citep{jacot2018ntk, yang2020tp2} and there is no feature learning in that limit.

\paragraph{\muP.} To avoid the lazy training phenomenon arising in the NTK parameterization, \citet{yang2020featureLearning} propose to adjust the scale of the  output layer by setting  $a_{L+1}=1$, while keeping $a_1 =0$ and $a_l =1/2$ for the intermediate layers $l \in [2,L]$. The learning rates are appropriately adjusted:  $c_l =-1$ for any layer $l$. With this parameterization, \citet{yang2020featureLearning} show that feature learning (see Definition~\ref{def:feature-learning} in Appendix~\ref{sec:muP} for a precise statement) occurs at every layer.

\paragraph{Integrable Parameterizations (IPs).}\label{def:ip}
The limits investigated in \citet{araujo2019mfLimit, sirignano2021mfAnalysis, pham2020global,e2020banach} are associated to a scale multiplier in $1/m$ for all layers except the first one. This corresponds to the choice $a_1 = 0$ and $a_l = 1$ for $l \in [2,L+1]$. We choose the adjective \quoting{integrable} in reference to the absolute convergence of sums of the form $(1/m)\sum_{q} x_q$  for \iid random variables with finite expectation. 
Integrable parameterizations really refer to a class of ac-parameterizations, because various choices for the learning rate exponents $c_l$ are admissible. 

\paragraph{Naive-IP.}In the mean-field literature, integrable parameterizations often come with the standard learning rates corresponding to $c_1 = c_{L+1} = -1$ for the input/output layers and $c_l = -2$ for the intermediate layers $l \in [2, L]$, see \eg \citep[Remark~3.4]{araujo2019mfLimit}, \citep[Algorithm~1]{fang2020modeling}, \citep[Lemma~5.1]{e2020banach}, and \citep[Equation~4.3]{sirignano2021mfAnalysis}. Mean-field models with these learning rates are the natural counterparts of the infinite-width limits where sums are replaced by integrals, and we call the integrable parameterization with this specific choice of learning rates the \textit{Naive Integrable Parameterization}.\medskip

When $L=1$, \muP\ and the Naive-IP coincide. For deeper networks, in the setting of abc-parameterizations described in~\citep{yang2020featureLearning}, \muP\ and Naive-IP correspond to the same parameterization (same values for a and c) except that the weights of the intermediate layers are initialized with a standard deviation of $1/m$ for Naive-IP instead of $1/\sqrt{m}$ for \muP, that is they are downscaled by $1/\sqrt{m}$ compared to \muP. In Section~\ref{sec:ipllr-modified-muP}, we show that there is also a close relationship between \muP\ and IP with large initial learning rates.

We give below an intuitive explanation for the  choice $c_1 = c_{L+1} = -1$ and $c_l = -2$ for $l \in [2, L]$ for the scaling of the learning rates in Naive-IP. 
For $l \in [2,L]$, we have $h^l_t = m^{-1} (w^l(t)x^{l-1}_t +b^l(t))$, so that $\nabla_{w^l} f_t(\xi_t) = m^{-1} (\nabla_{h^l} f_t(\xi_t))\transpose{{(x^{l-1}_t)}}$. In addition $\nabla_{w^{L+1}} f_t(\xi_t) =
x^L_t / m$ and $\nabla_{w^1} f_t(\xi_t) = (\nabla_{h^1} f_t(\xi_t))\transpose{{(\xi_t)}}$. 
So for one step of SGD:
\begin{equation}
    \begin{aligned}\label{eq:ip-update-contrib-t}
    \Delta W^{1}(t+1) \xi_{t+1} &= - \eta \partial_2 \ell(y_t, f_t(\xi_t)) (\transpose{{\xi_{t}}} \xi_{t+1}) m^{-(1 + c_{1})} (m \nabla_{h^1} f_{t}(\xi_{t})). \\
    \Delta W^{l}(t+1) x^{l-1}_{t+1} &= -\eta \partial_2 \ell(y_t, f_t(\xi_t)) m^{-(2 + c_{l})} \frac{\transpose{{(x^{l-1}_t)}} x^{l-1}_{t+1}}{m} (m \nabla_{h^l} f_{t}(\xi_{t})),\quad\textrm{for}\  l \in [2, L] \\
    \transpose{{(\Delta W^{L+1}(t+1))}} x^{L}_{t+1} &= -\eta \partial_2 \ell(y_t, f_t(\xi_t)) m^{-(1 + c_{l})} \frac{\transpose{{(x^{L}_t)}} x^{L}_{t+1}}{m}.
\end{aligned}
\end{equation}
In addition, from the equations of backpropagation, we get
\begin{align*}
\nabla_{h^{L}_t} f_t(\xi_t) = \frac{1}{m}  w^{L+1}(t) \odot \sigma'(h^L_t)
     \quad \textrm{and} \quad \nabla_{h^l} f_t(\xi_t) = \frac{\transpose{{(w^{l}(t))}} \nabla_{h^{l+1}_t} f_t(\xi_t)}{m} \odot \sigma'(h^{l}_t),
\end{align*}
for $l \in [1, L-1]$,
so that, by a simple induction, $\nabla_{h^l} f_t(\xi_t) = O(1/m)$ for $l \in [1, L]$. In addition, the averaged inner products $\transpose{{(x^{l-1}_t)}} x^{l-1}_{t+1} / m$ in Equation~\eqref{eq:ip-update-contrib-t} converge as $m \rightarrow \infty$. This point is somewhat technical and is handled within the framework of the Tensor Program. The choice of $c_{l}$ in Naive-IP thus ensures that the updates are $O(1)$ when $m$ goes to infinity. 

We conclude this section by giving the definition of a training routine which consists in the combination of the base learning rate, the sequence of training samples and a loss function:
\begin{definition}[Training routine]\label{def:training-routine}
A training routine is the list consisting of the base learning rate $\eta > 0$, $(a_l,c_l)_{l \in [L+1]}$ in the ac-parameterization, the loss $\ell$ and the sequence of training samples $(\xi_0, y_0),  \ldots, (\xi_{T-1}, y_{T-1})$ used to train a network for $T$ steps. 
\end{definition}

\section{Deep Networks with Naive Integrable Parameterization are Trivial}\label{sec:ip-trivial}

In this section, we point out that, in the wide limit, neural networks in the Naive-IP remain at their initial value. We then prove that no choice for the learning rates exponents $(c_l)_{l \in [L+1]}$ which is constant in time can induce non-degenerate learning.

\subsection{No learning in Deep Networks with Naive Integrable Parameterization}\label{sec:trivial-deep-ip}

To start with, we show that the functions learned by networks with more than four layers in the naive integrable parameterization, as described in prior work~\citep{araujo2019mfLimit, rotskoff2019trainability, fang2020modeling, nguyen2020rigorous, e2020banach, sirignano2021mfAnalysis}, remain at there value at initialization in the infinite-width limit: they are identically equal to zero at any time step. Our proof of this result is based on the Tensor Program framework~\citep{yang2020tp3, yang2020featureLearning}, which requires some regularity assumptions on the activation function.

\begin{definition}(Pseudo-Lipschitz functions).\ \label{def:pseudo-Lipschitz}
A function $\psi : \mathbb{R}^k \rightarrow \mathbb{R}$ is  pseudo-Lipschitz of degree $p > 0$ if there exists a constant $K > 0$, such that, for any $x,y \in \mathbb{R}^k$,
\begin{align*}
    |\psi(x) - \psi(y)| \leq K ||x-y|| \left(1 + \sum_{r=1}^k |x_r|^p + \sum_{r=1}^k |y_r|^p \right).
\end{align*}
A function is pseudo-Lipschitz, if it is pseudo-Lipschitz of degree $p$ for some $p > 0$. 
\end{definition}
In particular, functions with polynomially bounded weak derivatives are pseudo-Lipschitz. In the next proposition, we require the activation function $\sigma$ and its derivative to be pseudo-Lipschitz.  

\begin{assumption}[Smooth activation]\label{ass:smooth-act}
The activation function $\sigma$ is differentiable and both $\sigma$ and its derivative $\sigma'$ are pseudo-Lipschitz and not identically zero.
\end{assumption}

\begin{proposition}[Naive-IP is trivial]\label{th:trivial-ip-mf-lr}
Let $L \geq 3$ and consider the naive integrable parameterization of a network with $L$-hidden layers, and an activation function satisfying Assumption~\ref{ass:smooth-act} and  $\sigma(0) = 0$. Then, for any training routine which has a loss satisfying Assumption~\ref{ass:loss}, the function learned by SGD remains at its value at initialization in the infinite-width limit:
\begin{align*}
    \forall \, t \geq 0, \quad \forall \, \xi \in \mathbb{R}^d, \quad \lim_{m \rightarrow \infty} f_t(\xi) = \lim_{m \rightarrow \infty} f_0(\xi) = 0 \quad \text{almost surely}.
\end{align*}
\end{proposition}

\begin{remark}\label{remark:ass-act-weaker} \ \\*
\begin{enumerate}[1.]
    \item In the above statement, ``almost surely''  is relative to the randomness of the initialization.
    
    \item The smoothness Assumption~\ref{ass:smooth-act} on $\sigma$ is met by common activation functions such as GeLU~\citep{hendrycks2020gelu}, ELU~\citep{clevert2016elu}, tanh and the sigmoid activations, but it 
    excludes \relu\ and all the other variants of Leaky \relu. This assumption is required to apply~\citep[Theorem 7.4]{yang2020featureLearning} (which we recall in Appendix~\ref{sec:tp-formalism}) which is the main theoretical result of the Tensor Program series (\citealp{yang2019tp1, yang2020tp2, yang2020tp3, yang2020featureLearning}), but the result is likely to hold with weaker assumptions, as observed numerically in Section~\ref{sec:numerical}, and we leave this for future work. 
    
    \item The assumption $\sigma(0) = 0$ is met by the activation functions mentioned above (except the sigmoid) and is necessary to prove that the network does not move at any layer. Without this assumption, learning is degenerate but not trivial at all layers. It is trivial at step $t=1$ at all layers except the last two: the coordinates of $h^L_1$ and $f_1(\xi)$ converge, with $m$, to quantities which are not 0 but which are independent of the input $\xi$ to the network, similarly to the effect described in Section~\ref{sec:ip-bias}.
\end{enumerate}
\end{remark}

\noindent
The proof of Proposition~\ref{th:trivial-ip-mf-lr}, presented in Appendix~\ref{app:trivial-ip}, proceeds by induction over $t$ to show that the forward and backward passes vanish at any time step. For any time $t$, we proceed again by induction over $l$ (from $l=1$ to $l=L+1$ for the forward pass and from $l=L+1$ to $l=1$ for the backward pass) to prove this vanishing occurs given the magnitudes of the previous forward and backward passes. The informal idea of the proof is the following: essentially, the multiplications of the activation vectors by $m^{-1/2} U^l$ yield vectors whose coordinates are distributed as a Gaussian with finite variance as $m \rightarrow \infty$ for $l \geq 2$ (see Appendix~\ref{app:tp-gaussian-mat-mul} for more details). At initialization, since $w^l(0) = m^{-1} U^l$ for $l \geq 2$ for IPs, the coordinates of $h^l_0$ converge towards $0$ as fast as $m^{-1/2}$ and that of $x^l_0$ towards $\sigma(0)$ for $\sigma$ continuous at $0$. For the same reasons, $f_0(\xi_0)$ converges to $0$. In the first backward pass, multiplications by $\transpose{{(W^l(0))}}$ also yield vectors whose coordinates are in $O(m^{-1/2})$. In contrast to the forward pass, these scales propagate from $l=L$ to $l=1$ and thus compound with depth, and since the last layer's gradient $x^L_0 /m$ is in $O(m^{-1})$, all the gradients' coordinates vanish as $m \rightarrow \infty$ and there is no learning. This reasoning can be repeated at later time steps as there are no correlations between the initial weight matrices and the vectors they multiply because of the degeneracy of the (pre)-activations (their coordinates become equal to the constant $\sigma(0)$ as $m \rightarrow \infty$). Those informal calculations are made rigorous by the Tensor Program.

Proposition~\ref{th:trivial-ip-mf-lr} shows that the parameters of  neural networks in the integrable parameterization are stuck in a stationary point of the objective function in the infinite-width limit, and no learning occurs. It might appear obvious that using larger learning rates to correct the scale with $m$ of the weight updates can avoid this pitfall, but as discussed in the following Section~\ref{sec:no-constant-lr}---where we study which choices of learning rates can lead to stable learning with homogeneous activation functions---the issue is more subtle.

\subsection{No stable learning with learning rates constant over time}\label{sec:no-constant-lr}

As $m$ grows, to compensate the vanishing gradients in the first SGD step, one can use larger learning rates than in the Naive-IP. Yet, as explained below, exponents $(c_l)_{l \in [L+1]}$ for the learning rates which allow to escape the stationary point at initialization will induce an explosion of the pre-activations, if the same values of the exponents are used in the subsequent gradient steps. Indeed, the next informal statement of Theorem~\ref{th:formal-no-constant-lr} shows that, with IPs, one cannot have non-trivial and stable learning with learning rate scales $c_l$ constant in time.

\begin{theorem}[Informal]\label{th:informal-no-constant-lr}
Consider an $L$-hidden layer fully-connected neural network with $L \geq 3$ in the integrable parameterization. Assume that the contributions of the first and second updates $\Delta W^l(1) x^{l-1}_1$ and $\Delta W^l(2) x^{l-1}_2$ are non-vanishing and non-exploding with $m$ at every layer $l$. Then, the learning rates scales $c_l$ cannot have the same value at $t=0$ and $t=1$.  
\end{theorem}
\noindent
In a nutshell, one needs large learning rates to escape the initial stationary point, but keeping those initial values at later time steps would make the pre-activations blow-up as $m \rightarrow \infty$. The formal version of the previous~\autoref{th:informal-no-constant-lr} is given in~\autoref{th:formal-no-constant-lr} below. For this formal statement, we introduce some definitions and assumptions.

\begin{assumption}[Smooth non-negative homogeneous activation]\label{ass:smooth-homogeneous-act}
The activation function $\sigma$ is non-negative, not identically zero and it is positively $p$-homogeneous with $p \geq 2$, \ie $\sigma(\lambda z)=\lambda^p \sigma(z)$ for any $\lambda>0$ and $z\in\mathbb{R}$. Additionally, $\sigma$ has faster growth on the positive part of the real line: $\exists z > 0 \ s.t.\ \sigma(z) > \sigma(-z)$.
\end{assumption}

\begin{remark}
\ 
\begin{enumerate}
    \item While the homogeneity assumption is core to the calculation of scales with integrable parameterization, the fact that $p \geq 2$, and that $\sigma$ is non-negative and has faster growth on the positive part of the real line are simply here to avoid cumbersome technical difficulties in the proofs. It is clear that $\relu^p$ satisfies Assumption~\ref{ass:smooth-homogeneous-act} for any $p \geq 2$. 
    
    \item With the assumption that $p \geq 2$, $\sigma$ also satisfies Assumption~\ref{ass:smooth-act}, so that the rules of the Tensor Program can be applied. 
\end{enumerate}

\end{remark}

\begin{definition}[Scales of first updates with homogeneity]\label{def:gamma-p}
Let $p > 0$. We define the following exponents:
\begin{align*}
   \gamma_1(p) &= \gamma_{L+1}(p) = - \frac{1}{2}
   \left(1+\sum_{k=0}^{L-1} p^k \right),\\
   \textrm{and}\quad \gamma_l(p) &= -1 - \frac{1}{2} \sum_{k=0}^{L-1} p^k   ,  \quad \textrm{for}\  l \in [2,L] .
\end{align*}
\end{definition}

\begin{theorem}[Formal version]\label{th:formal-no-constant-lr}
Consider an $L$-hidden layer fully-connected neural network with $L \geq 3$ in the integrable parameterization, and with no bias terms, except for the first layer. Assume that the activation function $\sigma$ satisfies Assumption~\ref{ass:smooth-homogeneous-act}, the loss $\ell$ satisfies Assumption~\ref{ass:loss} and  that $\lim_{m \rightarrow \infty} \partial_2 \ell(y_0, f_0(\xi_0)) \neq 0$, and $\lim_{m \rightarrow \infty} \partial_2 \ell(y_1, f_1(\xi_1)) \neq 0$ almost surely. Assume further that $\xi_0, \xi_1, \xi_2 \in \mathbb{R}^d$ are all distinct vectors such that $\transpose{\xi_0} \xi_1 \neq 0$ and $\transpose{\xi_1} \xi_2 \neq 0$. Finally assume that:
\begin{align}\label{eq:ass-delta1-stable}
    \begin{cases}
        \frac{1}{m} ||\Delta W^l(1) x^{l-1}_1 ||^2 = \Theta(1), \quad l \in [1, L] \\
        \transpose{{(\Delta W^{L+1}(1))}} x^L_1 = \Theta(1)
    \end{cases}
\end{align}
and 
\begin{align}\label{eq:ass-delta2-stable}
    \begin{cases}
        \frac{1}{m} ||\Delta W^l(2) x^{l-1}_2 ||^2 = \Theta(1), \quad l \in [1, L] \\
        \transpose{{(\Delta W^{L+1}(2))}} x^L_2 = \Theta(1)
    \end{cases}
\end{align}
Then, one necessarily has that:
\begin{enumerate}[(i)]
    \item at $t=0$, $c_l = \gamma_l(p)$ for any $l \in [1, L+1]$ (see Definition~\ref{def:gamma-p}),
    \item at $t=1$, $c_1 = c_{L+1} = -1$, and $c_l = -2$ for $l \in [2, L]$.
\end{enumerate}
\end{theorem}
Let us comment briefly on the hypotheses of Theorem~\ref{th:formal-no-constant-lr}.
The proof of Theorem~\ref{th:formal-no-constant-lr} relies on an analysis of the SGD steps involving both~\cite[Theorem 7.4]{yang2020featureLearning} and the homogeneity property of the activation function.
The requirement that $p \geq 2$ allows to satisfy the smoothness assumption of~\cite[Theorem 7.4]{yang2020featureLearning} and the removal of the bias terms allows to fully exploit homogeneity. In Section~\ref{sec:numerical}, we  numerically check that the result still holds with $\sigma = \relu$, which is $p=1$ homogeneous. The corresponding scales for the learning rates  in the ReLU case are $\gamma_1(1) = -(L+1)/2$, $\gamma_l(1) = -(L+2)/2$ and $\gamma_{L+1}(1) = -(L+1)/2$. 

We give below an informal explanation for the values of the learning rates appearing in \autoref{th:formal-no-constant-lr} in the case of a positively 1-homogeneous activation function. As previously mentioned in Section~\ref{sec:trivial-deep-ip}, each multiplication by $W^l(0) = m^{-1} U^l$ or its transpose yields a factor in $m^{-1/2}$ for $l \geq 2$. Because of the homogeneity property, this scale  propagates from layer to layer starting from layer $2$, and the coordinates of $h^l_0$ and $x^l_0$ are thus in $\Theta(m^{-(l-1)/2})$ for $l \in [1, L]$. For the backward pass, the first gradient $\nabla_{x^L} f_0(\xi_0) = U^{L+1} / m$ has coordinates in $\Theta(m^{-1})$, and, as already discussed in Section~\ref{sec:trivial-deep-ip}, from $l=L$ to $l=2$, each multiplication by $\transpose{{(W^l(0))}}$ yields an additional factor in $m^{-1/2}$ and those compound with depth so that the coordinates of $\nabla_{h^l} f_0(\xi_0)$ are in $\Theta(m^{-1}m^{-(L-l)/2})$. Therefore, calling $\tildex^l_0 := m^{(l-1)/2} x^{l}_0$, and $d\tildeh^l_0 := m^{1+(L-l)/2} \nabla_{h^l} f_{0}(\xi_{0})$, we have after the first weight update
\begin{align*}
    \Delta W^{1}(1) \xi_{1} &= - \eta \partial_2 \ell(y_0, f_t(\xi_0)) (\transpose{{\xi_{0}}} \xi_{1}) m^{-c_{1}} m^{-(L+1)/2} d\tildeh^1_0,\\
    \Delta W^{l}(1) x^{l-1}_{1} &= -\eta \partial_2 \ell(y_0, f_0(\xi_0)) m^{-c_{l}} m^{-2} m^{-(L-l)/2 - (l-2)/2} \frac{\transpose{{(\tildex^{l-1}_0)}} x^{l-1}_{1}}{m} d\tildeh^l_0, \quad l \in [2, L], \\
    \transpose{{(\Delta W^{L+1}(1))}} x^{L}_{1} &= -\eta \partial_2 \ell(y_0, f_0(\xi_0)) m^{-c_{l}} m^{-1} m^{-(L-1)/2} \frac{\transpose{{(\tildex^{L}_0)}} x^{L}_{1}}{m}.
\end{align*}
Since $d\tildeh^l_0$ and $\tildex^l_0$ have coordinates in $\Theta(1)$ by design, and since averaged inner products of the type $\transpose{{(\tildex^{l-1}_0)}} x^{l-1}_1 / m$ converge to finite expectations (by the rules of the Tensor Program, see~\citealp[Theorem 7.4]{yang2020featureLearning}), we see that the choice $c_1 = -(L+1)/2$, $c_l = -(L+2)/2$ for $l \in [2, L]$, and $c_{L+1} = -(L+1)/2$ is the only way to ensure that the updates induce contributions which have coordinates in $\Theta(1)$ at $t=1$. Given this choice for the learning rate scales $c_1, \ldots, c_{L+1}$ at $t=0$, we readily get that the coordinates of $h^l_1$ and $x^l_1$ are in $\Theta(1)$ because the contributions $W^l(0)x^{l-1}_1$ have coordinates in $O(m^{-1/2})$ for intermediate layers, and in $O(1)$ for the input and output layers. From the Equations~\eqref{eq:ip-update-contrib-t} with $t=1$, we see that for the second gradient step, $m\nabla_{h^l} f_1(\xi_1)$ has coordinates in $\Theta(1)$ because the multiplications by $\transpose{{(W^l(1))}}$ do not yield a factor in $m^{-1/2}$ due to the scale correction introduced in the first update. At $t=1$, this leads to the choice $c_1 = c_{L+1} = -1$, and $c_l = -2$ for $l \in [2, L]$,  in order to have update contributions with coordinates in $\Theta(1)$ at $t=2$. 
These informal calculations are made rigorous in the proof of \autoref{th:formal-no-constant-lr} using the Tensor Program~\citep{yang2020tp3}.

\section{Large Initial Learning Rates Induce Learning}\label{sec:llr-learning}

In this section, we show that with positively homogeneous activation functions, using large initial learning rates (polynomial in $m$) allows the network to escape from the initial stationary point and to initiate a non-trivial training phase in the infinite-width limit.
Because we use the homogeneity property extensively for our results, in all this section, as in Section~\ref{sec:no-constant-lr}, we consider a version of integrable parameterizations where the bias terms are removed except for the first layer.

As observed in Section~\ref{sec:no-constant-lr}, beyond the fact that IPs require large learning rates (for the first gradient step) to be trained, one crucial characteristic of IPs is that no choice of learning rate scales ($c_l$) which are constant in time can induce a favorable learning behavior: one has to first use large learning rates to escape the stationary point at initialization ($t=0$) and then revert to the Naive-IP learning rates for $t \geq 1$ to induce stable learning. 

\begin{definition}[IP with large initial learning rates]\label{def:ipllr}
Let $\sigma$ be a positively $p$-homogeneous activation function with $p > 0$. We define \emph{the integrable parameterization with large initial learning rates} (IP-LLR) as the integrable parameterization of an $L$-hidden layer fully connected-network with activation $\sigma$ such that:
\begin{enumerate}[(i)]
    \item At $t=0$: $c_l = \gamma_l(p)$, for $l \in [1,L+1]$;
    \item At $t \geq 1$: $c_1 = c_{L+1} = -1$ and  $c_l = -2$, for $l \in [2,L]$,
\end{enumerate}
where the values of the $\gamma_l(p)$ are given in Definition~\ref{def:gamma-p}.
\end{definition}

\begin{remark}
\ 
\begin{enumerate}
    \item The definition means that $\Delta w^l(1) = -\eta m^{-\gamma_l(p)} \nabla_{w^l} \ell(y_0, f_0(\xi_0))$ for the first weight update after the forward-backward pass at time $t=0$, and for $t \geq 1$, the $(t+1)$-th weight update is  $\Delta w^l(t+1) = -\eta m^{-2} \nabla_{w^l} \ell(y_t, f_t(\xi_t))$ for $l \in [2, L]$, and $\Delta w^1(t+1) = -\eta m^{-1} \nabla_{w^1} \ell(y_t, f_t(\xi_t))$, $\Delta w^{L+1}(t+1) = -\eta m^{-1} \nabla_{w^{L+1}} \ell(y_t, f_t(\xi_t))$ after the forward-backward pass at time $t$.

    \item We give the definition with an arbitrary degree of homogeneity $p$ (the values of the $\gamma_l(p)$ are given in Definition~\ref{def:gamma-p}) as for some theorems where we use the Tensor Program for the proof, we need sufficient smoothness of the activation function, which is achieved only when $p \geq 2$, but we always use $\sigma = \relu$ (which corresponds to $p=1$) in our informal derivations and numerical experiments. Note that since the values of $c_1, \ldots, c_{L+1}$ at $t=0$ depend on $p$, the definition of an IP-LLR parameterization also implicitly depends on the degree of homogeneity $p$.  
    
    \item Since $a_1 = 0$ for IPs, we leverage the homogeneity property only for layers $l \in [2, L]$ (see Appendix~\ref{app:forward-backward-homog} for more details), so that we might as well assume $L \geq 2$ whenever we study IP-LLR. 
\end{enumerate}
\end{remark}

\subsection{Non-trivial and Stable Learning for Integrable Parameterizations}\label{sec:ipllr-main-result}

\begin{theorem}[Non-trivial and non-exploding learning of IP-LLR]\label{th:non-trivial-ipllr}
Consider the IP-LLR parameterization of an $L$-hidden layer neural network with no bias terms, except for the first layer, and with an activation function $\sigma$ satisfying Assumption~\ref{ass:smooth-homogeneous-act} and a loss function $\ell$ satisfying Assumption~\ref{ass:loss}. Let $\xi \in \mathbb{R}^d$ be an input to the network, and assume $\partial_2 \ell(y_0, 0) \neq 0$. Then, one has:
\begin{align*}
    (i) \ \ &f_0(\xi) \xrightarrow[m \rightarrow \infty]{a.s.} 0. \\
    (ii) \ \ &f_1(\xi) \xrightarrow[m \rightarrow \infty]{a.s.} \scalarlim{f_1}(\xi), \quad  0 < |\scalarlim{f_1}(\xi)| < \infty \ a.s. \\
    (iii) \ \ &f_2(\xi) \xrightarrow[m \rightarrow \infty]{a.s.} \scalarlim{f_2}(\xi), \quad  |\scalarlim{f_2}(\xi)| < \infty \ a.s.
\end{align*}
\end{theorem}

\begin{remark}\label{remark:non-trivial-ipllr}
\ 
\begin{enumerate}[1.]
    \item We show in our numerical experiments (see Section~\ref{sec:numerical}) that with $\sigma = \relu$ (\ie $p=1$), the choice of learning rates for IP-LLR is indeed able to induce learning for networks deeper than four layers without creating instabilities. 
    
    \item A similar result could be obtained with more general assumptions on the activation function $\sigma$, namely that $\sigma$ is twice differentiable almost everywhere and that $\sigma(0) = 0$ and $\sigma'(0) \neq 0$ (which is the case for many activation functions such as GeLU, ELU, tanh), but at the cost of a more technical proof. The idea in this case is that because of the scaling in $1/m$ which makes the forward pass vanish at initialization, one can recover the homogeneity property by linearizing $\sigma$ around 0: $\sigma(h) \simeq \sigma'(0) h$. This linearization also provides the right value $|\sigma'(0)|^{-1}$ for the standard deviation of the initial Gaussians in order to avoid vanishing or explosion at initialization with the depth $L$. See more details in Remark~\ref{remark:homog-first-forward}.
    
    \item  For positively $p$-homogeneous activations with $p \geq 2$, we have $\sigma'(0) = 0$ and the behavior of the network is inherently different from that of a network where the first forward pass can effectively be linearized (the setting described in the previous point). This difference appears clearly in the numerical experiments presented in Section~\ref{sec:numerical} where we also discuss the reasons for such a qualitatively different behavior.
    
    \item In IP-LLR, the initial gradient direction will be determined by the first sample $(\xi_0, y_0)$ fed to the network. To avoid giving too much importance to a single sample, one can in practice average the gradients over a batch of many training samples instead, which is what we do in our numerical experiments in Section~\ref{sec:numerical}. 
\end{enumerate}
\end{remark}

\noindent
The idea of the proof essentially lies in the informal calculations of Section~\ref{sec:no-constant-lr} which are made rigorous using the framework of the Tensor Program. Point $(ii)$ stems from the fact that at $t=1$, the output is the difference between two expectations in the limit $m \rightarrow \infty$, which can both be shown to be different from $0$ and of opposite signs.

\subsection{IP-LLR is a Modified $\mu$P}\label{sec:ipllr-modified-muP}

In this section, we analyze the behavior of IP-LLR more in detail and show that this model is actually equivalent to a modification of \muP{} where the initial weights are removed from the first weight update for all of the intermediate layers. We first show an equivalence at finite-width in Section~\ref{sec:ipllr-muP-finite-width} with mild assumptions, and then extend those results to the infinite-width limit in Section~\ref{sec:ipllr-muP-infinite-width} with slightly more restrictive assumptions on the activation function $\sigma$. Since we study the IP-LLR parameterization, we consider positively $p$-homogeneous activation functions, and only the degree of homogeneity allowed will vary between Sections~\ref{sec:ipllr-muP-finite-width} and~\ref{sec:ipllr-muP-infinite-width}. In short, the main idea behind this equivalence is that since IP-LLR and \muP\ are both designed to have maximal update contributions at $t=0$, they will induce the same update at initialization, and the only difference at later time steps is that the initial weights of IP-LLR contribute vanishingly to the pre-activations whereas those of \muP\ contribute in $\Theta(1)$.

\subsubsection{Finite-Width Equivalence}\label{sec:ipllr-muP-finite-width}

As explained in Section~\ref{sec:parameterizations-nn} in the examples of ac-parameterizations, from the point of view of abc-parameterizations (see \citealp{yang2020featureLearning}), both \muP\ and Naive-IP follow the same training procedure for the effective weights $W^l$, the only difference being the standard deviation at initialization which is downscaled by $1/\sqrt{m}$ for Naive-IP compared to \muP. In this regard, since IP-LLR is a modification of Naive-IP where large learning rates are used at initialization, it comes as no surprise that the learning dynamics of IP-LLR and \muP\ are closely related. We detail this relationship in this section.

Recall that for \muP\ one has $W^1_{\muP}(0) = U^1$, $W^l_{\muP}(0) = m^{-1/2} U^l$ for $l \in [2, L]$, and $W^{L+1}_{\muP}(0) = m^{-1}U^{L+1}$ whereas for any integrable parameterization, one has $W^1_{\IP}(0) = U^1$, $W^l_{\IP}(0) = m^{-1} U^l$ for $l \in [2, L+1]$. Consider the following \emph{hybrid parameterization} (HP) which consists in training with the maximal update parameterization \muP\ all along, but simply replacing, for all intermediate layers $l \in [2, L]$, the first update $W^l(1) = W^l(0) + \Delta W^l(1)$ by $W^l(1) = m^{-1}U^l + \Delta W^l(1)$. In other words, this simply consists in using the weight pre-factors of \muP\ for the intermediate layers in the initial forward and backward passes, and then using the pre-factors from IP for the initial weights of the intermediate layers in any subsequent update. 

\begin{proposition}[Finite width equivalence between IP-LLR and HP]\label{th:finite-equivalence-ipllr-hp}
Consider the IP-LLR and HP parameterizations with a $p$-homogeneous activation function $\sigma$ with $p \geq 1$ and without any bias term except at the first layer. Let us sub/super-script the variables of each model with IP and HP respectively. Assume the full sequence of training samples $(\xi_0, y_0), \ldots, (\xi_s, y_s), \ldots$ and the loss $\ell$ are the same for both parameterizations. Assume further that $\partial_2 \ell(y_0, f_0^{HP}(\xi_0)) \neq 0$, and denote by $\eta$ the base learning rate of the IP-LLR parameterization. Finally consider the following schedule for the base learning rate of HP: 
\begin{align*}
    \eta_{\HP}(0) &= {\partial_2 \ell(y_0, f_0^{IP}(\xi_0)) \over \partial_2 \ell(y_0, f_0^{HP}(\xi_0))} \eta,  \\
    \eta_\HP(s) &= \eta, \qquad \qquad \qquad \qquad s \geq 1.
\end{align*}
Then one has:
\begin{align*}
    \forall t \geq 1, \quad \forall \xi \in \mathbb{R}^d, \quad f^\HP_t(\xi) = f^\IP_t(\xi).
\end{align*}
\end{proposition}
The proof, presented in Appendix~\ref{app:proof-equivalence-ip-muP-finite}, simply shows inductively that the effective weight matrices for both models are equal for all $t \geq 1$. Since the Tensor Program is not needed here as we consider only finite-width networks, we can work with any positively homogeneous activation function (not necessarily smooth, so that $p=1$ is not precluded).

\subsubsection{Infinite-Width Equivalence}\label{sec:ipllr-muP-infinite-width}

Similarly to HP, we now consider another hybrid parameterization where the initial weights $W^l(0)$ are simply replaced by 0 in the first update of the intermediate layers. We thus consider the following \emph{hybrid parameterization with zero re-initialization} (HPZ): we train with \muP\ all along, but simply replace, for all intermediate layers $l \in [2, L]$, the first update $W^l(1) = W^l(0) + \Delta W^l(1)$ by $W^l(1) = \Delta W^l(1)$. In other words, this simply consists in using the weight pre-factors of \muP\ for the intermediate layers in the initial forward and backward passes, and then forgetting the contribution of the initial weights of the intermediate layers in any subsequent update. As already discussed in Section~\ref{sec:trivial-deep-ip}, the contribution of the initial weights of the intermediate layers $m^{-1}U^l$ vanishes as $m \rightarrow \infty$ for IP, so that HPZ is simply the infinite-width equivalent of HP.

\begin{theorem}[HPZ and IP-LLR are equivalent]\label{th:hpz-ipllr-equivalence}
Consider the IP-LLR and HPZ parameterizations with a $p$-homogeneous activation function $\sigma$ with $p \geq 2$, and with no bias terms except at the first layer. Let us sub/super-script the variables of each models with IP and HPZ respectively. Assume that the training routine is the same for both parameterizations, and assume further that the loss $\ell$ satisfies Assumption~\ref{ass:loss}. Then, one has:
\begin{align*}
    \forall t \geq 0, \quad \forall \xi \in \mathbb{R}^d, \quad \lim_{m \rightarrow \infty}f^\HPZ_t(\xi) = \lim_{m \rightarrow \infty} f^\IP_t(\xi) \quad \text{almost surely}.
\end{align*}
\end{theorem}
\noindent
The proof, presented in Appendix~\ref{app:proof-equivalence-ip-muP-infinite}, proceeds by induction to show that the quantities appearing in the forward and backward passes at every layer are the same for both models at every time step in the infinite-width limit. We use the Tensor Program framework for this proof so we need smoothness of $\sigma$ ($p \geq 2$) for this result.
\\ \\
In essence, \autoref{th:hpz-ipllr-equivalence} shows that the IP-LLR parameterization is equivalent to \muP\ where we simply forget the initialization after the first forward and backward passes. Said differently, IP-LLR is the same as \muP, except that IP-LLR re-initializes the weights of the intermediate layers $l \in [2, L]$ at $t=1$  with $W^l(1) = \Delta W^l(1)$, \ie with the first update computed after the first forward-backward pass. It is not entirely clear whether forgetting the initial weights in one step is beneficial or detrimental to learning. On the one hand, it would seem like forgetting the random initialization could make the network learn faster and be more robust to perturbations (but this is only speculative at this point, and we leave this open for future work), on the other hand the large rank of the initial weight matrices with \iid Gaussian entries might increase the stability of the training dynamics.
In other words, while the randomness from initialization propagates to every layer at every times step for \muP, it is forgotten in one step of SGD for IP-LLR in the infinite-width limit.  
We explore the comparative performance of \muP\ and IP-LLR in Section~\ref{sec:numerical} but there appears to be no clear-cut indication towards one model or the other.

Another interesting difference between IP-LLR and \muP\ is that for any intermediate layer $l \in [2, L]$, while $(W^l_{jq}(t) - W^l_{jq}(0)) / W^l_{jq}(0) = \Theta(m^{-1/2})$ for \muP, so that the effective weights only move infinitesimally (in the infinite-width limit) relatively to their initial values, we have $(W^l_{jq}(t) - W^l_{jq}(0)) / W^l_{jq}(0) = \Theta(1)$ for IP-LLR so that the effective weights actually move in the infinite-width limit (see more details in Remark~\ref{remark:muP-first-weight-updates}).

\section{Alternative Methods for Escaping the Initial Stationary Point}\label{sec:alternative-ip}

As discussed in Section~\ref{sec:llr-learning}, using large initial learning rates in combination with a positively homogeneous activation function allows escaping the initial stationary point and induces stable learning. In this section, we introduce two alternatives to escape this initial stationary point and discuss the properties of the resulting models. In contrast to the setting of Section~\ref{sec:llr-learning}, in all this section, we consider IPs with bias terms at every layer. 

A first alternative to escape the initial stationary point, which we discuss in Section~\ref{sec:ip-non-centered}, is to simply initialize the weight matrices with \iid Gaussian distributions which are not-centered around $0$, as suggested by~\citet{nguyen2020rigorous}. This method is able to escape the stationary point without large initial learning rates and without any homogeneity assumption on the activation function. It turns out that the computations in that setting are well described within the Tensor Program framework and we show that, as highlighted in~\citep[Corollary 37]{nguyen2020rigorous}, a collapse phenomenon occurs, where all the individual entries in the weight matrix of an intermediate layer evolve by the same deterministic quantity in the infinite-width limit.

Another alternative is to remove the pre-factor $m^{-1}$ in front of the bias terms of layers $l \geq 2$. Indeed, as observed in Section~\ref{sec:trivial-deep-ip}, the vanishing of the forward pass and the weight updates in integrable parameterizations is mostly due to the multiplications by the weight matrices $m^{-1} U^l$ which results in pre-activations whose coordinates are $\Theta(m^{-1/2})$ for $l \in [2, L]$. Since the bias terms are decoupled from the input to the layer, re-scaling them appropriately avoids vanishing of the forward pass for IPs. Escaping the initial stationary point can then be achieved without any homogeneity assumption on the activation function $\sigma$. However, one issue which arises then is that the bias terms have the dominant contribution to the pre-activations, and since the input signal propagates through the network via the weight multiplications, the output of the trained network is only ``weakly'' dependent on its input and the training data. 
Let us now study in more details these two alternatives.

\subsection{Using Non-Centered i.i.d. Initialization}\label{sec:ip-non-centered}

In this section, we consider the following modified version of IPs which we call \emph{IP-non-centered} : the forward pass is computed exactly as in IPs but the weight matrices of layers $l \geq 2$ are initialized with $w^l_{jq}(0) = U^l_{jq} + u_l \sim \mathcal{N}(u_l, 1)$ \iid over $(j,q)$ with $u_l \neq 0$. This simply consists in setting $w^l(0) = U^l + u_l J$ for $l \in [2, L]$ and $w^{L+1}(0) = U^{L+1} + u_{L+1} \mathbf{1}$ where $J$ is the square matrix full of ones (whose variable size is the same as $U^2$ and thus equal to $m$) and $\mathbf{1}$ is the vector (of variable size equal to $m$) full of ones. As we will see shortly, the effect of this type of initialization is similar to removing the pre-factor in $m^{-1}$ on the bias terms in that the vanishing of the matrix multiplications $m^{-1} U^l x^{l-1}_t$ is offset by the appearance of an additional term in the expression of $h^l_t$ whose coordinates are all equal and \emph{depend} on the input data.

\subsubsection{First forward Pass}\label{sec:ip-non-centered-forward-0}

As for any IP, $h^1_0$ is a Gaussian vector with \iid coordinates following $\mathcal{N}(0,||\xi||^2+1)$ at any width, and for the second layer we have
\begin{align*}
    h^2_0 = m^{-1} (U^2 x^1_0 + v^2) + u_2 m^{-1}J x^1_0.
\end{align*}
The coordinates of $m^{-1} J x^1_0$ are all equal to $(1/m) \sum_{q=1}^m \sigma(h^{1}_{0,q})$, which converges almost surely, by the law of large numbers, towards $\mathbb{E}[\sigma(Z_1)]$ where $Z_1 \sim \mathcal{N}(0, ||\xi||^2 +1)$. When $\sigma = \relu$ this expectation is tractable as shown in Appendix~\ref{app:exp-relu} and equal to $\sqrt{(||\xi||^2 +1) / (2\pi)}$. On the other hand, the coordinates of $m^{-1} (U^2 x^1_0 + v^2)$ simply converge to $0$. The term $u_2 m^{-1}J x^1_0$ thus offsets the vanishing of the term $m^{-1}( U^2 x^1_0 + v^2)$. In the infinite-width limit, we thus have that $x^2_{0,j} \simeq \sigma(u_2 \mathbb{E}[\sigma(Z_1)])$ for any $j \in [1, m]$. We thus already see that the coordinates of $h^2_0$ all converge almost surely to the same deterministic constant $C_2 = u_2 \mathbb{E}[\sigma(Z_1)]$ and the coordinates of $x^2_0$ towards $\sigma(C_2)$. 

\paragraph{Degeneracy in intermediate layers.}
An easy induction gives that for any $l \in [2, L]$, for any coordinate $j$, and for large $m$
\begin{align}
    h^l_{0, j} &\simeq u_l \sigma \left(u_{l-1} \sigma \left(\ldots \sigma \left(u_2 \mathbb{E} \left[\sigma \left(Z_1 \right) \right] \right) \right) \right) =: C_l \label{eq:C_l},\\
    x^l_{0, j} &\simeq \sigma \left(u_l \sigma \left(u_{l-1} \sigma \left(\ldots \sigma \left(u_2 \mathbb{E} \left[\sigma \left(Z_1 \right) \right] \right) \right) \right) \right) = \sigma(C_l), \nonumber
\end{align}
so that the coordinates of the (pre-)activations of any intermediate layer are all equal to the same deterministic constant for large $m$. Finally, the output of the first forward pass is $f_0(\xi) = m^{-1}( \transpose{{(U^{L+1})}} x^L_0 +v^{L+1}) + u_{L+1} (1/m) \sum_{q=1}^m x^L_{0,q}$ and converges almost surely towards the constant $u_{L+1} \sigma \left(u_{L} \sigma \left(\ldots \sigma \left(u_2 \mathbb{E} \left[\sigma \left(Z_1 \right) \right] \right) \right) \right)$ (this is made rigorous within the framework of the Tensor Program). 

If $\sigma = \relu$ we see that to avoid vanishing of the first forward pass, one must set $u_l > 0$ for $l \in [2, L]$,  and we then get that the coordinates of $h^l_0$ are roughly all equal to $u_l u_{l-1} \ldots u_2 \sqrt{(||\xi||^2 +1) / (2\pi)}$. This suggests that to avoid vanishing or explosion with the depth $L$, one should set $u_l = 1$ for $l \in [2, L]$.

\subsubsection{First Backward Pass}\label{sec:ip-non-centered-backward-0}

We show here that the same degeneracy as in the first forward pass is also at play in the first backward pass. We have $\nabla_{x^L} f_0(\xi_0) = W^{L+1}(0) = m^{-1}(U^{L+1} + u_{L+1} \mathbf{1})$, so that the coordinates of $m \nabla_{x^L} f_0(\xi_0)$ are not deterministic in the infinite-width limit and simply follow $\mathcal{N}(u_{L+1}, 1)$ i.i.d. We have $m \nabla_{h^L} f_0(\xi_0) = m \nabla_{x^L} f_0(\xi_0) \odot \sigma'(h^L_0)$ and as shown in Section~\ref{sec:ip-non-centered-forward-0} the coordinates of $h^L_0$ are roughly all equal to the same constant for large $m$, so that the coordinates of $m \nabla_{h^L} f_0(\xi_0)$ are in $\Theta(1)$. 

\paragraph{Degeneracy for layers $l \in [1, L-1]$.}
Using the equations of backpropagation, we have
\begin{align*}
    m \nabla_{x^{L-1}} f_0(\xi_0) &=  m^{-1} \transpose{{(U^L)}} (m \nabla_{h^L} f_0(\xi_0)) + u_{L}m^{-1} J (m \nabla_{h^L} f_0(\xi_0)).
\end{align*}
The multiplication by $m^{-1} \transpose{{(U^L)}}$ yield a vector whose coordinates converge to $0$, and the coordinates of $m^{-1} J (m \nabla_{h^L} f_0(\xi_0))$ are all equal to $(1/m) \sum_{q=1}^m d\tildeh^L_{0,q}$ where $d\tildeh^L_{0} := m\nabla_{h^L} f_0(\xi_0) = m \nabla_{x^L} f_0(\xi_0) \odot \sigma'(h^L_0)$. We thus have that $(1/m) \sum_{q=1}^m d\tildeh^L_{0,q}$ converges almost surely to the constant $u_{L+1} \sigma'(C_L)$, where $C_L$ is defined in Equation~\eqref{eq:C_l}.
Because $m \nabla_{h^{L-1}} f_0(\xi_0) = m \nabla_{x^{L-1}} f_0(\xi_0) \odot \sigma'(h^{L-1}_0)$, we get that the coordinates of $m \nabla_{h^{L-1}} f_0(\xi_0)$ are roughly all equal to the constant $u_{L+1} u_L \sigma'(C_L) \sigma'(C_{L-1})$ for large $m$.
An easy induction then yields that for any $l \in [1, L-1]$, and for any coordinate $j$
\begin{align*}
    d\tildex^{l}_{0,j} &\simeq u_{L+1} \ldots u_{l+1} \sigma'(C_L) \ldots \sigma'(C_{l+1}), \\
    d\tildeh^{l}_{0,j} &\simeq u_{L+1} \ldots u_{l+1} \sigma'(C_L) \ldots \sigma'(C_{l}),
\end{align*}
as $m \rightarrow \infty$, where $d\tildex^{l}_{0} := m \nabla_{x^l} f_0(\xi_0)$, $d\tildeh^{l}_{0} := m \nabla_{h^l} f_0(\xi_0)$, and $C_l$  is defined in Equation~\eqref{eq:C_l}. Note that all the $C_l$ depend on $\xi_0$ through $Z_1 \sim \mathcal{N}(0, ||\xi_0||^2+1)$, and the coordinates of $m \nabla_{h^l} f_0(\xi_0)$ are in $\Theta(1)$ for all $l \in [1, L]$. Again, the products of $u_l$ which appear in the backward pass strongly suggest setting $u_l = 1$ for any $l \in [2, L]$ to avoid issues with increasing depth $L$.

\subsubsection{First parameter updates}\label{sec:ip-non-centered-delta-W-1}

Now that we have described the first forward and backward passes, we can give the formulas for the first weight updates of IP-non-centered. We have:

\begin{align*}
    \Delta W^1(1) &= - \eta m^{-(1 + c_1)} \partial_2 \ell(y_0, f_0(\xi_0)) (m \nabla_{h^1} f_0(\xi_0)) \transpose{\xi_0} \\
    \Delta B^1(1) &= - \eta m^{-(1 + c_1)} \partial_2 \ell(y_0, f_0(\xi_0)) (m \nabla_{h^1} f_0(\xi_0)) \\
    \Delta W^l(1) &= - \eta m^{-(2 + c_l)} \partial_2 \ell(y_0, f_0(\xi_0)) { (m \nabla_{h^l} f_0(\xi_0)) \transpose{{(x^{l-1}_0)}} \over m}, && l \in [2, L] \\
    \Delta B^l(1) &= - \eta m^{-(3 + c_l)} \partial_2 \ell(y_0, f_0(\xi_0))  (m \nabla_{h^l} f_0(\xi_0)) && l \in [2, L]\\
    \Delta W^{L+1}(1) &= - \eta m^{-(1 + c_{L+1})} \partial_2 \ell(y_0, f_0(\xi_0)) {x^L_0 \over m}, \\
    \Delta B^{L+1}(1) &= - \eta m^{-(2 + c_{L+1})} \partial_2 \ell(y_0, f_0(\xi_0)).
\end{align*}

\paragraph{Choice of learning rates and update contributions.}
To ensure non-vanishing and non-exploding updates for both the weights and the bias terms, one must choose different learning rate exponents $c_l$ for the weights and for the bias terms for layers $l \geq 2$.
To make things simpler, we simply choose $c_l = -2$ for $l \in [2, L]$ and $c_1 = c_{L+1} = -1$ (which are the learning rates of Naive-IP) for both the weights and the bias terms, which implies that the updates of the bias terms contribute vanishingly to the second forward pass as $m \rightarrow \infty$ for layers $l \in [2, L+1]$, but this is offset by the non-centered initialization. 

\paragraph{Degeneracy of the weight updates.}
With the choice of learning rate exponents of the Naive-IP, 
all the entries of $\Delta w^l(1)$ are equal to the same deterministic constant for large $m$ for $l \in [3, L-1]$. In other words, for those layers $l \in [3, L-1]$, there is a \emph{collapse to a single parameter per layer} (since the contribution of the centered initialization vanishes for large $m$) which evolves by a deterministic quantity. We recover a result proved by \citet{nguyen2020rigorous} (see \citealp[Corollary 37]{nguyen2020rigorous}). In Section~\ref{sec:trivial-deep-ip} and Proposition~\ref{th:trivial-ip-mf-lr}, we have additionally shown that this translation is $0$ when the \iid initialization is centered around $0$. In fact, a slightly more precise statement can be made: although the coordinates of $\Delta w^2(1)$ dot not become equal to deterministic constants for large $m$, the coordinates of $\Delta W^2(1) x^1_1$ all become equal to the same deterministic constant in the large-width limit because the term $\transpose{{(x^{1}_0)}} x^1_1 /m$ converges to a finite expectation.

\subsubsection{Collapse to Deterministic Dynamics}
Repeating the same calculations as in Sections~\ref{sec:ip-non-centered-forward-0} and~\ref{sec:ip-non-centered-backward-0} shows that the choice of learning rate exponents as in Naive-IP (see Section~\ref{def:ip}), \ie $c_1 = c_{L+1} = -1$, and $c_l = -2$ for $l \in [2, L]$ leads to non-vanishing and non-exploding updates for the weights at any time step for IP-non-centered, and deterministic dynamics as summarized in the following informal theorem:

\begin{theorem}[Informal]\label{th:ip-non-centered-deterministic-informal}
Consider IP-non-centered with the Naive-IP learning rates at every time step, and let $t \geq 0$ and $\xi \in \mathbb{R}^d$ be an input to the network. Then, one has that:
\begin{enumerate}[(i)]
    \item for any $l \in [2, L-1]$, the coordinates of $h^l_t$ (\textit{resp}. $x^l_t$) all converge to the same deterministic constant,
    \item for any $l \in [2, L-1]$, the coordinates of $m \nabla_{x^l_t} f_t(\xi_t)$ (\textit{resp}. $m \nabla_{h^l_t} f_t(\xi_t)$) all converge to the same deterministic constant,
    \item for any $l \in [3, L-1]$, the entries of $(W^l(t) - W^l(0))$ all converge to the same deterministic constant.
\end{enumerate}
\end{theorem}

The rigorous version of this theorem, and its proof, formalized within the framework of the Tensor Program, are presented in Appendix~\ref{app:ip-non-centered-formal}.

\subsection{Not Scaling the Bias Terms}\label{sec:ip-bias}

In this section, we consider a version of IPs where we remove the pre-factor $1/m$ for the bias terms of layers $l \geq 2$. We thus consider the following computations in the forward pass:
\begin{equation}\label{eq:ip-bias}
    \begin{aligned}
        h^1_t &= w^1(t) \xi + b^1(t), \\
        h^l_t &= \left(m^{-1}w^l(t) x^{l-1}_t \right) + b^l(t), \qquad l \in [2, L] \\
        f_t(\xi) &= \left(m^{-1}\transpose{{(w^{L+1}(t))}} x^{L}_t \right) + b^{L+1}(t),
    \end{aligned}
\end{equation}
which in other terms simply means that $B^l(t) = b^l(t)$ for $l \in [1, L+1]$. We use the same initialization for the bias terms as in IPs: $b^l(0) = v^l$ for $l \in [1, L+1]$, where the entries of $v^l$ are \iid following $\mathcal{N}(0, 1)$. We call \emph{IP-bias} the modified version of the integrable parameterization described by Equations~\eqref{eq:ip-bias}. 

\paragraph{Gaussian first forward pass.}
For the first forward pass we have that the pre-activation of the first layer $h^1_0$ is the same as in IPs at initialization and thus has \iid Gaussian coordinates. On the other hand, $h^l_0 \simeq b^l(0) = v^l \sim \mathcal{N}(0, 1)$ as $m \rightarrow \infty$, so that the coordinates of the pre-activations of all the intermediate layers now behave as standard Gaussians in the large-width limit. Note that in contrast to IP-non-centered, the coordinates of $h^l_0$ \textbf{do not} depend on the input data for $l \geq 2$ in the large-width limit. Similarly, we have $f_0(\xi) \simeq v^{L+1} \sim \mathcal{N}(0, 1)$ (which does not depend on the input $\xi$) as $m \rightarrow \infty$.

\paragraph{First parameter updates.}
The first backward pass still vanishes as in integrable parameterizations because of the multiplications by $\transpose{{(W^l(0))}} = m^{-1/2} (m^{-1/2} U^l)$. Indeed, we have $\nabla_{x^L} f_0(\xi) = W^{L+1}(0) = m^{-1} U^{L+1}$, and $\nabla_{h^L} f_0(\xi) = m^{-1} U^{L+1} \odot \sigma'(h^L_0)$, so that the coordinates of $\nabla_{x^L} f_0(\xi)$ and $\nabla_{h^L} f_0(\xi)$ are in $\Theta(m^{-1})$. For $l \in [1, L-1]$, we have that $m \nabla_{x^l} f_0(\xi) = m^{-1/2} \left(m^{-1/2} \transpose{{(U^{l+1})}} \right) (m \nabla_{h^{l+1}} f_0(\xi))$, and an easy induction shows that the coordinates of $\nabla_{x^l} f_0(\xi)$ and $\nabla_{h^l} f_0(\xi)$ are in $\Theta(m^{-1} m^{-(L-l)/2})$ for any $l \in [1, L]$. Note that as in the forward pass, the backward pass at $t=0$ also does not depend on the first training input input $\xi_0$ except for $\nabla_{h^1} f_0(\xi_0)$.
We get the following formulas for the first weight and bias updates at $t=0$:
\begin{equation}
    \begin{aligned}\label{eq:ip-bias-updates-1}
    \Delta W^1(1) &= - \eta m^{-c_1} \partial_2 \ell(y_0 ,f_0(\xi_0)) (\nabla_{h^1} f_0(\xi_0)) \transpose{\xi_0}, \\
    \Delta B^1(1) &= - \eta m^{-c_1} \partial_2 \ell(y_0 ,f_0(\xi_0)) \nabla_{h^1} f_0(\xi_0), \\
    \Delta W^l(1) &= - \eta m^{-(2 + c_l)}  \partial_2 \ell(y_0 ,f_0(\xi_0)) {(m \nabla_{h^l} f_0(\xi_0)) x^{l-1}_0 \over m}, &&l \in [2, L], \\
    \Delta B^l(1) &= - \eta m^{-c_l}  \partial_2 \ell(y_0 ,f_0(\xi_0)) \nabla_{h^l} f_0(\xi_0)), &&l \in [2, L], \\
    \Delta W^{L+1}(1) &= -\eta m^{-(1 + c_{L+1})} \partial_2 \ell(y_0 ,f_0(\xi_0)) x^L_0 / m, \\
    \Delta B^{L+1}(1) &= -\eta m^{-c_{L+1}} \partial_2 \ell(y_0 ,f_0(\xi_0)).
\end{aligned}
\end{equation}

\paragraph{Initial learning rates.}
Because the backward pass vanishes in the infinite-width limit, the learning rate exponents $c_l$ still need to be chosen carefully in order to escape the initial stationary point. However, the two following points stand out: (1) because the first forward pass does not vanish as in the Naive-IP, the choice of $c_l$ does not require any homogeneity property, and needs not be as large (in absolute value) as for IP-LLR (see the values in the case $p=1$ in the comment after \autoref{th:formal-no-constant-lr}); (2) Because we removed the pre-factor $m^{-1}$ from the bias terms, $\Delta W^l(1)$ and $\Delta B^l(1)$ do not have compatible magnitudes, which suggests setting a separate learning rate exponent $\epsilon_l$ for the bias terms, different from $c_l$ for layers $l \in [2, L+1]$, in order to have non-trivial updates for both the weights and the bias terms. In light of the previous comment and of the update formulas of Equations~\eqref{eq:ip-bias-updates-1}, we set, at $t=0$, the learning rate exponents for the weights to 
\begin{equation}\label{eq:ip-bias-initial-lr-weights}
    \begin{aligned}
    c_1 &= -1 - (L-1) / 2 = -(L+1)/2, \\
    c_l &= -2 - (L-l)/2 = -(L- l + 4)/2, \\
    c_{L+1} &= -1,
\end{aligned}    
\end{equation}
and for the bias terms to
\begin{equation}\label{eq:ip-bias-initial-lr-bias}
    \begin{aligned}
        \epsilon_1 &= c_1 = -(L+1)/2, \\
        \epsilon_l &= -(L-l + 2)/2, \\
        \epsilon_{L+1} &= 0.
    \end{aligned}
\end{equation}
One may compare the learning rates exponents for the weights with those of IP-LLR with a degree of homogeneity $p=1$, which are $c_1 = c_{L+1} = -(L+1)/2$, and $c_l = -(L+2)/2$, where the absolute value of the exponent does not decrease with the layer $l$ for intermediate layers.
Even when the learning rates are appropriately scaled as in Equations~\eqref{eq:ip-bias-initial-lr-weights} and~\eqref{eq:ip-bias-initial-lr-bias}, $\Delta W^l(1)$ does not depend on the first training input for $l \in [2, L+1]$.
We thus get the following informal theorem, whose formal version within the framework of the Tensor Program is given in Appendix~\ref{app:ip-bias-formal}.

\begin{theorem}[Informal]\label{th:ip-bias-no-input-dependence-informal}
Consider the IP-bias as in Equations~\eqref{eq:ip-bias}, with the initial learning rates as in Equations~\eqref{eq:ip-bias-initial-lr-weights} and~\eqref{eq:ip-bias-initial-lr-bias}. Then, for any input $\xi \in \mathbb{R}^d$ to the network, $h^l_0(\xi), x^l_0(\xi)$ for $l \geq 2$, and $f_0(\xi)$ do not depend on $\xi$ in the limit $m \rightarrow \infty$. In addition, $\Delta W^l(1)$ also does not depend on the first training input $\xi_0$ in the infinite-width limit $m \rightarrow \infty$ for $l \in [3, L+1]$.
\end{theorem}

\paragraph{Learning rates at step $t \geq 1$.} 
Repeating the calculations of the forward pass with the updates of Equation~\eqref{eq:ip-bias-updates-1} and with the learning rates for the weights and bias terms as described in Equations~\eqref{eq:ip-bias-initial-lr-weights} and~\eqref{eq:ip-bias-initial-lr-bias}, we readily get that the coordinates of the second forward pass are in $\Theta(1)$. Then, it is direct to see that the choice of the Naive-IP learning rate exponents $c_1 = c_{L+1} = -1$, and $c_l = -2$ for $l \in [2, L]$ for the weights and $\epsilon_l = -1$ for $l \in [1, L]$ and $\epsilon_{L+1} = 0$ for the bias terms yields non-vanishing and non-exploding updates for the weights and the bias terms at $t=1$. 

\paragraph{Degeneracy at time $t \geq 1$.} 
It follows that the same choice of learning rate exponents as at $t=1$ also induce non-vanishing and non-exploding updates in the limit $m \rightarrow \infty$ at later time steps $t \geq 1$. With this choice of learning rates we thus get, for any $t \geq 1$, and for $l \in [2, L]$, 
\begin{align*}
    h^l_t \simeq v^l + \sum_{s=1}^{t-1} \Delta W^l(s) x^{l-1}_t.
\end{align*}
With the choice of learning rates prescribed above for $t \geq 1$, the products $\Delta W^l(s) x^{l-1}_t$ are finite and their numerical value strongly depends on the values of $\eta$ and $\partial_2 \ell(y_s, f_s(\xi_s))$. Typically, their product is rather small (\eg $\leq 10^{-2}$), and this means that the initial bias term $v^l$ has the dominant contribution to $h^l_t$. Therefore, in addition to \autoref{th:ip-bias-no-input-dependence-informal}, it can be also be argued that the forward pass in the intermediate layers only weakly depends on the training data and on the input to the network at time steps $t$.

\section{Numerical Experiments}\label{sec:numerical}

In this section we investigate numerically the behavior of the models previously introduced in this work, namely Naive-IP, IP-LLR, IP-bias, IP-non-centered and \muP. In contrast to the theoretical analysis carried out in Sections~\ref{sec:ip-trivial},~\ref{sec:llr-learning}, and~\ref{sec:alternative-ip}, we examine the performance of the models on a multi-class classification task (instead of a single output prediction) and we train them using mini-batch SGD (instead of single-sample SGD). In addition to these two points, we adopt the following slight modifications compared to our theoretical setting.

\paragraph{Standard deviation of initial weights.}
In our numerical experiments, we allow the initial Gaussian weight matrices $U^l$ and vectors $v^l$ to have entries drawn from $\mathcal{N}(0, \delta_l^2)$ where $\delta_l$ can be different from $1$ for $l \in [1, L]$, but is independent of $m$. As hinted in Remark~\ref{remark:non-trivial-ipllr} and explained more in detail in Remarks~\ref{remark:tilde-pos-finite-variance} and~\ref{remark:homog-first-forward}, this is to avoid issues (vanishing or explosion of the forward/backward pass) with the depth $L$. The choices of the standard deviation of the Gaussian depend on the activation function and are summarized in Table~\ref{tab:std-vs-act}.

\begin{table}[!hbtp]
    \centering
    \begin{tabular}{|c|c|c|c|c|}
        \hline 
        \textbf{activation} & \relu  & GeLU & ELU & tanh \\ \hline
         \textbf{init. std} & $\sqrt{2}$ & $2$ & $1$  & $1$ \\
        \hline 
    \end{tabular}
    \caption{Standard deviation $\delta_l$ of the initial Gaussian entries of layers $l \in [1, L]$ for different choices of activation functions.}\label{tab:std-vs-act}
\end{table}

\paragraph{Re-scaling the standard deviation of the first layer.} 
All the models we consider have $a_1 = 0$ so that, as mentioned in Section~\ref{sec:ip-bias}, the coordinates of $h^1_0$ follow $\mathcal{N}(0,||\xi||^2+1)$ and the variance is equal to $\sum_{k=1}^d \xi_k^2 + 1$. To avoid having too large a variance when the (fixed) dimension $d$ is large, we re-scale the standard deviation of the first layer's weights and bias term at initialization by dividing it by $\sqrt{d+1}$, that is we use the Gaussian law $\mathcal{N}(0, \delta_1^2 / (d+1))$ to initialize the entries of $w^1(0)$ and $b^1(0)$. 

\paragraph{Calibrating the initial base learning rates for IP-LLR.} 
As discussed in Section~\ref{sec:ipllr-modified-muP}, IP-LLR basically amounts to training with \muP\ but forgetting the initialization in the intermediate layers for the first update. We thus roughly have $W^l(1) \simeq \Delta W^l(1)$ for any $l \in [2, L]$, and the base learning rate $\eta$ directly influences the magnitude of $\Delta W^l(1)$ and thus that of $h^l_1$. Typical values for the learning rates, the initial loss derivative $\partial_2 \ell(y_0, 0)$, and the averaged inner products involved in the second forward pass are rather small (\eg $\leq 10^{-1}$), and this will cause the pre-activations of the second forward pass to be of small magnitude, and this effect compounds quickly with depth as the pre-activations of layer $(l-1)$ are then multiplied by $\Delta W^l(1)$. This will in turn lead to very small values for the second weight updates $\Delta W^l(2)$ and can considerably slow down learning in practice. To overcome this issue, we simply calibrate the initial values of the base learning rates $\eta_l$ (but cap them at a value of $500$ to avoid too large initial updates) of layers $l \in [2, L]$ at $t=0$, so that the magnitude of the pre-activation of the intermediate layers in the second forward pass is equal to $1$ on average over the second training batch. 

Note that this calibration results in base learning rates $\eta_l$ which \textbf{do not} depend on $m$ (they do depend on $L$ however) in the large-width limit as the coordinates of $h^l_1$ have non-zero and finite values for large $m$. In contrast, this is not possible with the Naive-IP as the coordinates of $h^l_1$ converge to zero as fast as some power of $m$, which would result in the base learning rate $\eta_l$ depending on $m$ which is prohibited (by definition of the base learning rate).
\\ \\
All the points above can be handled within the framework of the Tensor Program, but they would unnecessarily over-complicate the analysis and the formulas, which is why we used a simpler setting in our theoretical analysis.

\subsection{Experimental Setup}\label{sec:exp-setup}

We evaluate the performance of the different models on two datasets: MNIST\footnote{\url{http://yann.lecun.com/exdb/mnist/}}, containing 60,000 training samples and 10,000 test samples, and CIFAR-10\footnote{\url{https://www.cs.toronto.edu/~kriz/cifar.html}}, containing 50,000 training samples and 10,000 test samples. Both datasets consist in a $10$-class image classification task. Since we consider only fully-connected networks, we use gray-scale images which we also flatten for both datasets, which means the input dimension is $d=28 \times 28=784$ for MNIST and $d=32 \times 32 = 1024$ for CIFAR-10.

We train for $600$ SGD steps on MNIST and $1200$ steps on CIFAR-10 using a base learning rate $\eta = 0.01$, a batch-size $B=512$, and the cross-entropy loss, which satisfies Assumption~\ref{ass:loss}.
For each experiment, we run $N_{\text{trials}} = 5$ trials with different random initializations. The hyperparameters are summarized in Table~\ref{tab:hyper-param}.
\begin{table}[!hbtp]
    \centering
    \begin{tabular}{|l|c|c|c|c|c|c|c|}
        \hline 
        $L$ & $m$ & $d_{\text{MNIST}}$ & $d_{\text{CIFAR}}$ & $\ell$ & $\eta$ & $B$ & $N_{\text{trials}}$ \\ \hline
         $6$ & $1024$ & $784$  & $1024$ & cross-ent. & $0.01$ & $512$ & $5$ \\
        \hline 
    \end{tabular}
    \caption{Hyperparameters for training models.}\label{tab:hyper-param}
\end{table}

\subsection{Naive-IP is Trivial but Large Initial Learning Rates Induce Learning}\label{sec:exp-ip-vs-ipllr}

In this section we compare the numerical performance of Naive-IP and IP-LLR on MNIST for different activation functions. Essentially, the results we present corroborate Proposition~\ref{th:trivial-ip-mf-lr} and \autoref{th:non-trivial-ipllr}, except that numerical evidence tends to show that those results hold with less restrictive assumptions on the activation function than what we consider in the theoretical part, as already hinted in Point 2 of Remark~\ref{remark:non-trivial-ipllr}. 

As observed in Figure~\ref{fig:loss-ip-std-vs-llr}, while the loss (averaged over a batch) stays at its initial value for Naive-IP, we observe a decrease for IP-LLR whose strength depends on the choice of activation function. Similarly, Figure~\ref{fig:output-ip-std-vs-llr} depicts the evolution of the mean absolute output during training, that is, we plot for any step $t$ the quantity $(1/B) \sum_{i=1}^B (1/10) \sum_{k=1}^{10} \left|f_{k,t} \left(\xi_{t}^{(i)} \right) \right|$, where $\xi_t^{(i)}$ is the $i$-th sample in the batch at time $t$ and for any class label $k \in [1, 10]$, $f_{k,t}(\xi)$ is the $k$-th entry of the output of the model (logits for class $k$) on input $\xi$. We also observe here that there is no change in the output for the Naive-IP which stays equal to 0 during the course of training, whereas for IP-LLR, the mean absolute output value increases from its initial value, equal to 0, to some positive quantity whose value depends on the activation function. The solid line in both plots denotes the mean of the metric of interest over multiple (5) random trials while the shaded area represents a 95\% confidence interval around the mean. There is no shaded area for Naive-IP since the output of the network is equal to the deterministic constant $0$ at any time step for large $m$, as stated in Proposition~\ref{th:trivial-ip-mf-lr}.
\begin{figure}[t]
\centering
    \subfloat[Naive-IP]{{\includegraphics[width=0.45\linewidth]{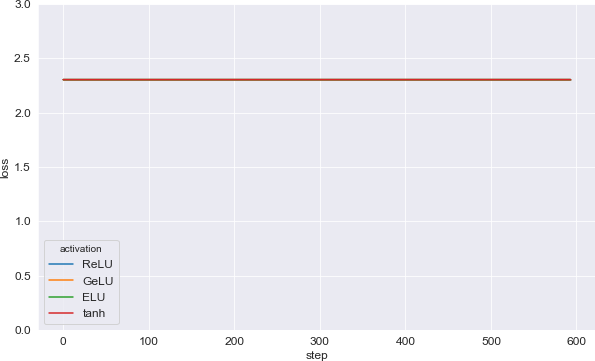}}}
    \qquad
    \subfloat[IP-LLR]{{\includegraphics[width=0.45\linewidth]{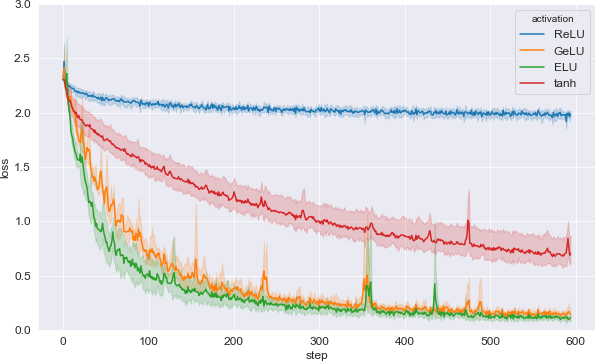} }}
    \caption{Loss \textit{vs.}~number of optimization (SGD) steps on MNIST for different activation functions.
    }
\label{fig:loss-ip-std-vs-llr}
\end{figure}
\begin{figure}[t]
\centering
    \subfloat[Naive-IP]{{\includegraphics[width=0.45\linewidth]{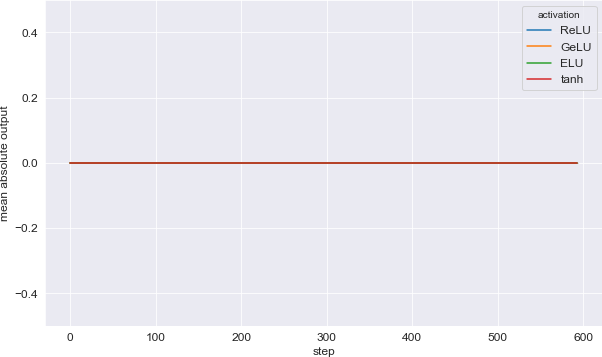}}}
    \qquad
    \subfloat[IP-LLR]{{\includegraphics[width=0.45\linewidth]{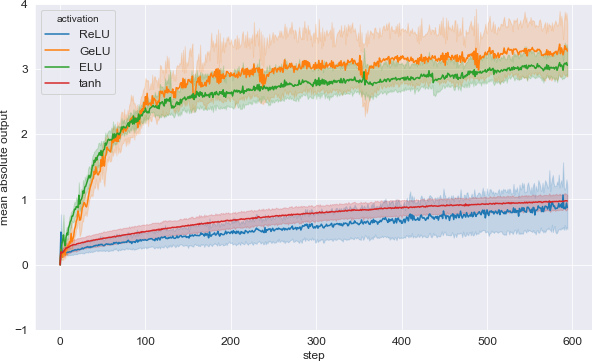}}}
    \caption{Mean absolute output \textit{vs.}~number of optimization (SGD) steps on MNIST for different activation functions.
    }
\label{fig:output-ip-std-vs-llr}
\end{figure}

Finally, we show in Table~\ref{tab:acc-ip-std-vs-llr} the test accuracy (averaged over 5 random runs) \emph{at the end of training} for the Naive-IP and IP-LLR for different activation functions. The Naive-IP has the same test accuracy of $0.098$ independently of the activation function, which is roughly equal to that of random guessing which would yield an accuracy of $0.10$ as there are 10 classes. In contrast, IP-LLR has higher-than-chance test accuracy for every choice of activation function, and while \relu{}  appears to perform poorly, all other activations perform relatively well with ELU and GeLU achieving an error lower than 5\%. 
\begin{table}[!hbtp]
    \centering
    \begin{tabular}{|c|c|c|c|c|}
        \hline 
        \backslashbox{model}{activation} &
        \relu & GeLU & ELU & tanh \\ \hline
         Naive-IP & $0.098$ & $0.098$ & $0.098$ & $0.098$ \\
        \hline 
        IP-LLR & $0.113$ & $0.956$ & $0.964$ & $0.932$ \\
        \hline 
    \end{tabular}
    \caption{Test accuracies on MNIST for various activation functions.}\label{tab:acc-ip-std-vs-llr}
\end{table}

\subsection{IP-LLR \textit{vs.}~\muP}

We compare the numerical performance of IP-LLR and \muP\ on both MNIST and CIFAR-10, and investigate the reasons behind the differences observed between different models and different non-linearities.

As observed in Tables~\ref{tab:acc-all-mnist} and~\ref{tab:acc-all-cifar-10}, the performance, as measured by the accuracy on the test set, is consistent across activation functions for \muP\ whereas the gaps are larger for IP-LLR. However, the best test accuracy for \muP\ and IP-LLR are comparable: the former achieves $0.975$ test accuracy on MNIST and $0.419$ test accuracy on CIFAR-10 with $\sigma = \text{GeLU}$ while the latter achieves $0.964$ test accuracy on MNIST and $0.383$ test accuracy on CIFAR-10 with $\sigma = \text{ELU}$. 
\begin{table}[t]
    \centering
    \begin{tabular}{|c|c|c|c|c|}
        \hline 
        \backslashbox{model}{activation} &
        \relu & GeLU & ELU & tanh \\ \hline
        IP-LLR & $0.113$ & $0.956$ & $0.964$ & $0.932$ \\ \hline
        \muP & $0.954$ & $0.975$ & $0.928$ & $0.905$ \\
        \hline 
    \end{tabular}
    \caption{Test accuracies on MNIST for various activation functions.}\label{tab:acc-all-mnist}
\end{table}

\begin{table}[t]
    \centering
    \begin{tabular}{|c|c|c|c|c|}
        \hline 
        \backslashbox{model}{activation} &
        \relu & GeLU & ELU & tanh \\ \hline
        IP-LLR & $0.100$ & $0.329$ & $0.383$ & $0.284$ \\ \hline
        \muP & $0.407$ & $0.419$ & $0.356$ & $0.304$ \\
        \hline 
    \end{tabular}
    \caption{Test accuracies on CIFAR-10 for various activation functions.}\label{tab:acc-all-cifar-10}
\end{table}

\paragraph{Performance and rank collapse.}
The consistency of \muP\ across activation functions and the lack of consistency for IP-LLR can be explained by (or at least correlated with) the diversity, measured in terms of rank, of the {(pre-)activations} at different layers on large batches of samples. Indeed, as shown in~\citep{Daneshmand2020BNRank}, the rank of the family of pre-activations (considered over large batches) has a dramatic impact on the observed performance of models. In fact, the authors argue that this might be the reason behind the empirical success of batch normalization: it allows the rank of these families of pre-activations to remain large even when the number of hidden layers $L$ is large, whereas they show there is a collapse in the rank without the batch-normalization operation, which coincides with poor accuracy. This problem is exacerbated in IP-LLR because the contribution of the initial weight matrices (which are full-rank) vanishes after the first gradient step, thereby lowering considerably the rank of the family of pre-activations. Two effects are then at play: (1) the choice of the activation function $\sigma$ can induce large differences in the rank of the family of vectors $\left(W^l(1)\sigma(h) \right)_{h \in \mathcal{S}}$, where $\mathcal{S}$ is a large set of vectors; (2) the impact of the activation function on (1) is compounding with depth and can lead to dramatically small rank (equal to $1$ in the worst case) towards the last layers of the network. 

In Figures~\ref{fig:ranks-x-h} we plot  the rank (the $y$-axis is in log-scale) of the families $\left(h^l_1(\xi) \right)_{\xi \in \mathcal{S}}$ and $\left(x^l_1(\xi) \right)_{\xi \in \mathcal{S}}$ for $l \in [1, L]$, where $\mathcal{S}$ is the set comprised of the first 5,000 training inputs of MNIST. The numerical \quoting{rank} is computed as in~\citep{Daneshmand2020BNRank} with \texttt{torch.matrix_rank()} which regards singular values below $\sigma_{\text{max}} \times m \times 10^{-7}$ as zero. We observe that for IP-LLR, the rank of those families with $\sigma =  \relu$ is one order of magnitude smaller than for other activation functions after layer $l=4$ and even collapses to 1 in the last layers, which might explain its poor performance, whereas for \muP\ all activation functions induce comparable ranks which remain at least on the order of $10^2$ at any layer. We believe the latter fact is due to the non-vanishing contributions of the initial Gaussian matrices which are full-rank (with probability 1). In contrast, it would seem like IP-LLR is much more sensitive to the choice of activation function and we identify the vanishing of the contribution of the initial weights for intermediate layers as a probable cause for this effect.

Whether the difference between \relu{} and other activation functions for IP-LLR is actually due to the difference between the homogeneity property with $\sigma=\relu$ and the effective linearization property for other activation functions (as highlighted in Remark~\ref{remark:non-trivial-ipllr}) or to other inherent characteristics of the activation functions is still an open question and we leave it for future work.

\begin{figure}[t]
\centering
    \subfloat[IP-LLR]{{\includegraphics[width=0.45\linewidth]{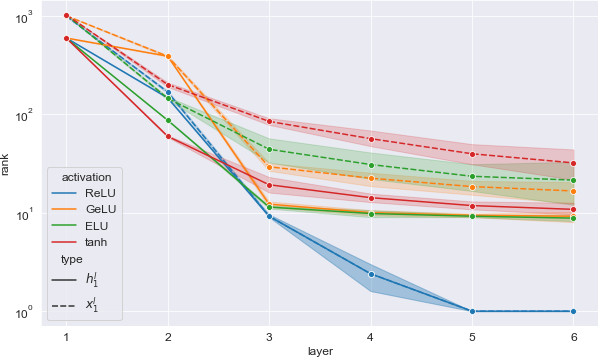}}}
    \qquad
    \subfloat[\muP]{{\includegraphics[width=0.45\linewidth]{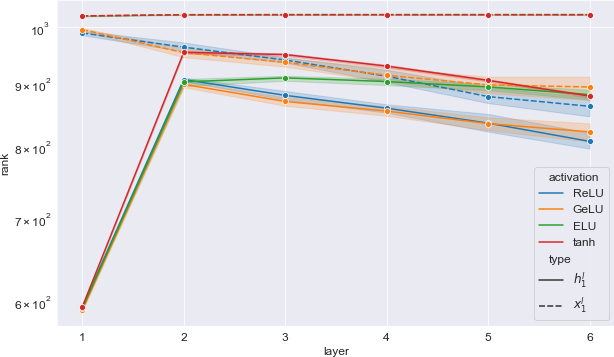}}}
    \caption{Ranks (log-scale) of the families of pre-activations ($h^l_1(\xi)$) and activations ($x^l_1(\xi)$) at time $t=1$ on MNIST \textit{vs.} layer $l$ for different activation functions.
    }
\label{fig:ranks-x-h}
\end{figure}

\subsection{Learning is Degenerate for IP-bias and IP-non-centered}\label{sec:exp-ip-alternatives}

In this section we show numerically that IP-non-centered and IP-bias (see Sections~\ref{sec:ip-non-centered} and~\ref{sec:ip-bias} respectively) are able to escape the initial stationary point but that the resulting dynamics do not seem effective as observed through the evolution of the training loss.

Figure~\ref{fig:output-alternatives} shows that both models are indeed able to escape the initial stationary point as the magnitude of the output evolves non-trivially during training but in contrast Figure~\ref{fig:loss-alternatives}, depicting the training losses on MNIST and CIFAR-10 for both models, shows that learning is very slow for those models and that the dynamics are not effective in reducing the training loss.

Additionally, as summarized in Table~\ref{tab:acc-best-all}, the slow decrease of the training loss translates into poor test accuracy at the end of training comparatively with IP-LLR and \muP, even with the best choice of activation function.

\begin{table}[h]
    \centering
    \begin{tabular}{|c|c|c|c|c|}
        \hline 
        \backslashbox{dataset}{model} &
        IP-LLR & \muP & IP-bias & IP-non-centered \\ \hline
         MNIST & $0.964$ & $0.975$ & $0.113$ & $0.209$ \\
        \hline 
        CIFAR-10 & $0.383$ & $0.419$ & $0.100$ & $0.154$ \\
        \hline 
    \end{tabular}
    \caption{Test accuracies (averaged over 5 random runs) at the end of training on MNIST and CIFAR-10. For each model, we show the maximum (averaged) accuracy over all activation functions. For each model, the activation function which performs best is the same for both datasets and the pairing model $\rightarrow$ activation is the following: IP-LLR $\rightarrow$ ELU, \muP $\rightarrow$ GeLU, IP-bias $\rightarrow$ GeLU, IP-non-centered $\rightarrow$ ELU.}\label{tab:acc-best-all}
\end{table}

\begin{figure}[p]
\centering
    \subfloat[IP-bias / MNIST]{{\includegraphics[width=0.4\linewidth]{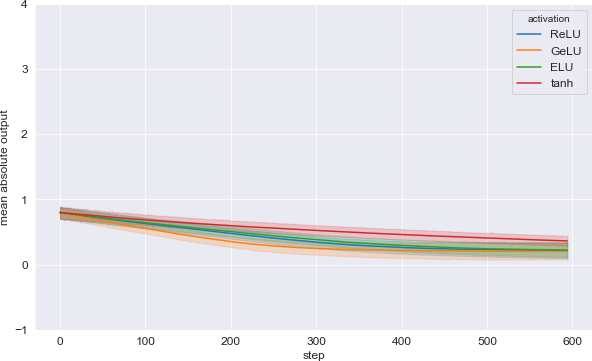}}}
    \qquad
    \subfloat[IP-bias / CIFAR-10 ]{{\includegraphics[width=0.4\linewidth]{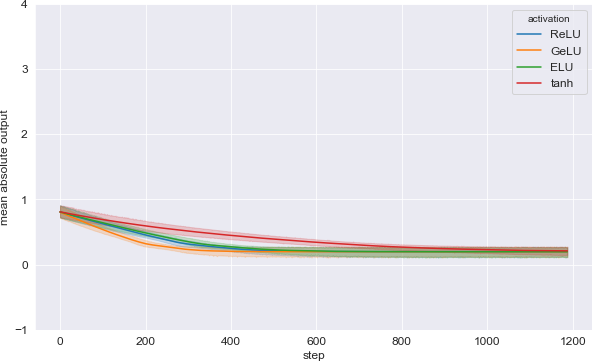}}}
    \vskip\baselineskip
    \subfloat[IP-non-centered / MNIST]{{\includegraphics[width=0.4\linewidth]{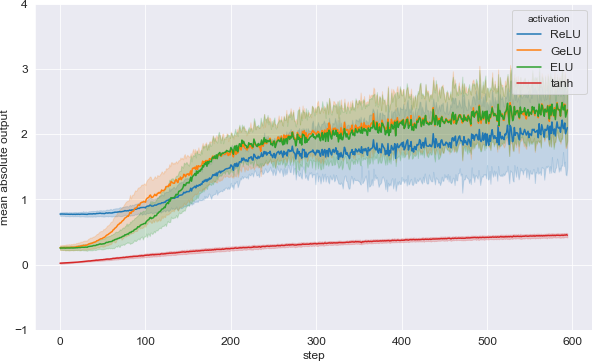}}}
    \qquad
    \subfloat[IP-non-centered / CIFAR-10 ]{{\includegraphics[width=0.4\linewidth]{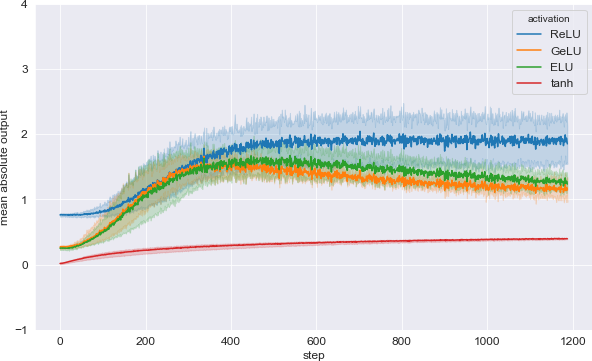}}}
    \caption{Mean absolute value of the output during training.}
\label{fig:output-alternatives}
\end{figure}

\begin{figure}[p]
\centering
    \subfloat[IP-bias / MNIST]{{\includegraphics[width=0.4\linewidth]{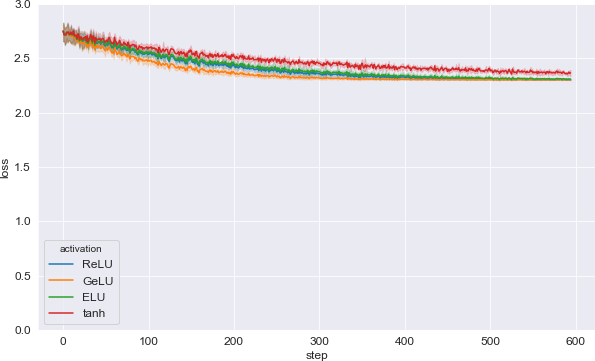}}}
    \qquad
    \subfloat[IP-bias / CIFAR-10 ]{{\includegraphics[width=0.4\linewidth]{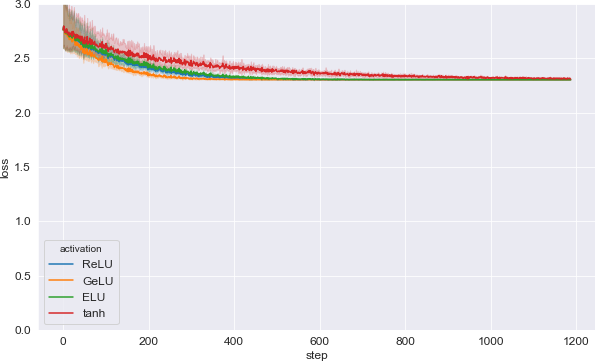}}}
    \vskip\baselineskip
    \subfloat[IP-non-centered / MNIST]{{\includegraphics[width=0.4\linewidth]{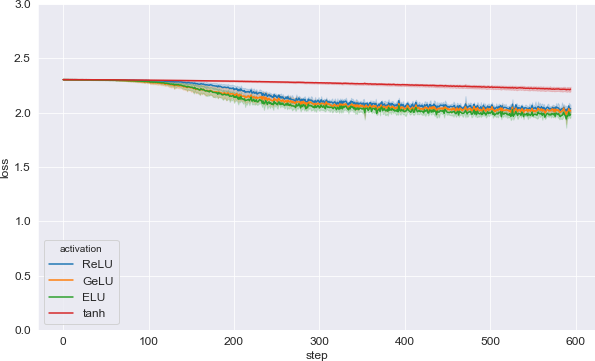}}}
    \qquad
    \subfloat[IP-non-centered / CIFAR-10 ]{{\includegraphics[width=0.4\linewidth]{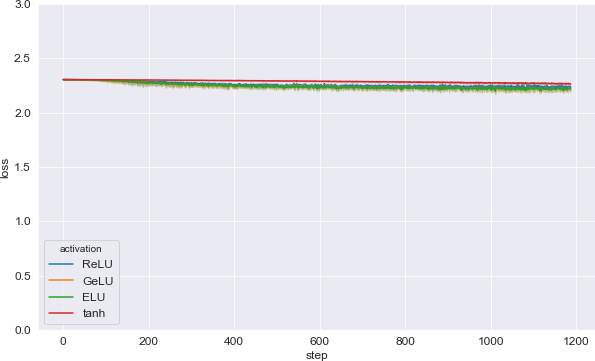}}}
    \caption{Loss \emph{vs.}~number of SGD steps.}
\label{fig:loss-alternatives}
\end{figure}

\section{Conclusion}\label{sec:conclusion}

Recent research has shown that the parameterization of a neural network has a dramatic impact on its training dynamics, and therefore, on the type of functions that it is able to learn. Until now, the parameterizations used by practitioners have been restricted to standard schemes which rely on the analysis of the the first forward and backward passes. In the present work, pushing the analysis beyond the first gradient step (which is made possible by the Tensor Program framework), we have studied how to train neural networks with parameterizations that enjoy radically different behaviors, such as forgetting the contribution of the initial weights after the first weight update. 
 
The parameterizations we have analyzed, which we refer to as \emph{integrable parameterizations}, have been previously described with tools from the \emph{mean-field} literature, and we have deepened our understanding of these models with a different perspective. Indeed, we have shown that these parameterizations are trivial for deep networks with centered \iid initialization and a constant learning rate: they are stuck at initialization. This observation led us to explore various ways to escape this initial stationary point and initiate learning. Among those methods, we found that the only one that does not lead to a degenerate behaviour is to use large learning rates for the first gradient step. We proved that in the infinite-width limit the resulting dynamic is equivalent to a modification of \muP{} where the initial weights are removed after the first gradient step. Importantly, the random fluctuations around the limit---which are ignored in the mean-field description---turn out to actually be essential for our analysis, since it is by amplifying them that we are able to escape the stationary point. 

Extending our theoretical results to a more general class of activation functions requires more thorough technical work and is left as an open problem. Also, analyzing rigorously the impact of the presence or absence of the initial weight matrices on the learning behavior appears to be an interesting avenue for future research. Finally, understanding the generalization properties of IP-LLR and \muP\ remains an important open question but is beyond the scope of this paper.

\section*{Acknowledgements}
Karl Hajjar and Christophe Giraud receive respectively full and partial support from
the Agence Nationale de la Recherche (ANR), reference ANR-19-CHIA-0021-01 “BiSCottE". 

\appendix

\section*{Appendix}\label{app:appendix}

\section{Notations}\label{app:notations}

We introduce here some additional notations that will come in handy in the text and equations presented in the Appendix.

\paragraph{Hat matrices.} We define the following matrices and output weight vector (see Definition~\ref{def:ac-param} for the definitions of the matrices $U^l$): 
\begin{align}\label{eq:hatW}
    \begin{cases}
        \hatW^1 = U^1 \\
        \hatW^l = m^{-1/2} U^l, \quad l \in [2, L+1].
    \end{cases}
\end{align}
The pre-factor in $m^{-1/2}$ is the natural re-scaling of the \iid Gaussian matrices when their input dimension grows to infinity due to the central limit theorem (CLT).

\paragraph{Omegas.} For any ac-parameterization, we define $\omega_1 := m^{-a_1}$, and for any $l \in [2, L+1]$, $\omega_l := m^{1/2 - a_l}$. To avoid blow-up or vanishing in the first layer, all the parameterizations we study have $\omega_1 = 1$. This is the case for integrable parameterizations, the NTK parameterization and for \muP. For integrable parameterizations we also have $\omega_l = m^{-1/2}$ for $l \in [2, L+1]$, but for \muP, $\omega_l = 1$ if $l \in [2, L]$ and $\omega_{L+1} = m^{-1/2}$ (see Section~\ref{sec:muP} for a detailed description of $\muP$). 

Those $\omega_l$ naturally appear in the calculations as the magnitudes of the first forward pass of an ac-parameterization of a neural network. The term $m^{-a_l}$ comes from the scaling pre-factor of the effective weights, and the added $m^{1/2}$ appears when expressing the computation in function of the naturally scaled $\hatW^l$: $W^l(0) = \omega_l \hatW^l$.

\paragraph{Scalar limits.} For any scalar $\omega$ which depends on $m$, we denote by $\scalarlim{\omega}$ the almost sure limit (when it exists) of this scalar as $m \rightarrow \infty$. 

\paragraph{Gradients.} We define for any $t$ and $l$, 
\begin{align*}
    \begin{cases}
        dh^l_t := \nabla_{h^l_t} f_t(\xi_t) \\
        dx^l_t := \nabla_{x^l_t} f_t(\xi_t) \\
        dw^l(t) := \nabla_{w^l(t)} f_t(\xi_t) \\
        db^l(t) := \nabla_{b^l(t)} f_t(\xi_t) \\
        \chi_t := \partial_2 \ell(y_t, f_t(\xi_t)).
    \end{cases}
\end{align*}
The equations of backpropagation give:
\begin{align*}
    &dx^L_t = W^{L+1}(t) \\
    &dw^L(t) = m^{-a_{L+1}} x^L_t \\
    &dh^l_t = dx^l_t \odot \sigma'(h^l_t) \\
    &dx^{l-1}_t = \transpose{{(W^l(t))}} dh^l_t  \\
    &dw^l(t) = m^{-a_l} dh^l_t \transpose{{(x^{l-1}_t)}}, \\
    &db^l(t) = m^{-a_l} dh^l_t.
\end{align*}
As noted in Definition~\ref{def:ac-param} Remark~\ref{remark:ac-param}, one has for $l \in [1, L]$,
\begin{align}
    &\Delta w^l(t) = - \eta m^{-c_l} \chi_t dw^l(t) = - \eta m^{-(a_l + c_l)} \chi_t dh^l_t \transpose{{(x^{l-1}_t)}} \label{eq:delta-w}, \\
    &\Delta W^l(t) = m^{-a_l} \Delta w^l(t) = - \eta m^{-(2a_l + c_l)} \chi_t dh^l_t \transpose{{(x^{l-1}_t)}}, \label{eq:delta-W} \\
    &\Delta B^l(t) = m^{-a_l} \Delta b^l(t) = - \eta m^{-(2a_l + c_l)} \chi_t dh^l_t, \label{eq:delta-B}
\end{align}
and for $l=L+1$
\begin{align}
    &\Delta w^{L+1}(t) = - \eta m^{-c_l} \chi_t dw^{L+1}(t) = - \eta m^{-(a_{L+1} + c_{L+1})} \chi_t x^L_t \label{eq:delta-w-L+1}, \\
    &\Delta W^{L+1}(t) = m^{-a_{L+1}} \Delta w^{L+1}(t) = - \eta m^{-(2a_{L+1} + c_{L+1})} \chi_t x^L_t, \label{eq:delta-W-L+1} \\
    &\Delta B^{L+1}(t) = m^{-a_{L+1}} \Delta b^{L+1}(t) = - \eta m^{-(2a_{L+1} + c_{L+1})} \chi_t. \label{eq:delta-B-L+1}
 \end{align}
 
\paragraph{Z variables.} 
As described in Section~\ref{sec:tp-formalism}, the variables $Z$ with a superscript will be used to denote the random variable whose law describes the evolution of all coordinates of a given vector of the forward or backward pass at a given layer in the limit $m \rightarrow \infty$.

\paragraph{Tilde variables.} For $z \in \{ h^l_t, x^l_t, dh^l_t, dx^l_t  \}$, we will use $\Tilde{z}$ to denote a variable \textit{\quoting{without scale}}, \ie such that $Z^{\Tilde{z}}$ has positive and finite variance (see Definition~\ref{def:tilde-variables}). When we do so, we always have $z = \lambda \Tilde{z}$ for some scalar $\lambda$ (which might depend on $m$). The tilde variables of the backward pass for $t \geq 1$ might have different expressions in different contexts or in different proofs, but we still use the same notation every time as the exact definition should always be clear from the context. 

\section{An overview of the Tensor Program technique}

The Tensor Program technique, first introduced by in~\citet{yang2019tp1}, was initially developed to better understand the behavior at initialization of networks whose weights are initialized \iid with standard Gaussians as the number of units in each layer grows to infinity. Since the output of a hidden unit in layer $l \geq 2$ is given by $\sum_{q=1}^m W^l_{pq}(0) x^{l-1}_{0,q}$, the magnitude of the weights need to be downscaled by some negative power of $m$ to avoid blow-up as $m \rightarrow \infty$. Scalings which have naturally appeared in the literature are $m^{-1/2}$ and $m^{-1}$, and lead to different types of limits. 

Using a first version of the Tensor Program (referred to as NETSOR), it is shown in~\citep{yang2019tp1} that the output at initialization of a neural network of \textbf{any architecture} (fully-connected, recurrent, convolutional, with normalization, attention, ...) whose weights are initialized with $W^l(0) = m^{-1/2} U^l$ for $l \geq 2$ (\ie $a_l = 0$ and $b_l=1/2$ for $l \geq 2$ in the ac-parameterization) is a Gaussian process in the infinite-width limit. 

Going further, and in the light of the recent literature on the neural tangent kernel,~\citet{yang2020tp2} studies the first backward pass of networks initialized as above in the limit where $m \rightarrow \infty$ and has shown that the neural tangent kernel at initialization, defined as $K(\xi, \Bar{\xi}) := \left<\nabla_{\theta} f_0(\theta(0); \xi), \nabla_{\theta} f_0(\theta(0); \Bar{\xi}) \right>$ converges to a deterministic limit for any architecture. 

Finally, and most importantly for our work, the Tensor Program is extended in~\citep{yang2020tp3} to cover the forward and backward passes of networks of any architecture \textbf{at any time step} and not just at initialization. The crucial step taken in~\citep{yang2020tp3} is to be able to describe the evolution of quantities where both a weight matrix $W^l$ and its transpose $\transpose{{(W^l)}}$ are involved.~\citep{yang2020featureLearning} then applies the results and theorems of~\citep{yang2020tp3} in the particular context of ac-parameterizations (or rather abc-parameterizations as defined by~\citealp{yang2020featureLearning}) to describe the infinite-width limits of neural networks with different parameterizations. 

\subsection{Intuition behind the technique}\label{sec:tp-intuition}
To explain the intuition behind the Tensor Program technique and how it comes into play for neural networks, let us first look at the forward pass of a fully-connected network with $L$ hidden layers after $t$ steps of SGD. Assume single samples $(\xi_0, y_0), \ldots (\xi_{t-1}, y_{t-1})$ are used at each step for simplicity. Consider a neural network in any ac-parameterization and an input $\xi$ to the network. Using Equation~\eqref{eq:delta-W} for the updates, the forward pass of the network at time $t$ is given by:
\begin{align*}
    h^1_t &= W^1(0)\xi -\eta m^{-(2a_1 + c_1)} \sum_{s=0}^{t-1} \chi_s \left(\transpose{\xi_s} \xi \right) dh^1_s \\
    h^l_t &= W^l(0) x^{l-1}_t -\eta m^{-(2a_l + c_l)} \sum_{s=0}^{t-1} \chi_s \left(\transpose{(x^{l-1}_s)} x^{l-1}_t \right) dh^l_s && l \in [2, L] \\
    f_t(\xi) &= \transpose{{(W^{L+1}(0))}} x^L_t - \eta m^{-(2a_{L+1} + c_{L+1})} \sum_{s=0}^{t-1} \chi_s \transpose{(x^{L}_s)} x^{L}_t.
\end{align*}
To understand what happens in the forward pass, one thus needs to understand the behavior of the multiplication by \iid Gaussian matrices, that of vectors $dh^l_s$ of the backward pass as well as that of the inner products $\transpose{(x^{l-1}_s)} x^{l-1}_t$. As $m \rightarrow \infty$, the sums defining the matrix multiplications and inner products involve an infinity of terms and one must therefore understand how those quantities scale in the limit. 
\\ \\
Before we dive into the matrix multiplications, let us look more precisely at what the vectors $dh^l_s$ look like. We have:
\begin{align*}
    dh^l_s &= dx^l_s \odot \sigma'(h^l_s) \\
    dx^l_s& = \transpose{{(W^{l+1}(0))}}  dh^{l+1}_s -\eta m^{-(2a_l + c_l)} \sum_{u=0}^{s} \chi_u \left(\transpose{(dh^{l+1}_u)} dh^{l+1}_s \right) x^l_u  && l \in [2, L].
\end{align*}
We observe that inner products appear again, and that in contrast with the forward pass, it is now the multiplication by the transpose of \iid Gaussian matrices which appears. 
\\ \\
We already see that two main quantities appear in the calculations: The initial \iid Gaussian matrices, and vectors which are generated either $(i)$ through the multiplication of another vector with a Gaussian matrix or its transpose, or $(ii)$ through some form of non-linearity involving other vectors as well as the activation function $\sigma$ and/or its derivative $\sigma'$. Before trying to understand how the inner products behave, let us first dive into the multiplication by \iid Gaussian matrices. 

\subsubsection{Multiplication by \iid Gaussian matrices}\label{app:tp-gaussian-mat-mul}

The multiplication of a random vector by an \iid Gaussian matrix can happen in two different scenarios: $(i)$ the input vector is independent of the Gaussian weights, and $(ii)$ the input vector is correlated with the Gaussian weights, which, in the case of neural networks, will translate into saying that the transpose of the weight matrix is used somewhere to compute the input vector. 
\\ \\
\textbf{Independent input vector.} Consider a list $(x_q)_{q \in \mathbb{N}^*}$ of \iid random variables with finite first and second moments, independent of $U^l$, and consider multiplying this vector by the \iid Gaussian matrix $U^l$. At any finite-width $m$ the $p$-th entry of $U^l x$ is given by 
\begin{align*}
    \sum_{q=1}^m U^l_{pq} x_q \underset{m \rightarrow \infty}{\simeq} m^{1/2} \mathcal{N}(0, \mathbb{E}[x_1^2])
\end{align*}
The terms $(U^l_{pq} x_q)_{q \geq 1}$ are \iid with mean 0 and finite variance  $\mathbb{E}[x_1^2]$ because $x_q$ is independent of $U^l_{pq}$. Therefore, by a central limit argument, the sum will behave like $m^{1/2} \mathcal{N}(0, \mathbb{E}[x_1^2])$ for large $m$. It is thus natural to scale the sum by $m^{-1/2}$, or equivalently to consider $\hatW^l = m^{-1/2} U^l$ (as defined in  Equation~\ref{eq:hatW}) for matrix multiplications. 
\\ \\
With the above result in mind, we take a look at the first forward pass at initialization of a network where all the weight matrices are initialized as $W^l(0) = \hatW^l$ (\ie $a_1 = 0$, $a_l=1/2$, $l \in [2, L+1]$). We consider an input $\xi \in \mathbb{R}^d$ to the network and compute the pre-activations of each layer recursively. For the first layer, we get that for any $p \in [m]$, 
\begin{align*}
    h^1_{0,p} &= (\hatW^1 \xi)_p = (U^1 \xi)_p \\
    &= \sum_{q=1}^d \xi_q U^1_{pq} \sim \mathcal{N}(0, ||\xi||^2)
\end{align*}
Since the $(U^1_{pq})_q$ are \iid standard Gaussians, the linear combination above is also a Gaussian with mean $0$ and variance $\sum_q \xi_q^2 = ||\xi||^2$. Note that since the lists $(U^1_{pq})_q$ are independent for different $p$, the vector $h^1_0$ has \iid coordinates all distributed as $\mathcal{N}(0, ||\xi||^2)$. We also note that adding a bias term initialized as $\mathcal{N}(0,1)$ would simply change the variance to $||\xi||^2 + 1$. 
\\ \\
Then for the second layer we get that for any $p \in [m]$:
\begin{align*}
    h^2_{0,p} = \frac{1}{\sqrt{m}} \sum_{q=1}^m U^2_{pq} \sigma(h^1_{0,q}) \xrightarrow[m \rightarrow \infty]{law} \mathcal{N}(0, \mathbb{E}[\sigma(h^1_{0,1})^2])
\end{align*}
The terms $(U^2_{pq} \sigma(h^1_{0,q}))_q$ are \iid with mean zero, and by a central limit argument, we have that the coordinates of $h^2_0$ converge in law towards $\mathcal{N}(0, \mathbb{E}[\sigma(h^1_{0,1})^2])$ where $\mathbb{E}[\sigma(h^1_{0,1})^2])$ is simply $\mathbb{E}[\sigma(Z)^2])$ with $Z \sim \mathcal{N}(0, ||\xi||^2)$. Those coordinates are also independent (and Gaussian at any finite width $m$) \textbf{conditionally} on $h^1_0$ because the lists $(U^2_{pq})_q$ are independent (Gaussians) for different $p$. The different coordinates of $h^2_0$ are identically distributed at any finite width $m$ and remain so in the limit. They are not strictly speaking independent at finite width but the intuition is that they become so in the limit $m \rightarrow \infty$ as they also become Gaussian, and that is how they should be thought of in the context of the Tensor Program. 
\\ \\
Repeating the calculations above at every layer, we can intuitively describe the forward pass in the infinite-width limit by describing the law of a single random variable $Z_l$ for each layer (whose law is the common law of all the coordinates of the pre-activations $h^l_0$), and by the hand-wavy calculations above, we get the following recursion for the variables $Z$:
\begin{align*}
    Z_1 &\sim \mathcal{N}(0, ||\xi||^2) \\
    Z_{l+1} &\sim \mathcal{N}(0, \mathbb{E}[\sigma(Z_l)^2]), && l \in [1, L]
\end{align*}
Having discussed the case where the input vectors are not correlated with the weight matrix, we now move on to the case where there is some correlation between the two. 
\\ \\
\textbf{Correlated input vector.} As the simplest form of correlation, we consider a vector $x = \transpose{{(\hatW^l)}} z$ where $(z_q)_{q \in \mathbb{N}^*}$ is a list of \iid random variables independent of $\hatW^l$ with finite first and second moments, and we consider the result of the multiplication $h = \hatW^l x$. For any $p \in [m]$, we have
\begin{align*}
    h_p &= \sum_{q=1}^m \sum_{r=1}^m \hatW^l_{pq} \hatW^l_{rq} z_r \\
    &= \left[\frac{1}{m} \sum_{q=1}^m \left(U^l_{pq} \right)^2 \right] z_p + \frac{1}{\sqrt{m}} \sum_{r \neq p} z_r \left(\frac{1}{\sqrt{m}} \sum_{q=1}^m U^l_{pq} U^l_{rq} \right)
\end{align*}
By the law of large numbers, the first term will converge almost surely to $z_p$ as $m \rightarrow \infty$. For the second term, the intuition is that for any $r \neq p$ the terms $(1/ \sqrt{m}) \sum_{q} U^l_{pq} U^l_{rq}$ become distributed as independent Gaussians as $m$ becomes large by a central limit argument. Then, by another central limit argument, intuitively, the sum over $r \neq p$ should also becomes distributed as $\mathcal{N}(0, \mathbb{E}[z_1^2])$. In the limit $m \rightarrow \infty$, we thus expect the coordinates of $h = \hatW^l \transpose{{(\hatW^l)}} z$ to be the sum of two terms: a first term distributed as $z_1$ where the correlation between the entries of $\hatW^l$ and $\transpose{{(\hatW^l)}}$ comes into play, and a second term distributed as $\mathcal{N}(0, \mathbb{E}[z_1^2])$ which is purely Gaussian and where the correlation between the entries of $\hatW^l$ and $\transpose{{(\hatW^l)}}$ has no effect. 
\\ \\
The aim of the Tensor Program series~\citep{yang2019tp1, yang2020tp2, yang2020tp3} is to formalize those intuitions into theorems and rigorous calculations. Of course, the calculations become more complex when we introduce non-linearities and consider later steps in training than the initialization, but what the Tensor Program shows is that the intuitions above still hold. 
\\ \\
To summarize, the intuition is that in the large-width limit, the coordinates of pre-activation vectors become \iid and we thus only need to track the law of a single real-valued random variable. Therefore, any average of some function of the coordinates should converge to an expectation in the limit $m \rightarrow \infty$ by a law of large number argument. Finally, any multiplication by $\hatW^l$ yields two terms where one is purely Gaussian and the other depends on the expression of the vector that is multiplied by $\hatW^l$ in function of $\transpose{{(\hatW^l)}}$. 

\subsection{Mathematical formalism}\label{sec:tp-formalism}
The mathematical formalism of the Tensor Program goes beyond neural network computations and describes the evolution of any computational systems (with some restrictions) in the limit $m \rightarrow \infty$. The computational system is comprised of different vectors whose dimensions are equal to $m$ which can be generated from a set of initial vectors in various ways. The Tensor Program is defined by the sequence of mathematical operations which produce the vectors from previously generated vectors. The operations are the same at any given width $m$, only the size of the vectors and matrices involved change with $m$, and the aim of the Tensor Program is to provide the tools (formalism and theorems) to be able to described the behavior of the system in the limit $m \rightarrow \infty$. As described in the intuitions of the previous section~\ref{sec:tp-intuition}, the coordinates of vectors in the program are roughly \iid as $m \rightarrow \infty$ and variables $Z$ are introduced to described the common law of the coordinates in the limit $m \rightarrow \infty$. 

\paragraph{Initial vectors.} 
Consider a set $\mathcal{V} := \left\{v^1, \ldots, v^N \right\} \in (\mathbb{R}^m)^N$ of \textit{initial vectors} such that:
\begin{enumerate}[(i)]
    \item the coordinates $(v_p)_{p \in [m]}$ are \iid for any $v \in \mathcal{V}$ and any $m$. We call $Z^v$ a real-valued random variable whose law is the same as that of all the coordinates.
    
    \item The joint law of $Z^{\mathcal{V}} := (Z^{v^1}, \ldots, Z^{v^N})$ is a Gaussian $\mathcal{N} \left(\mu_{\text{init}}, \Sigma_{\text{init}} \right)$ for any $m$ (the variables $Z^v$ do not actually depend on $m$, but this is simply to say that at any width $m$ and for any $p \in [m]$, the law of $(v^1_p, \ldots, v^N_p)$ is the same $N$-dimensional Gaussian). 
\end{enumerate}

\paragraph{Initial scalars.} 
Similarly, we define a list of initial scalars $\theta_1, \ldots, \theta_M$ which can depend on $m$ and for which the only requirement is that each $\theta_r$ converges almost surely to some finite limit $\scalarlim{\theta}_r$ as $m \rightarrow \infty$. 

\paragraph{Initial Gaussian matrices.} 
Consider a set $\mathcal{W} := \left \{\hatW^1, \ldots, \hatW^P \right \} \in (\mathbb{R}^{m \times m})^{P}$,  such that $\hatW^r_{pq} \sim \mathcal{N}(0, 1/m)$ \iid over $p,q$ for any $r$, and the $(\hatW^r)_{r \in [P]}$ are independent of each other and independent of the vectors in $\mathcal{V}$. Since we consider a more general setting than neural networks, we do not index those matrices by $l$ and can have $P \neq L$ but for neural networks, those initial matrices will always be the initialization of the weight matrices of the intermediate layers $l \in [2, L]$, appropriately scaled. 

\paragraph{Generation of new vectors/scalars.} Given previously generated vectors $v^1, \ldots, v^k$, previously generated scalars $\theta_1, \ldots, \theta_r$, and a non-linearity $\psi(\cdot \, ; \, \cdot) : \mathbb{R}^k \times \mathbb{R}^r \rightarrow \mathbb{R}$, we can, in the following ways, generate:
\begin{description}
\item[\texttt{MatMul}] a vector $z = \hatW v$ for any $v \in \{v^1, \ldots, v^k \}$ and $\hatW \in \mathcal{W}$.
\item[\texttt{NonLin}] a vector $z = \psi(v^1, \ldots, v^k \, ; \, \theta_1, \ldots, \theta_r)$ where $\psi$ is taken element-wise, \ie $z_p = \psi(v^1_p, \ldots, v^k_p \, ; \, \theta_1, \ldots, \theta_r)$ for any $p \in [m]$ and for any $m$.
\item[\texttt{Moment}] a scalar $\omega = \frac{1}{m} \sum_{p=1}^m \psi(v^1_p, \ldots, v^k_p \, ; \, \theta_1, \ldots, \theta_r) \in \mathbb{R}$.
\end{description}
The non-linearity used does not have to actually depend on all the previous vectors and/or scalars, but we present the operations this way for simplicity. 
\\ \\
Given those operations, the Tensor Program framework allows to seamlessly describe the infinite-width limit of the computational system defining a given Tensor Program by tracking recursively the laws of the variables $Z$ whose law represents the common law of the coordinates of a given vector. Indeed, every vector $z$ in the program (initial or generated using previous vectors in the program) will roughly have \iid coordinates in the limit $m \rightarrow \infty$, and the Tensor Program associates a real-valued random variable $Z^{z}$ to the vector $z$. Then, associated with the operations on vectors and scalars above are the following operations on the corresponding variables $Z$ which come as their natural counterparts in the infinite-width limit to track the evolution of the laws of the variables $Z$:
\begin{description}
\item[\texttt{ZInit}] For initial vectors $v \in \mathcal{V}$, define $\dotZ^{v} = 0$ and $\hatZ^{v} = Z^{v}$. The purpose of those notations will become clear in the $\texttt{ZMatMul}$ section. 

\item[\texttt{ZMoment}] Given a scalar $\omega = (1/m) \sum_{p=1}^m \psi(z^1_p, \ldots, z^k_p \,; \, \theta_1, \ldots, \theta_r)$, define
\begin{align}\label{label:zmoment}
    \scalarlim{\omega} = \mathbb{E} \left[\psi(Z^{z^1}, \ldots, Z^{z^k} \, ; \, \scalarlim{\theta}_1, \ldots, \scalarlim{\theta}_1) \right] 
\end{align}

\item[\texttt{ZNonLin}] Given $z =\psi(z^1, \ldots, z^k \,; \, \theta_1, \ldots, \theta_r)$, define:
\begin{align}\label{eq:znonlin}
    Z^{z} = \psi(Z^{z^1}, \ldots, Z^{z^k} \,; \, \scalarlim{\theta}_1, \ldots \scalarlim{\theta}_r)
\end{align}

\item[\texttt{ZMatMul}] Given $z = \hatW v$ for a previous vector $v$ and $\hatW \in \mathcal{W}$, $Z^{z} = \hatZ^{z} + \dotZ^{z}$ is the sum of two terms:
\begin{description}

\item[\texttt{ZHat}]
$\hatZ^{z} \sim \mathcal{N} \left(0, \mathbb{E} \left[\left(Z^v \right)^2 \right] \right)$ is a purely Gaussian term. Additionally, if we let $\mathcal{W}_\hatW$ be the set of all vectors in the program of the form $\hatW u$ for some $u$ in the program, the vector $Z^{\mathcal{W}_\hatW} = (Z^h)_{h \in \mathcal{W}_\hatW}$ is defined to be jointly Gaussian with covariance matrix given by:
\begin{align*}
    \text{cov}(Z^{Wx}, Z^{Wy}) = \mathbb{E}[Z^x Z^y]
\end{align*}
Moreover, the vector $Z^{\mathcal{W}_\hatW}$ is defined to be mutually independent of the list of $Z^{u}$ for $u$ in $\{ \hatZ^{v} : v \in \mathcal{V} \cup_{W \in \mathcal{W} \cup \transpose{\mathcal{W}}, W \neq \hatW}  \mathcal{W}_W \}$ where $\transpose{\mathcal{W}} := \{\transpose{{\hatW}} : \hatW \in \mathcal{W} \}$, and $\mathcal{W}_W$ is the set of vectors in the program of the form $Wu$ for some vector $u$ in the program.

\item[\texttt{ZDot}] $\dotZ^{z}$ comes from the potential interactions (correlations) between $\hatW$ and $\transpose{{\hatW}}$ in the computation of $z$. One can always unwind the expression of $Z^{v}$ and express it in function of the $\hatZ^{\transpose{{\hatW}} y}$ for some $x$ in the program, that is we can always write $Z^{v}$ as $Z^{v} = \phi(\hatZ^{\transpose{\hatW}y^1}, \ldots, \hatZ^{\transpose{\hatW}y^k} \, ,\hatZ^{x^1}, \ldots \hatZ^{x^r} \, ; \, \scalarlim{\theta}_1, \ldots, \scalarlim{\theta}_s)$ with $x^1, \ldots, x^r$ such that $\transpose{{\hatW}}$ is never used in the computation of those vectors. Then, define:
\begin{align}\label{eq:zdot}
    \dotZ^{z} = \sum_{j=1}^k \mathbb{E} \left[ \frac{\partial Z^v}{\partial Z^{\transpose{{\hatW}} y_j}} \right] Z^{y_j}
\end{align}
where $\partial Z^v / \partial Z^{\transpose{{\hatW}} y_j}$ is simply defined as the $j$-th partial derivative of $\phi$ above when expressing $Z^v$ as required for $\dotZ$. As noted in~\citep{yang2020featureLearning}, if $\phi$ is not everywhere differentiable, one can leverage Stein's lemma to replace the formula in Equation~\eqref{eq:zdot} by a linear algebra formula. 
\end{description}
\end{description}
Now that we have introduced the necessary concepts and described the content of a Tensor Program, we can move on to present the main theorem derived in~\citep{yang2020featureLearning} which connects the mathematical operations used at finite-width with the infinite-width limit of the computational system defining a Tensor Program. The \quoting{master theorem} formulated in~\citep{yang2020featureLearning} is surprisingly simple (although the proof is much more intricate) yet very powerful, and goes as follows (see ~\citealp[Theorem 7.4]{yang2020featureLearning}):
\begin{theorem}[Master Theorem]\label{th:master-theorem}
Given a Tensor Program, for any vectors $x^1, \ldots, x^k$ and scalars $\theta_1, \ldots, \theta_r$ in the program, and for any pseudo-Lipschitz non-linearity  $\psi$ (see Definition~\ref{def:pseudo-Lipschitz}, page \pageref{def:pseudo-Lipschitz}), one has that:
\begin{align*}
    \frac{1}{m} \sum_{p=1}^m \psi(x^1_p, \ldots, x^k_p \, ;\, \theta_1, \ldots, \theta_r) \xrightarrow[m \rightarrow \infty]{a.s.} \mathbb{E} \left[\psi \left(Z^{x^1}, \ldots, Z^{x^k} \, ;\, \scalarlim{\theta}_1, \ldots, \scalarlim{\theta}_r \right) \right] 
\end{align*}
\end{theorem}

\begin{remark}
\
\begin{enumerate}[1.]
    \item The theorem essentially states that even though the coordinates of vectors in the program are not rigorously \iid\!, they appear so from the perspective of the average by a suitable non-linearity so that a law of large number type of result holds. Note that for neural networks, even though the coordinates of the (pre-)activations follow the same law when using \iid initialization for the weights, it is not \textit{a priori} clear that we can consider them as independent copies, and thus that we can summarize the computations using a single real-valued variable, but the master theorem shows that from the perspective of averaging, this is in fact the case in the infinite-width limit.
    
    \item In~\citep{yang2020featureLearning}, different versions of the Tensor Program are presented in the sense that different classes of non-linearities are allowed. These differences induce minor subtleties in the master theorem and in the proofs. However, most of the results in the main text of the paper require that the non-linearities be pseudo-Lipschitz (which is the stronger assumption), both in \texttt{NonLin} and in the master theorem. The Assumption~\ref{ass:smooth-act} on the activation function $\sigma$ and its derivative $\sigma'$ ensures that any quantity appearing in the forward or backward computation of a neural network can be expressed as pseudo-Lipschitz non-linearity.
    
    \item What the Tensor Program and its master theorem show is that to understand the behavior of the computational system in the infinite-width limit, one simply needs to track the operations on the variables $Z$ which mimic the recursive operations in the computational system. Then, quantities which involve sum over coordinates such as inner products between the vectors in the program (which occur in the forward and backward passes of a neural network, as well as in the computation of the neural tangent kernel), or norm computations are easily described, when properly re-normalized, through expectations involving the corresponding variables $Z$. The main difficulty is that it is actually hard (computationally and in the mathematical formulation) to track the correlations between different $Z$ because, as explained in~\citep{yang2020featureLearning}, of the necessary unwinding in the definition of $\dotZ$, so that the computational graph associated with the operations on the variables $Z$ is hard to implement in practice. 
\end{enumerate}
\end{remark}

\subsection{The maximal update parameterization \muP}\label{sec:muP}

We close this section by presenting briefly the maximal update parameterization considered in~\citep{yang2020featureLearning}. To quantify the learning abilities of a given parameterization,~\citet{yang2020featureLearning} introduce the notions of \textit{feature learning} and \textit{feature kernel evolution} at a given layer $l \in [1, L]$, which we recall below. Both these definitions concern the large-width limit of the networks: 

\begin{definition}[Feature Learning]\label{def:feature-learning}
An ac-parameterization is said to admit \textbf{feature learning} at the $l$-th layer if the quantity $\Delta x^l_t(\xi) := x^l_t(\xi) - x^l_0(\xi)$ is such that there exists a training routine for which, almost surely, there exists a constant $C > 0$ such that $||\Delta x^l_t(\xi)||^2 / m \geq C$ for large enough $m$.
\end{definition}

\begin{definition}[Kernel Evolution]\label{def:kernel-evolution}
An ac-parameterization is said to \textbf{evolve the feature kernel} at the $l$-th layer if the quantity $\Delta F^l_t(\xi, \Bar{\xi}) := \left[ \transpose{{x^l_t(\xi)}} x^l_t(\Bar{\xi}) - \transpose{{x^l_0(\xi)}} x^l_0(\Bar{\xi}) \right ]  / m$ is such that there exists a training routine for which, almost surely, there exists a constant $C > 0$ such that for large enough $m$, $\Delta F^l_t(\xi, \Bar{\xi}) \geq C$.
\end{definition}
\noindent
\citep{yang2020featureLearning} goes about categorizing whether different ac-parameterizations admit feature learning or not. One of the striking result presented is that there is essentially a dichotomy (depending on the values of $(a_l,c_l)_{l \in [L+1]}$) among ac-parameterizations: an ac-parameterization either admits feature learning (and evolves the feature kernel) or is in the kernel regime, meaning that the quantities in definitions~\ref{def:feature-learning} and ~\ref{def:kernel-evolution} converge to $0$ almost surely so that in the infinite width limit, the evolution of the prediction function $f_t$ is deterministic and depends only on the previous prediction function $f_{t-1}$ and the loss at time $(t-1)$ through a (deterministic) kernel $K(\xi, \Bar{\xi}) = \lim_{m \rightarrow \infty} \transpose{{(x^L_0(\xi))}} x^L_0(\Bar{\xi}) / m$ (or a rescaled version thereof).
\\ \\
The categorization result proved in~\citep{yang2020featureLearning} holds for a certain class of ac-parameterizations which are deemed \textit{stable} and \textit{non-trivial}. Stable refers to the fact that the pre-activations and output ($h^l_0$ and $f_0(\xi)$ respectively) at initialization do not blow-up as $m \rightarrow \infty$ at any layer. As already hinted in Section~\ref{sec:tp-intuition}, this corresponds to having $a_1 = 0$ and $a_l \geq 1/2$ for $l \in [2, L+1]$. Non-trivial refers to the fact that the pre-activations of all layers do not converge to 0 almost surely as $m \rightarrow \infty$ at initialization. This corresponds to having $a_1 \leq 0$ and $a_l \leq 1/2$ for $l \in [2, L]$. It is mentioned in~\citep{yang2020featureLearning} that those parameterizations for which the pre-activations of the intermediate layers converge to $0$ almost surely should stay at their initialization throughout the course of training, and we actually prove in Section~\ref{sec:ip-trivial}, using the Tensor Program technique, that this is the case when $L \geq 3$ in the setting where $a_1 = 0$ and $a_l = 1$ for $l \in [2, L+1]$ (\ie integrable parameterizations) unless one uses large (polynomial in $m$) initial learning rates, a scenario which is not covered in~\citep{yang2020featureLearning}. We show that in this case, integrable parameterizations are only trivial at initialization (the pre-activations of all layers except the first one converge to $0$ in the infinite-width limit) and are actually in a feature learning regime at all layers after the first gradient step ($t \geq 1$).
\\ \\
The maximal update parameterization \muP~introduced in~\citep{yang2020featureLearning} is the result of the analysis of the values of $a_l$, and $c_l$ for which the parameterization admits feature learning at every layer, and maximally so in the sense that if we were to reduce the value of $a_l$ then the $\Delta x^l_t$ introduced in Definition~\ref{def:feature-learning} or the pre-activations $h^l_t$ would blow-up as $m \rightarrow \infty$. In essence, \muP~corresponds to the values of $a_l$, and $c_l$ for which $\Delta x^l_t$ is as large as possible (with regards to its dependency on $m$) at every layer without creating any instabilities (pre-activations or updates blowing-up) in the limit $m \rightarrow \infty$. A quick analysis of the updates at $t=0$ shows that the choice $a_1 = 0$, $a_l = 1/2$ for $l \in [2, L]$, and $a_{L+1} = 1$ associated with $c_l = -1$ for all $l \in [L+1]$ achieves this, and it is rigorously shown in~\citep{yang2020featureLearning} that this choice of ac-parameterization induces an update such that, $||\Delta W^l(t) x^{l-1}_t||^2 / m =\Theta(1)$. We thus adopt the following definition for \muP~which is the same as in~\citep[Definition 5.1]{yang2020featureLearning} but re-parameterized to remove the redundant b in the abc-parameterization:
\begin{definition}[\muP]\label{def:muP}
The maximal update parameterization \muP\ is defined by the following choice of parameterization:
\begin{align*}
    &a_1 = 0, &&c_1 = -1, \\
    &a_l = 1/2, &&c_l = -1, \qquad l \in [2, L],\\
    &a_{L+1} = 1, &&c_{L+1} = -1.
\end{align*}
\end{definition}

\section{Useful preliminary results}\label{sec:prelim-res}
We show in this section a couple of useful results which will prove helpful in the proofs. 
\subsection{Positive finite moments of pseudo-Lipschitz functions of Gaussians}\label{app:pos-finite-moment}

\begin{lemma}[Positive finite moments with polynomially bounded non-linearities]\label{th:pos-finite-moment}
Let $\phi$ be a polynomially bounded non-linearity which is not almost everywhere $0$, and let $Z \sim \mathcal{N}(0, v^2)$ with $v^2 < \infty$. Then, for any $p \in \mathbb{R}_+$:
\begin{enumerate}[(i)]
    \item $0 \leq \mathbb{E}[|\phi (Z)|^p] < \infty$,
    \item if in addition $v^2 > 0$, $0 < \mathbb{E}[|\phi (Z)|^p] < \infty$.
\end{enumerate}
\end{lemma}

\begin{proof}
If $v^2=0$, and then $\phi(Z)= \phi(0)$ almost surely, so that $\mathbb{E}[|\phi(Z)|^p] = |\phi(0)|^p < \infty$.
\\ \\
Now, assume $v^2 > 0$. Since $\phi$ is bounded by a polynomial of some degree $r > 0$, $|\phi(z)|  \leq C(1+ |z|^r)$ for some $C >0$. Then, $|\phi(z)|^p = \exp(p \ln(|\phi(z)|)) \leq C^p (1+ |z|^{r})^p$. Since $v^2 > 0$, we have
\begin{align*}
    \mathbb{E}[|\phi(Z)|^p] &= \frac{1}{\sqrt{2 \pi v^2}}\int_{\mathbb{R}} |\phi(z)|^p e^{-z^2/2v^2} \mathrm{d}z \\
    &\leq \frac{1}{\sqrt{2 \pi v^2}}\int_{\mathbb{R}} C^p (1+ |z|^{r})^p e^{-z^2/2v^2} \mathrm{d}z < \infty.
\end{align*}
Finally, since $\phi$ is not almost everywhere $0$, neither is $|\phi|^p$ which shows the integral in the first equality above is not 0, and gives $\mathbb{E}[|\phi(Z)|^p] > 0$. 
\end{proof}

\subsection{The $Z$ dots are 0 in the first forward-backward pass}\label{app:dotZ-first-forward}

\begin{lemma}[$\dotZ=0$ in the first forward-backward pass]\label{th:dotZ-first-forward}
Consider an ac-parameterization of an $L$-hidden layer fully-connected neural network with $a_1 \geq 0$ and $a_l \geq 1/2$ for $l \in [2, L+1]$, and with a non-linearity satisfying Assumption~\ref{ass:smooth-act}. Then for any $l \geq 2$, $\dotZ^{\hatW^l x^{l-1}_0} = 0$, and for any $l \in [1, L]$, $\dotZ^{\transpose{{(\hatW^l)}} dh^{l}_0} = 0$.
\end{lemma}

\begin{remark}
This lemma applies to the NTK, \muP, and integrable parameterizations (in particular IP-LLR) as well as HP and HPZ.
\end{remark}

\begin{proof}
Consider any ac-parameterization of a fully-connected neural network which has $a_1 \geq 0$ and $a_l \geq 1/2$ for $l \in [2, L+1]$, and with a non-linearity satisfying Assumption~\ref{ass:smooth-act}. Define $\omega_1 = m^{-a_1}$ and $\omega_l = m^{-(a_l - 1/2)}$ for $l \geq 2$, and the initial scalar $\alpha_{L+1} := m^{-a_{L+1}}$. The conditions on the $a_l$ guarantee that the $\omega_l$ converge almost surely to either 0 or 1 and and $\alpha_{L+1}$ converges almost surely to 0, which allows applying the rules of the Tensor Program.
\\\\
For any $l \in [2, L]$, since the computation of $x^{l-1}_0$, and thus of $Z^{x^{l-1}_0}$ do not involve $\transpose{{(\hatW^l)}}$, $\dotZ^{\hatW^lx^{l-1}_0} = 0$ as per the \ZDot\ rule of the Tensor Program. In addition, $Z^{h^l_1} = \omega_1(\hatZ^{\hatW^1 \xi} + \hatZ^{v^1})$ and by definition, $\hatZ^{\hatW^1 \xi} \sim \mathcal{N}(0, ||\xi||^2)$ and $\hatZ^{v^1} \sim \mathcal{N}(0,1)$ are independent Gaussians, which shows that $Z^{h^l_1} \sim \mathcal{N}(0, \scalarlim{\omega}_1^2 (||\xi||^2 + 1))$ whose variance is finite because $\scalarlim{\omega}_1^2 \in \{0, 1\}$. By Lemma~\ref{th:pos-finite-moment}, this also shows that $\mathbb{E}[(Z^{x^1_0})^2] < \infty$. Let $l \in [2, L]$ and assume that $\mathbb{E}[(Z^{h^{l-1}_0})^2] < \infty$ and $\mathbb{E}[(Z^{x^{l-1}_0})^2] < \infty$. We have $h^l_0 = \omega_l \hatW^l x^{l-1}_0 + m^{-a_l} v^l$. Since $m^{-2a_l}$ converges to $0$ almost surely, we can consider it as an initial scalar in the program, which gives by \ZNonLin\ $Z^{h^l_0} = \scalarlim{\omega}_l \hatZ^{\hatW^l x^{l-1}_0} + 0 \times \hatZ^{v^l_1}$. $\hatZ^{v^l_1} \sim \mathcal{N}(0,1)$ by definition since $v^l$ is an initial vector in the program, so that $Z^{h^l_0} = \scalarlim{\omega}_l \hatZ^{\hatW^l x^{l-1}_0} \sim \mathcal{N}(0, \scalarlim{\omega}_l^2 \mathbb{E}[Z^{x^{l-1}_0})^2])$ whose variance is finite by the induction hypothesis and because $\scalarlim{\omega}_l \in \{ 0, 1\}$. Then by Lemma~\ref{th:pos-finite-moment}, we also get that $\mathbb{E}[(Z^{x^l_0})^2] < \infty$, which concludes the induction.
\\\\
Let us now deal with the first backward pass for any ac-parameterization. The result will essentially boil down to having the expectation of the derivatives defining the $\dotZ$ being 0 because the weight matrices are initialized with 0 mean and because of an independence argument.  We have $dx^L_0 = W^{L+1}(0) = m^{-a_{L+1}} U^{L+1}$, and $dh^L_0 = dx^L_0 \odot \sigma'(Z^{h^L_0})$. By \ZNonLin\ we thus have
\begin{align*}
    Z^{dx^L_0} &= \scalarlim{\alpha}_{L+1} Z^{U^{L+1}},\\
    Z^{dh^L_0} &= \scalarlim{\alpha}_{L+1} Z^{U^{L+1}} \sigma'(Z^{h^L_0}).
\end{align*}
Now let $l \in [1, L]$. $dx^{l-1}_0 = \transpose{{(\hatW^l)}} dh^l_0$ gives 
\begin{align*}
    Z^{\transpose{{(\hatW^l)}} dh^l_0} = \hatZ^{\transpose{{(\hatW^l)}} dh^l_0} + \dotZ^{\transpose{{(\hatW^l)}} dh^l_0},
\end{align*}
and to understand what $\dotZ^{\transpose{{(\hatW^l)}} dh^l_0}$ is, we need to expand the expression of $Z^{dh^l_0}$ in function of variables which were generated with $\hatW^l$. So far, the only variable where $\hatW^l$ was used is $h^l_0 = \omega_l \hatW^lx^{l-1}_0$ (with the convention that $x^{0}_0 = \xi_0$). We thus need to expand the expression of $Z^{dh^l_0}$ in function of $\hatZ^{\hatW^lx^{l-1}_0}$. We have, for $l=L$
\begin{align*}
    Z^{dh^L_0} &= \scalarlim{\alpha}_{L+1} Z^{U^{L+1}} \sigma'(\scalarlim{\omega}_L Z^{\hatW^l x^{L-1}_0}) \\ 
     &= \scalarlim{\alpha}_{L+1} \hatZ^{U^{L+1}} \sigma'(\scalarlim{\omega}_L \hatZ^{\hatW^l x^{L-1}_0}),
\end{align*}
where the last equality stems from the fact that $Z^{\hatW^L x^{L-1}_0} = \hatZ^{\hatW^L x^{L-1}_0}$ in the first forward pass, and the fact that $U^{L+1}$ is an initial vector in the program which gives by definition $\hatZ^{U^{L+1}} = Z^{U^{L+1}}$. We can formally write this as 
\begin{align*}
    Z^{dh^L_0} = \Psi(\hatZ^{\hatW^L x^{L-1}_0}, \hatZ^{U^{L+1}}; \scalarlim{\alpha}_{L+1}, \scalarlim{\omega}_L),
\end{align*}
where $\Psi(z_1, z_2; \theta_1, \theta_2) := \theta_1 z_2 \sigma'(\theta_2 z_1)$ is a pseudo-Lipschitz function because $\sigma'$ is, and we have
\begin{align*}
    \frac{\partial \Psi}{\partial z_1}(z_1, z_2; \theta_1, \theta_2) = \theta_1 \theta_2 z_2 \sigma''(\theta_2 z_1).
\end{align*}
We get that by definition
\begin{align*}
    \dotZ^{\transpose{{(\hatW^L)}} dh^L_0} &= \mathbb{E} \left[\frac{\partial Z^{dh^L_0}}{\partial \hatZ^{\hatW^L x^{L-1}_0}} \right] Z^{x^{L-1}_0} \\
    &= \mathbb{E} \left[\frac{\partial \Psi}{\partial z_1}(\hatZ^{\hatW^L x^{L-1}_0}, \hatZ^{U^{L+1}}; \scalarlim{\alpha}_{L+1}, \scalarlim{\omega}_L) \right] Z^{x^{L-1}_0} \\
    &= \scalarlim{\alpha}_{L+1} \scalarlim{\omega}_L \mathbb{E}[Z^{U^{L+1}} \sigma''(\scalarlim{\omega}_L \hatZ^{\hatW^L x^{L-1}_0})] Z^{x^{L-1}_0} \\
    &= \scalarlim{\alpha}_{L+1} \scalarlim{\omega}_L \underbrace{\mathbb{E}[\hatZ^{U^{L+1}}]}_{0} \underbrace{\mathbb{E}[\sigma''(\scalarlim{\omega}_L \hatZ^{\hatW^L x^{L-1}_0})]}_{< \infty} \underbrace{Z^{x^{L-1}_0}}_{< \infty \ a.s.},
\end{align*}
where the last equality stems from the fact that by \ZHat, $\hatZ^{\hatW^L x^{L-1}_0}$ is independent of $\hatZ^{U^{L+1}}$ because $U^{L+1}$ is an initial vector in the program. The fact that the second expectation finite is because $\scalarlim{\omega}_L \in \{0,1\}$, $\sigma''$ is polynomially bounded, and $\hatZ^{\hatW^L x^{L-1}_0}$ is a Gaussian with mean 0 and finite variance since $\mathbb{E}[(Z^{x^{L-1}_0})^2] < \infty$. This gives $\dotZ^{\transpose{{(\hatW^l)}} dh^L_0} = 0$. 
\\\\
Now suppose $l \in [1, L-1]$ and assume $\dotZ^{\transpose{{(\hatW^{l+1})}} dh^{l+1}_0} = 0$ which gives $Z^{\transpose{{(\hatW^{l+1})}} dh^{l+1}_0} = \hatZ^{\transpose{{(\hatW^{l+1})}} dh^{l+1}_0}$. We have
\begin{align*}
    Z^{dh^l_0} &= Z^{dx^l_0} \sigma'(Z^{h^l_0}) \\
    &= \scalarlim{\omega}_{l+1} Z^{\transpose{{(\hatW^{l+1})}} dh^{l+1}_0} \sigma'(\scalarlim{\omega}_l Z^{\hatW^l x^{l-1}_0}) \\ 
    &= \scalarlim{\omega}_{l+1} \hatZ^{\transpose{{(\hatW^{l+1})}} dh^{l+1}_0} \sigma'(\scalarlim{\omega}_l \hatZ^{\hatW^l x^{l-1}_0})
\end{align*}
where we have used that previous $\dotZ$ are 0 to replace the $Z$ with $\hatZ$. We can once more formally write this as 
\begin{align*}
    Z^{dh^l_0} = \Psi(\hatZ^{\hatW^l x^{l-1}_0}, \hatZ^{\transpose{{(\hatW^{l+1})}} dh^{l+1}_0}; \scalarlim{\omega}_{l+1}, \scalarlim{\omega}_l)
\end{align*}
with exactly the same $\Psi$ as for $l=L$. We get that by definition
\begin{align*}
    \dotZ^{\transpose{{(\hatW^l)}} dh^l_0} &= \mathbb{E} \left[\frac{\partial Z^{dh^l_0}}{\partial \hatZ^{\hatW^l x^{l-1}_0}} \right] Z^{x^{l-1}_0} \\
    &= \mathbb{E} \left[\frac{\partial \Psi}{\partial z_1}(\hatZ^{\hatW^l x^{l-1}_0}, \hatZ^{\transpose{{(\hatW^{l+1})}} dh^{l+1}_0}; \scalarlim{\omega}_{l+1}, \scalarlim{\omega}_l) \right] Z^{x^{l-1}_0} \\
    &= \scalarlim{\omega}_{l+1} \scalarlim{\omega}_l \mathbb{E}[\hatZ^{\transpose{{(\hatW^{l+1})}} dh^{l+1}_0} \sigma''(\scalarlim{\omega}_{l} \hatZ^{\hatW^l x^{l-1}_0})] Z^{x^{l-1}_0}\\
    &= \scalarlim{\omega}_{l+1} \scalarlim{\omega}_l \underbrace{\mathbb{E}[\hatZ^{\transpose{{(\hatW^{l+1})}} dh^{l+1}_0}]}_{0} \underbrace{\mathbb{E}[\sigma''(\scalarlim{\omega}_l \hatZ^{\hatW^l x^{l-1}_0})]}_{< \infty} \underbrace{Z^{x^{l-1}_0}}_{< \infty a.s.} \\
    &= 0
\end{align*}
Where the first expectation is 0 because by definition $\hatZ^{\transpose{{(\hatW^{l+1})}} dh^{l+1}_0}$ is a Gaussian with 0 mean and an easy induction (from $l=L$ to $l=1$) shows that, as for the forward pass, $\mathbb{E}[(Z^{dx^l_0})^2] < \infty$ and $\mathbb{E}[(Z^{dh^l_0})^2] < \infty$, which implies that $\hatZ^{\transpose{{(\hatW^{l+1})}} dh^{l+1}_0}$ has finite variance. The second expectation is finite because $\scalarlim{\omega}_l \in \{0,1\}$, $\hatZ^{\hatW^l x^{l-1}_0}$ is a Gaussian with 0 mean by definition and finite variance, and because $\sigma''$ is polynomially bounded since $\sigma'$ is pseudo-Lipschitz.  
\end{proof}

\subsection{Gaussian output in the infinite-width limit}\label{app:gaussian-output}

\begin{lemma}[Gaussian output]\label{th:gaussian-output}
For every $m \in \mathbb{N}^*$, let $x^m$ and $w^m$ be independent random vectors in $\mathbb{R}^m$ such that
\begin{align*}
    \begin{cases}
    \frac{1}{m} ||x^m||^2 \xrightarrow[m \rightarrow \infty]{a.s.} \sigma^2_\infty\\
    w^m_j \sim \mathcal{N}(0, 1/m) \text{ \iid over } j=1,\ldots,m.
    \end{cases}
\end{align*}
Then 
\begin{align*}
    \transpose{{\left(w^m \right)}} x^m \xrightarrow[m \rightarrow \infty]{law} \mathcal{N}(0, \sigma^2_\infty)    
\end{align*}
\end{lemma}

\begin{proof}
Consider two sequences of independent vectors of growing dimension $(w^m)_m$ and $(x^m)_m$ as in Lemma~\ref{th:gaussian-output}.
Conditionally on $x^m$, the random variable $\transpose{{\left(w^m \right)}} x^m$ follows a Gaussian $\mathcal{N}(0,||x^m||^2 /m)$ distribution. Since $||x^m||^2 /m$ converges to $\sigma^2_\infty$ almost surely, the conditional distribution of $\transpose{{\left(w^m \right)}} x^m$ given $x^m$ converges to a Gaussian $\mathcal{N}(0,\sigma^2_\infty)$ distribution. The lemma  follows. 
\end{proof}

\subsection{Convergence of the coordinates to the limiting distribution $Z$}\label{app:cv-law-limit-z}

\begin{lemma}[Convergence to the limit distribution]\label{th:cv-law-limit-z}
For any vector $h$ in the Tensor Program we have for any $\alpha \in \mathbb{N}^*$,
\begin{align*}
    h_\alpha \xrightarrow[m \rightarrow \infty]{law} Z^h
\end{align*}
\end{lemma}

\begin{remark} 
\
\begin{enumerate}[1.]
    \item Let $h^1, \ldots, h^k$ be $k$ vectors in the program, let $\theta_1, \ldots, \theta_p$ be $p$ scalars in the program, and let $\phi : \mathbb{R}^{k+p} \rightarrow \mathbb{R}$ be a pseudo-Lipschitz function. Then applying the previous Lemma~\ref{th:cv-law-limit-z} to $h := \phi(h^1, \ldots, h^k; \theta_1, \ldots, \theta_p)$ (which is in the program by \NonLin), shows that for any $\alpha$, $\phi(h^1_\alpha, \ldots, h^k_\alpha; \theta_1, \ldots, \theta_p)$ converges in law to $Z^h = \phi(Z^{h^1}, \ldots, Z^{h^k}; \scalarlim{\theta}_1, \ldots, \scalarlim{\theta}_p)$.
    
    \item A stronger form of convergence can occur depending on the parameterization we look at and the context. Indeed, if for example $Z^h$ turns out to be a constant, then we already get convergence in probability instead of in law. If in addition the convergence is \quoting{fast enough}, it can occur almost surely. 
\end{enumerate}

\end{remark}

\begin{proof}
Let $h$ be a vector in the program, and consider the corresponding random variable $Z^h$. All we need is to prove that for any $\alpha\in  \mathbb{N}^*$ and any bounded 1-Lipschitz function $\phi$, we have
$\mathbb{E}[\phi(h_\alpha)]\to \mathbb{E}[\phi(Z^h)]$, as $m$ goes to infinity. We first observe that the Master Theorem \ref{th:master-theorem} ensures the convergence
\begin{align*}
    \frac{1}{m} \sum_{\beta=1}^m \phi(h_\beta) \xrightarrow[m \rightarrow \infty]{a.s.} \mathbb{E}[\phi(Z^h)].
\end{align*}
Secondly, for any $m$, the distribution of $h_1, \ldots, h_m$ is exchangeable by symmetry, so that we get
\begin{align*}
    \mathbb{E}[\phi(h_\alpha)] = \mathbb{E} \left[\frac{1}{m} \sum_{\beta=1}^m \phi(h_\beta) \right] \xrightarrow[m \rightarrow \infty]{} \mathbb{E}[\phi(Z^h)],
\end{align*}
where the convergence is obtained by dominated convergence, which concludes the proof.
\end{proof}

\section{Proof of the triviality of IPs: Proposition~\ref{th:trivial-ip-mf-lr}}\label{app:trivial-ip}

\begin{proof}
Fix a time $t \geq 0$ and an input $\xi \in \mathbb{R}^d$ for the whole proof. We first show that the coordinates of the (pre-)activations of any layer $l \geq 2$ converge to $0$ almost surely at initialization. To that end, we prove that the corresponding $Z$'s are equal to 0. Then we show a similar result for the backward pass, and finally conclude the proof by an induction.

\subsection{Proof at $t=0$}

\subsubsection{First forward pass}

\textbf{Tensor program setup:} We consider a Tensor Program as defined in 
\begin{align*}
    \begin{cases}
        \hatW^{l+1} = U^{L+1},\\
        U^1 \xi_0, \ldots, U^1 \xi_t, U^1\xi, \\
        v^1, \ldots, v^{L},
    \end{cases}
\end{align*}
and the initial scalars
\begin{align*}
    \begin{cases}
        \chi_0, \ldots, \chi_t, \\
        \omega:= m^{-1/2}, \nu := m^{-1}, \tau := m^{-2}, \\
        m^{-1} v^{L+1};
    \end{cases}
\end{align*}
and with initial weight matrices
\begin{align*}
    \hatW^2, \ldots, \hatW^L.
\end{align*}
Recall that the $\hatW^l$ are defined in Equation~\eqref{eq:hatW} of Appendix~\ref{app:notations}. Note that for any $m \in \mathbb{N}^*$ and $j \in [m]$, we have
\begin{align*}
    \left(U^{L+1}_j, (U^1 \xi_0)_j, \ldots, (U^1 \xi_t)_j, (U^1 \xi)_j, v^1_j, \ldots, v^L_j \right)  \sim \mathcal{N} \left( 0,
    \begin{pmatrix}
        1 & 0 & 0 \\
        0 & M & 0 \\
        0 & 0 & I_L
    \end{pmatrix}
    \right) ,
\end{align*}
where $M := \text{Gram}(\xi_0, \ldots, \xi_t, \xi) = (\transpose{\xi_r} \xi_s)_{0 \leq r,s \leq t+1}$ and $I_L$ is the identity matrix of size $L \times L$.
where we have set $\xi_{t+1} := \xi$. 
\\\\
\textbf{Convergence of the initial scalars:} $\omega, \nu, \tau$ as well as $m^{-1} v^{L+1}$ all converge almost surely towards $0$. For the $\chi_s$ we will show below in the proof that they all converge to constants almost surely, thereby meeting the requirements of the Tensor Program. It is important to note that there is no circular logic to prove the $\chi_s$ converge almost surely. Indeed, each time we apply the master theorem to prove the convergence of $f_s(\xi_s)$ to a constant almost surely and thus that of $\chi_s$, we apply it to a restricted Tensor Program where only the scalars $(\chi_r)_{0 \leq r < s}$ appear (and there is no such scalar needed to prove the convergence of $\chi_0$ as shown below) which will already have been proved to converge almost surely.  
\\\\
\textbf{1st forward pass:} We drop the dependency of the forward and backward passes on $\xi$ for brevity. $h^1_0 = U^1 \xi + v^1$ is th sum of two initial vectors in the program and has iid Gaussian coordinates $\mathcal{N}(0, ||\xi||^2 + 1)$. By definition, $\hatZ^{h^1_0} = \hatZ^{U^1\xi} + \hatZ^{v^1} \sim \mathcal{N}(0, ||\xi||^2 + 1)$ since the two Gaussians appearing in the sum are independent. By \NonLin, we have that since $x^1_0 = \sigma(h^1_0)$, $Z^{x^1_0} = \sigma(Z^{h^1_0})$. Note that $\mathbb{E}[\sigma(Z^{h^1_0})^2] < \infty$ since $Z^{h^1_0}$ is Gaussian with finite variance and $\sigma$ is pseudo-Lipschitz and thus polynomially bounded. 
\\ \\
Since $L \geq 2$, we can write $h^2_0 = m^{-1/2} \hatW^2 x^1_0 + m^{-1}v^2$ (otherwise there is no $h^2_0$ and we simply have $f_0(\xi_0) = m^{-1} \transpose{{(U^2)}} x^1_0$), which implies by \NonLin\ that $Z^{h^2_0} = \scalarlim{\omega} Z^{\hatW^2 x^1_0} + \scalarlim{\nu} Z^{v^2}$ with $\scalarlim{\omega} = \scalarlim{\nu} =  0$ and
\begin{align*}
    Z^{\hatW^2 x^1_0} = \hatZ^{\hatW^2 x^1_0} + \dotZ^{\hatW^2 x^1_0}.
\end{align*}
$\dotZ^{\hatW^2 x^1_0} = 0$ by Lemma~\ref{th:dotZ-first-forward}, and $\hatZ^{\hatW^2 x^1_0} \sim \mathcal{N}(0, \mathbb{E}[(Z^{x^1_0})^2])$ and $0 \leq \mathbb{E}[(Z^{x^1_0})^2] < \infty$.  We thus have $Z^{h^2_0} = \scalarlim{\omega} \hatZ^{\hatW^2(0) x^1_0} = 0$. Similarly, we also get that $\scalarlim{\\nu} Z^{v^2} = 0$. We then have by \ZNonLin\ $Z^{x^2_0} = \sigma(Z^{h^2_0}) = \sigma(0) = 0$.\\

\noindent
Let $l \in [2, L-1]$ and assume $Z^{h^l_0} = 0$. Then, $Z^{x^l_0} = \sigma(Z^{h^l_0}) = 0$, and since $h^{l+1}_0 = \omega \hatW^{l+1} x^l_0 + \nu v^{l+1}$, by \ZNonLin, $Z^{h^{l+1}_0} = \scalarlim{\omega} Z^{\hatW^{l+1} x^l_0} + \scalarlim{\nu} Z^{v^{l+1}}$ where by \ZMatMul,
\begin{align*}
    Z^{\hatW^{l+1} x^l_0} = \hatZ^{\hatW^{l+1} x^l_0} + \dotZ^{\hatW^{l+1} x^l_0},
\end{align*}
and $\dotZ^{\hatW^{l+1} x^l_0} = 0$ by Lemma~\ref{th:dotZ-first-forward}. By \ZHat, $\hatZ^{\hatW^{l+1} x^l_0} \sim \mathcal{N}(0, \mathbb{E}[(Z^{x^l_0})^2])$, and since $\scalarlim{\omega} = 0$, $\scalarlim{\omega} \hatZ^{\hatW^{l+1} x^l_0} = 0$. Similarly, $\scalarlim{\nu} Z^{v^{l+1}} = 0$. Then, by \ZNonLin\ $Z^{x^{l+1}_0} = \sigma(Z^{h^{l+1}_0}) = \sigma(0) = 0$, which concludes the induction. \\ \\
We thus have only to deal with the last layer $L+1$ to finish the first forward pass. We have $f_0(\xi) = m^{-1} (\transpose{{(U^{L+1}(0))}} x^L_0 +v^{L+1}) = (1/m) \sum_{i=1}^m U^{L+1}_i x^L_{0,i} + m^{-1}v^{L+1}$. Since $U^{L+1}$ and $x^L_0$ are vectors in the program, $(1/m) \sum_{i=1}^m U^{L+1}_i x^L_{0,i}$ is a scalar in the program by the \Moment\ rule, and it therefore converges almost surely to $\mathbb{E}[Z^{U^{L+1}} Z^{x^L_0}]$ by the Master Theorem. Now because $U^{L+1}$ is an initial vector in the program, by definition, $Z^{U^{L+1}} = \hatZ^{U^{L+1}} \sim \mathcal{N}(0,1)$ is independent of $Z^{x^L_0}$. We thus get $\mathbb{E}[Z^{U^{L+1}} Z^{x^L_0}] = \mathbb{E}[Z^{U^{L+1}}] \mathbb{E}[Z^{x^L_0}] = 0$. On the other hand, $m^{-1} v^L$ is an initial scalar in the program which converges to $0$ almost surely, so that $f_0(\xi)$ converges almost surely to 0. 

\subsubsection{First backward pass}

\textbf{1st backward pass:} We can apply the previous reasoning of the forward pass with $\xi_0$ instead of $\xi$ and we get that $f_0(\xi_0) \rightarrow 0$ almost surely. Therefore, since $\chi_0 = \partial_2 \ell(y_0, f_0(\xi_0))$ and $\partial_2 \ell(y_0, \cdot)$ is continuous by assumption, $\chi_0 \rightarrow \partial_2 \ell(y_0, 0) =: \scalarlim{\chi}_0$ almost surely. We have $dx^L_0 = m^{-1} U^{L+1}$ which makes it a vector in the program by \NonLin, and $Z^{dx^L_0} = \scalarlim{\nu} Z^{U^{L+1}}$. Since $Z^{U^{L+1}} \sim \mathcal{N}(0, 1)$ has finite variance and $\scalarlim{\nu} = 0$, we have $Z^{dx^L_0} = 0$. $dh^L_0 = dx^L_0 \odot \sigma'(h^L_0)$ implies by \ZNonLin\ $Z^{dh^L_0} = Z^{dx^L_0}  \sigma'(Z^{h^L_0}) = 0 \times \sigma'(0) = 0$. 
\\ \\
One has:
\begin{align*}
    Z^{m dx^{L-1}_0} = \scalarlim{\omega} ( \hatZ^{\transpose{{(\hatW^L)}} (m dh^L_0)} + \dotZ^{\transpose{{(\hatW^L)}} (m dh^L_0)}),
\end{align*}
where $m dh^L_0 = U^{L+1} \odot \sigma'(h^L_0)$. By Lemma~\ref{th:dotZ-first-forward}, $\dotZ^{\transpose{{(\hatW^L)}} (mdh^L_0)} = 0$ (essentially, $\hatW^l$ never appears in the computation of $dh^L_0$), and by \ZHat, $\hatZ^{\transpose{{(\hatW^L)}} (mdh^L_0)} \sim \mathcal{N}(0, \mathbb{E}[(Z^{m dh^L_0})^2])$, and by independence of $Z^{U^{L+1}}$ and $Z^{h^L_0}$, 
\begin{align*}
    \mathbb{E}[(Z^{m dh^L_0})^2]) = \mathbb{E}[(Z^{U^{L+1}})^2]) \mathbb{E}[\sigma'(Z^{h^L_0})^2]) = \sigma'(0)^2
\end{align*}
which is finite. Since $\scalarlim{\omega} = 0$ we get $Z^{m dx^{L-1}_0} = 0$. $dh^{L-1}_0 = dx^{L-1}_0 \odot \sigma'(h^{L-1}_0)$ implies by \ZNonLin\ $Z^{m dh^{L-1}_0} = Z^{m dx^{L-1}_0}  \sigma'(Z^{h^{L-1}_0}) = 0 \times \sigma'(Z^{h^{L-1}_0}) = 0$.
\\ \\
Let $l \in [2, L]$ (which is non-empty since $L \geq 2$) and assume $Z^{m dx^l_0} = Z^{m dh^l_0} = 0$. $m dx^{l-1}_0 = \omega \transpose{{(\hatW^l)}} (m dh^l_0)$ implies by \ZMatMul
\begin{align*}
    Z^{m dx^{l-1}_0} = \scalarlim{\omega} ( \hatZ^{\transpose{{(\hatW^l)}} (m dh^l_0)} + \dotZ^{\transpose{{(\hatW^l)}} (m dh^l_0)}).
\end{align*}
By Lemma~\ref{th:dotZ-first-forward}, $\dotZ^{\transpose{{(\hatW^l)}} (m dh^l_0)} = 0$, and by \ZHat, $\hatZ^{\transpose{{(\hatW^l)}} (m dh^l_0)} \sim \mathcal{N}(0, \mathbb{E}[(Z^{m dh^l_0})^2])$. By the assumption above, $\mathbb{E}[(Z^{m dh^l_0})^2]) = 0$, and since $\scalarlim{\omega} = 0$ we get $Z^{m dx^{l-1}_0} = 0$. $dh^{l-1}_0 = dx^{l-1}_0 \odot \sigma'(h^{l-1}_0)$ implies by \ZNonLin\ $Z^{m dh^{l-1}_0} = Z^{m dx^{l-1}_0}  \sigma'(Z^{h^{l-1}_0}) = 0 \times \sigma'(Z^{h^{l-1}_0})$. $Z^{h^{l-1}_0}$ is not $0$ if $l=2$, but since it is Gaussian with finite variance, and $\sigma'$ is pseudo-Lipschitz by assumption,  $\sigma'(Z^{h^{l-1}_0})$ is finite almost surely, and $Z^{m dh^{l-1}_0} = 0$ almost surely, which concludes the induction.

\subsection{Induction step}

\textbf{Induction:} Since we proved the result of the theorem for $t=0$ in the first forward pass, we might as well assume $t \geq 1$. Let $s \in [0, t-1]$ be an integer. In all that follows, for any $r \in [0, s]$, for $z \in \{ h^l_r, x^l_r, dh^l_r, dx^l_r  \}$, we use $z$ to denote $z(\xi_r)$. We make the following induction hypothesis: for any $r \in [0, s]$
\begin{align*}
    \begin{cases}
        Z^{h^1_r} = Z^{U^1 \xi_r + v^1}  \sim \mathcal{N}(0, ||\xi_r||^2 + 1) \\
        Z^{h^l_r} = 0 \ \text{ almost surely}, \ l \in [2, L] \\
        f_r(\xi_r), f_r(\xi) \rightarrow 0 \ \text{ almost surely } \\
        \chi_r \rightarrow \scalarlim{\chi}_r := \partial_2 \ell(y_r, 0) \  \text{ almost surely }\\
        Z^{mdx^l_r} = Z^{mdh^l_r} = 0 \ \text{ almost surely}, \ l \in [1, L-1], \\
        Z^{mdx^L_r} = U^{L+1}.
    \end{cases}
\end{align*}
The aim is then to prove the same claims for $r=s+1$. Let us first start with the expressions of $\Delta W^l(s+1)$ and $\Delta B^l(s+1)$. We will use Equation~\eqref{eq:delta-W} and the fact that $c_l + 2 \geq 0$ if $l \in [2, L]$, and $c_l + 1 \geq 0$ for $l =1$, and $l = L+1$. We have by Equations~\eqref{eq:delta-W} and~\eqref{eq:delta-W-L+1}
\begin{align*}
    &\Delta W^1(s+1) = -\eta m^{-c_l} \sum_{r=0}^s \chi_r dh^1_r \transpose{\xi_r}, \\
    &\Delta W^l(s+1) = -\eta m^{-(2 + c_l)}\sum_{r=0}^s \chi_r dh^l_r \transpose{{(x^{l-1}_r)}}, \quad l \in [2, L], \\
    &\Delta W^{L+1}(s+1) = -\eta m^{-(1 +c_{L+1})} \sum_{r=0}^s \chi_r x^L_r / m,
\end{align*}
and by Equations~\eqref{eq:delta-B} and~\eqref{eq:delta-B-L+1}
\begin{align*}
    &\Delta B^1(s+1) = -\eta m^{-c_l} \sum_{r=0}^s \chi_r dh^1_r, \\
    &\Delta B^l(s+1) = -\eta m^{-(2 + c_l)}\sum_{r=0}^s \chi_r dh^l_r, \quad l \in [2, L], \\
    &\Delta B^{L+1}(s+1) = -\eta m^{-(1 +c_{L+1})} \sum_{r=0}^s \chi_r / m.
\end{align*}
In the following, we use for $z \in \{ h^l_{s+1}, x^l_{s+1}, dh^l_{s+1}, dx^l_{s+1}  \}$, we use $z$ to denote $z(\xi)$ (and not $z(\xi_{s+1})$ for now). Using that in the Naive-IP, $c_1 = c_{L+1} = -1$, and $c_l = -2$ for $l \in [2, L]$, we have 
\begin{align*}
    &\Delta W^1(s+1) \xi + \Delta B^1(s+1) = -\eta \sum_{r=0}^s (\transpose{\xi_s} \xi + 1) \chi_r (m dh^1_r), \\
    &\Delta W^l(s+1) x^{l-1}_{s+1} + \Delta B^l(s+1) = -\eta \sum_{r=0}^s \chi_r \frac{(\transpose{{(x^{l-1}_{r})}}  x^{l-1}_{s+1}) + 1}{m}  (mdh^l_r) , \ \ l \in [2, L], \\
    &\transpose{{(\Delta W^{L+1}(s+1))}} x^{L}_{s+1} + \Delta B^{L+1}(s+1) = -\eta \sum_{r=0}^s \chi_r \frac{\transpose{{(x^L_r)}} x^{L}_{s+1} + 1}{m}.
\end{align*}
To prove the claims above for $r=s+1$, we will first induct from $l=1$ to $l=L$ for the forward pass and then induct from $l=L$ to $l=1$ for the backward pass.

\subsubsection{Forward pass at step $s+1$}

\textbf{Forward pass at step $(s+1)$:} $h^1_{s+1} = U^{1}\xi + v^1 + \Delta W^1(s+1) \xi + \Delta b^1(s+1)$ and by \ZNonLin 
\begin{align*}
    Z^{h^1_{s+1}} &= \hatZ^{U^{1}\xi} - \eta \sum_{r=0}^s (\transpose{\xi_s} \xi + 1) \scalarlim{\chi}_r \underbrace{Z^{m dh^1_r}}_{0 \ a.s.} \\
    &= \hatZ^{U^{1}\xi} = Z^{h^1_0(\xi)} \ \ \text{ almost surely}.
\end{align*}
Note that the scalars $(\chi_r)_{0 \leq r \leq s}$ are now valid scalars in the program by the induction hypothesis which allows applying the Tensor Program rules with those scalars as well as the master theorem. This gives $Z^{h^1_{s+1}} \sim \mathcal{N}(0, ||\xi||^2 + 1)$, and we then have $Z^{x^1_{s+1}} = \sigma(\hatZ^{h^1_0(\xi)}) = Z^{x^1_0(\xi)}$ for which we have already proven  $\mathbb{E}[(Z^{x^1_0(\xi)})^2] < \infty$. 
\begin{align*}
    h^2_{s+1} = \omega \hatW^2 x^1_{s+1} + \tau v^2 - \eta \sum_{r=0}^s \chi_r \frac{\transpose{{(x^{1}_{r})}}  x^{1}_{s+1} + 1}{m}  (mdh^2_r).
\end{align*}
Because $x^1_{s+1}$ is a vector in the program, by \ZMatMul
\begin{align*}
    Z^{\hatW^2 x^1_{s+1}} = \hatZ^{\hatW^2 x^1_{s+1}} + \dotZ^{\hatW^2 x^1_{s+1}},
\end{align*}
and because $Z^{x^1_{s+1}} = \sigma(\hatZ^{U^1 \xi + v^1})$ is only a function of the initial vectors $U^1 \xi$ and $v^1$, and not of any vector computed used $\transpose{{(\hatW^2)}}$, $\dotZ^{\hatW^2 x^1_{s+1}} = 0$ by \ZDot, and $\hatZ^{\hatW^2 x^1_{s+1}} \sim \mathcal{N}(0, \mathbb{E}[(Z^{x^1_{s+1}})^2])$ is a Gaussian with finite variance by \ZHat. $\transpose{{(x^{1}_{r})}}  x^{1}_{s+1} / m$ is a valid scalar in the program by the moment rule, and by the Master theorem, 
\begin{align*}
    (\transpose{{(x^{1}_{r})}}  x^{1}_{s+1} + 1)/ m \xrightarrow[m \rightarrow \infty]{a.s.} \mathbb{E}[Z^{x^1_r} Z^{x^1_{s+1}}] = \mathbb{E}[\sigma(Z^{U^1 \xi_r + v^1}) \sigma(Z^{U^1 \xi + v^1})],
\end{align*}
and because $U^1\xi_r$, $v^1$ and  $U^1 \xi$ are initial vectors in the program, $(Z^{U^1 \xi_r + v^1}, Z^{U^1 \xi + v^1})$ is jointly Gaussian by definition with finite covariance matrix 
\begin{align*}
    \begin{pmatrix}
        ||\xi_r||^2 + 1 & \transpose{\xi_r} \xi +1 \\
        \transpose{\xi} \xi_r +1 & ||\xi||^2 +1
    \end{pmatrix},
\end{align*}
which ensures the expectation above is finite because $\sigma$ is polynomially bounded since it is pseudo-Lipschitz. We thus habe
\begin{align*}
    Z^{h^2_{s+1}} &= 0 \times \underbrace{\hatZ^{\hatW^2 x^1_{s+1}}}_{< \infty} + 0 \times Z^{v^2} - \eta \sum_{r=0}^s \underbrace{\scalarlim{\chi}_r}_{< \infty} \underbrace{\mathbb{E}[Z^{x^1_r} Z^{x^1_{s+1}}]}_{<\infty} \underbrace{Z^{mdh^2_r}}_{0} \\
    Z^{h^2_{s+1}} &= 0.
\end{align*}
We then get $Z^{x^2_{s+1}} = \sigma(0) = 0$ and thus $\mathbb{E}[(Z^{x^2_{s+1}})^2] = 0$. 
\\\\
Let $l \in [2, L-1]$ and assume $Z^{h^l_{s+1}} = 0$. 
\begin{align*}
    h^{l+1}_{s+1} = \omega \hatW^{l+1} + \tau v^{l+1} + x^l_{s+1} - \eta \sum_{r=0}^s \chi_r \frac{\transpose{{(x^{l}_{r})}}  x^{l}_{s+1} + 1}{m}  (mdh^{l+1}_r).
\end{align*}
Now, since $x^l_{s+1}$ is a vector in the program, $(\transpose{{(x^{l}_{r})}}  x^{l}_{s+1} +1 )/ m$ is a scalar in the program by the  \Moment\ operation, which converges almost surely, by the Master Theorem, to 
\begin{align*}
    \mathbb{E}[Z^{x^l_r} Z^{x^l_{s+1}}] = \mathbb{E}[\sigma(Z^{h^l_r}) \sigma(Z^{h^l_{s+1}})] = \sigma(0)^2 = 0.
\end{align*}
By \ZNonLin,
\begin{align*}
    Z^{h^{l+1}_{s+1}} &= \scalarlim{\omega}  Z^{\hatW^{l+1} x^l_{s+1}} + \scalarlim{\tau} Z^{v^{l+1}} - \eta \sum_{r=0}^s \underbrace{\scalarlim{\chi}_r}_{< \infty} \underbrace{\mathbb{E}[Z^{x^l_r} Z^{x^l_{s+1}}]}_{<\infty} \underbrace{Z^{m dh^{l+1}_r}}_{0}.
\end{align*}
On the other hand, 
\begin{align*}
    Z^{\hatW^{l+1} x^l_{s+1}} = \hatZ^{\hatW^{l+1} x^l_{s+1}} + \dotZ^{\hatW^{l+1} x^l_{s+1}},
\end{align*}
and since $Z^{x^l_{s+1}} = \sigma(Z^{h^l_{s+1}}) = \sigma(0) = 0$ is a constant almost surely, the derivatives defining $\dotZ$ are equal to 0 (its expression as a function of the previous $\hatZ$ is a constant because any $\hatZ$ gets multiplied by 0) so that $\dotZ^{\hatW^{l+1} x^l_{s+1}} = 0$, and $\hatZ^{\hatW^{l+1} x^l_{s+1}} \sim \mathcal{N}(0, \mathbb{E}[(Z^{x^l_{s+1}})^2]) = 0$. With $\scalarlim{\omega}_{l+1} = 0$ and $\scalarlim{\tau} = 0$, this yields $Z^{h^{l+1}_{s+1}} = 0$, and therefore $\mathbb{E}[(Z^{x^{l+1}_{s+1}})^2] = \mathbb{E}[\sigma(Z^{h^{l+1}_{s+1}})^2] = \sigma(0)^2 = 0$. 
\\\\
We now deal with the last layer $l=L+1$ in the forward pass. 
\begin{align*}
    f_{s+1}(\xi) = m^{-1} \transpose{{(U^{L+1})}} x^L_{s+1} -\eta \sum_{r=0}^s \chi_r \frac{\transpose{{(x^L_r)}} x^{L}_{s+1} + 1}{m}.
\end{align*}
Since $U^{L+1}, x^L_{s+1}, x^L_r$ are vectors in the program, by the Master Theorem, we have:
\begin{align*}
    m^{-1} \transpose{{(U^{L+1})}} x^L_{s+1} \xrightarrow[m \rightarrow \infty]{a.s.} \mathbb{E}[Z^{U^{L+1}} Z^{x^L_{s+1}}] = \sigma(0) \underbrace{\mathbb{E}[Z^{U^{L+1}}]}_{0} = 0,
\end{align*}
and 
\begin{align*}
    \frac{\transpose{{(x^L_r)}} x^{L}_{s+1} + 1}{m} \xrightarrow[m \rightarrow \infty]{a.s.} \mathbb{E}[Z^{x^{L}_r} Z^{x^{L}_{s+1}}] = \sigma(0)^2 = 0.
\end{align*}
We thus get 
\begin{align*}
    \sum_{r=0}^s \chi_r \frac{\transpose{{(x^L_r)}} x^{L}_{s+1} + 1}{m} \xrightarrow[m \rightarrow \infty]{a.s.} \sum_{r=0}^s \underbrace{ \scalarlim{\chi}_r}_{< \infty} \times 0 = 0.
\end{align*}
This shows that 
\begin{align*}
    f_{s+1}(\xi) \xrightarrow[m \rightarrow \infty]{a.s.} 0.
\end{align*}
Doing the exact same reasoning as above with $\xi_{s+1}$ instead of $\xi$ for $r=s+1$ gives us the first 3 claims of the induction hypothesis for $r=s+1$. 

\subsubsection{Backward pass at step $s+1$}

\textbf{Backward pass at step $(s+1)$:} the fourth claim $\chi_{s+1} \rightarrow \scalarlim{\chi}_{s+1} = \partial_2 \ell(y_{s+1}, 0)$ is a consequence of the fact that $f_{s+1}(\xi_{s+1}) \rightarrow 0$ almost surely, combined with the facts that $\chi_{s+1} = \partial_2 \ell(y_{s+1}, f_{s+1}(\xi_{s+1}))$ and that $\partial_2 \ell(y_{s+1}, \cdot)$ is continuous by assumption. In all the rest of this proof, for $z \in \{ h^l_{s+1}, x^l_{s+1}, dh^l_{s+1}, dx^l_{s+1} \}$ we now use $z$ to denote $z(\xi_{s+1})$ and not $z(\xi)$ anymore. 
\\\\
$mdx^L_{s+1} = w^{L+1}(s+1) = U^{L+1} - \eta \sum_{r=0}^s \chi_r x^L_r$ yields by \ZNonLin
\begin{align*}
    Z^{mdx^L_{s+1}} &=  Z^{U^{L+1}} - \eta \sum_{r=0}^s \scalarlim{\chi}_r \underbrace{Z^{x^L_r}}_{0} \\
    Z^{dx^L_{s+1}} &= Z^{U^{L+1}}.
\end{align*}
We thus have $Z^{dx^L_{s+1}} = \scalarlim{\tau} Z^{m dx^L_{s+1}} = 0$, and $Z^{dh^L_{s+1}} = Z^{dx^L_{s+1}} \sigma'(Z^{h^L_{s+1}}) = 0 \times \sigma'(0) = 0$ almost surely. 
\\\\
One has:
\begin{align*}
    m dx^{L-1}_{s+1} = \omega \transpose{{(\hatW^L)}}  (m dh^L_{s+1}) - \eta \sum_{r=0}^s \chi_r \frac{\transpose{{(m dh^l_r)}} m dh^L_{s+1}}{m} x^{L-1}_r,
\end{align*}
so that 
\begin{align*}
    Z^{m dx^{L-1}_{s+1}} = \scalarlim{\omega} Z^{\transpose{{(\hatW^L)}}  (m dh^L_{s+1})} - \eta \sum_{r=0}^s  \scalarlim{\chi}_r \mathbb{E}[Z^{m dh^l_r} Z^{m dh^l_{s+1}}]  Z^{x^{L-1}_r}.
\end{align*}
Now, we have $\mathbb{E}[Z^{m dh^l_r} Z^{m dh^l_{s+1}}] = \mathbb{E}[(Z^{U^{L+1}})^2] \sigma'(0)^2 = \sigma'(0)^2$ which is finite. On the other hand, because $Z^{m dh^L_{s+1}} = Z^{U^{L+1}}$ does not depend on $Z^{\hatW^L}$, we  get that $\dotZ^{\transpose{{(\hatW^L)}}  (m dh^L_{s+1})} = 0$ and $\hatZ^{\transpose{{(\hatW^L)}}  (m dh^L_{s+1})} \sim \mathcal{N}(0,1)$ so that $\scalarlim{\omega} Z^{\transpose{{(\hatW^L)}}  (m dh^L_{s+1})} = 0$. It  follows that $Z^{m dx^{L-1}_{s+1}} = 0$, and since $Z^{m dh^{L-1}_{s+1}} = Z^{m dx^{L-1}_{s+1}}  \sigma'(Z^{h^{L-1}_0})$ we also get $Z^{m dh^{L-1}_{s+1}} = 0$. 
\\\\
Let $l \in [2, L]$ and assume $Z^{m dx^l_{s+1}} = Z^{m dh^l_{s+1}} = 0$. Then 
\begin{align*}
    m dx^{l-1}_{s+1} = \omega \transpose{{(\hatW^l)}} (mdh^l_{s+1}) - \eta \sum_{r=0}^s \chi_r \frac{\transpose{{(mdh^l_r)}} mdh^l_{s+1}}{m} x^{l-1}_r.
\end{align*}
Since $\transpose{{(m dh^l_r)}}$ and $m dh^l_{s+1}$ are vectors in the program, $\transpose{{(m dh^l_r)}} m dh^l_{s+1} / m$ is a scalar in the program which converges almost surely, by the Master Theorem, to $\mathbb{E}[Z^{m dh^l_r} Z^{m dh^l_{s+1}}] = 0$. On the other hand $\dotZ^{\transpose{{(\hatW^l)}} (mdh^l_{s+1})} = 0$ because $Z^{m dh^l_{s+1}}$ is a constant (its expression in function of the previous $\hatZ$ is constant equal to 0), and $\hatZ^{\transpose{{(\hatW^l)}} (mdh^l_0)} \sim \mathcal{N}(0, \mathbb{E}[(Z^{m dh^l_0})^2])$ is almost surely 0 because $\mathbb{E}[(Z^{m dh^l_0})^2]) = 0$. By \ZNonLin\ we have
\begin{align*}
    Z^{dx^{l-1}_{s+1}} &= \underbrace{\scalarlim{\omega}}_{0} \underbrace{\hatZ^{\transpose{{(\hatW^l)}} dh^l_0}}_{0} - \eta \sum_{r=0}^s \underbrace{\scalarlim{\chi}_r}_{ < \infty} \underbrace{\mathbb{E}[Z^{dh^l_r} Z^{dh^l_{s+1}}]}_{0} \underbrace{Z^{x^{l-1}_r}}_{0} \\
    Z^{dx^{l-1}_{s+1}} &= 0.
\end{align*}
Finally, $Z^{dh^{l-1}_{s+1}} = Z^{dx^{l-1}_{s+1}} \sigma'(Z^{h^{l-1}_{s+1}})$  yields $Z^{dh^{l-1}_{s+1}} = 0$ because $Z^{h^{l-1}_{s+1}} = 0$. This proves the last claim of the induction hypothesis for $r=s+1$ and thus concludes the induction and therefore the proof. 
\end{proof}

\section{Preliminaries on positively homogeneous functions}\label{sec:pos-homog}

In this section we give a description of activation functions $\sigma$ satisfying Assumption~\ref{ass:smooth-homogeneous-act}. The fact that $\sigma$ is positively $p$-homogeneous translates as
\begin{align*}
    \sigma(z) =
    \begin{cases}
    \alpha z^p \quad \text{if } z \geq 0 \\
    \beta |z|^p \quad \text{if } z < 0.
    \end{cases}
\end{align*}
Additionally, one has
\begin{align*}
    \sigma'(z) =
    \begin{cases}
    \alpha p z^{p-1} \quad \text{if } z \geq 0 \\
    - \beta p |z|^{p-1} \quad \text{if } z < 0,
    \end{cases}
\end{align*}
so that $\sigma'$ is positively $(p-1)$-homogeneous with $\sigma'(0) = 0$. Since $p \geq 2$, both $\sigma$ and $\sigma'$ are continuous and $\sigma'$ is differentiable everywhere except at $0$ if $p=2$. It is immediate to check that both $\sigma$ and $\sigma'$ are pseudo-Lipschitz and that $\sigma$, $\sigma'$ and $\sigma''$ are also polynomially bounded functions. The non-negativity assumption on $\sigma$ gives $\alpha \geq 0, \beta \geq 0$, the fact that $\sigma$ is not identically $0$ leads to $\alpha >0 \text{ or } \beta > 0$, and finally the fact that $\sigma$ has faster growth on the positive part of the real line yields $\alpha > \beta \geq 0$. One  notices that the faster growth assumption is stronger than the assumption that $\sigma$ is not identically zero, and the latter could thus be gotten rid of. The conditions on $\alpha$ and $\beta$ can thus simply be summarized as 
\begin{align}\label{ass:alpha-beta}
    \alpha > \beta \geq 0    
\end{align}
With these conditions, we have that $\sigma(z) > 0$ for $z > 0$, and $\sigma'(z) z \geq 0$ for $z \neq 0$, that is $\text{sign}(\sigma'(z)) = \text{sign}(z)$. 

\section{Preliminaries for~\autoref{th:formal-no-constant-lr} and \autoref{th:non-trivial-ipllr}}

In all this section since we assume positive homogeneity of the activation function, we also consider parameterizations with no bias terms except at the first layer.

\subsection{Tilde variables}

\begin{definition}[Scaleless variables at initialization]\label{def:tilde-variables}
Let $\xi \in \mathbb{R}^d$ be an input vector. Independently of any parameterization, we consider the following variables \quoting{without scale} at initialization :
\begin{align*}
    &\begin{cases}
        \tildeh^1_0(\xi) := U^1 \xi + v^1 \\
        \tildex^1_0(\xi) := \sigma(\tildeh^1_0(\xi))
    \end{cases}
    &&\forall l \in [2,L],
    \begin{cases}
        \tildeh^l_0(\xi) := \hatW^l \tildex^{l-1}_0(\xi) \\
        \tildex^l_0(\xi) := \sigma(\tildeh^l_0(\xi))
    \end{cases}
\end{align*}
and define $\tildef_0(\xi) := \transpose{{(\hatW^{L+1})}} \tildex^L_0$, as well as
\begin{align*}
    &\begin{cases}
        d\tildex^L_0(\xi) := U^{L+1} \\
        d\tildeh^L_0(\xi) := d\tildex^L_0(\xi) \odot \sigma'(\tildeh^L_0(\xi))
    \end{cases}
    \forall l \in [L-1],
    \begin{cases}
        d\tildex^l_0(\xi) := \transpose{{(\hatW^{l+1})}} d\tildeh^{l+1}_0(\xi) \\
        d\tildeh^l_0(\xi) := d\tildex^l_0(\xi) \odot \sigma'(\tildeh^l_0(\xi))
    \end{cases}
\end{align*}
where the $\hatW^l$ are defined in Equation~\eqref{eq:hatW}.
\end{definition}

\begin{remark}
The tilde variables are independent of the choice of parameterization because, independently of the parameterization, $\hatW^l_{pq} = m^{-1/2} U^l_{pq} \sim \mathcal{N}(0, 1/m)$ for $l \in [2, L+1]$ and $\hatW^1_{pq} = U^1_{pq} \sim \mathcal{N}(0, 1)$. Those variables essentially reproduce the computations that take place in the forward (without any bias terms except at the first layer) and backward passes of any ac-parameterization but the magnitudes (the multiplying scalars $\omega_l$) have been set to $1$, essentially removing the additional scales which lead to explosion or vanishing as $m \rightarrow \infty$. The tilde variables of the forward pass at initialization correspond to the NTK parameterization. However this is not the case for the backward pass as the backward pass of NTK vanishes at initialization whereas the corresponding tilde variables have positive ($> 0$) variance as shown in Lemma~\ref{th:tilde-pos-variance} below.
\end{remark}

\begin{lemma}[Scaleless variables have positive and finite second moment]\label{th:tilde-pos-variance}
Let $\xi \in \mathbb{R}^d$ be an input vector, and consider a non-linearity $\sigma$ satisfying Assumption~\ref{ass:smooth-act}. Then, dropping the dependency of the tilde variables on $\xi$, one has that for any $l \in [1,L]$, and for any $z \in \{\tildeh^l_0, \tildex^l_0, d\tildeh^l_0, d\tildex^l_0 \}$, the second moment is positive and finite: $ 0 < \mathbb{E}[(Z^{z})^2] < \infty$. More precisely, one has:
\begin{align*}
    &Z^{\tildeh^1_0} \sim \mathcal{N}(0, ||\xi||^2 + 1),  &&0 < \mathbb{E}[(Z^{\tildex^1_0})^2] < \infty  \\
    &Z^{\tildeh^l_0} \sim \mathcal{N}(0, V^2_{h,l}),  &&0 < V^2_{h,l}:= \mathbb{E}[(Z^{\tildex^{l-1}_0})^2] < \infty, &&& l \in [2, L], \\
    &0 < \mathbb{E}[(Z^{\tildex^{l}_0})^2] < \infty,  &&\, &&& l \in [2, L], \\
    &\tildef_0(\xi) \xrightarrow[m \rightarrow \infty]{law} \mathcal{N}(0, V^2_f),  &&0 < V^2_f:= \mathbb{E}[(Z^{\tildex^{L}_0})^2] < \infty, \\
    &Z^{d\tildex^L_0} \sim \mathcal{N}(0, 1) &&\, &&&\, \\
    &Z^{d\tildex^l_0} \sim \mathcal{N}(0, V^2_{dx,l}),  &&0 < V^2_{dx,l}:= \mathbb{E}[(Z^{d\tildeh^{l+1}_0})^2] < \infty, &&& l \in [1, L-1],\\
    &0 < \mathbb{E}[(Z^{d\tildeh^{l}_0})^2] < \infty,  &&\, &&& l \in [1, L].
\end{align*}
\end{lemma}

\begin{remark}\label{remark:tilde-pos-finite-variance}
As shown in Appendix~\ref{app:exp-relu}, those expectations, as well as the means (first and second moment) are tractable with $\sigma = \relu$ and have simple expressions (for the first forward and backward passes). As shown in Appendices~\ref{app:relu-tilde-forward} and~\ref{app:relu-tilde-backward}, the recursive formulas for the variances of the forward and backward passes can be unrolled, and to avoid explosion or vanishing with the depth $L$, one must initialize the \iid Gaussian entries with a standard deviation of $\sqrt{2}$ to preserve the norm of the input signal. 
\end{remark}

\begin{proof}
Let $\xi \in \mathbb{R}^d $ be an input vector. We omit the dependency of the forward and backward passes on $\xi$ for simplicity. We first induct from $l=1$ to $l=L$ for the forward pass and then from $l=L$ to $l=1$ for the backward pass. $\tildeh^1_0 = U^1 \xi + v^1$ is the sum of two initial vectors in the program, which follows two independent Gaussian laws by definition: $Z^{U^1 \xi} \sim \mathcal{N}(0, ||\xi||^2)$, and $Z^{v^1} \sim \mathcal{N}(0, 1)$ independently of $Z^{U^1 \xi}$. We thus have $Z^{\tildeh^1_0} \sim \mathcal{N}(0, ||\xi||^2 + 1)$, which shows its  variance is finite and $>0$, and by Lemma~\ref{th:pos-finite-moment}, $0 < \mathbb{E}[(Z^{\tildex^1_0})^2] < \infty$ since $Z^{\tildex^1_0} = \sigma(Z^{\tildeh^1_0})$. 
\\\\
Now let $l \in [1, L-1]$ and assume $Z^{\tildeh^l_0} \sim \mathcal{N}(0, V^2_{h,l})$ with $0 < V^2_{h,l} < \infty$, and $0 < \mathbb{E}[(Z^{\tildex^l_0})^2] < \infty$. By \ZMatMul, $Z^{\tildeh^{l+1}_0} = Z^{\hatW^{l+1} x^l_0}$ which is equal to $\hatZ^{\hatW^{l+1} x^l_0}$ by Lemma~\ref{th:dotZ-first-forward}. now by definition, $\hatZ^{\hatW^{l+1} \tildex^l_0} \sim \mathcal{N}(0, \mathbb{E}[(Z^{\tildex^l_0})^2])$, and the variance is $>0$ and finite by the induction hypothesis, so that $0 < \mathbb{E}[(Z^{\tildeh^{l+1}_0})^2] < \infty$. Now by Lemma~\ref{th:pos-finite-moment} again, since $Z^{\tildex^{l+1}_0} = \sigma(Z^{\tildeh^{l+1}_0})$, we also get that $0 < \mathbb{E}[(Z^{\tildex^{l+1}_0})^2] < \infty$ which concludes the induction for the first $L$ layers of the forward pass. 
\\\\
$\tildef_0(\xi) = \transpose{{(\hatW^{l+1})}} \tildex^L_0$ and $\hatW^{L+1}_{j} \sim \mathcal{N}(0, 1/m)$ for every $m$, and by the Master Theorem, since $||\tildex^L_0||^2/m$ is a scalar in the program defined by the moment operation, it converges almost surely to $\mathbb{E}[(Z^{\tildex^L_0})^2]$. Finally, since $\tildex^L_0$ is computed using only the $\hatW^l$ for $l \leq L$, $\tildex^L_0$ is independent of $\hatW^{l+1}$. By Lemma~\ref{th:gaussian-output}, $\tildef_0(\xi)$ converges in law towards $\mathcal{N}(0, \mathbb{E}[(Z^{\tildex^L_0})^2])$, and $0 <\mathbb{E}[(Z^{\tildex^L_0})^2] < \infty$ by the previous induction. 
\\\\
$Z^{d\tildex^L_0} = Z^{U^{L+1}}$ and since $U^{L+1}$ is an initial vector in the program whose coordinates are iid following $\mathcal{N}(0, 1)$, we have by definition $Z^{U^{L+1}} \sim \mathcal{N}(0, 1)$. $Z^{d\tildeh^L_0} = Z^{d\tildex^L_0} \sigma'(Z^{\tildeh^L_0}) = \hatZ^{U^{l+1}} \sigma'(\hatZ^{\hatW^L \tildex^{L-1}_0})$.  Now by definition in \ZHat, $\hatZ^{\hatW^L \tildex^{L-1}_0}$ is independent of $\hatZ^{U^{l+1}}$ since $U^{L+1}$ is an initial vector in the program. This yields
\begin{align*}
    \mathbb{E} [(Z^{d\tildeh^L_0})^2] &= \mathbb{E} \left[(\hatZ^{U^{L+1}})^2 \right] \, \mathbb{E} \left[\sigma'(Z^{\tildeh^L_0})^2 \right] \\
    &= 1 \times \mathbb{E} \left[\sigma'(Z^{\tildeh^L_0})^2 \right].
\end{align*}
By assumption, $\sigma'$ is pseudo-Lipschitz and thus polynomially bounded, and is not almost everywhere $0$. By the induction above, $Z^{\tildeh^L_0} \sim \mathcal{N}(0, \mathbb{E}[(Z^{\tildex^{L-1}_0})^2])$ with $0 < \mathbb{E}[(Z^{\tildex^{L-1}_0})^2] < \infty$. By Lemma~\ref{th:pos-finite-moment} we thus have $0 < \mathbb{E} [\sigma'(Z^{\tildeh^L_0})^2] < \infty$, which shows $0 < \mathbb{E} [(Z^{d\tildeh^L_0})^2] < \infty$. 
\\\\
Now let $l \in [2, L]$ and assume $Z^{d\tildex^l_0} \sim \mathcal{N}(0, V^2_{dx,l})$ with $0 < V^2_{dx,l}< \infty$, and assume $0 < \mathbb{E}[(Z^{d\tildeh^l_0})^2] < \infty$. $Z^{d\tildex^{l-1}_0} = Z^{\transpose{{(\hatW^l)}} d\tildeh^l_0}$ and $Z^{\transpose{{(\hatW^l)}} d\tildeh^l_0} = \hatZ^{\transpose{{(\hatW^l)}} d\tildeh^l_0}$ by Lemma~\ref{th:dotZ-first-forward}. By definition, $\hatZ^{\transpose{{(\hatW^l)}} d\tildeh^l_0} \sim \mathcal{N}(0, \mathbb{E}[(Z^{d\tildeh^l_0})^2])$, so that $\mathbb{E}[(Z^{d\tildex^{l-1}_0})^2] = \mathbb{E}[(Z^{d\tildeh^l_0})^2]$ and thus $0 < \mathbb{E}[(Z^{d\tildex^{l-1}_0})^2] < \infty$ by the induction hypothesis. We have
\begin{align*}
    Z^{d\tildeh^{l-1}_0} = Z^{d\tildex^{l-1}_0} \sigma'(Z^{\tildeh^{l-1}_0}) = \hatZ^{\transpose{{(\hatW^l)}} d\tildeh^l_0} \sigma'(\hatZ^{\hatW^{l-1} \tildex^{l-2}_0})
\end{align*}
if $l \geq 3$, and 
\begin{align*}
    Z^{d\tildeh^{1}_0} = Z^{d\tildex^{1}_0} \sigma'(Z^{\tildeh^{1}_0}) = \hatZ^{\transpose{{(\hatW^2)}} d\tildeh^2_0} \sigma'(\hatZ^{\hatW^{1} \xi_0 + v^1})
\end{align*}
if $l=2$. In any case, the random variable inside $\sigma'$ is independent of the other variable in the product. We thus get
\begin{align*}
        \mathbb{E} [(Z^{d\tildeh^{l-1}_0})^2] &= \underbrace{\mathbb{E} \left[(Z^{d\tildex^{l-1}_0})^2 \right]}_{ >0, \, < \infty}  \, \underbrace{\mathbb{E} \left[\sigma'(Z^{\tildeh^{l-1}_0})^2 \right]}_{>0, \, < \infty}
\end{align*}
where the bounds on the second expectation are obtained using Lemma~\ref{th:pos-finite-moment}. This concludes the induction for the backward pass and thus the proof.
\end{proof}

\subsection{Expression of the forward and backward passes of ac-parameterizations in function of the tilde variables with homogeneity}\label{app:forward-backward-homog}

\begin{lemma}[Forward pass with homogeneity at $t=0$]\label{th:homog-first-forward}
Consider any ac-parameterization of an $L$-hidden layer neural network with a $p$-homogeneous activation function, and $p \geq 1$. Let $\xi \in \mathbb{R}^d$ be an input to the network. Then, omitting the dependency of the forward pass and the tilde variables on $\xi$, one has:
\begin{align}
    h^l_0 &= \gamma_{f,l} \tildeh^l_0, \qquad l \in [1,L], \\
    x^l_0 &= (\gamma_{f,l})^p \tildex^l_0, \qquad l \in [1,L], \\
    f_0(\xi) &= \gamma_{f,L+1} \tildef_0(\xi),
\end{align}
where, for any $l \in [1, L+1]$
\begin{align*}
    \gamma_{f,l}:= \left(\prod_{k=1}^l \omega_k^{p^{l-k}} \right).
\end{align*}
\end{lemma}
\begin{remark}\label{remark:homog-first-forward}
\ 
\begin{enumerate}
    \item $(\gamma_{f,l})^p = \left(\prod_{k=1}^l \omega_k^{p^{l-k+1}} \right)$.
    
    \item When $p=1$, $\gamma_{f,l}$ and $(\gamma_{f,l})^p$ simply reduce to $\omega_l \ldots \omega_1$. 
    
    \item For integrable parameterizations, for any $l \in [1, L+1]$, $\gamma_{f,l} = m^{-\sum_{k=0}^{l-2}p^k/2}$. The latter term is $1$ when $l = 1$, and otherwise $m^{-(l-1)/2}$ if $p=1$ and $m^{-(p^{l-1} - 1)/2(p-1)}$ if $p>1$. For \muP, $\gamma_{f,l} = 1$ for any $l \in [1, L]$ because $\omega_l = 1$ for \muP\ if $l \in [1, L]$.
    
    \item Instead of homogeneity, assume $\sigma$ is differentiable, has non-zero derivative in $0$ and $\sigma(0) = 0$. Also assume that $\omega_1 = 1$ (\ie $a_1 = 0$) and $\omega_l \rightarrow 0$ (\ie $a_l > 1/2$) for $l \in [2, L]$, which is the case in integrable parameterizations. Then, we have $h^1_0 = \tildeh^1_0$, and $h^2_0 = \omega_2 \tildeh^2_0$, so that $x^2_0 = \sigma(\omega_2 \tildeh^2_0)$ and as $m \rightarrow \infty$, $x^2_0 \simeq \omega_2 \sigma'(0) \tildeh^2_0$. Then similarly, we have for $h^3_0 \simeq \omega_3 \omega_2 \sigma'(0) \hatW^3 \tildeh^2_0$ and $x^3_0 \simeq \omega_3 \omega_2 \sigma'(0)^2 \hatW^3 \tildeh^2_0$. An easy induction then gives $h^l_0 = \sigma'(0)^{l-2} (\omega_l \ldots \omega_2) \hatW^l \ldots \hatW^2 \tildeh^2_0$. This thus resembles the case of a $p=1$ positively homogeneous function, except that the first forward pass is effectively linearized after layer $1$, but the magnitude of the forward pass at different layers is also well understood in this case so that the learning rates for the first update can be chosen appropriately (\eg for integrable parameterizations). In particular, the initial learning rates of IP-LLR for $p=1$ will also produce non-trivial weight updates at $t=0$ in this setting, which will in turn induce learning. Finally, setting the initial standard deviations of the weight matrices equal to $|\sigma'(0)|^{-1}$ instead of $1$ for the intermediate layers avoids problems with the depth $L$.
    
\end{enumerate}
\end{remark}

\begin{proof}
$h^1_0 = m^{-a_1}$ implies that $h^1_0 = \omega_1 (U^1\xi + v^1) = \omega_1\tildeh^1_0$, which entails $x^1_0 = \omega_1^p \tildex^1_0$ because $\sigma$ is positively $p$-homogeneous and $\omega_1 \geq 0$. Now let $l \in [1, L-1]$ and assume $h^l_0 = (\prod_{k=1}^l \omega_k^{p^{l-k}}) \tildeh^l_0$, and $x^l_0 = (\prod_{k=1}^l \omega_k^{p^{l-k+1}}) \tildex^l_0$. Then 
\begin{align*}
    h^{l+1}_0 &= \omega_{l+1} \hatW^{l+1}(0) x^l_0 \\
    &=\omega_{l+1} \left(\prod_{k=1}^l \omega_k^{p^{l-k+1}} \right) \hatW^{l+1}(0)  \tildex^l_0 \\
    &= \left(\prod_{k=1}^{l+1} \omega_k^{p^{l+1-k}} \right) \tildeh^{l+1}_0
\end{align*}
Since $\sigma$ is positively homogeneous, we have
\begin{align*}
    x^{l+1}_0 &= \sigma(h^{l+1}_0) \\
    &= \left(\prod_{k=1}^{l+1} \omega_k^{p^{l+1-k}})^p \sigma(\tildeh^{l+1}_0 \right) \\
    &= \left(\prod_{k=1}^{l+1} \omega_k^{p^{l+2-k}} \right) \sigma(\tildeh^{l+1}_0)
\end{align*}
This concludes the induction and gives the result for any $l \in [1, L]$. To conclude, we compute the expression of $f_0(\xi) = \omega_{L+1} \transpose{{(\hatW^{L+1}(0))}} x^L_0 = \omega_{L+1} (\prod_{k=1}^{L} \omega_k^{p^{L-k+1}}) \transpose{{(\hatW^{L+1}(0))}} \tildex^L_0 = (\prod_{k=1}^{L+1} \omega_k^{p^{L+1-k}}) \tildef_0(\xi)$.
\end{proof}

\begin{lemma}[Backward pass with homogeneity at $t=0$]\label{th:homog-first-backward}
Consider any ac-parameterization of an $L$-hidden layer neural network with a positively $p$-homogeneous activation function, and $p \geq 1$. Let $\xi_0 \in \mathbb{R}^d$ be the first training input. Then, omitting the dependency of the forward and backward passes, as well as that of the tilde variables on $\xi_0$, one has for any $l \in [1, L]$:
\begin{align}
    dx^l_0 &= m^{-a_{L+1}} \gamma_{b,l} \left( \prod_{k=l+1}^L \gamma_{f,k} \right)^{p-1} d\tildex^{l}_0, \label{eq:dx_l_0_homogeneous}\\ 
    dh^l_0 &= m^{-a_{L+1}} \gamma_{b,l} \left( \prod_{k=l}^L \gamma_{f,k} \right)^{p-1} d\tildeh^{l}_0, \label{eq:dh_l_0_homogeneous}
\end{align}
where, for any $l \in [1, L]$,
\begin{align*}
    \gamma_{b,l} = \prod_{k=l+1}^L \omega_k.
\end{align*}
\end{lemma}
\begin{remark}\label{remark:homog-first-backward}
\
\begin{enumerate}
    \item By swapping the products, one has that
    \begin{align*}
        \prod_{k=l+1}^L \gamma_{f,k} = \prod_{k=1}^{L} \omega_k^{\sum_{r=\max(k, l+1)}^{L} p^{r-k}}.
    \end{align*}
    
    \item When $p=1$, $\left(\prod_{k=l}^L \gamma_{f,k} \right)^{p-1} = 1$ for any $l \in [1, L+1]$.
    
    \item For integrable parameterizations, $\gamma_{b,l} = m^{-(L-l)/2}$ for any $l \in [1, L]$. For \muP, $\gamma_{b,l} = 1$ for any $l \in [1, L]$. 
    
    \item For $l=L$, $\gamma_{b,L} = 1$, $\prod_{k=l+1}^L \gamma_{f,k} = 1$, $\prod_{k=l}^L \gamma_{f,k} = \gamma_{f,L}$. 
\end{enumerate}
\end{remark}

\begin{proof}
$dx^L_0 = W^L(0) = m^{-a_{L+1}} U^{L+1} = m^{-a_{L+1}} d\tildex^L_0$,
\begin{align*}
    dh^L_0 &= dx^L_0 \odot \sigma'(h^L_0) \\
    &= m^{-a_{L+1}} d\tildex^L_0 \odot \sigma'(\gamma_{f,L} \tildeh^L_0) \\
    &= m^{-a_{L+1}} (\gamma_{f,L})^{p-1} d\tildex^L_0 \odot \sigma'(\tildeh^L_0)
\end{align*}
where the second equality stems from Lemma~\ref{th:homog-first-forward} and the last equality stems from $\omega_L \ldots \omega_1 > 0$ and the positive $(p-1)$-homogeneity of $\sigma'$. Let $l \in [2,L]$ and assume that $dx^l_0$ satisfies Equation~\eqref{eq:dx_l_0_homogeneous} and $dh^l_0$ satisfies Equation~\eqref{eq:dh_l_0_homogeneous}. Then 
\begin{align*}
    dx^{l-1}_0 &= \omega_l \transpose{{(\hatW^l)}} dh^l_0 \\
    &= m^{-a_{L+1}} \omega_l \gamma_{b,l} \left(\prod_{k=l}^L \gamma_{f,k} \right)^{p-1} \transpose{{(\hatW^l)}} d\tildeh^l_0 \\
    &= m^{-a_{L+1}} \gamma_{b,l-1} \left(\prod_{k=(l-1)+1}^L \gamma_{f,k} \right)^{p-1} d\tildex^{l-1}_0,
\end{align*}
and 
\begin{align*}
    dh^{l-1}_0 &= dx^l_0 \odot \sigma'(h^l_0) \\
    &= m^{-a_{L+1}} \gamma_{b,l-1} \left(\prod_{k=l}^L \gamma_{f,k} \right)^{p-1} d \tildex^{l-1}_0 \odot \sigma'(\gamma_{f,l} \tildeh^{l-1}_0) \\
    &= m^{-a_{L+1}} \gamma_{b,l-1} \left(\prod_{k=l}^L \gamma_{f,k} \right)^{p-1} (\gamma_{f, l-1})^{p-1} d \tildex^l_0 \odot \sigma'( \tildeh^l_0) \\
    &= m^{-a_{L+1}} \gamma_{b,l-1} \left(\prod_{k=l-1}^L \gamma_{f,k} \right)^{p-1} d \tildeh^{l-1}_0,
\end{align*}
where we have used Lemma~\ref{th:homog-first-forward} in the second equality, the positive $(p-1)$-homogeneity of $\sigma'$ combined with $\omega_l \ldots \omega_1 > 0$ in the third equality and the definition of $d\tildeh^{l-1}_0$ in the last. This thus concludes the proof by induction. 
\end{proof}

\begin{lemma}[Weight updates with homogeneity at $t=0$]\label{th:homog-first-weight-updates}
Consider any ac-parameterization of an $L$-hidden layer neural network with a positively $p$-homogeneous activation function, and $p \geq 1$. Let $\xi_0 \in \mathbb{R}^d$ be the first training input. Then, omitting the dependency of the forward and backward passes, as well as that of the tilde variables on $\xi_0$, one has:
\begin{align*}
    \Delta W^1(1) &= -\eta \chi_0 m^{-(a_{L+1} +2a_1 + c_1)} \omega_1^{p^{L} - 1} \left(\prod_{k=2}^L \omega_k^{p^{L-k+1}} \right) d\tildeh^1_{0} \transpose{\xi_{0}},   \\
    \Delta B^1(1) &= -\eta \chi_0 m^{-(a_{L+1} +2a_1 + c_1)} \omega_1^{p^{L} - 1} \left(\prod_{k=2}^L \omega_k^{p^{L-k+1}} \right) d\tildeh^1_{0},  \\
    \Delta W^l(1) &= -\eta \chi_0 m^{-(a_{L+1} + 2a_l + c_l - 1)} \left( \prod_{k=1}^L \omega_k^{p^{L-k+1}} \right) \omega_l^{-1} \frac{d\tildeh^l_{0} \transpose{{(\tildex^{l-1}_{0})}}}{m}, \qquad l \in [2,L] \nonumber, \\
    \Delta W^{L+1}(1) &= -\eta \chi_0 m^{-(2a_{L+1} + c_{L+1} - 1)} \left( \prod_{k=1}^L \omega_k^{p^{L-k+1}} \right) \tildex^{L}_{0} / m.
\end{align*}
\end{lemma}

\begin{remark}
For $p=1$, we have 
\begin{align*}
    \omega_1^{p^{L} - 1} \left(\prod_{k=2}^L \omega_k^{p^{L-k+1}} \right) &= \omega_1 \ldots \omega_L \\
    \left( \prod_{k=1}^L \omega_k^{p^{L-k+1}} \right) \omega_l^{-1} &= \omega_1 \ldots \omega_{l-1} \omega_{l+1} \ldots \omega_L \\
    \prod_{k=1}^L \omega_k^{p^{L-k+1}} &= \omega_1 \ldots \omega_L
\end{align*}
\end{remark}

\begin{proof}
Before we begin with the proof, we start with a first basic result which will be used repeatedly in the proof. Let $N \in \mathbb{N}^*$. By Equation~\eqref{eq:delta-W}, we have
\begin{align*}
    (p-1) \sum_{r=0}^N p^r = \sum_{r=1}^{N+1}p^r - \sum_{r=0}^N p^r = p^{N+1} - 1.
\end{align*}
Now that this is established, let us look at the update for the first layer. We have
\begin{align*}
    \Delta W^1(1) &= -\eta m^{-(2a_1 + c_1)} \chi_0 dh^1_0 \transpose{\xi_0} \\
    &= - \eta  m^{-(2a_1 + c_1 + a_{L+1})} \gamma_{b,1} \left(\prod_{k=1}^L \gamma_{f,k} \right)^{p-1} \chi_0 d\tildeh^1_0 \transpose{\xi_0},
\end{align*}
where we have used Lemmas~\ref{th:homog-first-forward} and \ref{th:homog-first-backward} in the second equality. Now, we have
\begin{align*}
   \gamma_{b,1} = \prod_{k=2}^L \omega_k,
\end{align*}
and by the first point in Remark~\ref{remark:homog-first-backward}, we have (with $l=1$)
\begin{align*}
    \left(\prod_{k=1}^L \gamma_{f,k} \right)^{p-1} &= \prod_{k=1}^L \omega_k^{(p-1)\sum_{r=k}^L p^{r-k}} \\
    &= \prod_{k=1}^L \omega_k^{(p-1)\sum_{r=0}^{L-k} p^{r}} \\
    &= \prod_{k=1}^L \omega_k^{p^{L-k+1} - 1}.
\end{align*}
It  follows that
\begin{align*}
    \gamma_{b,1} \left(\prod_{k=1}^L \gamma_{f,k} \right)^{p-1} &= \omega_1^{p^L-1} \prod_{k=2}^L \omega_k^{p^{L-k+1}}
\end{align*}
The formula for $\Delta B^1(1)$  follows from the expression of $dh^1_0$ in function of $d\tildeh^1_0$ and from Equation~\eqref{eq:delta-B}.
\\ \\
Let $l \in [2,L]$
\begin{align*}
    \Delta W^{l}(1) &= -\eta m^{-(2a_l + c_{l})} \chi_0 dh^l_0 \transpose{{(x^{l-1}_0)}} \\
    &= -\eta m^{-(2a_l + c_{l} + a_{L+1} - 1)} \chi_0  \, \gamma_{b,l} \, (\gamma_{f,l-1})^p \left(\prod_{k=l}^L \gamma_{f,k} \right)^{p-1} \frac{d\tildeh^l_{0} \transpose{{(\tildex^{l-1}_{0})}}}{m}
\end{align*}
Now, we have
\begin{align*}
    \gamma_{b,l} &= \prod_{k=l+1}^L \omega_k.
\end{align*}
In addition, by the first point of Remark~\ref{remark:homog-first-forward}, we have
\begin{align*}
    (\gamma_{f,l-1})^p = \left(\prod_{k=1}^{l-1} \omega_k^{p^{l-k}} \right),
\end{align*}
and by the first point in Remark~\ref{remark:homog-first-backward}
\begin{align*}
    \left(\prod_{k=l}^L \gamma_{f,k} \right)^{p-1} &= \left(\prod_{k=1}^{l-1} \omega_k^{(p-1) \sum_{r=l}^L p^{r-k}} \right) \times \left(\prod_{k=l}^L \omega_k^{(p-1) \sum_{r=k}^L p^{r-k}} \right) \\
    &= \left(\prod_{k=1}^{l-1} \omega_k^{(p-1) p^{l-k} \sum_{r=l}^L p^{r-l}} \right) \times \left(\prod_{k=l}^L \omega_k^{(p-1) \sum_{r=0}^{L-k} p^{r}} \right) \\
    &= \left(\prod_{k=1}^{l-1} \omega_k^{(p-1) p^{l-k} \sum_{r=0}^{L-l} p^{r}} \right) \times \left(\prod_{k=l}^L \omega_k^{(p-1) \sum_{r=0}^{L-k} p^{r}} \right).
\end{align*}
Let us now look, for each $k \in [1, L]$, at the power of $\omega_k$ which appears in the product $\gamma_{b,l} (\gamma_{f, l-1})^p \left(\prod_{k=l}^L \gamma_{f,k} \right)^{p-1}$. If $k \in [1, l-1]$, the exponent for $\omega_k$ is equal to 
\begin{align*}
    p^{l-k} + (p-1)p^{l-k} \sum_{r=0}^{L-l} p^r &= p^{l-k} \left((p-1)\sum_{r=0}^{L-l} p^r + 1  \right) \\
    &= p^{l-k} \left(p^{L-l+1} - 1 + 1  \right) \\
    &= p^{L-k+1}.
\end{align*}
If $k=l$, the exponent for $\omega_l$ is equal to 
\begin{align*}
    (p-1)\sum_{r=0}^{L-l}p^r &= p^{L-l+1} - 1.
\end{align*}
If $k \in [l+1, L]$, the 
exponent for $\omega_k$ is equal to 
\begin{align*}
    1 + (p-1)\sum_{r=0}^{L-k}p^r = 1 + p^{L-k+1} - 1 = p^{L-k+1}.
\end{align*}
Thus, for every $k \neq l$, the exponent for $\omega_k$ is equal to $p^{L-k+1}$, and for $k=l$, the exponent for $\omega_l$ is equal to $p^{L-l+1} - 1$. It  follows that
\begin{align*}
    \gamma_{b,l} (\gamma_{f, l-1})^p \left(\prod_{k=l}^L \gamma_{f,k} \right)^{p-1} &= \left(\prod_{k=1}^L \omega_{k}^{p^{L-k+1}} \right) w_l^{-1}.
\end{align*}
Finally, 
\begin{align*}
    \Delta W^{L+1}(1) &= -\eta m^{-(2a_{L+1} + c_{L+1})} \chi_0 x^L_0 \\
    &= -\eta m^{-(2a_{L+1} + c_{L+1} - 1)} \chi_0 (\gamma_{f,L})^p \tildex^L_0 / m,
\end{align*}
where we have used Lemma~\ref{th:homog-first-forward} in the second equality. From the first point of Remark~\ref{remark:homog-first-forward}, we get that 
\begin{align*}
    (\gamma_{f,L})^p = \prod_{k=1}^L \omega_k^{p^L-k+1},
\end{align*}
which concludes the proof. 

\end{proof}

\begin{corollary}[Weight updates of IP with homogeneity at $t=0$]\label{th:homog-first-weight-updates-ip}
Consider an integrable parameterization of an $L$-hidden layer neural network with no bias terms except at the first layer, and a positively $p$-homogeneous activation function, and $p \geq 1$. Let $\xi_0 \in \mathbb{R}^d$ be the first training input. Then, omitting the dependency of the forward and backward passes, as well as that of the tilde variables on $\xi_0$, one has:
\begin{align*}
    \Delta W^1(1) &= -\eta \chi_0 m^{-(c_1 - \gamma_{1}(p))} d\tildeh^1_{0} \transpose{\xi_{0}},   \\
    \Delta B^1(1) &= -\eta \chi_0 m^{-(c_1 - \gamma_{1}(p))} d\tildeh^1_{0},   \\
    \Delta W^l(1) &= -\eta \chi_0 m^{-(c_l - \gamma_{l}(p))} \frac{d\tildeh^l_{0} \transpose{{(\tildex^{l-1}_{0})}}}{m}, \qquad l \in [2,L] \nonumber, \\
    \Delta W^{L+1}(1) &= -\eta \chi_0 m^{-(c_{L+1} - \gamma_{L+1}(p))} \tildex^{L}_{0} / m,
\end{align*}
where the $\gamma_l(p)$ are given in Definition~\ref{def:gamma-p}.
\end{corollary}

\begin{proof}
For integrable parameterizations, $\omega_1 = 1$, $\omega_l = m^{-1/2}$ for $l \in [2, L]$, and $a_{L+1} = 1$. 
For the first layer, we have $a_{L+1} + 2a_1 + c_1 = 1 $. On the other hand, 
\begin{align*}
    \omega_1^{p^L-1} \left( \prod_{k=2}^L \omega_k^{p^{L-k+1}} \right) &= \prod_{k=2}^L m^{-p^{L-k+1}/2} \\
    &= m^{-\sum_{k=2}^{L} p^{L-k+1}/2} \\
    &= m^{-\sum_{k=1}^{L-1} p^{k}/2} \\
    &= m^{-1/2 (\sum_{k=0}^{L-1} p^{k} - 1)},
\end{align*}
so that 
\begin{align*}
    m^{-(a_{L+1} + 2a_1 + c_1)} \omega_1^{p^L-1} \left( \prod_{k=2}^L \omega_k^{p^{L-k+1}} \right) &= m^{-c_1} m^{-1/2 (\sum_{k=0}^{L-1} p^{k} - 1)} m^{-1} \\
    &= m^{-c_1} m^{-1/2 (\sum_{k=0}^{L-1} p^{k} + 1)} \\
    &= m^{-c_1} m^{\gamma_1(p)},
\end{align*}
by Definition~\ref{def:gamma-p},
which gives the result for the first layer's update ($\Delta W^1(1)$ and $\Delta B^1(1)$). Let $l \in [2, L]$. $a_{L+1} + 2a_l - 1 = 1 + 2 - 1 = 2$. On the other hand,
\begin{align*}
    \left(\prod_{k=1}^L \omega_k^{p^{L-k+1}} \right) \omega_l^{-1} &= m^{-1/2 (\sum_{k=0}^{L-1} p^{k} - 1)} m^{1/2} \\ 
    &= m^{-1/2 \sum_{k=0}^{L-1} p^{k} + 1},
\end{align*}
so that 
\begin{align*}
    m^{-(a_{L+1} + 2a_l + c_l)} \left(\prod_{k=1}^L \omega_k^{p^{L-k+1}} \right) \omega_l^{-1} &= m^{-c_l} m^{-1/2 \sum_{k=0}^{L-1} p^{k} + 1} m^{-2} \\
    &= m^{-c_l} m^{-1/2 \sum_{k=0}^{L-1} p^{k} - 1} \\
    &= m^{-c_l} m^{\gamma_l(p)},
\end{align*}
by Definition~\ref{def:gamma-p}, which proves the result for the updates of the intermediate layers. Finally, we have $2a_{L+1} - 1 = 2 - 1 = 1$, and on the other hand, because $\omega_1 = 1$, as in the first update, we find
\begin{align*}
    \prod_{k=1}^L \omega_k^{p^{L-k+1}} = m^{-1/2 (\sum_{k=0}^{L-1} p^{k} - 1)},
\end{align*}
so that 
\begin{align*}
    m^{-(2a_{L+1} + c_{L+1} - 1)} \prod_{k=1}^L \omega_k^{p^{L-k+1}} &= m^{-c_{L+1}} m^{-1/2 (\sum_{k=0}^{L-1} p^{k} - 1)} m^{-1} \\
    &= m^{-c_{L+1}} m^{-1/2 (\sum_{k=0}^{L-1} p^{k} + 1)} \\
    &= m^{-c_{L+1}} m^{\gamma_{L+1}(p)},
\end{align*}
by Definition~\ref{def:gamma-p},
which gives the result for the last layer's update and therefore concludes the proof. 
\end{proof}

\begin{corollary}[Weight updates of IP-LLR at $t=0$]\label{th:ipllr-first-weight-updates}
Consider an IP-LLR parameterization of an $L$-hidden layer neural network with a $p$-homogeneous activation function, and $p \geq 1$. Let $\xi_0 \in \mathbb{R}^d$ be the first training input. Then, omitting the dependency of the forward and backward passes of IP-LLR, as well as that of the tilde variables on $\xi_0$, one has:
\begin{align*}
    \Delta W^1(1) &= -\eta \chi_0 d\tildeh^1_{0} \transpose{\xi_{0}},   \\
    \Delta B^1(1) &= -\eta \chi_0 d\tildeh^1_{0},   \\
    \Delta W^l(1) &= -\eta \chi_0 \frac{d\tildeh^l_{0} \transpose{{(\tildex^{l-1}_{0})}}}{m}, \ \ l \in [2,L], \\
    \Delta W^{L+1}(1) &= -\eta \chi_0 \tildex^{L}_{0} / m.
\end{align*}
\end{corollary}

\begin{proof}
This is a simple consequence of Corollary~\ref{th:homog-first-weight-updates-ip} and the fact that for IP-LLR $c_l = \gamma_l(p)$ at $t=0$ by definition (see Definition~\ref{def:ipllr}) for any $l \in [1, L+1]$.
\end{proof}

\begin{lemma}[Weight updates of $\mu$P at $t=0$]\label{th:muP-first-weight-updates}
Consider the $\mu$P parameterization given in Definition~\ref{def:muP} with a differentiable activation function $\sigma$. Let $\xi_0 \in \mathbb{R}^d$ be the first training input. Then, omitting the dependency of the forward and backward passes of $\mu$P, as well as that of the tilde variables on $\xi_0$, one has:
\begin{align*}
    \Delta W^1(1) &= -\eta \chi_0 d\tildeh^1_{0} \transpose{\xi_{0}}   \\
    \Delta B^1(1) &= -\eta \chi_0 d\tildeh^1_{0}   \\
    \Delta W^l(1) &= -\eta \chi_0 \frac{d\tildeh^l_{0} \transpose{{(\tildex^{l-1}_{0})}}}{m}, \ \ l \in [2,L] \\
    \Delta W^{L+1}(1) &= -\eta \chi_0 \tildex^{L}_{0} / m
\end{align*}
\end{lemma}

\begin{remark}\label{remark:muP-first-weight-updates}
\ 
\begin{enumerate}
    \item Although the formulas are identical with those for IP-LLR when the activation function is positively $p$-homogeneous, this \textbf{does not} mean that the weight updates are exactly equal. Indeed, although the tilde variables do not depend on the choice of parameterization and will thus be the same in $\mu$P as in IP-LLR, the variable $\chi_0$ which appears in the formulas is parameterization-dependent as it depends on $f_0(\xi)$ which itself depends on the choice of parameterization. 
    
    \item There is no strong assumption on the activation function here (\eg homogeneity) as \muP\ is designed to have such updates which induce feature learning at all layers.
    
    \item Note that the coordinates of $\Delta W^l(1)$ are in $\Theta(m^{-1})$ whereas that of $W^l(0)$ are in $\Theta(m^{-1/2})$ for $l \in [2, L]$, so that paradoxically, even though \muP\ is designed to produce \quoting{maximal updates} (in a certain sense), we have that $\Delta W^l_{jq}(1) / W^l_{jq}(0) = \Theta(m^{-1/2}) \rightarrow 0$ as $m \rightarrow \infty$: the relative displacement of the weights is zero in the infinite-width limit. More generally, we have that for \muP\ $(W^l_{jq}(t) - W^l_{jq}(0)) / W^l_{jq}(0) \rightarrow 0$ as $m \rightarrow \infty$ if $t \geq 1$, which means that weights of the intermediate layers do not move away from their initialization in the infinite-width limit for \muP, even if the (pre-)activations of every layer are maximally updated. This is in stark contrast with IP-LLR for which both $W^l(0)$ and $\Delta W^l(1)$ are in $\Theta(m^{-1})$ for the intermediate layers $l \in [2, L]$: the weights do move relatively to their initialization in the infinite-width limit.
\end{enumerate}

\end{remark}

\begin{proof}
\muP\ is designed so that its forward pass has $h^l_0 = \tildeh^l_0$ for any $l \in [1, L]$. Indeed, the choice of pre-factors for the weights with \muP\  lead to the same recursive equations for the forward pass as the tilde variables, except for $f_0(\xi)$ which is equal to $m^{-1/2} \tildef_0(\xi)$. For the backward pass, one has that for \muP, $dx^L_0 = W^{L+1}(0) = m^{-1} U^{L+1} = m^{-1} d\tildex^L_0$. We then have
\begin{align*}
    dh^L_0 &= dx^L_0 \odot \sigma'(h^L_0) \\
    &= m^{-1} d\tildex^L_0 \odot \sigma'(\tildeh^L_0) \\
    &= m^{-1} d\tildeh^L_0.
\end{align*}
Let $l \in [1, L-1]$, and assume that $dx^{l+1}_0 = m^{-1} d\tildex^{l+1}_0$ and $dh^{l+1}_0 = m^{-1} d\tildeh^{l+1}_0$. Then, we have 
\begin{align*}
    dx^l_0 &= \transpose{{(W^{l+1}(0))}} dh^{l+1}_0 \\
    &= m^{-1} \transpose{{(\hatW^{l+1})}} d \tildeh^{l+1}_0 \\
    &= m^{-1} d\tildex^l_0.
\end{align*}
Similarly, we have 
\begin{align*}
    dh^l_0 &= dx^l_0 \odot \sigma'(h^l_0) \\
    &= m^{-1} d\tildex^l_0 \odot \sigma'(\tildeh^l_0) \\
    &= m^{-1} d\tildeh^l_0,
\end{align*}
which proves by induction that for any $l \in [1, L]$, $dx^l_0 = m^{-1} d\tildex^l_0$ and $dh^l_0 = m^{-1} d\tildeh^l_0$ for \muP. Recall that for \muP, $a_1 = 0$, $a_l = 1/2$ for $l \in [2, L]$  and $a_{L+1} = 1$, and $c_l = -1$ for any $l \in [1, L+1]$. Now by Equations~\eqref{eq:delta-W} and~\eqref{eq:delta-B}, the first weight updates give:
\begin{align*}
    \Delta W^1(1) &= -\eta \chi_0 m^{-c_1} m^{-1} d\tildeh^1_0 \transpose{\xi_0} \\
    &= -\eta \chi_0 d\tildeh^1_0 \transpose{\xi_0}, 
\end{align*}
and
\begin{align*}
    \Delta B^1(1) &= -\eta \chi_0 m^{-c_1} m^{-1} d\tildeh^1_0 \\
    &= -\eta \chi_0 d\tildeh^1_0, 
\end{align*}
For $l \in [2, L]$, we have
\begin{align*}
    \Delta W^l(1) &= -\eta \chi_0 m^{-(1 + c_l)} m^{-1} d\tildeh^l_0 \transpose{{(x^{l-1}_0)}} \\
    &= -\eta \chi_0 \frac{  d\tildeh^l_0 \transpose{{(\tildex^{l-1}_0)}}}{m}.
\end{align*}
Finally,
\begin{align*}
    \Delta W^{L+1}(1) &= -\eta \chi_0 m^{-(2 + c_{L+1})} x^L_0 \\
    &= -\eta \chi_0 \tildex^L_0 / m,
\end{align*}
which concludes the proof.
\end{proof}

\section{Dynamics of the infinite-width limit of IP-LLR}\label{sec:ipllr-inf-width}

\begin{lemma}[IP-LLR is zero at initialization]\label{th:ipllr-forward-backward-0}
Consider the IP-LLR parameterization with a positively $p$-homogeneous activation function, and $p \geq 2$. Then, for any input vector $\xi \in \mathbb{R}^d$, one has that $\tildeh^l_0(\xi), \tildex^l_0(\xi), d\tildex^l_0, d\tildeh^l_0$ are vectors in the Tensor Program program for any $l \in [2, L]$, and additionally:
\begin{align*}
    f_0(\xi) &\xrightarrow[m \rightarrow \infty]{a.s.} 0 \\
    \chi_0 &\xrightarrow[m \rightarrow \infty]{a.s.} \scalarlim{\chi}_0 := \partial_2 \ell(y_0, 0)
\end{align*}
\end{lemma}

\begin{remark}
The result on the almost sure convergence of $\chi_0$ ensures that the latter is a valid initial scalar in the Tensor Program defining the computations associated with the IP-LLR parameterization. 
\end{remark}

\begin{proof}
Because $\sigma$ and $\sigma'$ are pseudo-Lipschitz (since $p \geq 2$, see Appendix~\ref{sec:pos-homog}), the tilde variables of the first forward and backward passes $(\tildeh^l_0, \tildex^l_0, d\tildex^l_0, d\tildex^l_1)$ are vectors in the program given Definition~\ref{def:tilde-variables} by the \ZNonLin\ and \ZMatMul\ rules. Additionally, by Lemma~\ref{th:homog-first-forward},
\begin{align*}
    f_0(\xi_0) &= m^{-\sum_{k=0}^{L-1} p^k /2} m^{-1/2} \transpose{{(U^{L+1})}} \tildex^l_0 \\
    &= m^{1/2} m^{-\sum_{k=0}^{L-1} p^k /2} m^{-1} \transpose{{(U^{L+1})}} \tildex^L_0 \\
    &= m^{-\sum_{k=1}^{L-1} p^k /2} m^{-1} \transpose{{(U^{L+1})}} \tildex^L_0
\end{align*}
Now, $m^{-1} \transpose{{(U^{L+1})}} \tildex^L_0 \rightarrow \mathbb{E}[Z^{U^{L+1}} Z^{\tildex^L_0}]$ almost surely by the master theorem, and $Z^{\tildex^L_0} = \sigma(Z^{\hatW^L \tildex^{L-1}_0})$. By the Lemma~\ref{th:dotZ-first-forward}, $Z^{\hatW^L \tildeh^{L-1}_0} = \hatZ^{\hatW^L \tildeh^{L-1}_0}$, and by the \ZHat\ rule, the latter variable is independent of $Z^{U^{L+1}}$ since $Z^{U^{L+1}}$ is an initial vector in the program. This gives $\mathbb{E}[Z^{U^{L+1}} Z^{\tildex^L_0}] = \mathbb{E}[Z^{U^{L+1}}] \mathbb{E}[ Z^{\tildex^L_0}] = 0 \times \mathbb{E}[ Z^{\tildex^L_0}]$. By Lemma~\ref{th:tilde-pos-variance}, and Lemma~\ref{th:pos-finite-moment}, $\mathbb{E}[ Z^{\tildex^L_0}] < \infty$ because $\sigma$ is polynomially bounded. We thus get $\mathbb{E}[Z^{U^{L+1}} Z^{\tildex^L_0}] = 0$, and since $m^{-\sum_{k=1}^{L-1} p^k /2} \in (0,1]$, $f_0(\xi_0) \rightarrow 0$ almost surely. Recall that by definition (see Appendix~\ref{app:notations}) $\chi_0 = \partial_2 \ell(y_0, f_0(\xi_0))$. Since $f_0(\xi_0) \rightarrow 0$ almost surely, and since $\partial_2 \ell(y_0, \cdot)$ is continuous by assumption, we have that $\chi_0 \rightarrow \partial_2 \ell(y_0, 0) =: \scalarlim{\chi}_0$, which concludes the proof.
\end{proof}

\begin{definition}[Tilde variables in the backward pass after initialization]\label{def:tilde-backward-t}
For any ac-parameterization with $a_{L+1} = 1$, define for any $t \geq 1$,
\begin{align*}
    &d\tildex^L_t = m dx^L_t, \ d\tildeh^L_t = d\tildex^L_t \odot \sigma'(h^L_t), \\
    &d\tildex^l_t = \transpose{{(W^{l+1}(t))}} d\tildeh^{l+1}_t, \ \ l \in[1, L-1], \\
    &d\tildeh^l_t = d\tildex^l_t \odot \sigma'(h^l_t), \ \ l \in[1, L-1].
\end{align*}
\end{definition}

\begin{remark}
\ 
\begin{enumerate}
    \item One could in general define $d\tildex^l_t$ to be equal to $m^{a_{L+1}} dx^l_t$ but since all the ac-parameterizations we study in this paper, \ie integrable parameterizations, \muP, or hybrid versions thereof have $a_{L+1} = 1$, we limit the formulas to this case. The tilde variables are the right quantity to look at because of the term $m^{-a_{L+1}}$ which appears in the gradient \wrt to $x^L_t$ and then propagate to all the other variables of the backward pass by the equations of backpropagation.
    
    \item Recall that in the definition above, it is implicitly assumed that the computations of the forward and backward passes at any time step $s$ are done with the input $\xi = \xi_s$. 
\end{enumerate}
\end{remark}

\begin{lemma}[Relationship between tilde and non-tilde variables]\label{th:tilde-backward-t-ipllr}
For any ac-parameterization with $a_{L+1} = 1$, for any $t \geq 1$, and for any $\xi$, dropping the dependency of the forward and backward passes on $\xi$ at time $t$, one has:
\begin{align*}
    \forall l \in [1, L+1], \ dx^l_t = m^{-1} d\tildex^l_t, \ dh^l_t = m^{-1} d\tildeh^l_t.
\end{align*}
\end{lemma}

\begin{proof}
$dx^L_t = m^{-1} d\tildex^L_t$. $dh^L_t = dx^L_t \odot \sigma'(h^L_t) = m^{-1} d\tildex^L_t \odot \sigma'(h^L_t) = m^{-1} d\tildeh^L_t$. Now let $l \in [2, L]$ and assume $dx^l_t = m^{-1} d\tildex^l_t, \ dh^l_t = m^{-1} d\tildeh^l_t$. Then $dx^{l-1}_t = \transpose{{(W^l(t))}} dh^l_t = m^{-1} \transpose{{(W^l(t))}} d\tildeh^l_t = m^{-1} d\tildex^{l-1}_t$, and $dh^{l-1}_t = dx^{l-1}_t \sigma(h^{l-1}_t) = m^{-1} d\tildex^{l-1}_t \sigma(h^{l-1}_t) = m^{-1}d\tildeh^{l-1}_t$ which concludes the proof by induction. 
\end{proof}

\begin{lemma}[Weight updates for IP-LLR at any time step]\label{th:update-t-ipllr}
Consider the IP-LLR parameterization with a positively $p$-homogeneous activation function, and $p \geq 1$, and let $t \geq 1$. Then, dropping the dependency of the forward and backward passes on $\xi_t$ at time $t$, one has:
\begin{align*}
    \Delta W^{L+1}(t+1) &= -\eta \chi_t x^L_t / m,  \\
    \Delta W^l(t+1) &= -\eta \chi_t \frac{d\tildeh^l_t \transpose{{(x^{l-1}_t)}}}{m}, \qquad l \in [2, L], \\
    \Delta W^1(t+1) &= -\eta \chi_t d\tildeh^1_t \transpose{\xi_t}, \\
    \Delta B^1(t+1) &= -\eta \chi_t d\tildeh^1_t.
\end{align*}
\end{lemma}

\begin{proof}
Using Equation~\eqref{eq:delta-W}, we have $\Delta W^{L+1}(t) = - \eta \chi_t m^{-(2a_{L+1} + c_{L+1})} x^L_t = -\eta \chi_t x^L_t / m$ because $2a_{L+1} + c_{L+1} = 2-1=1$ in IP-LLR since $t \geq 1$. For $l \in [2, L]$
\begin{align*}
    \Delta W^l(t) &= -\eta \chi_t m^{-(2a_l + c_l)} dh^l_t \transpose{{(x^{l-1})}} \\
    &= -\eta \chi_t \frac{d\tildeh^l_t \transpose{{(x^{l-1}_t)}}}{m},
\end{align*}
by Lemma~\ref{th:tilde-backward-t-ipllr} and because $2a_l + c_l = 2-2 = 0$ for $t \geq 1$ in IP-LLR. $\Delta W^1(t) = -\eta \chi_t m^{-(2a_1 + c_1)} dh^1_t \xi_t = -\eta \chi_t d\tildeh^1_t \transpose{\xi_t}$ by Lemma~\ref{th:tilde-backward-t-ipllr} and because $2a_1 + c_1 = 0 - 1 = -1$ for $t \geq 1$ in IP-LLR. Finally, by Equation~\eqref{eq:delta-B}, we have $\Delta B^1(t) = - \eta \chi_t m^{-(2a_1 + c_1)} dh^1_t = - \eta \chi_t d\tildeh^1_t$ by Lemma~\ref{th:tilde-backward-t-ipllr} and because $2a_1 + c_1 = -1$. 
\end{proof}

\begin{theorem}[Weights in IP-LLR at time $t$]\label{th:weights-ipllr-t}
Consider the IP-LLR parameterization with a positively $p$-homogeneous activation function, and $p \geq 1$. Then, for any $t \geq 1$, one has:
\begin{enumerate}[(i)]
    \item $W^1(t) = U^1 - \eta \chi_0 d\tildeh^1_0 \transpose{\xi_0} - \eta \left(\sum_{s=1}^{t-1} \chi_s d\tildeh^1_s \transpose{\xi_s} \right)$,
    
    \item $B^1(t) = v^1 - \eta \chi_0 d\tildeh^1_0 - \eta \left(\sum_{s=1}^{t-1} \chi_s d\tildeh^1_s \right)$,
    
    \item $W^l(t) = \omega_l \hatW^l - \eta \chi_0 \frac{d\tildeh^l_0 \transpose{{(\tildex^{l-1}_0)}}}{m} - \eta \left(\sum_{s=1}^{t-1} \chi_s \frac{d\tildeh^l_s \transpose{{(x^{l-1}_s)}}}{m} \right)$, \qquad $l \in [2, L]$,
    
    \item $W^{L+1}(t) = U^{L+1}/m - \eta \chi_0 \tildex^L_0 / m - \eta \left(\sum_{s=1}^{t-1} \chi_s x^L_s / m\right)$.
\end{enumerate}
\end{theorem}

\begin{proof}
We have already seen the formulas are correct for $t=1$ by Corollary~\ref{th:ipllr-first-weight-updates}. Then, by Lemma~\ref{th:update-t-ipllr}, an easy induction immediately yields the result. 
\end{proof}

\begin{lemma}[Backward pass of IP-LLR at time $t$]\label{th:backward-ipllr-t}
Consider the IP-LLR parameterization with a positively $p$-homogeneous activation function, and $p \geq 1$. Then, for any $t \geq 1$, dropping the dependency of the forward pass at time $t$ on $\xi_t$, and of the previous forward and backward passes on the corresponding $\xi_s$,  one has:
\begin{enumerate}[(i)]
    \item $d\tildex^L_t = w^{L+1}(t) = U^{L+1} - \eta \chi_0 \tildex^L_0 - \eta \sum_{s=1}^{t-1} \chi_s x^L_s$,
    \item $d\tildex^{l-1}_t = \omega_l \transpose{{(\hatW^l)}} d\tildeh^l_t - \eta \chi_0 \frac{\transpose{{(d\tildeh^l_0)}} d\tildeh^l_t}{m} \tildex^{l-1}_0 -\eta \sum_{s=1}^{t-1} \chi_s \frac{\transpose{{(d\tildeh^l_s)}} d\tildeh^l_t}{m} x^{l-1}_s$, \qquad $l \in [2, L]$.
\end{enumerate}
\end{lemma}

\begin{proof}
By definition, we have
\begin{align*}
    d\tildex^L_t &= m dx^L_t \\
    &= m W^{L+1}(t) \\
    &= U^{L+1} - \eta \chi_0 \tildex^L_0 - \eta \sum_{s=1}^{t-1} \chi_s x^L_s
\end{align*}
where the last equality stems from \autoref{th:weights-ipllr-t}.
\\\\
Let $l \in [2, L]$, we have:
\begin{align*}
    d\tildex^{l-1}_t &= \transpose{{(W^l(t))}} d\tildeh^{l}_t \\
    &= \omega_l \transpose{{(\hatW^l)}} d\tildeh^l_t - \eta \chi_0 \frac{\transpose{{(d\tildeh^l_0)}} d\tildeh^l_t}{m} \tildex^{l-1}_0 -\eta \sum_{s=1}^{t-1} \chi_s \frac{\transpose{{(d\tildeh^l_s)}} d\tildeh^l_t}{m} x^{l-1}_s
\end{align*}
where the second equality stems from \autoref{th:weights-ipllr-t}. 
\end{proof}

\begin{lemma}[Z for the forward pass of IP-LLR at time $t=1$]\label{th:z-forward-ipllr-1}
Consider the IP-LLR parameterization with a positively $p$-homogeneous activation function, and $p \geq 2$. Let $\xi \in \mathbb{R}^d$ be an input to the network. Then, for any $l \in [1, L]$, $h^l_1(\xi), x^l_1(\xi), d\tildex^l_1, d\tildeh^l_1$ are vectors in the program, $f_1(\xi)$ is a scalar in the program, and $\chi_1$ is a valid initial scalar in the program. Additionally, dropping the dependency of the forward pass at time $t=1$ on $\xi$, and of the first forward and backward passes on $\xi_0$, one has:
\begin{enumerate}[(i)]
    \item $Z^{h^1_1} = Z^{W^1(1) \xi + B^1(1)} = Z^{U^1\xi} + Z^{v^1} - \eta \scalarlim{\chi}_0 (\transpose{\xi_0} \xi + 1) Z^{d\tildeh^1_0}$,
    
    \item $Z^{h^l_1} = Z^{W^l(1)x^{l-1}_1} = \scalarlim{\omega}_l Z^{\hatW^l x^{l-1}_1} - \eta \scalarlim{\chi}_0 \mathbb{E}[Z^{\tildex^{l-1}_0} Z^{x^{l-1}_1}]  Z^{d\tildeh^l_0}$, \qquad $l \in [2, L]$,

    \item $f_1(\xi) = \transpose{{(W^{L+1}(1))}} x^L_1 \xrightarrow[m \rightarrow \infty]{a.s.} \mathbb{E}[Z^{U^{L+1}} Z^{x^L_1}] - \eta \scalarlim{\chi}_0 \mathbb{E}[Z^{\tildex^L_0} Z^{x^L_1}]$.
\end{enumerate}
\end{lemma}

\begin{proof}
By \autoref{th:weights-ipllr-t}, with $t=1$, one has that $h^1_1 = U^1 \xi + v^1 - \eta \chi_0 (\transpose{\xi}_0 \xi + 1) d\tildeh^1_0$. By Lemma~\ref{th:ipllr-forward-backward-0}, $d\tildeh^1_0$ is a vector in the Tensor Program and $\chi_0$ is a valid initial scalar in the program which has an almost sure limit $\scalarlim{\chi}_0 := \partial_2 \ell(y_0, 0)$ as $m \rightarrow \infty$. In addition, $U^1 \xi$ and $v^1$ are initial vectors in the program, which thus shows that $h^1_1$ is a vector in the program by the \NonLin\ operation. This also gives that $x^1_1 = \sigma(h^1_1)$ is a vector in the program since $\sigma$ is pseudo-Lipschitz (see Appendix~\ref{sec:pos-homog}). Moreover, by \ZNonLin, we have $Z^{h^1_1} = Z^{U^1 \xi} + Z^{v^1} - \eta \scalarlim{\chi}_0 (\transpose{\xi}_0 \xi + 1) Z^{d\tildeh^1_0}$. 
Let $l \in [2, L]$ and assume that $h^{l-1}_1, x^{l-1}_1$ are vectors in the program. Then, by \autoref{th:weights-ipllr-t} with $t=1$, we get
\begin{align*}
    h^l_1 = \omega_l \hatW^l x^{l-1}_1 - \eta \chi_0 \frac{\transpose{{(\tildex^{l-1}_0)}} x^{l-1}_1}{m} d\tildeh^l_0.
\end{align*}
$\transpose{{(\tildex^{l-1}_0)}} x^{l-1}_1 / m$ is a scalar in the program by the \Moment\ operation, and thus by the \MatMul\ and \NonLin\ operations, $h^l_1$ is a vector in the program and thus so is $x^l_1 = \sigma(h^l_1)$, which proves by induction that this is the case for any $l \in [2,L]$. By \ZNonLin\ we thus have
\begin{align*}
    Z^{h^l_1} = \scalarlim{\omega}_l Z^{\hatW^l x^{l-1}_1} - \eta \scalarlim{\chi}_0 \mathbb{E}[Z^{\tildex^{l-1}_0} Z^{x^{l-1}_1}] Z^{d\tildeh^l_0}.
\end{align*}
We then have by \autoref{th:weights-ipllr-t} with $t=1$, 
\begin{align*}
    f_1(\xi) = m^{-1} \transpose{{(U^{L+1})}} x^L_1 - \eta \chi_0 \frac{\transpose{{(\tildex^L_0)}} x^L_1}{m}
\end{align*}
$U^{L+1} - \eta \chi_0 \tildex^L_0$ is a vector in the program by the \NonLin\ operation, and the quantity $m^{-1} \transpose{{(U^{L+1} - \eta \chi_0 \tildex^{L}_0)}} x^L_1$ is thus a scalar in the program by the \Moment\ operation, and by the master theorem, we get $f_1(\xi) \rightarrow \mathbb{E}[Z^{U^{L+1}} Z^{x^L_1}] - \eta \scalarlim{\chi}_0 \mathbb{E}[Z^{\tildex^L_0} Z^{x^L_1}]$ almost surely, since both expectations are finite by Lemma~\ref{th:Z0-dist-moments}. Since we did the previous reasoning with an arbitrary $\xi$, we also get that $h^l_1(\xi_1), x^l_1(\xi_1)$ are vectors in the program for any $l \in [1, L]$ and that the formulas in $(i)$, $(ii)$, and $(iii)$ hold when the input is $\xi_1$. In particular, $f_1(\xi_1)$ converges to a finite almost sure limit $\scalarlim{f_1}(\xi_1)$, and thus the continuity of $\partial_2 \ell(y_1, \cdot)$ ensures the almost sure convergence of $\chi_1$ towards $\scalarlim{\chi}_1 := \partial_2 \ell(y_1, \scalarlim{f_1}(\xi_1))$, which means $\chi_1$ is a valid initial scalar in the Tensor Program. Then, dropping the dependency of the second forward pass (at $t=1$) on $\xi_1$, we get by Lemma~\ref{th:backward-ipllr-t} with $t=1$:
\begin{align*}
    d\tildex^L_1 = U^{L+1} - \eta \chi_0 \tildex^L_0
\end{align*}
which is a vector in the program by \NonLin. Then $d\tildeh^L_1 = d\tildex^L_1 \odot \sigma'(h^L_1)$ is also a vector in the program since $\sigma'$ is pseudo-Lipschitz. Let $l \in [2, L-1]$ and assume that $d\tildex^{l+1}_1$ and $d\tildeh^{l+1}_1$ are vectors in the program. Then by Lemma~\ref{th:backward-ipllr-t} with $t=1$, we have
\begin{align*}
    d\tildex^{l}_1 = \omega_{l+1} \transpose{{(\hatW^{l+1})}} d\tildeh^{l+1}_1 - \eta \chi_0 \frac{\transpose{{(d\tildeh^{l+1}_0)}} d\tildeh^{l+1}_1}{m} \tildex^l_0
\end{align*}
$\transpose{{(d\tildeh^{l+1}_0)}} d\tildeh^{l+1}_1 / m$ is a scalar in the program by the \Moment\ operation and by \MatMul\ and \NonLin\ we thus get that $d\tildex^l_1$ is a vector in the program. Then $d\tildeh^l_1 = d\tildex^l_1 \odot \sigma'(h^l_1)$ is also a vector in the program since $\sigma'$ is pseudo-Lipschitz, which concludes the induction and with it the proof.
\end{proof}

\begin{theorem}[Z for the forward pass of IP-LLR at time $t$]\label{th:z-forward-ipllr-t}
Consider the IP-LLR parameterization with a positively $p$-homogeneous activation function, and $p \geq 2$. Let $\xi \in \mathbb{R}^d$ be an input to the network. Then, for any $l \in [1, L]$, $h^l_s(\xi), x^l_s(\xi), d\tildex^l_s, d\tildeh^l_s$ are vectors in the program, $f_s(\xi)$ is a scalar in the program, and $\chi_s$ is a valid initial scalar in the program. Additionally, dropping the dependency of the forward pass at time $t$ on $\xi$, and of the previous forward and backward passes on the corresponding $\xi_s$, one has:
\begin{enumerate}[(i)]
    \item $Z^{h^1_t} = Z^{W^1(t) \xi + B^1(t)} = Z^{U^1\xi} + Z^{v^1} - \eta \scalarlim{\chi}_0 (\transpose{\xi_0} \xi + 1) Z^{d\tildeh^1_0} - \eta \left(\sum_{s=1}^{t-1} \scalarlim{\chi}_s (\transpose{\xi_s} \xi + 1) Z^{d\tildeh^1_s}  \right)$,
    
    \item for any $l \in [2, L]$,
\end{enumerate}
\vspace{-1.5em}
\begin{align*}
    Z^{h^l_t} = Z^{W^l(t)x^{l-1}_t} = \scalarlim{\omega}_l Z^{\hatW^l x^{l-1}_t} - \eta \scalarlim{\chi}_0 \mathbb{E}[Z^{\tildex^{l-1}_0} Z^{x^{l-1}_t}]  Z^{d\tildeh^l_0}  - \eta \left(\sum_{s=1}^{t-1} \scalarlim{\chi}_s \mathbb{E}[Z^{x^{l-1}_s} Z^{x^{l-1}_t}]  Z^{d\tildeh^l_s } \right),
\end{align*}
\begin{enumerate}[(i)]
    \setcounter{enumi}{2}
    \item $f_t(\xi) = \transpose{{(W^{L+1}(t))}} x^L_t \xrightarrow[m \rightarrow \infty]{a.s.} \mathbb{E}[Z^{U^{L+1}} Z^{x^L_t}] - \eta \scalarlim{\chi}_0 \mathbb{E}[Z^{\tildex^L_0} Z^{x^L_t}] -
    \eta \left(\sum_{s=1}^{t-1} \scalarlim{\chi}_s \mathbb{E}[Z^{x^L_s} Z^{x^L_t}]  \right)$.
\end{enumerate}
\end{theorem}

\begin{proof}
We prove that the vectors and scalars in the claim of the theorem are part of the program by induction. Then the formulas of $(i)$, $(ii)$, and $(iii)$ are a simple consequence of the \ZNonLin\ operation. The case $t=1$ has been treated in Lemma~\ref{th:z-forward-ipllr-1}. Let $t \geq 1$ and assume that the vectors and scalars in the claim of the theorem are part of the program for any $s \in [1, t]$. By \autoref{th:weights-ipllr-t}, one has that
\begin{align*}
    h^1_{t+1} &= W^1(t+1) \xi + B^1(t+1)\\
    &=U^1 \xi + v^1 - \eta \chi_0 (\transpose{\xi}_0 \xi + 1) d\tildeh^1_0 - \eta \left(\sum_{s=1}^t \chi_s (\transpose{\xi}_s \xi + 1) d\tildeh^1_s \right)     
\end{align*}
By the induction hypothesis and \NonLin, we thus get that $h^1_{t+1}$ is a vector in the program and thus so is $x^1_{t+1} = \sigma(h^1_{t+1})$ since $\sigma$ is polynomially bounded. Let $l \in [2, L]$ and assume that $h^{l-1}_{t+1}, x^{l-1}_{t+1}$ are vectors in the program. Then, by \autoref{th:weights-ipllr-t}, we get
\begin{align*}
    h^l_{t+1} = \omega_l \hatW^l x^{l-1}_{t+1} - \eta \chi_0 \frac{\transpose{{(\tildex^{l-1}_0)}} x^{l-1}_{t+1}}{m} d\tildeh^l_0 - \eta \left( \sum_{s=1}^t \chi_s \frac{\transpose{{(x^{l-1}_s)}} x^{l-1}_{t+1}}{m} d\tildeh^l_s \right).
\end{align*}
For any $s \in [1, t]$, $\transpose{{(x^{l-1}_s)}} x^{l-1}_1 / m$ and $\transpose{{(\tildex^{l-1}_0)}} x^{l-1}_1 / m$ are scalars in the program by the induction hypothesis and the \Moment\ operation. Thus by the \MatMul\ and \NonLin\ operations, $h^l_{t+1}$ is a vector in the program and thus so is $x^l_{t+1} = \sigma(h^l_{t+1})$, which proves by induction that this is the case for any $l \in [2,L]$. We then have by \autoref{th:weights-ipllr-t}, 
\begin{align*}
    f_{t+1}(\xi) = m^{-1} \transpose{{\left(U^{L+1} - \eta \chi_0 \tildex^L_0 - \eta \sum_{s=1}^t \chi_s x^L_s \right)}} x^L_{t+1}
\end{align*}
$U^{L+1} - \eta \chi_0 \tildex^L_0 - \eta \sum_{s=1}^t \chi_s x^L_s$ is a vector in the program by the induction hypothesis and the \NonLin\ operation. Then, by the \Moment\ operation, $f_{t+1}(\xi)$ is a scalar in the program since $x^L_{t+1}$ is also a vector in the program, and by the master theorem, we have 
\begin{align*}
    f_{t+1}(\xi) \xrightarrow[m \rightarrow \infty]{a.s.} \mathbb{E}[Z^{U^{L+1}} Z^{x^L_{t+1}}] - \eta \scalarlim{\chi}_0 \mathbb{E}[Z^{\tildex^L_0} Z^{x^L_{t+1}}] - \eta \sum_{s=1}^t \scalarlim{\chi}_s \mathbb{E}[Z^{x^L_s} Z^{x^L_{t+1}}].
\end{align*}
The limit is finite by Lemma~\ref{th:Z0-dist-moments} since by an easy induction any $Z$ which appears is a polynomially bounded function of a Gaussian vector with finite covariance matrix. Since we did the previous reasoning with an arbitrary $\xi$, we also get that $h^l_{t+1}(\xi_{t+1}), x^l_{t+1}(\xi_{t+1})$ are vectors in the program for any $l \in [1, L]$. In particular, $f_{t+1}(\xi_{t+1})$ converges to an almost sure limit $\scalarlim{f}_{t+1}(\xi_{t+1})$, and thus the continuity of $\partial_2 \ell(y_{t+1}, \cdot)$ ensures the almost sure convergence of $\chi_{t+1}$ towards $\scalarlim{\chi}_{t+1} := \partial_2 \ell(y_{t+1}, \scalarlim{f}_{t+1}(\xi_{t+1}))$, which means $\chi_{t+1}$ is a valid initial scalar in the Tensor Program. Then, dropping the dependency of the forward pass at $t+1$ on $\xi_{t+1}$, we get by Lemma~\ref{th:backward-ipllr-t}:
\begin{align*}
    d\tildex^L_{t+1} = U^{L+1} - \eta \chi_0 \tildex^L_0 - \eta \sum_{s=1}^t \chi_s x^L_s
\end{align*}
which is a vector in the program by \NonLin. Then $d\tildeh^L_{t+1} = d\tildex^L_{t+1} \odot \sigma'(h^L_{t+1})$ is also a vector in the program since $\sigma'$ is pseudo-Lipschitz. Let $l \in [2, L-1]$ and assume that $d\tildex^{l+1}_{t+1}$ and $d\tildeh^{l+1}_{t+1}$ are vectors in the program. Then by Lemma~\ref{th:backward-ipllr-t}, we have
\begin{align*}
    d\tildex^{l}_{t+1} = \omega_{l+1} \transpose{{(\hatW^{l+1})}} d\tildeh^{l+1}_{t+1} - \eta \chi_0 \frac{\transpose{{(d\tildeh^{l+1}_0)}} d\tildeh^{l+1}_{t+1}}{m} \tildex^l_0 - \eta \sum_{s=1}^t \chi_s \frac{\transpose{{(d\tildeh^{l+1}_s)}} d\tildeh^{l+1}_{t+1}}{m} x^l_s
\end{align*}
$\transpose{{(d\tildeh^{l+1}_s)}} d\tildeh^{l+1}_{t+1} / m$ is a scalar in the program for any $s \in [0, t]$ by the \Moment\ operation and by \MatMul\ and \NonLin\ we thus get that $d\tildex^l_{t+1}$ is a vector in the program. Then $d\tildeh^l_{t+1} = d\tildex^l_{t+1} \odot \sigma'(h^l_{t+1})$ is also a vector in the program since $\sigma'$ is pseudo-Lipschitz, which concludes the induction. Then we get the claims of $(i)$, $(ii)$ and $(iii)$ simply by applying the \ZNonLin\ rule to the formulas derived above for the pre-activations $h^l_{t+1}$. 
\end{proof}

\begin{corollary}[Z for the forward pass of IP-LLR at time $t$]\label{th:z-forward-ipllr-t-0}
Consider the IP-LLR parameterization with a positively $p$-homogeneous activation function, and $p \geq 2$. Then, for any $t \geq 1$, and for any input $\xi \in \mathbb{R}^d$, dropping the dependency of the forward pass at time $t$ on $\xi$, and of the previous forward and backward passes on the corresponding $\xi_s$,  one has:
\begin{enumerate}[(i)]
    \item $Z^{h^1_t} = Z^{W^1(t) \xi + B^1(t)} = Z^{U^1\xi} + Z^{v^1} - \eta \scalarlim{\chi}_0 (\transpose{\xi_0} \xi + 1) Z^{d\tildeh^1_0} - \eta \left(\sum_{s=1}^{t-1} \scalarlim{\chi}_s (\transpose{\xi_s} \xi + 1) Z^{d\tildeh^1_s}  \right)$
    
    \item for any $l \in [2, L]$,
\end{enumerate}
\vspace{-2em}
\begin{align*}
    Z^{h^l_t} = Z^{W^l(t)x^{l-1}_t} = - \eta \scalarlim{\chi}_0 \mathbb{E}[Z^{\tildex^{l-1}_0} Z^{x^{l-1}_t}]  Z^{d\tildeh^l_0}  - \eta \left(\sum_{s=1}^{t-1} \scalarlim{\chi}_s \mathbb{E}[Z^{x^{l-1}_s} Z^{x^{l-1}_t}]  Z^{d\tildeh^l_s } \right),  
\end{align*}
\begin{enumerate}[(i)]
    \setcounter{enumi}{2}
    \item $f_t(\xi) = \transpose{{(W^{L+1}(t))}} x^L_t \xrightarrow[m \rightarrow \infty]{a.s.} \mathbb{E}[Z^{U^{L+1}} Z^{x^L_t}] - \eta \scalarlim{\chi}_0 \mathbb{E}[Z^{\tildex^L_0} Z^{x^L_t}] - 
        \eta \left(\sum_{s=1}^{t-1} \scalarlim{\chi}_s \mathbb{E}[Z^{x^L_s} Z^{x^L_t}]  \right).$
\end{enumerate}
\end{corollary}

\begin{proof}
The formulas are readily obtained by \autoref{th:z-forward-ipllr-t} coupled with the fact that we have $\scalarlim{\omega}_l Z^{\hatW^l x^{l-1}_t} = 0$ for any $l \in [2, L]$, and $t \geq 1$, which stems from \autoref{th:ipllr-initial-weight-vanish-t-geq-1}.
\end{proof}

\begin{remark}\label{remark:no-circ-logic-vanish}
Note that there is no circular logic here since only \autoref{th:z-forward-ipllr-t} is used to prove the results of Appendix~\ref{sec:th-ipllr-initial-weight-vanish} (and in particular \autoref{th:ipllr-initial-weight-vanish-t-geq-1}), so that using \autoref{th:ipllr-initial-weight-vanish-t-geq-1} for Corollary~\ref{th:z-forward-ipllr-t-0} does not lead to any issue.
\end{remark}

\begin{theorem}[$Z$s of backward pass of IP-LLR at time $t$]\label{th:z-backward-ipllr-t}
Consider the IP-LLR parameterization with a positively $p$-homogeneous activation function, and $p \geq 2$. Then, for any $t \geq 1$, dropping the dependency of the forward pass at time $t$ on $\xi_t$, and of the previous forward and backward passes on the corresponding $\xi_s$,  one has:
\begin{enumerate}[(i)]
    \item $Z^{d\tildex^L_t} = Z^{w^{L+1}(t)} = Z^{U^{L+1}} - \eta \scalarlim{\chi}_0 Z^{\tildex^L_0} - \eta \sum_{s=1}^{t-1} \scalarlim{\chi}_s Z^{x^L_s}$,
    \item $Z^{d\tildex^{l-1}_t} = \scalarlim{\omega}_l Z^{\transpose{{(\hatW^l)}} d\tildeh^l_t} - \eta \scalarlim{\chi}_0 \mathbb{E}[Z^{d\tildeh^l_0} Z^{d\tildeh^l_t}] Z^{\tildex^{l-1}_0} -\eta \sum_{s=1}^{t-1} \scalarlim{\chi}_s \mathbb{E}[Z^{d\tildeh^l_s} Z^{d\tildeh^l_t}] Z^{x^{l-1}_s}$, \quad $l \in [2, L]$.
\end{enumerate}
\end{theorem}

\begin{proof}
We have already proved in \autoref{th:z-forward-ipllr-t} that for any $s \in [1, t]$ the vectors of the forward ($h^l_s, x^l_s$ for $l \in [1, L]$) and the backward pass ($d\tildex^l_s, d\tildeh^l_s$ for $l \in [1, L]$) at time $s$ are part of the program and similarly at $t=0$ by Lemma~\ref{th:ipllr-forward-backward-0}. Then, claims $(i)$ and $(ii)$ readily follow from applying the \ZNonLin\ rule to the formulas of Lemma~\ref{th:backward-ipllr-t}.
\end{proof}

\begin{corollary}[$Z$s of backward pass of IP-LLR at time $t$]\label{th:z-backward-ipllr-t-0}
Consider the IP-LLR parameterization with a positively $p$-homogeneous activation function, and $p \geq 2$. Then, for any $t \geq 1$, dropping the dependency of the forward pass at time $t$ on $\xi_t$, and of the previous forward and backward passes on the corresponding $\xi_s$,  one has:
\begin{enumerate}[(i)]
    \item $Z^{d\tildex^L_t} = Z^{w^{L+1}(t)} = Z^{U^{L+1}} - \eta \scalarlim{\chi}_0 Z^{\tildex^L_0} - \eta \sum_{s=1}^{t-1} \scalarlim{\chi}_s Z^{x^L_s}$
    \item $Z^{d\tildex^{l-1}_t} = - \eta \scalarlim{\chi}_0 \mathbb{E}[Z^{d\tildeh^l_0} Z^{d\tildeh^l_t}] Z^{\tildex^{l-1}_0} -\eta \sum_{s=1}^{t-1} \scalarlim{\chi}_s \mathbb{E}[Z^{d\tildeh^l_s} Z^{d\tildeh^l_t}] Z^{x^{l-1}_s}$, \qquad $l \in [2, L]$.
\end{enumerate}
\end{corollary}

\begin{proof}
The formulas are readily obtained by \autoref{th:z-backward-ipllr-t} and the fact that $Z^{\transpose{{(\hatW^l)}} d\tildeh^l_t} = 0$ for any $l \in [2, L]$ and $t \geq 1$, which stems from \autoref{th:ipllr-initial-weight-vanish-t-geq-1}.
\end{proof}

\begin{remark}
Note that a similar statement can be made as in Remark~\ref{remark:no-circ-logic-vanish} regarding circular logic since only \autoref{th:z-backward-ipllr-t} is used to prove the results of Appendix~\ref{sec:th-ipllr-initial-weight-vanish}.
\end{remark}

\subsection{Second forward pass of IP-LLR ($t=1$)}

In this section, we prove that for IP-LLR, we have $0 < \mathbb{E}[Z^{\tildex^l_0} Z^{x^l_1}] < \infty$ for any $l \in [1, L]$ under the assumption that $\scalarlim{\chi}_0 := \lim_{m \rightarrow \infty} \chi_0 \neq 0$. To obtain those results, we use the formulas from Corollary~\ref{th:z-forward-ipllr-t-0} for $t=1$, which are obtained using the main result from Appendix~\ref{sec:ipllr-induct-dynamics}, namely \autoref{th:ipllr-initial-weight-vanish-t-geq-1}. We choose to put Appendix~\ref{sec:ipllr-induct-dynamics} towards the end of the Appendix section as its main result is quite intuitive: any multiplication by matrices with pre-factors in $m^{-1}$ result in a vector whose coordinates (the corresponding $Z$) converge to $0$ almost surely at any time step. The proof however requires a long and cumbersome induction and we thus leave it for the later stages of the Appendix so as not to break the narrative of the Appendix. 
\\ \\
The finiteness of the expectations $\mathbb{E}[Z^{\tildex^l_0} Z^{x^l_1}]$ is a simple consequence of Lemma~\ref{th:Z0-dist-moments}, but the fact that they are $>0$ requires more work as we will see below. Since we work with IP-LLR, recall that we consider a bias term at the first layer only. 

\begin{lemma}[1st layer of forward pass of IP-LLR at $t=1$]\label{th:ipllr-forward-1-0}
Consider the IP-LLR parameterization with an activation function $\sigma$ satisfying Assumption~\ref{ass:smooth-homogeneous-act}. Let $\xi$ be an input to the network, and assume $\scalarlim{\chi}_0 \neq 0$. Then, dropping the dependency of the first forward-backward pass on $\xi_0$, and that of the second forward pass on $\xi$, one has:
\begin{enumerate}[(i)]
    \item $Z^{h^1_1} = Z^{\tildeh^{1}_0(\xi)} -\eta \scalarlim{\chi}_0 (\transpose{\xi_0} \xi + 1)Z^{d\tildeh^1_0} = Z^{\tildeh^{1}_0(\xi)} - \eta \scalarlim{\chi}_0 (\transpose{\xi_0} \xi + 1) Z^{d\tildex^1_0} \sigma'(Z^{\tildeh^1_0})$,
    
    \item 
    $(Z^{\tildeh^1_0}, Z^{\tildeh^1_0(\xi)},  Z^{d\tildex^1_0}) \sim \mathcal{N} \left(0, 
    \begin{pmatrix}
        ||\xi_0||^2 + 1 & \transpose{\xi_0} \xi +1 & 0 \\
        \transpose{\xi} \xi_0 + 1 & ||\xi||^2 +1 & 0 \\
        0 & 0 & \mathbb{E}[(Z^{d\tildeh^2_0})^2]
    \end{pmatrix} \right)$,
    
    \item $0 < \mathbb{E}[Z^{\tildex^1_0} Z^{x^1_1}] < \infty$.
\end{enumerate}
\end{lemma}

\begin{proof}
We have by Corollary~\ref{th:z-forward-ipllr-t-0} at time $t=1$
\begin{align*}
    Z^{x^1_1} &= \sigma \left(Z^{h^1_1(\xi)} \right) \\
    &= \sigma \left(Z^{\tildeh^1_0(\xi)} - \eta \scalarlim{\chi}_0 (\transpose{\xi_0} \xi + 1) Z^{d\tildeh^1_0} \right).
\end{align*}
Moreover, since $d\tildeh^1_0 = d\tildex^1_0 \odot \sigma'(\tildeh^1_0)$, since all the vectors are part of the Tensor Program, by \ZNonLin\ we have $Z^{d\tildeh^1_0} = Z^{d\tildex^1_0} \sigma'(Z^{\tildeh^1_0})$, so that
\begin{align*}
    Z^{x^1_1} &= \sigma \left(Z^{\tildeh^1_0(\xi)} - \eta \scalarlim{\chi}_0 (\transpose{\xi_0} \xi + 1) Z^{d\tildex^1_0} \sigma'(Z^{\tildeh^1_0}) \right).
\end{align*}
Finally, we have
\begin{align*}
    Z^{\tildex^1_0} = \sigma(Z^{\tildeh^1_0}).
\end{align*}
From the rules of \ZInit\ and \ZHat, we have that
\begin{align*}
    (Z^{\tildeh^1_0}, Z^{\tildeh^1_0(\xi)}, Z^{d\tildex^1_0}) \sim \mathcal{N} \left(0, 
    \begin{pmatrix}
        S & 0 \\
        0 & \mathbb{E}[(Z^{d\tildeh^2_0})^2]
    \end{pmatrix} \right),
\end{align*}
with 
\begin{align*}
    S := 
    \begin{pmatrix}
          ||\xi_0||^2+1 & \transpose{\xi_0}\xi + 1 \\
        \transpose{\xi_0}\xi + 1 & ||\xi||^2+1
    \end{pmatrix}.
\end{align*}
By Lemma~\ref{th:tilde-pos-variance}, $\mathbb{E}[(Z^{d\tildeh^2_0})^2] < \infty$, so that the covariance matrix is finite and thus $Z^{\tildex^1_0} Z^{x^1_1}$ is a polynomially bounded function of a Gaussian vector which shows that the expectation is finite by Lemma~\ref{th:Z0-dist-moments}. It is also non-negative since $\sigma$ is non-negative. To prove that it is positive, one needs only prove that the integrand is not almost everywhere 0. 
By Lemma~\ref{th:tilde-pos-variance}, $\mathbb{E}[(Z^{d\tildeh^2_0})^2] > 0$ so that the covariance matrix is invertible if and only if $S$ is invertible. We have 
\begin{align*}
    \text{det}(S) = \left(||\xi_0||^2 ||\xi||^2 - (\transpose{\xi_0}\xi)^2 \right) + ||\xi_0 - \xi||^2,
\end{align*}
which is the sum of two non-negative terms by Cauchy-Schwarz's inequality, and is thus $0$ if and only if both terms are zero. The first term is zero only when $\xi$ and $\xi_0$ are proportional, and if in addition the second term is zero than $\xi = \xi_0$.  The distribution of the Gaussian vector appearing in $Z^{\tildex^1_0} Z^{x^1_1}$ thus depends on whether or not $\xi_0$ and $\xi$ are equal. 
\\ \\
\textbf{Case when $\xi = \xi_0$}. Then, calling $\lambda := -\eta \scalarlim{\chi}_0 (\transpose{\xi_0} \xi +1)$, we have
\begin{align*}
    \mathbb{E}[Z^{\tildex^1_0} Z^{x^1_1}] = \int \sigma(z) \sigma \left(z - \lambda u \sigma'(z) \right) p_z(z) p_u(u) \mathrm{d}z \mathrm{d}u,
\end{align*}
where $p_z$ and $p_u$ are the densities of the two Gaussians $\mathcal{N}(0, ||\xi_0||^2 +1)$ and $\mathcal{N}(0, \mathbb{E}[(Z^{d\tildeh^2_0})^2])$ respectively, which are not degenerate, so that $p_z(z) > 0$ for any $z$ and similarly for $p_u(u)$. Since $Z^{d\tildex^1_0}$ and $-Z^{d\tildex^1_0}$ have the same distribution and since it is independent of $Z^{\tildeh^1_0}$, we can assume $\lambda \geq 0$ W.L.O.G (if $\lambda \leq 0$ we can always do the change of variable $u \leftarrow -u$ in the integral above since $p_u(-u) = p_u(u)$). Consider the point $(z^*, u^*) := (1, -1)$, at which the integrand in the integral above is $>0$, because $\sigma$ and $\sigma'$ are $>0$ on the positive part of the real line (see Appendix~\ref{sec:pos-homog}) and $\lambda \geq 0$. 
The integral is then positive, because the integrand is a continuous function, since $\sigma$ and $\sigma'$ are continuous (see again Appendix~\ref{sec:pos-homog}). 
\\ \\
\textbf{Case when $\xi \neq \xi_0$}. Then, we have
\begin{align*}
        \mathbb{E}[Z^{\tildex^1_0} Z^{x^1_1}] = \int \sigma(u) \sigma \left(v - \lambda z \sigma'(u) \right) p_{u,v}(u,v) p_z(z) \mathrm{d}u \mathrm{d}v \mathrm{d}z,
\end{align*}
where $p_{u,v}$ and and $p_z$ are the densities of non-degenerate Gaussians and are thus well-defined and positive everywhere. Again, we can assume $\lambda \geq 0$ W.L.O.G. We consider the point $(u^*, v^*, z^*) = (1, 1, -1)$ at which the integrand is $>0$ since $\sigma$ and $\sigma'$ are positive on the positive part of the real line. 
Hence, the integral is $>0$ because the integrand is a continuous function, since $\sigma$ and $\sigma'$ are continuous, which concludes the proof. 
\end{proof}

\begin{lemma}[Intermediate layer of forward pass of IP-LLR at $t=1$]\label{th:ipllr-forward-1-l}
Consider the IP-LLR parameterization with an activation function $\sigma$ satisfying Assumption~\ref{ass:smooth-homogeneous-act}. Let $\xi$ be an input to the network, let $l \in [2, L]$, and assume $\scalarlim{\chi}_0 \neq 0$. Then, dropping the dependency of the first forward-backward pass on $\xi_0$, and that of the second forward pass on $\xi$, one has:
\begin{enumerate}[(i)]
    \item $Z^{h^l_1} = -\eta \scalarlim{\chi}_0 \mathbb{E}[Z^{\tildex^{l-1}_0} Z^{x^{l-1}_1}] Z^{d\tildeh^l_0} = -\eta \scalarlim{\chi}_0 \mathbb{E}[Z^{\tildex^{l-1}_0} Z^{x^{l-1}_1}] \hatZ^{d\tildex^l_0} \sigma'(\hatZ^{\tildeh^l_0})$,
    
    \item $Z^{\tildeh^l_0}$ and $Z^{d\tildex^l_0}$ are independent,
    
    \item $0 < \mathbb{E}[Z^{\tildex^l_0} Z^{x^l_1}] < \infty$.
\end{enumerate}
\end{lemma}

\begin{proof}
We prove the result by induction on $l$, the case of $l=1$ has already been dealt with in Lemma~\ref{th:ipllr-forward-1-0}. Let $l \in [1, L-1]$, and assume $0 < \mathbb{E}[Z^{\tildex^l_0} Z^{x^l_1}] < \infty$. Calling $\lambda := -\eta \scalarlim{\chi}_0 \mathbb{E}[Z^{\tildex^l_0} Z^{x^l_1}]$, we have $\lambda \neq 0$ by assumption and by the induction hypothesis. Then, by Corollary~\ref{th:z-forward-ipllr-t-0} with $t=1$, we have
\begin{align*}
    Z^{h^{l+1}_1} = - \lambda Z^{d\tildeh^{l+1}_0}.
\end{align*}
Moreover, $d\tildeh^{l+1}_0 = d\tildex^{l+1}_0 \odot \sigma'(\tildeh^{l+1}_0)$, and since all the vectors are part of the Tensor Program, we have by \ZNonLin\ $Z^{d\tildeh^{l+1}_0} = Z^{d\tildex^{l+1}_0} \sigma'(Z^{\tildeh^{l+1}_0})$. On the other hand, by Lemma~\ref{th:dotZ-first-forward}, we have $Z^{d\tildeh^{l+1}_0} = \hatZ^{d\tildeh^{l+1}_0}$ and $Z^{\tildeh^{l+1}_0} = \hatZ^{\tildeh^{l+1}_0}$, and finally by the \ZHat\ rule, since $\tildeh^{l+1}_0 = \hatW^{l+1} \tildex^l_1$ and $d\tildex^{l+1}_0 = U^{L+1}$ if $l=L-1$ and $\transpose{{(\hatW^{l+2})}} d\tildeh^{l+2}$ otherwise, we get that $\hatZ^{\tildeh^{l+1}_0}$ and $\hatZ^{d\tildeh^{l+1}_0}$ are independent. In addition, we have
\begin{align*}
    \mathbb{E}[Z^{\tildex^{l+1}_0} Z^{x^{l+1}_1}] &= \mathbb{E}[\sigma(Z^{\tildeh^{l+1}_0}) \sigma(- \lambda Z^{d\tildex^{l+1}_0} \sigma'(Z^{\tildeh^{l+1}_0}))].
\end{align*}
The expectation is non-negative because $\sigma$ is and it is finite by Lemma~\ref{th:Z0-dist-moments} because the integrand is a polynomially bounded function of the Gaussian vector $(Z^{\tildeh^{l+1}_0}, Z^{d\tildex^{l+1}_0})$ (and thus of $Z_0$, see Definition~\ref{def:Z0}). Using the positive $p$-homogeneity of $\sigma$ and the fact that $\text{sign}(\sigma'(z)) = \text{sign}(z)$ (see Appendix~\ref{sec:pos-homog}), and calling $\epsilon = \text{sign}(\lambda) \in \{-1,1\}$, we have
\begin{align*}
    \mathbb{E}[Z^{\tildex^{l+1}_0} Z^{x^{l+1}_1}] &= \mathbb{E}\left[\mathbb{E}\left[Z^{\tildex^{l+1}_0} Z^{x^{l+1}_1} \Big| Z^{\tildeh^{l+1}_0}\right]\right]\\
    &=  |\lambda|^p \,  \mathbb{E}\left[\sigma(Z^{\tildeh^{l+1}_0}) |\sigma'(Z^{\tildeh^{l+1}_0})|^p \mathbb{E}\left[\sigma(-\epsilon \, \text{sign}(Z^{\tildeh^{l+1}_0}) Z^{d\tildex^{l+1}_0}) \Big| Z^{\tildeh^{l+1}_0} \right] \right],
\end{align*}
Now since $\epsilon \, \text{sign}(Z^{\tildeh^{l+1}_0}) \in \{-1, 1\}$, $Z^{d\tildex^{l+1}_0}$ and $\epsilon \, \text{sign}(Z^{\tildeh^{l+1}_0}) Z^{d\tildex^{l+1}_0}$ have the same distribution conditionally on $Z^{\tildeh^{l+1}_0}$, so that
\begin{align*}
    \mathbb{E}\left[\sigma(-\epsilon \, \text{sign}(Z^{\tildeh^{l+1}_0}) Z^{d\tildex^{l+1}_0}) \Big| Z^{\tildeh^{l+1}_0} \right] &= \mathbb{E}\left[\sigma(Z^{d\tildex^{l+1}_0}) \Big| Z^{\tildeh^{l+1}_0} \right] \\
    &= \mathbb{E}\left[\sigma(Z^{d\tildex^{l+1}_0}) \right].
\end{align*}
We thus get
\begin{align*}
    \mathbb{E}[Z^{\tildex^{l+1}_0} Z^{x^{l+1}_1}] &= |\lambda|^p \mathbb{E}[\sigma(Z^{\tildeh^{l+1}_0}) |\sigma'(Z^{\tildeh^{l+1}_0})|^p] \ \mathbb{E}[ \sigma(Z^{d\tildex^{l+1}_0})],
\end{align*}
and both expectations are positive because they are non-negative and their integrands are $>0$ on the positive part of the real line and the Gaussians involved have non-zero density on this subset of $\mathbb{R}$ as they are not degenerate by Lemma~\ref{th:tilde-pos-variance}. This proves $\mathbb{E}[Z^{\tildex^{l+1}_0} Z^{x^{l+1}_1}] > 0$ and concludes the proof by induction. 
\end{proof}

\begin{lemma}[Last layer of forward pass of IP-LLR at $t=1$]\label{th:ipllr-forward-1-L}
Consider the IP-LLR parameterization with an activation function $\sigma$ satisfying Assumption~\ref{ass:smooth-homogeneous-act}. Let $\xi$ be an input to the network, and assume $\scalarlim{\chi}_0 \neq 0$. Then, dropping the dependency of the first forward-backward pass on $\xi_0$, and that of the second forward pass on $\xi$, one has:
\begin{align*}
   &(i) \quad f_1(\xi) = \transpose{{(W^{L+1}(1))}} x^L_1 \xrightarrow[m \rightarrow \infty]{a.s.} \scalarlim{f_1}(\xi) := \mathbb{E}[Z^{U^{L+1}} Z^{x^L_1}] - \eta \scalarlim{\chi}_0 \mathbb{E}[Z^{\tildex^L_0} Z^{x^L_1}], \\
   &(ii) \quad Z^{U^{L+1}} \text{ and } Z^{\tildeh^L_0} \text{ are independent},\\
    &(ii) \quad 0 < \scalarlim{f_1}(\xi) < \infty.
\end{align*}
\end{lemma}

\begin{proof}
Claim $(i)$ comes from Lemma~\ref{th:z-forward-ipllr-1}, in which we have already proved that the limit $\scalarlim{f_1}(\xi)$ is finite as a result of Lemma~\ref{th:Z0-dist-moments} and the fact that the integrands are polynomially bounded functions of the Gaussian vector $(Z^{\tildeh^L_0}, Z^{U^{L+1}})$ which has finite (and diagonal as we will see shortly) covariance matrix. In addition, by Lemma~\ref{th:dotZ-first-forward}, we have $Z^{\tildeh^L_0} = \hatZ^{\tildeh^L_0}$ and by definition in \ZInit\ $Z^{U^{L+1}} = \hatZ^{U^{L+1}}$. Finally, by the \ZHat\ rule, the latter two random variables are independent since $\tildeh^L_0 = \hatW^L \tildex^{L-1}_0$. Let $\epsilon := \text{sign}(\scalarlim{\chi}_0)$ and $\lambda_l := \mathbb{E}[Z^{\tildex^{l}_0} Z^{x^{l}_1}]$ for $l \in \{L-1, L\}$. We have $\lambda_{L-1}, \lambda_L > 0$ by Lemma~\ref{th:ipllr-forward-1-l}, and using again the fact that $\text{sign}(\sigma'(z)) = \text{sign}(z)$ and the positive $p$-homogeneity of $\sigma$, we have
\begin{align*}
    \mathbb{E}[Z^{U^{L+1}} Z^{x^L_1}] &= \mathbb{E}\left[\mathbb{E}\left[Z^{U^{L+1}} Z^{x^L_1} \Big| Z^{\tildeh^{L}_0}\right]\right]\\
    &=  \mathbb{E}\left[ \mathbb{E}\left[Z^{U^{L+1}} \sigma(-\eta \scalarlim{\chi}_0 \lambda_{L-1} Z^{U^{L+1}} \sigma'(Z^{\tildeh^L_0})) \Big| Z^{\tildeh^{L}_0} \right] \right] \\
    &=  |\eta \lambda_{L-1} \scalarlim{\chi}_0|^p \,  \mathbb{E}\left[ |\sigma'(Z^{\tildeh^{L}_0})|^p \mathbb{E}\left[Z^{U^{L+1}} \sigma(- \epsilon \, \text{sign}(Z^{\tildeh^L_0}) Z^{U^{L+1}} ) \Big| Z^{\tildeh^{L}_0} \right] \right].
\end{align*}
Since $Z^{U^{L+1}}$ and $-Z^{U^{L+1}}$ have the same distribution, and it is independent of $Z^{\tildeh^L_0}$, and since $\epsilon \, \text{sign}(Z^{\tildeh^L_0}) \in \{-1, 1\}$, we have 
\begin{align*}
    \mathbb{E}\left[- \epsilon \, \text{sign}(Z^{\tildeh^L_0}) Z^{U^{L+1}} \sigma(- \epsilon \, \text{sign}(Z^{\tildeh^L_0}) Z^{U^{L+1}} ) \Big| Z^{\tildeh^{L}_0} \right] &= \mathbb{E}\left[ Z^{U^{L+1}} \sigma(Z^{U^{L+1}} ) \Big| Z^{\tildeh^{L}_0} \right] \\
    &= \mathbb{E}\left[ Z^{U^{L+1}} \sigma(Z^{U^{L+1}} ) \right],
\end{align*}
so that 
\begin{align*}
    &\mathbb{E}\left[ |\sigma'(Z^{\tildeh^{L}_0})|^p \mathbb{E}\left[Z^{U^{L+1}} \sigma(- \epsilon \, \text{sign}(Z^{\tildeh^L_0}) Z^{U^{L+1}} ) \Big| Z^{\tildeh^{L}_0} \right] \right] = \\
    & \qquad - \epsilon \, \mathbb{E}\left[ \text{sign}(Z^{\tildeh^L_0}) |\sigma'(Z^{\tildeh^{L}_0})|^p \right] \mathbb{E}\left[ Z^{U^{L+1}} \sigma(Z^{U^{L+1}} ) \right].
\end{align*}
We thus get
\begin{align*}
    \mathbb{E}[Z^{U^{L+1}} Z^{x^L_1}] &= - \epsilon \, |\eta \lambda_{L-1} \scalarlim{\chi}_0|^p \, \mathbb{E}\left[ \text{sign}(Z^{\tildeh^L_0}) |\sigma'(Z^{\tildeh^{L}_0})|^p \right] \mathbb{E}\left[ Z^{U^{L+1}} \sigma(Z^{U^{L+1}} ) \right]
\end{align*}
We now prove that both expectations are positive. This is where the assumption that $\alpha > \beta$ (see Appendix~\ref{sec:pos-homog}) appears to be crucial. We start with the first one. Since $Z^{\tildeh^L_0}$ has a zero-mean Gaussian distribution with positive variance (by Lemma~\ref{th:tilde-pos-variance}), its density $p_z$ is positive everywhere and symmetric, and we have
\begin{align*}
    \mathbb{E}\left[ \text{sign}(Z^{\tildeh^L_0}) |\sigma'(Z^{\tildeh^{L}_0})|^p \right] &= \int_{z=0}^{+\infty} (\alpha p)^p z^{p(p-1)} p_z(z) \mathrm{d}z + \int_{z=-\infty}^0 -(\beta p)^p (-z)^{p(p-1)} p_z(z) \mathrm{d}z \\
    &= (\alpha p)^p \int_{z=0}^{+\infty} z^{p(p-1)}p_z(z) \mathrm{d}z - (\beta p)^p \int_{z=0}^{+\infty} z^{p(p-1)}\mathrm{d}z \\
    &= (\alpha^p - \beta^p)p^p \int_{z=0}^{+\infty} z^{p(p-1)}p_z(z) \mathrm{d}z.
\end{align*}
The second equality stems from the change of variable $z \leftarrow -z$ in the second integral and from the symmetry of $p_z$ with respect to $z=0$. The last integral is $>0$ because its integrand is $>0$ on the corresponding domain, and $\alpha^p - \beta^p > 0$ since $\alpha > \beta$ by assumption and $p > 0$. For the second expectation, we get with a similar reasoning that
\begin{align*}
    \mathbb{E}\left[ Z^{U^{L+1}} \sigma(Z^{U^{L+1}} ) \right] &= \int_{u=0}^{+\infty} u \alpha u^p p_u(u)  \mathrm{d}u + \int_{u=-\infty}^{0} u \beta (-u)^p p_u(u)  \mathrm{d}u \\
    &= (\alpha - \beta) \int_{u=0}^{+\infty} u^{p+1} p_u(u) \mathrm{d}u,
\end{align*}
which shows the expectation is $>0$. 
\\ \\
We now look at the second term in $\scalarlim{f_1}(\xi)$: $-\eta \scalarlim{\chi}_0 \mathbb{E}[Z^{\tildex^L_0} Z^{x^L_1}] = - \epsilon \eta |\scalarlim{\chi}_0| \lambda_{L}$. Summing this up with the first term, we get
\begin{align*}
    \scalarlim{f_1}(\xi) = - \epsilon \left[ \underbrace{|\eta \lambda_{L-1} \scalarlim{\chi}_0|^p \, \mathbb{E}\left[ \text{sign}(Z^{\tildeh^L_0}) |\sigma'(Z^{\tildeh^{L}_0})|^p \right] \mathbb{E}\left[ Z^{U^{L+1}} \sigma(Z^{U^{L+1}} ) \right] + \eta |\scalarlim{\chi}_0| \lambda_L}_{> 0} \right]
\end{align*}
which concludes the proof.
\end{proof}

\begin{theorem}
[Non-trivial learning of IP-LLR at $t=1$]\label{th:non-trivial-learning-1}
Consider an IP-LLR parameterization of an $L$-hidden layer neural network with an activation function $\sigma$ satisfying Assumption~\ref{ass:smooth-homogeneous-act}. Let $\xi \in \mathbb{R}^d$ be an input to the network, and assume $\xi_0, \xi, \scalarlim{\chi}_0 \neq 0$. Then, one has:
\begin{align*}
    (i) \ \ &f_0(\xi) \xrightarrow[m \rightarrow \infty]{a.s.} 0 \\
    (ii) \ \ &f_1(\xi) \xrightarrow[m \rightarrow \infty]{a.s.} \scalarlim{f_1}(\xi) \neq 0
\end{align*}
\end{theorem}
\begin{proof}
Claim $(i)$ has already been proved in Lemma~\ref{th:ipllr-forward-backward-0}, and claim $(ii)$ has been proved in Lemma~\ref{th:ipllr-forward-1-L} above. 
\end{proof}

\begin{remark}\label{remark:ipl-llr-forward-1-lr}
Note that since only quantities of the first ($t=0$) forward and backward passes and second ($t=1$) forward pass appear in Lemmas~\ref{th:ipllr-forward-1-0},~\ref{th:ipllr-forward-1-l},~\ref{th:ipllr-forward-1-L}, and \autoref{th:non-trivial-learning-1} we only need to assume we have an integrable parameterization with $c_l = \gamma_l(p)$ for any $l \in [1, L+1]$ at $t=0$.
\end{remark}

\section{Proof that no constant learning rate is possible: \autoref{th:formal-no-constant-lr}}
In this section we prove the result of \autoref{th:formal-no-constant-lr} by splitting the proof in two steps. First we show in Lemma~\ref{th:ip-stable-learning-0} that to have stable and non-vanishing updates for integrable parameterizations at $t=1$, one must use the learning rate exponents $c_l = \gamma_l(p)$ for any $l \in [1, L+1]$ at $t=0$. Then we show some preliminary results on the second backward pass (at $t=1$) for integrable parameterizations when $c_l = \gamma_l(p)$ for any $l \in [1, L+1]$ at $t=0$, and some other preliminary results on the third forward pass (at $t=2$) when additionally one uses $c_1 = -1$, $c_l = -2$ for $l \in [2, L]$ and $c_{L+1} = -1$ at $t=1$. Then we show in Lemma~\ref{th:ip-stable-learning-1}, using those preliminary results, that assuming we have $c_l = \gamma_l(p)$ for any $l \in [1, L+1]$ at $t=0$, to have stable and non-vanishing updates at $t=2$ for integrable parameterizations, one must use the learning rate exponents $c_1 = -1$, $c_l = -2$ for $l \in [2, L]$ and $c_{L+1} = -1$ at $t=1$.

\subsection{Proof of the first implication for the learning rates at $t=0$}

\begin{lemma}[Learning rates for stable learning with IP at $t=0$]\label{th:ip-stable-learning-0}
Consider an $L$-hidden layer fully-connected neural network with $L \geq 3$ in the integrable parameterization, and with no bias terms, except for the first layer. Assume that the activation function $\sigma$ satisfies Assumption~\ref{ass:smooth-homogeneous-act},  and  that $\lim_{m \rightarrow \infty} \partial_2 \ell(y_0, f_0(\xi_0)) \neq 0$. Assume further that $\transpose{\xi_0} \xi_1 \neq 0$. Finally assume that Equation~\eqref{eq:ass-delta1-stable} holds:
\begin{align*}
    \begin{cases}
        \frac{1}{m} ||\Delta W^l(1) x^{l-1}_1 ||^2 = \Theta(1), \quad l \in [1, L] \\
        \transpose{{(\Delta W^{L+1}(1))}} x^L_1 = \Theta(1)
    \end{cases}
\end{align*}
Then, one necessarily has that at $t=0$, $c_l = \gamma_l(p)$ for any $l \in [1, L+1]$ (see Definition~\ref{def:gamma-p}).
\end{lemma}

\begin{proof}
With the notations introduced in Appendix~\ref{app:notations}, the assumptions on the limit of the loss terms at $t=0$ imply $\scalarlim{\chi}_0 \neq 0$. Let us consider the updates at $t=0$. By Corollary~\ref{th:homog-first-weight-updates-ip}, we have 
\begin{align*}
    \Delta W^1(1) \xi_1 = - m^{-( c_1 - \gamma_1(p))} \eta \chi_0 (\transpose{\xi_0} \xi) d\tildeh^1_0,
\end{align*}
so that 
\begin{align*}
    \frac{1}{m} ||\Delta W^1(1) \xi_1||^2 = m^{-2(c_1 - \gamma_1(p))} \left[  \eta \chi_0 (\transpose{\xi_0} \xi_1) \right]^2 \frac{1}{m} \sum_{q=1}^m \left(d\tildeh^1_{0,q} \right)^2.
\end{align*}
From the master theorem, we get that $\sum_{q=1}^m (d\tildeh^1_{0,q} )^2 / m$ converges almost surely towards $\mathbb{E}[(Z^{d\tildeh^1_0})^2]$ which is $>0$ and finite by Lemma~\ref{th:tilde-pos-variance}. On the other hand, $\left[  \eta \chi_0 (\transpose{\xi_0} \xi_1) \right]^2$ converges almost surely to $\left[  \eta \scalarlim{\chi}_0 (\transpose{\xi_0} \xi_1) \right]^2$, which is $>0$ by assumption, and finite. 
\\ \\
If $c_1 > \gamma_1(p)$, then $c_1 - \gamma_1(p) >0$, and $||\Delta W^1(1) \xi_1||^2 / m \rightarrow 0$ almost surely, which is impossible since by assumption, almost surely, there exits $A > 0$ such that for large enough $m$, $A \leq ||\Delta W^1(1) \xi_1||^2 / m$. 
\\ \\
If $c_1 < \gamma_1(p)$, then $c_1 - \gamma_1(p) <0$, and $||\Delta W^1(1) \xi_1||^2 / m \rightarrow \infty$ almost surely, which is impossible since by assumption, almost surely, there exits $B > 0$ such that for large enough $m$, $ ||\Delta W^1(1) \xi_1||^2 / m \leq B$. 
\\ \\
We thus have that $c_1 = \gamma_1(p)$. Let $l \in [1, L-1]$ and assume that $c_k = \gamma_k(p)$ for $k \in [1, l]$. Then by Lemmas~\ref{th:ipllr-forward-1-0} and \ref{th:ipllr-forward-1-l}, we have $0 < \mathbb{E}[Z^{\tildex^k_0} Z^{x^k_1}] < \infty$ for any $k \in [1, l]$. We have
\begin{align*}
    \frac{1}{m} ||\Delta W^{l+1}(1) x^l_1||^2 = m^{-2(c_{l+1} - \gamma_{l+1}(p))} \left[  \eta \chi_0 \frac{\transpose{{(\tildex^l_0)}} x^l_1}{m} \right]^2 \frac{1}{m} \sum_{q=1}^m \left(d\tildeh^{l+1}_{0,q} \right)^2.
\end{align*}
From the master theorem, we get that $\sum_{q=1}^m (d\tildeh^{l+1}_{0,q} )^2 / m$ converges almost surely towards $\mathbb{E}[(Z^{d\tildeh^{l+1}_0})^2]$ which is $>0$ and finite by Lemma~\ref{th:tilde-pos-variance}. On the other hand, $\left[  \eta \chi_0 \transpose{{(\tildex^l_0)}} x^l_1 /m \right]^2$ converges almost surely to $\left[  \eta \scalarlim{\chi}_0 \mathbb{E}[Z^{\tildex^l_0} Z^{x^l_1}] \right]^2$, which is $>0$ and finite. 
\\ \\
If $c_{l+1} > \gamma_{l+1}(p)$, then $c_{l+1} - \gamma_{l+1}(p) >0$, and $||\Delta W^{l+1}(1) x^l_1||^2 / m \rightarrow 0$ almost surely, which is impossible since by assumption, almost surely, there exits $A > 0$ such that for large enough $m$, $A \leq ||\Delta W^{l+1}(1) x^l_1||^2 / m$. 
\\ \\
If $c_{l+1} < \gamma_{l+1}(p)$, then $c_{l+1} - \gamma_1(p) <0$, and $||\Delta W^{l+1}(1) x^l_1||^2 / m \rightarrow \infty$ almost surely, which is impossible since by assumption, almost surely, there exits $B > 0$ such that for large enough $m$, $ ||\Delta W^{l+1}(1) x^l_1||^2 / m \leq B$. 
\\ \\
Therefore, we have $c_{l+1} = \gamma_{l+1}(p)$. By induction, we thus get that $c_l = \gamma_l(p)$ for any $l \in [1, L]$, which means in particular that $ 0 < \mathbb{E}[Z^{\tildex^L_0} Z^{x^L_1}] < \infty$ by Lemma~\ref{th:ipllr-forward-1-l}. Finally, we have 
\begin{align*}
    \transpose{{(\Delta W^{L+1}(1))}} x^L_1 = - m^{-(c_{L+1} - \gamma_{L+1}(p))} \eta \chi_0 \frac{\transpose{(\tildex^L_0)} x^L_1}{m}.
\end{align*}
The term $ \eta \chi_0 \transpose{(\tildex^L_0)} x^L_1 / m$ converges almost surely towards $\eta \scalarlim{\chi}_0 \mathbb{E}[Z^{\tildex^L_0} Z^{x^L_1}]$, whose absolute value is $>0$ and finite. Therefore, if $c_{L+1} > \gamma_{L+1}(p)$ then $c_{L+1} - \gamma_{L+1}(p) > 0$ so that $\transpose{{(\Delta W^{L+1}(1))}} x^L_1 \rightarrow 0$ almost surely, which is impossible since by assumption, almost surely, there exits $A > 0$ such that for large enough $m$, $A \leq |\transpose{{(\Delta W^{L+1}(1))}} x^L_1 |$. If $c_{L+1} < \gamma_{L+1}(p)$ then $c_{L+1} - \gamma_{L+1}(p) < 0$ so that $\transpose{{(\Delta W^{L+1}(1))}} x^L_1 \rightarrow \infty$ almost surely, which is impossible since by assumption, almost surely, there exits $B > 0$ such that for large enough $m$, $|\transpose{{(\Delta W^{L+1}(1))}} x^L_1| \leq B$. Thus, we must have $c_{L+1} = \gamma_{L+1}(p)$, which concludes the proof for the first part. 
\end{proof}

\subsection{Preliminaries on the second backward pass ($t=1$)}

Before we move on to the proof of the second part of the claim of \autoref{th:formal-no-constant-lr}, we stop and prove some preliminary results on the second backward pass (at $t=1$) which will come in handy later on. Similarly to what we did for $\mathbb{E}[Z^{\tildex^l_0} Z^{x^l_1}]$, we wish to prove that the quantity $0 < \mathbb{E}[Z^{d\tildeh^l_0} Z^{d\tildeh^l_1}] < \infty$ for any $l \in [2, L]$. 

\begin{lemma}[Backward pass of IP-LLR at $t=1$]\label{th:ipllr-backward-1}
Consider the IP-LLR parameterization of an $L$ hidden-layer network, and assume that the activation function $\sigma$ satisfies Assumption~\ref{ass:smooth-homogeneous-act},  and  that $\lim_{m \rightarrow \infty} \partial_2 \ell(y_0, f_0(\xi_0)) \neq 0$. Then, one has that for any $l \in [2, L]$,
\begin{align*}
    0 < \mathbb{E}[Z^{d\tildeh^l_0} Z^{d\tildeh^l_1}] < \infty
\end{align*}
\end{lemma}

\begin{remark}\label{remark:ipl-llr-1-lr}
Note that since only quantities of the first ($t=0$) and second ($t=1$) forward and backward passes appear, we only need to assume we have an integrable parameterization with $c_l = \gamma_l(p)$ for any $l \in [1, L+1]$ at $t=0$.
\end{remark}

\begin{proof}
We start with $l = L$, and then induct over $l$ from $l=L$ to $l=2$, and we recall that $\lim_{m \rightarrow \infty} \partial_2 \ell(y_0, f_0(\xi_0)) =: \scalarlim{\chi}_0$ by definition (see Appendix~\ref{app:notations}), which is thus $\neq 0$ by assumption.
\\ \\
\textbf{The case $l=L$.}  By Corollary~\ref{th:z-backward-ipllr-t-0}, we have $Z^{d\tildeh^L_0} = Z^{U^{L+1}} \sigma'(Z^{\tildeh^L_0})$ and $Z^{d\tildeh^L_1} = (Z^{U^{L+1}} - \eta \scalarlim{\chi}_0 \sigma(Z^{\tildeh^L_0})) \sigma'(Z^{h^L_1})$. We thus have 
\begin{align*}
    \mathbb{E}[Z^{d\tildeh^L_0} Z^{d\tildeh^L_1}] = \underbrace{\mathbb{E}[{(Z^{U^{L+1}})^2 \sigma'(Z^{\tildeh^L_0}) \sigma'(Z^{h^L_1})]}}_{:= A}  + \eta |\scalarlim{\chi}_0| \underbrace{\mathbb{E}[ -\epsilon Z^{U^{L+1}} \sigma'(Z^{\tildeh^L_0}) \sigma(Z^{\tildeh^L_0}) \sigma'(Z^{h^L_1})]}_{:=B} ,
\end{align*}
with $\epsilon := \text{sign}(\scalarlim{\chi}_0)$ and we deal with both terms separately. First, by Corollary~\ref{th:z-forward-ipllr-t-0} we re-write $Z^{h^L_1}$ as 
\begin{align*}
    Z^{h^L_1} = - \eta |\scalarlim{\chi}_0| \epsilon \lambda Z^{U^{L+1}} \sigma'(Z^{\tildeh^L_0}),
\end{align*}
where $\lambda := \mathbb{E}[Z^{\tildex^{L-1}_0} Z^{x^{L-1}_1}] > 0$ by Lemma~\ref{th:ipllr-forward-1-l}. Using the fact that $\text{sign}(\sigma'(z)) = \text{sign}(z)$ and the positive $(p-1)$-homogeneity of $\sigma'$, we have
\begin{align*}
    \sigma'(Z^{h^L_1}) = (\eta |\scalarlim{\chi}_0| \lambda)^{p-1} |\sigma'(Z^{\tildeh^L_0})|^{p-1} \sigma'(-\epsilon \text{sign}(Z^{\tildeh^L_0}) Z^{U^{L+1}}).
\end{align*}
The first term in $\mathbb{E}[Z^{d\tildeh^L_0} Z^{d\tildeh^L_1}]$ is thus equal to 
\begin{align*}
    A &= (\eta |\scalarlim{\chi}_0| \lambda)^{p-1} \mathbb{E}\left[ \mathbb{E}\left[ (Z^{U^{L+1}})^2 \sigma'(Z^{\tildeh^L_0}) |\sigma'(Z^{\tildeh^L_0})|^{p-1} \sigma'(-\epsilon \text{sign}(Z^{\tildeh^L_0}) Z^{U^{L+1}}) \Big | Z^{\tildeh^L_0} \right] \right] \\
    &= (\eta |\scalarlim{\chi}_0| \lambda)^{p-1} \mathbb{E}\left[ \sigma'(Z^{\tildeh^L_0}) |\sigma'(Z^{\tildeh^L_0})|^{p-1} \mathbb{E}\left[ (Z^{U^{L+1}})^2 \sigma'(-\epsilon \text{sign}(Z^{\tildeh^L_0}) Z^{U^{L+1}}) \Big | Z^{\tildeh^L_0} \right] \right] \\
    &= (\eta |\scalarlim{\chi}_0| \lambda)^{p-1} \mathbb{E}\left[ \sigma'(Z^{\tildeh^L_0}) |\sigma'(Z^{\tildeh^L_0})|^{p-1} \right] \mathbb{E}\left[ (Z^{U^{L+1}})^2 \sigma'(Z^{U^{L+1}}) \right].
\end{align*}
The third equality stems from the fact that $-\epsilon \text{sign}(Z^{\tildeh^L_0}) Z^{U^{L+1}}$ and $Z^{U^{L+1}}$ have the same distribution conditionally on $Z^{\tildeh^{L}_0}$, and from the fact that $(Z^{U^{L+1}})^2 = (-\epsilon \text{sign}(Z^{\tildeh^L_0}) Z^{U^{L+1}})^2$. We now show that both  expectations are $>0$. Calling $p_z$ the density of the Gaussian $Z^{\tildeh^L_0}$ which is symmetric and positive everywhere since $Z^{\tildeh^L_0}$ is not degenerate, the first term is equal to
\begin{align*}
    \mathbb{E}\left[ \sigma'(Z^{\tildeh^L_0}) |\sigma'(Z^{\tildeh^L_0})|^{p-1} \right] &= (\alpha p)^p \int_{z=0}^{+\infty} z^{p(p-1)} p_z(z) \mathrm{d}z - (\beta p)^p \int_{z=-\infty}^{0} (-z)^{p(p-1)} p_z(z) \mathrm{d}z \\
    &= (\alpha^p - \beta^p) p^p \int_{z=0}^{+\infty} z^{p(p-1)} p_z(z) \mathrm{d}z,
\end{align*}
where we have used the change of variable $z \leftarrow -z$ in the second equality, and the last quantity is $>0$ since $\alpha > \beta$. With similar calculations, we get
\begin{align*}
    \mathbb{E}\left[ (Z^{U^{L+1}})^2 \sigma'(Z^{U^{L+1}}) \right] &= (\alpha - \beta) p \int_{u=0}^{+\infty} u^{p+1} p_u(u)\mathrm{d}u > 0,
\end{align*}
where $p_u$ is the density of the standard Gaussian $Z^{U^{L+1}}$. This thus shows that $A > 0$. 
\\ \\
We now turn to the second term $B$. We have:
\begin{align*}
    B &= (\eta |\scalarlim{\chi}_0|  \lambda)^{p-1} \times && \mathbb{E} \left[ \sigma'(Z^{\tildeh^L_0}) \sigma(Z^{\tildeh^L_0}) |\sigma'(Z^{\tildeh^L_0})|^{p-1} \text{sign}(Z^{\tildeh^L_0}) \times \right.  \\
    & && \qquad \left. \mathbb{E} \left[(-\epsilon \text{sign}(Z^{\tildeh^L_0}) Z^{U^{L+1}}) \sigma'(-\epsilon \text{sign}(Z^{\tildeh^L_0}) Z^{U^{L+1}}) \Big | Z^{\tildeh^L_0} \right] \right] \\
    &= (\eta |\scalarlim{\chi}_0|  \lambda)^{p-1} \times && \mathbb{E} \left[ \sigma'(Z^{\tildeh^L_0}) \sigma(Z^{\tildeh^L_0}) |\sigma'(Z^{\tildeh^L_0})|^{p-1} \text{sign}(Z^{\tildeh^L_0}) \right] \mathbb{E} \left[Z^{U^{L+1}} \sigma'(Z^{U^{L+1}}) \right] \\
    &= (\eta |\scalarlim{\chi}_0|  \lambda)^{p-1} \times && \mathbb{E} \left[ \sigma(Z^{\tildeh^L_0}) |\sigma'(Z^{\tildeh^L_0})|^{p} \right] \mathbb{E} \left[Z^{U^{L+1}} \sigma'(Z^{U^{L+1}}) \right].
\end{align*}
We now prove again that both expectations are $>0$. The first integrand is non-negative everywhere and positive on the positive part of the real line where the Gaussian $Z^{\tildeh^L_0}$ has non-zero density, which shows the first expectation is $> 0$. The same argument holds for the second expectation since $Z^{U^{L+1}}$ and $\sigma'(Z^{U^{L+1}})$ are of the same sign, which also leads to a positive expectation, which finally gives $B > 0$, thereby concluding the proof. 
\\ \\
\textbf{The case $l \in [2, L-1]$.}\\
Let $l \in [2, L-1]$ and assume $0 < \nu := \mathbb{E}[Z^{d\tildeh^{l+1}_0} Z^{d\tildeh^{l+1}_1}] < \infty$. Calling $\epsilon := \text{sign}(\scalarlim{\chi}_0)$, on the one hand, we have by Corollary~\ref{th:z-backward-ipllr-t-0}
\begin{align*}
    Z^{d\tildex^l_1} = - \eta |\scalarlim{\chi}_0| \nu \epsilon \sigma(Z^{\tildeh^l_0}),
\end{align*}
and on the other hand, with $\lambda := \mathbb{E}[Z^{\tildex^{l-1}_0} Z^{x^{l-1}_1}]$, which is $>0$ by Lemmas~\ref{th:ipllr-forward-1-l} and ~\ref{th:ipllr-forward-1-0} (if $l=2$)
\begin{align*}
    \sigma'(Z^{h^l_1}) = (\eta |\scalarlim{\chi}_0| \lambda)^{p-1} |\sigma'(Z^{\tildeh^l_0})|^{p-1} \sigma'(-\epsilon \text{sign}(Z^{\tildeh^l_0}) Z^{d\tildex^l_0}).
\end{align*}
Recalling that $Z^{d\tildeh^l_0} = Z^{d\tildex^l_0} \sigma'(Z^{\tildeh^l_0})$ and $Z^{d\tildeh^l_1} = Z^{d\tildex^l_1} \sigma'(Z^{h^l_1})$, this leads to 
\begin{align*}
    \mathbb{E}[Z^{d\tildeh^{l}_0} Z^{d\tildeh^{l}_1}] = \eta |\scalarlim{\chi}_0| \nu (\eta |\scalarlim{\chi}_0| \lambda)^{p-1} \mathbb{E}[& (-\epsilon \text{sign}(Z^{\tildeh^l_0}) Z^{d\tildex^l_0}) \sigma'(-\epsilon \text{sign}(Z^{\tildeh^l_0}) Z^{d\tildex^l_0}) \\
    &\text{sign}(Z^{\tildeh^l_0}) \sigma'(Z^{\tildeh^l_0}) |\sigma'(Z^{\tildeh^l_0})|^{p-1} \sigma(Z^{\tildeh^l_0})],
\end{align*}
which, by conditioning on $Z^{\tildeh^l_0}$ and since $-\epsilon \text{sign}(Z^{\tildeh^l_0}) Z^{d\tildex^l_0}$ and $Z^{d\tildex^l_0}$ have the same distribution conditionally on $Z^{\tildeh^l_0}$, and since $\text{sign}(\sigma'(z)) = \text{sign}(z)$, gives 
\begin{align*}
    \mathbb{E} [Z^{d\tildeh^{l}_0} Z^{d\tildeh^{l}_1}] = \eta |\scalarlim{\chi}_0| \nu (\eta |\scalarlim{\chi}_0| \lambda)^{p-1} \mathbb{E}\left[ Z^{d\tildex^l_0} \sigma'(Z^{d\tildex^l_0}) \right] \mathbb{E}\left[ |\sigma'(Z^{\tildeh^l_0})|^{p} \sigma(Z^{\tildeh^l_0})\right].
\end{align*}
The term in front of the expectations is positive by assumption, and both expectations are positive because their integrands are both non-negative and positive on the positive part of the real line where the Gaussians $Z^{d\tildex^l_0}$ and $Z^{\tildeh^l_0}$ have non-zero density. The expectations are also finite by Lemma~\ref{th:Z0-dist-moments} because their integrands are polynomially bounded functions of some Gaussian vector with finite covariance variance matrix. By induction, we thus get that $ 0 < \mathbb{E} [Z^{d\tildeh^{l}_0} Z^{d\tildeh^{l}_1}] < \infty$ for any $l \in [2, L]$, which concludes the proof.
\end{proof}

\subsection{Preliminaries on the third forward pass ($t=2$)}

In this section we wish to prove that similarly to the second forward pass, the quantities the quantities $\mathbb{E}[Z^{x^l_1} Z^{x^l_2}]$ and $\mathbb{E}[Z^{\tildex^l_0} Z^{x^l_2}]$ (which appear in the third forward pass at $t=2$) are $>0$ for any $l \in [1, L]$ when using the IP-LLR learning rates at $t=0$ and $t=1$. We assume here that the training samples $\xi_0, \xi_1, \xi_2$ are all distinct, which is probably not necessary for the result to hold but simplifies somewhat some parts of the proof and is in any case a very natural assumption.

\begin{lemma}[Forward pass of IP-LLR at $t=2$]\label{th:ipllr-forward-2}
Consider the IP-LLR parameterization of an $L$ hidden-layer network, and assume that the activation function $\sigma$ satisfies Assumption~\ref{ass:smooth-homogeneous-act},  and  that $\lim_{m \rightarrow \infty} \partial_2 \ell(y_0, f_0(\xi_0)) \neq 0$ and $\lim_{m \rightarrow \infty} \partial_2 \ell(y_1, f_1(\xi_1)) \neq 0$. Assume further that the first three training samples $\xi_0, \xi_1, \xi_2$ are all distinct. Then, one has that for any $l \in [1, L]$,
\begin{align*}
    0 <\, &\mathbb{E}[Z^{x^l_1} Z^{x^l_2}] < \infty \\
    0 <\, &\mathbb{E}[Z^{\tildex^l_0} Z^{x^l_2}]  < \infty
\end{align*}
\end{lemma}

\begin{proof}
We start with the case $l=1$ and then induct over $l$ from $l=1$ to $l=L$ for both expectations simultaneously as the derivations are very similar.
\\ \\
\textbf{The case $l=1$.} \\
Let us first unwind the expressions of $Z^{h^1_1}$ and $Z^{h^1_2}$. We have
\begin{align*}
    Z^{h^1_1} = Z^{\tildeh^1_0(\xi_1)} - \eta \scalarlim{\chi}_0 (\transpose{\xi}_0 \xi_1 + 1) Z^{d\tildex^1_0} \sigma'(Z^{\tildeh^1_0}),
\end{align*}
and
\begin{align*}
    Z^{h^1_2} = Z^{\tildeh^1_0(\xi_2)} - \eta \scalarlim{\chi}_0 (\transpose{\xi}_0 \xi_2 +1) Z^{d\tildex^1_0} \sigma'(Z^{\tildeh^1_0}) - \eta \scalarlim{\chi}_1 (\transpose{\xi_1} \xi_2 +1) Z^{d\tildex^1_1} \sigma'(Z^{h^1_1}).
\end{align*}
\\ \\
\textbf{The case of $\mathbb{E}[Z^{x^1_1} Z^{x^1_2}]$.}\\
Recalling that $Z^{d\tildex^1_1} = -\eta \scalarlim{\chi}_0 \nu \sigma(Z^{\tildeh^1_0})$ where $\nu := \mathbb{E}[Z^{d\tildeh^2_0} Z^{d\tildeh^2_1}]$. With the assumption that $\xi_0, \xi_1, \xi_2$ are all distinct, the vector $(Z^{\tildeh^1_0}, Z^{\tildeh^1_0(\xi_1)}, Z^{\tildeh^1_0(\xi_2)}, Z^{d\tildex^1_0})$ has a non-degenerate Gaussian distribution, and we thus get
\begin{align*}
    \mathbb{E}[Z^{x^1_1} Z^{x^1_2}] = \int & \sigma \left(u_1 - \mu_0 z \sigma'(u_0) \right) \sigma \left(u_2 - \mu_1 z \sigma'(u_0) + \mu_2 \sigma(u_0) \sigma' \left(u_1 - \mu_0 z \sigma'(u_0) \right) \right) \times \\
    &q(u_0, u_1, u_2) p_z(z) \mathrm{d}(u_0, u_1, u_2) \mathrm{d}z,
\end{align*}
where $\mu_0 := \eta \scalarlim{\chi}_0(\transpose{\xi_0} \xi_1 + 1)$, $\mu_1 := \eta \scalarlim{\chi}_0(\transpose{\xi_0} \xi_2 + 1)$ and $\mu_2 := \eta^2 \scalarlim{\chi}_0 \scalarlim{\chi}_1 \nu (\transpose{\xi_1} \xi_2 +1)$, and $q$ and $p_z$ are the densities of non-degenerate Gaussians and are thus positive everywhere. Now the integrand is non-negative everywhere and we wish to show that it is positive at some given point of $\mathbb{R}^4$, and it is also a polynomially bounded function of $(Z^{\tildeh^1_0}, Z^{\tildeh^1_0(\xi_1)}, Z^{\tildeh^1_0(\xi_2)}, Z^{d\tildex^1_0})$ which shows that the expectation is finite. Since $Z^{d\tildex^1_0}$ and $-Z^{d\tildex^1_0}$ have the same distribution and it is independent of $(Z^{\tildeh^1_0}, Z^{\tildeh^1_0(\xi_1)}, Z^{\tildeh^1_0(\xi_2)})$, we can assume that $\mu_0 \geq 0$ W.L.O.G. Consider the point $(u_0^*, u_1^*, u_2^*, z^*)$ defined as $u_0^* = u_1^* = 1$, $z^* = -1$ and 
\begin{align*}
    u_2^* := |\mu_1| \sigma'(1) + |\mu_2| \sigma(1) \sigma' \left(1 + \mu_0 \sigma'(1)\right) + 1.
\end{align*}
We show below that the integrand is $> 0$ at $(u_0^*, u_1^*, u_2^*, z^*)$. Since it is also a continuous function of $(u_0, u_1, u_2, z)$, we get that the expectation is positive.
\\ \\
Let us now show that the integrand is $> 0$ at $(u_0^*, u_1^*, u_2^*, z^*)$. We have 
\begin{align*}
    u_1^* - \mu_0 z^* \sigma'(u_0^*) = 1 + \mu_0 \sigma'(1) \geq 1 > 0,
\end{align*}
and 
\begin{align*}
    -\mu_1 z^* \sigma'(u_0^*) &= \mu_1 \sigma'(1) \geq  -|\mu_1| \sigma'(1),
\end{align*}
and finally
\begin{align*}
    \mu_2 \sigma(u_0^*) \sigma' \left(u_1^* - \mu_0 z^* \sigma'(u_0^*) \right) = \mu_2 \sigma(1) \sigma' \left(1 + \mu_0 \sigma'(1) \right) \geq -|\mu_2| \sigma(1) \sigma' \left(1 + \mu_0 \sigma'(1) \right).
\end{align*}
With the choice for $u_2^*$, one has that $u_2^* - \mu_1 z \sigma'(u_0^*) + \mu_2 \sigma(u_0^*) \sigma' \left(u_1^* - \mu_0 z^* \sigma'(u_0^*) \right) \geq 1 > 0$, which concludes the proof because $\sigma$ is positive on the positive part of the real line.
\\ \\
\textbf{The case of $\mathbb{E}[Z^{\tildex^1_0} Z^{x^1_2}]$.}

We have
\begin{align*}
    \mathbb{E}[Z^{\tildex^1_0} Z^{x^1_2}] = \int & \sigma \left(u_0 \right) \sigma \left(u_2 - \mu_1 z \sigma'(u_0) + \mu_2 \sigma(u_0) \sigma' \left(u_1 - \mu_0 z \sigma'(u_0) \right) \right) \times \\
    &q(u_0, u_1, u_2) p_z(z) \mathrm{d}(u_0, u_1, u_2) \mathrm{d}z,
\end{align*}
As for the case of $\mathbb{E}[Z^{x^1_1} Z^{x^1_2}]$, we  show that the integrand is $> 0$ at the same point $(u_0^*, u_1^*, u_2^*, z^*)$ as above, and since it is also a continuous function of $(u_0, u_1, u_2, z)$, we get that the expectation is positive.
It is also finite by Lemma~\ref{th:Z0-dist-moments} because its integrand is a polynomially bounded function of $(Z^{\tildeh^1_0}, Z^{\tildeh^1_0(\xi_1)}, Z^{\tildeh^1_0(\xi_2)}, Z^{d\tildex^1_0})$.
\\ \\
\textbf{The case $l \in [2, L-1]$.}\\
Let $l \in [2, L-1]$ and assume $\tau := \mathbb{E}[Z^{x^{l-1}_1} Z^{x^{l-1}_2}] > 0$ and $\rho := \mathbb{E}[Z^{\tildex^{l-1}_0} Z^{x^{l-1}_2}] > 0$. Calling $\lambda:=\mathbb{E}[Z^{\tildex^{l-1}_0} Z^{x^{l-1}_1}]$ which is $> 0$ by Lemma~\ref{th:ipllr-forward-1-l}, and $\nu := \mathbb{E}[Z^{d\tildeh^{l+1}_0} Z^{d\tildeh^{l+1}_1}]$ which is also $>0$ by Lemma~\ref{th:ipllr-backward-1}, we have 
\begin{align*}
    Z^{h^{l}_1} = -\eta \scalarlim{\chi}_0 \lambda Z^{d\tildex^{l}_0} \sigma'(Z^{\tildeh^{l}_0}),
\end{align*}
and 
\begin{align*}
    Z^{h^{l}_2} = -\eta \scalarlim{\chi}_0 \rho Z^{d\tildex^{l}_0} \sigma'(Z^{\tildeh^{l}_0}) - \eta \scalarlim{\chi}_1 \tau Z^{d\tildex^{l}_1} \sigma'(Z^{h^{l}_1}).
\end{align*}
Finally recall that $Z^{d\tildex^{l}_1} = -\eta \scalarlim{\chi}_0 \nu \sigma(Z^{\tildeh^l_0})$, and let us call $\mu_0 := \eta \scalarlim{\chi}_0 \lambda$, $\mu_1 := \eta \scalarlim{\chi}_0 \rho$ and $\mu_2 := \eta^2 \scalarlim{\chi}_0 \scalarlim{\chi}_1 \tau \nu$. $\mu_0$ is $\neq 0$ because of the assumption on $\scalarlim{\chi}_0$. Since $Z^{d\tildex^l_0}$ and $-Z^{d\tildex^l_0}$ have the same distribution and it is independent of $Z^{\tildeh^l_0}$, we can assume $\mu_0 > 0$ W.L.O.G. Note then that since $\mu_1$ is of the same sign as $\mu_0$ ($\lambda \rho > 0$), this also implies $\mu_1 > 0$, and $\mu_2$ has the sign of $\scalarlim{\chi}_1$. By assumption, $\scalarlim{\chi}_1 \neq 0$, and by the induction hypothesis and Lemma~\ref{th:ipllr-backward-1} we have $\mu_2 \neq 0$.
\\ \\
\textbf{The case of $\mathbb{E}[Z^{x^l_1} Z^{x^l_2}]$.}\\
We have
\begin{align*}
    \mathbb{E}[Z^{x^l_1} Z^{x^l_2}] = \int & \sigma \left(-\mu_0 z \sigma'(u)\right) \sigma \left(- \mu_1 z \sigma'(u) + \mu_2 \sigma(u) \sigma' \left(- \mu_0 z \sigma'(u) \right) \right) p_u(u) p_z(z) \mathrm{d}u \mathrm{d}z,
\end{align*}
where $p_u$ and $p_z$ are the densities of non-degenerate Gaussians ($Z^{\tildeh^l_0}$ and $Z^{d\tildex^l_0}$ respectively) and are thus positive everywhere. Now the integrand is non-negative everywhere and we wish to show that it is positive at some given point of $\mathbb{R}^2$. The integrand is also a polynomially bounded function of $(Z^{\tildeh^l_0}, Z^{d\tildex^l_0})$ which shows that the expectation is finite by Lemma~\ref{th:Z0-dist-moments}. Let $z^* = -1$ and $u > 0$. Then, $-\mu_0 z^* \sigma'(u) = \mu_0 \sigma'(u) > 0$ so that $\sigma(-\mu_0 z^* \sigma'(u)) > 0$. On the other hand, $-\mu_1 z^* \sigma'(u) = \mu_1 \alpha p u^{p-1}$, and 
\begin{align*}
    \mu_2 \sigma(u) \sigma'(-\mu_0 z^* \sigma'(u)) &= \mu_2 \alpha u^p \alpha p (\mu_0 \alpha p)^{p-1} u^{(p-1)^2} \\
    &\geq - (\alpha p) |\mu_2| \alpha (\mu_0 \alpha p)^{p-1} u^{p-1} u^{(p-1)^2 + 1}.
\end{align*}
This leads to
\begin{align*}
    -\mu_1 z^* \sigma'(u) + \mu_2 \sigma(u) \sigma'(-\mu_0 z^* \sigma'(u)) &\geq \alpha p u^{p-1} \left[\mu_1 - |\mu_2| \alpha (\mu_0 \alpha p)^{p-1} u^{(p-1)^2 + 1} \right].
\end{align*}
The quantity in the bracket is $>0$ as soon as 
\begin{align*}
    u < \left[ \frac{\mu_1}{|\mu_2| \alpha (\mu_0 \alpha p)^{p-1} } \right]^{\frac{1}{(p-1)^2 + 1}} =: \varepsilon
\end{align*}
Calling $u^* := \epsilon /2$, we thus  get that the integrand is $>0$ at $(u^*, z^*)$, and since it is a continuous function of $(u, z)$, the integral is positive.
\\\\
\textbf{The case of $\mathbb{E}[Z^{\tildex^l_0} Z^{x^l_2}]$.}\\
We have
\begin{align*}
    \mathbb{E}[Z^{\tildex^l_0} Z^{x^l_2}] = \int & \sigma \left(u \right) \sigma \left(- \mu_1 z \sigma'(u) + \mu_2 \sigma(u) \sigma' \left(- \mu_0 z \sigma'(u) \right) \right) p_u(u) p_z(z) \mathrm{d}u \mathrm{d}z,
\end{align*}
The integrand is non-negative everywhere and with $z^* = -1$ and $u^* = \varepsilon /2$ as above, one  shows that the integrand is $>0$ at $(u^*, z^*)$ which in turn  implies that the expectation is positive.
It is also finite for the same reasons as $\mathbb{E}[Z^{x^l_1} Z^{x^l_2}]$. This now concludes the induction over $l \in [1, L-1]$ which thus shows that $\mathbb{E}[Z^{x^l_1} Z^{x^l_2}]$ and $\mathbb{E}[Z^{\tildex^l_0} Z^{x^l_2}]$ are $>0$ and finite for any $l \in [1, L-1]$. Those expectations are also finite as their integrands are polynomially bounded functions of Gaussian vectors which have finite covariance matrices.
\\ \\
\textbf{The case $l = L$.}\\
Let $\tau := \mathbb{E}[Z^{x^{L-1}_1} Z^{x^{L-1}_2}] > 0$ and $\rho := \mathbb{E}[Z^{\tildex^{L-1}_0} Z^{x^{L-1}_2}] > 0$ by the previous induction. Calling $\lambda:=\mathbb{E}[Z^{\tildex^{L-1}_0} Z^{x^{L-1}_1}]$ which is $> 0$ by Lemma~\ref{th:ipllr-forward-1-l}, we have 
\begin{align*}
    Z^{h^{L}_1} = -\eta \scalarlim{\chi}_0 \lambda Z^{U^{L+1}} \sigma'(Z^{\tildeh^{L}_0}),
\end{align*}
and 
\begin{align*}
    Z^{h^{L}_2} = -\eta \scalarlim{\chi}_0 \rho Z^{U^{L+1}} \sigma'(Z^{\tildeh^{L}_0}) - \eta \scalarlim{\chi}_1 \tau Z^{d\tildex^{L}_1} \sigma'(Z^{h^{L}_1}).
\end{align*}
Finally recall that $Z^{d\tildex^{L}_1} = Z^{U^{L+1}} -\eta \scalarlim{\chi}_0 \sigma(Z^{\tildeh^L_0})$, and let us call $\mu_0 := \eta \scalarlim{\chi}_0 \lambda$, $\mu_1 := \eta \scalarlim{\chi}_0 \rho$, $\mu_2 := \eta \scalarlim{\chi}_1 \tau$, and finally $\mu_3 := \eta^2 \scalarlim{\chi}_0 \scalarlim{\chi}_1 \tau$.
Since $Z^{d\tildex^l_0}$ and $-Z^{d\tildex^l_0}$ have the same distribution and it is independent of $Z^{\tildeh^l_0}$, we can assume $\mu_0 > 0$ W.L.O.G. Note then that since $\mu_1$ is of the same sign as $\mu_0$, this also implies $\mu_1 > 0$, and $\mu_2$ has the sign of $\scalarlim{\chi}_1$. In addition, with the assumptions and previous results, we have $\mu_2 \neq 0$ and $\mu_3 \neq 0$.
\\ \\
\textbf{The case of $\mathbb{E}[Z^{x^L_1} Z^{x^L_2}]$.}\\
We have
\begin{align*}
    \mathbb{E}[Z^{x^L_1} Z^{x^L_2}] = \int & \sigma \left( - \mu_0 z \sigma'(u) \right) \sigma \left(- \mu_1 z \sigma'(u) + (-\mu_2 z + \mu_3 \sigma(u)) \sigma' \left(- \mu_0 z \sigma'(u) \right) \right) \times \\
    &p_u(u) p_z(z) \mathrm{d}u \mathrm{d}z,
\end{align*}
where $p_u$ and $p_z$ are the densities of non-degenerate Gaussians ($Z^{\tildeh^L_0}$ and $Z^{U^{L+1}}$ respectively) and are thus positive everywhere. Now the integrand is non-negative everywhere and we wish to show that it is positive at some point of $\mathbb{R}^2$. The integrand is also a polynomially bounded function of $(Z^{\tildeh^L_0}, Z^{U^{L+1}})$ and the expectation is thus finite by Lemma~\ref{th:Z0-dist-moments}. We first take a closer look at the second term inside $\sigma$. Let $z \leq 0, u \geq 0$. We have
\begin{align*}
    -\mu_1 z \sigma'(u) = \mu_1 |z| \alpha p u^{p-1},
\end{align*}
as well as 
\begin{align*}
    -\mu_2 z + \mu_3 \sigma(u) = - \mu_2 z + \mu_3 \alpha u^p,
\end{align*}
and 
\begin{align*}
    \sigma'(-\mu_0 z \sigma'(u)) = \alpha p (\mu_0 \alpha p)^{p-1} |z|^{p-1} u^{(p-1)^2}.
\end{align*}
We thus get that 
\begin{align*}
    &- \mu_1 z \sigma'(u) + (-\mu_2 z + \mu_3 \sigma(u)) \sigma' \left(- \mu_0 z \sigma'(u) \right) = \\
    & \qquad \alpha p |z| u^{p-1} \left[ \underbrace{\mu_1 + (-\mu_2 z + \mu_3 \alpha |u|^p) (\mu_0 \alpha p)^{p-1} |z|^{p-2} |u|^{(p-1)(p-2)}}_{F(u,z)} \right]
\end{align*}
Because $p-2 \geq 0$, the function $F$ is continuous over $\mathbb{R}^2$, and we have $F(0,0) = \mu_1 > 0$. Therefore, there exists $u^* > 0$ and $z^* < 0$ such that $F(u^*, z^*) > 0$. With such a pair $(u^*, z^*)$ we get that the integrand is $>0$ at $(u^*, z^*)$, and since it is a continuous function of $(u,z)$, it  follows 
that the expectation is positive. 
\\ \\
\textbf{The case of $\mathbb{E}[Z^{\tildex^L_0} Z^{x^L_2}]$.}\\
A similar argument to the case of $\mathbb{E}[Z^{x^L_1} Z^{x^L_2}]$ applies and we  get that the expectation is positive, which concludes the proof.
\end{proof}

\subsection{Proof of the second implication}

\begin{lemma}[Learning rates for stable learning with IP at $t=1$]\label{th:ip-stable-learning-1}
Consider an $L$-hidden layer fully-connected neural network with $L \geq 3$ in the integrable parameterization, and with no bias terms, except at the first layer. Assume that the activation function $\sigma$ satisfies Assumption~\ref{ass:smooth-homogeneous-act},  and  that $\lim_{m \rightarrow \infty} \partial_2 \ell(y_0, f_0(\xi_0)) \neq 0$ and $\lim_{m \rightarrow \infty} \partial_2 \ell(y_1, f_0(\xi_1)) \neq 0$. Assume further that $\transpose{\xi_1} \xi_2 \neq 0$, that the first three training samples $\xi_0, \xi_1, \xi_2$ are all distinct, and that at $t=0$ (\ie to compute $\Delta W^l(1)$) $c_l = \gamma_l(p)$ (see Definition~\ref{def:gamma-p}) for any $l \in [1, L+1]$. Finally assume that Equation~\eqref{eq:ass-delta2-stable} holds:
\begin{align*}
    \begin{cases}
        \frac{1}{m} ||\Delta W^l(2) x^{l-1}_2 ||^2 = \Theta(1), \quad l \in [1, L] \\
        \transpose{{(\Delta W^{L+1}(2))}} x^L_2 = \Theta(1)
    \end{cases}
\end{align*}
Then, one necessarily has that at $t=1$, $c_1 = c_{l+1} = -1$ and  $c_l = -2$ for any $l \in [2, L]$.
\end{lemma}

\begin{proof}
We first treat the case $l=1$ and then induct over $l$ from $l=2$ to $l=L$ and conclude by the case $l=L+1$. Note that because of the assumptions, Lemma~\ref{th:ipllr-backward-1} holds and the claim of Lemma~\ref{th:ipllr-forward-2} will hold at layer $l$ as soon as we show $c_1 = -1$ and $c_k = -2$ for $k \in [2, L]$.
\\ \\
\textbf{The case $l=1$.}\\
We have
\begin{align*}
    \Delta W^1(2) \xi_2 = - \eta m^{-(1 +c_1)} \chi_1 (\transpose{{\xi_1}} \xi_2) d\tildex^1_1 \odot \sigma'(h^1_1),
\end{align*}
so that 
\begin{align*}
    \frac{1}{m} ||\Delta W^1(2) \xi_2||^2 = m^{-2(1+c_1)} (\eta  \chi_1 (\transpose{{\xi_1}} \xi_2))^2 \frac{1}{m} ||d\tildex^1_1 \odot \sigma'(h^1_1)||^2.
\end{align*}
Recall that $d\tildex^1_1 = -\eta \chi_0 (\transpose{{(d\tildeh^2_0)}} d\tildeh^2_1)/m \, \sigma(\tildeh^1_0)$, so that by the Master Theorem,
\begin{align*}
    \frac{1}{m} ||d\tildex^1_1 \odot \sigma'(h^1_1)||^2 \xrightarrow[m \rightarrow \infty]{a.s.} (\eta \scalarlim{\chi}_0 \nu)^2 \mathbb{E}[\sigma(Z^{\tildeh^1_0})^2 \sigma'(Z^{h^1_1})^2],
\end{align*}
where $\nu := \mathbb{E}[Z^{d\tildeh^2_0} Z^{d\tildeh^2_1}] > 0$ by Lemma~\ref{th:ipllr-backward-1}. The term in front of the expectation is $>0$ with the assumptions. On the other hand, the term $(\eta \chi_1 (\transpose{{\xi_1}} \xi_2))^2$ converges almost surely towards $(\eta \scalarlim{\chi}_1 (\transpose{{\xi_1}} \xi_2))^2$ which is also $>0$ with the assumptions. We show below that the expectation is $>0$, which proves that $c_1$ must be equal to $1$ since by assumption $\frac{1}{m} ||\Delta W^1(2) \xi_2||^2 = \Theta(1)$. Recall that 
\begin{align*}
    Z^{h^1_1} = Z^{\tildeh^1_0(\xi_1)} - \eta \scalarlim{\chi}_0(\transpose{\xi_0}\xi_1 + 1) Z^{d\tildex^1_0} \sigma'(Z^{\tildeh^1_0}).
\end{align*}
The integrand in the expectation is non-negative, and it simply remains to show that is not almost surely zero. Because $Z^{d\tildex^1_0}$ and $-Z^{d\tildex^1_0}$ have the same distribution, and since it is independent of $(Z^{\tildeh^1_0}, Z^{\tildeh^1_0(\xi_1)})$, we can assume W.L.O.G. that $\mu := \eta \scalarlim{\chi}_0(\transpose{\xi_0}\xi_1 + 1) \geq 0$. As usual, the vector $(Z^{\tildeh^1_0}, Z^{\tildeh^1_0(\xi_1)})$ has a Gaussian distribution which is degenerate only if $\xi_1 = \xi_0$, which is precluded by the assumptions. Note that in any case, the expectation is finite by Lemma~\ref{th:Z0-dist-moments} since its integrand is a polynomially bounded function of a Gaussian vector with finite covariance matrix. Since $\xi \neq \xi_0$ by assumption, we have
\begin{align*}
        \mathbb{E}[\sigma(Z^{\tildeh^1_0})^2 \sigma'(Z^{h^1_1})^2] = \int \sigma(u)^2 \sigma' \left(v - \mu z \sigma'(u) \right)^2 p_{u,v}(u,v) p_z(z) \mathrm{d}u \mathrm{d}v \mathrm{d}z,
\end{align*}
where $p_{u,v}$ and and $p_z$ are the densities of non-degenerate Gaussians ($(Z^{\tildeh^1_0}, Z^{\tildeh^1_0(\xi_1)})$ and $Z^{d\tildex^1_0}$ respectively) and are thus well-defined and positive everywhere. Again, one  sees that at point $(u^*, v^*, z^*) = (1, 1, -1)$ the integrand is $>0$, and since it is a continuous function, this proves 
that the expectation is positive. It is also finite by Lemma~\ref{th:Z0-dist-moments} since the integrand is a polynomially bounded function of a Gaussian vector with finite covariance matrix.
\\ \\
\textbf{The case $l \in [1, L-1]$}\\
Let $l \in [2, L-1]$. We have already shown that $c_1 = -1$. Assume now that $c_k = -2$ for $k \in [2, l-1]$ (note that if $l=2$ this means no additional assumption). Then we have
\begin{align*}
    \Delta W^{l}(2) x^{l-1}_2 = - \eta m^{-(2 +c_l)} \chi_1 \frac{\transpose{{(x^{l-1}_1)}} x^{l-1}_2}{m} d\tildex^l_1 \odot \sigma'(h^l_1),
\end{align*}
so that 
\begin{align*}
    \frac{1}{m} ||\Delta W^l(2) x^{l-1}_2||^2 = m^{-2(2+c_l)} \left(\eta  \chi_1 \frac{\transpose{{(x^{l-1}_1)}} x^{l-1}_2}{m} \right)^2 \frac{1}{m} ||d\tildex^l_1 \odot \sigma'(h^l_1)||^2.
\end{align*}
In addition, we have $d\tildex^l_1 = -\eta \chi_0 (\transpose{{(d\tildeh^{l+1}_0)}} d\tildeh^{l+1}_1)/m \sigma(\tildeh^l_0)$, so that by the Master Theorem,
\begin{align*}
    \frac{1}{m} ||d\tildex^1_1 \odot \sigma'(h^1_1)||^2 \xrightarrow[m \rightarrow \infty]{a.s.} (\eta \scalarlim{\chi}_0 \nu)^2 \mathbb{E}[\sigma(Z^{\tildeh^l_0})^2 \sigma'(Z^{h^l_1})^2],
\end{align*}
where $\nu := \mathbb{E}[Z^{d\tildeh^{l+1}_0} Z^{d\tildeh^{l+1}_1}]$ is such that $0 < \nu < \infty$ by Lemma~\ref{th:ipllr-backward-1}. Recall that 
\begin{align*}
    Z^{h^l_1} = - \eta \scalarlim{\chi}_0 \lambda Z^{d\tildex^l_0} \sigma'(Z^{\tildeh^l_0}),
\end{align*}
with $\lambda := \mathbb{E}[Z^{\tildex^{l-1}_0} Z^{x^{l-1}_1}]$ such that $0 < \lambda < \infty$ by Lemmas~\ref{th:ipllr-forward-1-0} and~\ref{th:ipllr-forward-1-l}, which leads to 
\begin{align*}
    \mathbb{E}[\sigma(Z^{\tildeh^l_0})^2 \sigma'(Z^{h^l_1})^2] &= \int \sigma(u)^2 \sigma'(-\mu z \sigma'(u))^2 p_u(u) p_z(z) \mathrm{d}u \mathrm{d}z,
\end{align*}
where $\mu := \eta \scalarlim{\chi}_0 \lambda$ which is $\neq 0$ with the assumptions, and $p_u$ and $p_z$ are the densities of two non-degenerate Gaussians ($Z^{\tildeh^l_0}$ and $Z^{d\tildex^l_0}$ respectively) and are thus positive everywhere. Since $Z^{d\tildex^l_0}$ and $-Z^{d\tildex^l_0}$ have the same distribution and it is independent of $Z^{\tildeh^l_0}$ we can assume $\mu > 0$ W.L.O.G. Then, we  see that at point $(u^*, z^*) = (1, -1)$ the integrand is $>0$, and since it is a continuous function, this proves 
that the expectation is positive. It is also finite by Lemma~\ref{th:Z0-dist-moments} since the integrand is a polynomially bounded function of a Gaussian vector with finite covariance matrix. The term $(\eta \scalarlim{\chi}_0 \nu)^2$ in front of the expectation is $>0$ and finite with the assumptions. Finally the term $(\eta  \chi_1 (\transpose{{(x^{l-1}_1)}} x^{l-1}_2) / m)^2$ converges almost surely towards $(\eta \scalarlim{\chi}_1 \tau)^2$ by the Master Theorem, where $\tau := \mathbb{E}[Z^{x^{l-1}_1} Z^{x^{l-1}_2}]$ is $>0$ and finite by Lemma~\ref{th:ipllr-forward-2}, which shows that $(\eta \scalarlim{\chi}_1 \tau)^2$ is $>0$ and finite with the assumptions. Since $||d\tildex^1_1 \odot \sigma'(h^1_1)||^2 / m = \Theta(1)$ by assumption, then $c_l$ must be equal to $-2$ otherwise $||d\tildex^1_1 \odot \sigma'(h^1_1)||^2 / m$ would either converge towards $0$ or diver towards $\infty$ almost surely. 
\\ \\
\textbf{The case $l = L$.}
We have already proved that at $t=1$, $c_1 = -1$ and $c_l = -2$ for $l \in [2, L]$. We have
\begin{align*}
    |\transpose{{(\Delta W^{L+1}(2))}} x^L_2| = \eta m^{-(1 + c_{L+1})} |\chi_1| \left|\frac{\transpose{{(x^L_1)}} x^L_2}{m} \right|.
\end{align*}
By the Master Theorem,
\begin{align*}
    \left|\frac{\transpose{{(x^L_1)}} x^L_2}{m} \right| \xrightarrow[m \rightarrow \infty]{a.s.} \mathbb{E}[Z^{x^L_1} Z^{x^L_2}]
\end{align*}
which is $>0$ and finite by Lemma~\ref{th:ipllr-forward-2}. On the other hand, $\eta |\chi_1|$ converges almost surely towards $\eta |\scalarlim{\chi}_1|$ which is also $>0$ and finite. This shows that since $|\transpose{{(\Delta W^{L+1}(2))}} x^L_2| = \Theta(1)$ then we must have $c_{L+1} = -1$ to avoid vanishing towards $0$ or explosion towards $+\infty$ as $m \rightarrow \infty$, which concludes the proof.
\end{proof}

\section{Proof of  the non-triviality of IP-LLR: \autoref{th:non-trivial-ipllr}}\label{app:non-trivial-ipllr}

\begin{proof}
Claims $(i)$ and $(ii)$ of \autoref{th:non-trivial-ipllr} have already been shown in \autoref{th:non-trivial-learning-1}. Claim $(iii)$ simply stems from Corollary~\ref{th:z-forward-ipllr-t-0} with $t=2$ and the fact that all the variables $Z$ which appear are polynomially bounded functions of the vector $Z_0$ (see Definition~\ref{def:Z0}) by a simple induction.
\end{proof}

\section{Proof of the equivalence between IP-LLR and \muP: Proposition~\ref{th:finite-equivalence-ipllr-hp} and \autoref{th:hpz-ipllr-equivalence}}\label{app:proof-equivalence-ip-muP}

In this section, we present the proofs of the equivalence between IP-LLR and hybrid versions of \muP\ both at finite-width and in the large-width limit. Because we need to use the homogeneity property, we consider a positively $p$-homogeneous activation function $\sigma$ and no bias terms except at the first layer for all the parameterizations we consider. We assume $p \geq 1$ for the finite-width case, which includes \relu, and $p \geq 2$ in the infinite-width case as we use the Tensor Program framework for the proof and thus require some smoothness. 

\subsection{Finite-width equivalence: Proposition~\ref{th:finite-equivalence-ipllr-hp}}\label{app:proof-equivalence-ip-muP-finite}
We start with a preliminary Lemma showing the equivalence at $t=1$ and then do the proof of Proposition~\ref{th:finite-equivalence-ipllr-hp} by induction.

\subsubsection{Equivalence at $t=1$}
\begin{lemma}[First weight updates of HP]\label{th:hp-first-update}
Consider the IP-LLR and HP parameterizations with a positively $p$-homogeneous activation function,  and $p \geq 1$, and no bias terms except at the first layer, and let us sub/super-script the variables of each models with IP and HP respectively. Assume the first training sample $(\xi_0, y_0)$ and the loss $\ell$ are the same for both parameterizations. Assume further that $\chi^\HP_0 \neq 0$, and simply denote by $\eta$ the base learning rate of the IP-LLR parameterization. Finally consider for HP the initial learning rate: $\eta_{\HP}(0) = (\chi^\IP_0/\chi^\HP_0) \eta$, and let $\xi \in \mathbb{R}^d$ be an input to both networks. Then, dropping the dependency of the weights at $t=1$ on $\eta$ and $\eta_\HP$, one has:
\begin{align*}
    \forall l \in [1, L+1], \quad W^l_\HP(1) &= W^l_\IP(1) \\
    B^1_{\HP}(1) &= B^1_{\IP}(1) \\
    f_1^{\HP}(\xi) &= f_1^{\IP}(\xi)
\end{align*}
\end{lemma}

\begin{proof}
By definition (see Section~\ref{sec:ipllr-muP-finite-width}), we have
\begin{align*}
    W^1_{\HP}(1) &= W^1_{\IP}(0) + \Delta
    W^1_{\muP}(1) \\
    B^1_{\HP}(1) &= B^1_{\IP}(0) + \Delta B^1_{\muP}(1) \\
    W^l_{\HP}(1) &= W^l_{\IP}(0) + \Delta W^l_{\muP}(1), \qquad l \in [2, L]\\
    W^{L+1}_{\HP}(1) &= W^{L+1}_{\IP}(0) + \Delta W^{L+1}_{\muP}(1)
\end{align*}
Using Corollaries~\ref{th:ipllr-first-weight-updates}, and Lemma~\ref{th:muP-first-weight-updates}, and the fact that $\eta_{\HP}(0) \chi_0^\muP = \eta \chi_0^\IP$ we have:
\begin{align*}
    \Delta W^1_\muP(1) &= -\eta_\HP(0) \chi_0^\muP d\tildeh^1_0 \transpose{\xi_0} \\
    &= -\eta \chi_0^\IP d\tildeh^1_0 \transpose{\xi_0} \\
    &= \Delta W^1_\IP(1),
\end{align*}
\begin{align*}
    \Delta B^1_\muP(1) &= -\eta_\HP(0) \chi_0^\muP d\tildeh^1_0 \\
    &= -\eta \chi_0^\IP d\tildeh^1_0 \\
    &= \Delta B^1_\IP(1),
\end{align*}
and, for $l \in [2, L]$
\begin{align*}
    \Delta W^l_\muP(1) &= -\eta_\HP \chi_0^\muP \frac{d\tildeh^l_0 \transpose{{(\tildex^{l-1}_0)}}}{m} \\
    &= -\eta \chi_0^\IP \frac{d\tildeh^l_0 \transpose{{(\tildex^{l-1}_0)}}}{m} \\
    &= \Delta W^l_\IP(1),
\end{align*}
and finally
\begin{align*}
    \Delta W^{L+1}_\muP(1) &= -\eta_\HP \chi_0^\muP \tildex^L_0 /m \\
    &= -\eta \chi_0^\IP \tildex^L_0 / m \\
    &= \Delta W^{L+1}_\IP(1),
\end{align*}
where the $\Delta W^{l}_\IP(1)$ and $\Delta B^1(1)$ are computed with the base learning rate $\eta$. We then get $W^l_\HP(1) = W^l_\IP(1)$ for all $l$, and it  follows that for any input $\xi$, $f_1^{\HP}(\xi) = f_1^{\IP}(\xi)$. 
\end{proof}

\subsubsection{Proof of Proposition~\ref{th:finite-equivalence-ipllr-hp}}

\begin{proof}
We first show by induction that the effective weight matrices and the effective biases of the first layer are the same for both parameterizations at any time step $\geq 1$, which will then immediately yield the result. We have already shown in Lemma~\ref{th:hp-first-update} that with the choice of initial learning rate for \HP, $W^l_\HP(1) = W^l_\IP(1)$ for all $l \in [1, L+1]$, and $B^l_\HP(1) = B^l_\IP(1)$ as well as $f_1^{\HP}(\xi) = f_1^{\IP}(\xi)$. 
\\\\
Now let $s \geq 1$, and assume that for all $l \in [1, L+1]$, $W^l_\HP(s) = W^l_\IP(s)$, and $B^1_{\HP}(s) = B^1_{\HP}(s)$. We want to show that this also holds true for the next time step $s+1$. An easy induction shows that since the effective weights of all layers are equal, and since by assumption the $s$-th training sample $(\xi_s, y_s)$ is the same for both parameterization, we get that for any $l \in [1, L+1]$, $x^l_{s, \HP} = x^l_{s, \IP}$, $h^l_{s, \HP} = h^l_{s, \IP}$, as well as $f^\HP_s(\xi_s) = f^\IP_s(\xi_s)$, and therefore $\chi^\HP_s = \chi^\IP_s$ since by assumption both parameterization use the same loss. This in turn will give by another easy induction that for any $l \in [1, L+1]$, $dx^l_{s, \HP} = dx^l_{s, \IP}$, $dh^l_{s, \HP} = dh^l_{s, \IP}$. Now, by Equation~\eqref{eq:delta-W} we have, on the one hand (recall that $s+1 \geq 2$ so that the base learning for both models for the $(s+1)$-th SGD step is $\eta$)
\begin{align*}
    \Delta W^1_\HP(s + 1) &= -\eta m^{-(2a_1^\muP + c_1^\muP)} dh^1_{s, \HP} \transpose{\xi_s} 
\end{align*}
and for $l \in [2, L]$
\begin{align*}
    \Delta W^l_\HP(s) &= -\eta m^{-(2a_l^\muP + c_l^\muP)} dh^l_{s, \HP} x^{l-1}_{s, \HP}
\end{align*}
and finally
\begin{align*}
        \Delta W^{L+1}_\HP(s) &= -\eta m^{-(2a_{L+1}^\muP + c_{L+1}^\muP)} x^L_{s, \HP}
\end{align*}
On the other hand, we have
\begin{align*}
    \Delta W^1_\IP(s) &= -\eta m^{-(2a_1^\IP + c_1^\IP)} dh^1_{s, \IP} \transpose{\xi_s} 
\end{align*}
and for $l \in [2, L]$
\begin{align*}
    \Delta W^l_\IP(s) &= -\eta m^{-(2a_l^\IP + c_l^\IP)} dh^l_{s, \IP} x^{l-1}_{s, \HP}
\end{align*}
and finally
\begin{align*}
        \Delta W^{L+1}_\IP(s) &= -\eta m^{-(2a_{L+1}^\IP + c_{L+1}^\IP)} x^L_{s, \IP}
\end{align*}
To see that the quantities are equal, we only need to observe that since $s+1 \geq 1$
\begin{align*}
    2a_1^\muP + c_1^\muP = -&1 = 2a_1^\IP + c_1^\IP \\
    2a_l^\muP + c_l^\muP = &0 = 2a_l^\IP + c_l^\IP\\
    2a_{L+1}^\muP + c_{L+1}^\muP = &1 = 2a_{L+1}^\IP + c_{L+1}^\IP
\end{align*}
(recall that for $s \geq 1$, $c_1^\IP = c_{L+1}^\IP = -1$, and $c_l^\IP = -2$ for $l \in [2, L]$). We thus find $\Delta W^l_\HP(s)=\Delta W^l_\IP(s)$ for all $l$, and since $W^l_\HP(s)=W^l_\IP(s)$ by assumption, we get $W^l_\HP(s+1)=W^l_\IP(s+1)$ for all $l$ which concludes the induction. 
\\\\
The effective weights being equal in both parameterizations for all time steps $\geq 1$, we  get that at time step $t \geq 1$, for any input $\xi \in \mathbb{R}^d$, the outputs $f_t^\HP(\xi)$ and $f_t^\IP(\xi)$ are the same, which concludes the proof. 
\end{proof}

\subsection{Infinite-width equivalence: \autoref{th:hpz-ipllr-equivalence}}\label{app:proof-equivalence-ip-muP-infinite}

In this section we prove \autoref{th:hpz-ipllr-equivalence} which states the equivalence between IP-LLR (see Definition~\ref{def:ipllr}) and HPZ (see Section~\ref{sec:ipllr-muP-infinite-width}). We start by a couple of preliminary results on the dynamics of HPZ, then proceed to prove the main induction step over $t$, and finally conclude by putting the results together to prove the theorem.

\subsubsection{Preliminary results}

\begin{lemma}[\muP\ is zero at initialization]\label{th:muP-forward-backward-0}
Consider the \muP\ parameterization with an activation function satisfying Assumption~\ref{ass:smooth-act} and a loss function $\ell$ satisfying Assumption~\ref{ass:loss}, and no bias terms except at the first layer. Let $\xi \in \mathbb{R}^d$ be an input to the network. One has:
\begin{align*}
    f_0(\xi) &\xrightarrow[m \rightarrow \infty]{a.s.} 0 \\
    \chi_0 &\xrightarrow[m \rightarrow \infty]{a.s.} \scalarlim{\chi}_0 := \partial_2 \ell(y_0, 0)
\end{align*}
\end{lemma}

\begin{remark}\label{remark:muP-hpz-0}
The result on the almost sure convergence of $\chi_0$ ensures that the latter is a valid initial scalar in the Tensor Program defining the computations associated with \muP\ (and thus and HPZ). Also note that the limit of $\chi_0$ is the same as for IP-LLR (see Lemma~\ref{th:ipllr-forward-backward-0}).
\end{remark}

\begin{proof}
\muP\ is designed so that $h^l_0 = \tildeh^l_0$ and $x^l_0 = \tildex^l_0$ for any $l \in [2, L]$, and as already proved in Lemma~\ref{th:ipllr-forward-backward-0}, the tilde variables are vectors in the Tensor Program. Since $f_0(\xi) = m^{-1} \transpose{{(U^{L+1})}} \tildex^L_0$ we get by the master theorem that $f_0(\xi)$ converges almost surely towards $\mathbb{E}[Z^{U^{L+1}} Z^{\tildex^L_0}]$. By Lemma~\ref{th:initial-W-vanish}, $Z^{\tildeh^L_0} = \hatZ^{\tildeh^L_0}$ and $Z^{U^{L+1}} = \hatZ^{U^{L+1}}$ by definition, and by the ZHat rule, $\hatZ^{\tildeh^L_0}$ and $\hatZ^{U^{L+1}}$ are independent, and since $\mathbb{E}[Z^{U^{L+1}}] = 0$ and $\mathbb{E}[(Z^{\tildex^L_0})^2] < \infty$ we get that $f_0(\xi)$ converges almost surely towards $0$. The result on the limit of $\chi_0$ is then simply a consequence of the fact that $\partial_2 \ell(y_0, \cdot)$ is continuous by assumption. 
\end{proof}

\begin{lemma}[Weight updates for \muP\ at any time step]\label{th:update-t-muP}
Consider the \muP\ parameterization with a differentiable activation function $\sigma$ and no bias terms except at the first layer, and let $t \geq 1$. Then, dropping the dependency of the forward and backward passes on $\xi_t$ at time $t$, one has:
\begin{align*}
    \Delta W^{L+1}(t+1) &= -\eta \chi_t x^L_t / m,  \\
    \Delta W^l(t+1) &= -\eta \chi_t \frac{d\tildeh^l_t \transpose{{(x^{l-1}_t)}}}{m}, \qquad l \in [2, L], \\
    \Delta W^1(t+1) &= -\eta \chi_t d\tildeh^1_t \transpose{\xi_t}, \\
    \Delta B^1(t+1) &= -\eta \chi_t d\tildeh^1_t.
\end{align*}
\end{lemma}

\begin{remark}\label{remark:weight-updates-hpz}
Because HPZ and \muP\ have the same parameterization for $t \geq 1$ (see Section~\ref{sec:ipllr-muP-infinite-width}), the formulas above for the updates are the same for HPZ, the only difference is that, at finite width, the $x^l_t$ and $d\tildeh^l_t$ differ from HPZ to \muP\ because $W^l_{\HPZ}(t) = W^l_{\muP}(t) - W^l_{\muP}(0)$. Note that the formulas are also exactly the same as for IP-LLR (see Lemma~\ref{th:update-t-ipllr}) but again the quantities $x^l_t$ and $d\tildeh^l_t$ differ for \muP\ and IP-LLR because of the initial weight contribution in $W^l(t)$ which is different for the intermediate layers of both parameterizations.
\end{remark}

\begin{proof}
By Equation~\eqref{eq:delta-W-L+1}, we have
\begin{align*}
    \Delta W^{L+1}(t+1) &= -\eta m^{-(2a_{L+1} + c_{L+1})} \chi_t x^L_t \\
    &= -\eta \chi_t x^L_t / m,
\end{align*}
because $2a_{L+1} + c_{L+1} = 2 - 1 = 1$ for \muP. For $l \in [2, L]$, we have by Equation~\eqref{eq:delta-W}
\begin{align*}
    \Delta W^l(t+1) &= -\eta \chi_t m^{-(2a_l + c_l)} dh^l_t \transpose{{(x^{l-1}_t)}} \\
    &= -\eta \chi_t \frac{d\tildeh^l_t \transpose{{(x^{l-1}_t)}}}{m},
\end{align*}
because $dh^l_t = m^{-1} d\tildeh^l_t$ and $2a_l + c_l = 1 -1 = 0$ for \muP. Finally, for $l=1$ we have again by Equation~\eqref{eq:delta-W} 
\begin{align*}
    \Delta W^1(t+1) &= -\eta \chi_t m^{-(2a_1 + c_1)} dh^1_t \transpose{\xi_t} \\
    &= -\eta \chi_t d\tildeh^1_t \transpose{\xi_t},
\end{align*}
because $2a_1 + c_1 = -1$ for \muP\ and $dh^1_t = d\tildeh^1_t$. A similar argument holds for $\Delta B^1(t+1)$, which concludes the proof.
\end{proof}

\begin{theorem}[Weights in HPZ at time $t$]\label{th:weights-hpz-t}
Consider the HPZ parameterization with a differentiable activation function $\sigma$ and no bias terms except at the first layer. Then, for any $t \geq 1$, one has:
\begin{enumerate}[(i)]
    \item $W^1(t) = U^1 - \eta \chi_0 d\tildeh^1_0 \transpose{\xi_0} - \eta \left(\sum_{s=1}^{t-1} \chi_s d\tildeh^1_s \transpose{\xi_s} \right)$,
    
    \item $B^1(t) = v^1 - \eta \chi_0 d\tildeh^1_0 - \eta \left(\sum_{s=1}^{t-1} \chi_s d\tildeh^1_s \right)$,
    
    \item $W^l(t) = - \eta \chi_0 \frac{d\tildeh^l_0 \transpose{{(\tildex^{l-1}_0)}}}{m} - \eta \left(\sum_{s=1}^{t-1} \chi_s \frac{d\tildeh^l_s \transpose{{(x^{l-1}_s)}}}{m} \right)$, \qquad $l \in [2, L]$,
    
    \item $W^{L+1}(t) = U^{L+1}/m - \eta \chi_0 \tildex^L_0 / m - \eta \left(\sum_{s=1}^{t-1} \chi_s x^L_s / m\right)$.
\end{enumerate}
\end{theorem}

\begin{proof}
The formulas are correct at $t=1$ by definition of HPZ and by Lemma~\ref{th:muP-first-weight-updates} which gives the first weight updates for \muP. Then, an easy induction using Lemma~\ref{th:update-t-muP} yields the result.
\end{proof}

\begin{lemma}[Backward pass of HPZ at time $t$]\label{th:backward-hpz-t}
Consider the HPZ parameterization with a differentiable activation function $\sigma$ and no bias terms except at the first layer. Then, for any $t \geq 1$, dropping the dependency of the forward pass at time $t$ on $\xi_t$, and of the previous forward and backward passes on the corresponding $\xi_s$,  one has:
\begin{enumerate}[(i)]
    \item $d\tildex^L_t = w^{L+1}(t) = U^{L+1} - \eta \chi_0 \tildex^L_0 - \eta \sum_{s=1}^{t-1} \chi_s x^L_s$,
    \item $d\tildex^{l-1}_t = - \eta \chi_0 \frac{\transpose{{(d\tildeh^l_0)}} d\tildeh^l_t}{m} \tildex^{l-1}_0 -\eta \sum_{s=1}^{t-1} \chi_s \frac{\transpose{{(d\tildeh^l_s)}} d\tildeh^l_t}{m} x^{l-1}_s$, \qquad $l \in [2, L]$.
\end{enumerate}
\end{lemma}

\begin{proof}
By definition, we have
\begin{align*}
    d\tildex^L_t &= m dx^L_t \\
    &= m W^{L+1}(t) \\
    &= U^{L+1} - \eta \chi_0 \tildex^L_0 - \eta \sum_{s=1}^{t-1} \chi_s x^L_s
\end{align*}
where the last equality stems from \autoref{th:weights-hpz-t}.
\\\\
Let $l \in [2, L]$, we have:
\begin{align*}
    d\tildex^{l-1}_t &= \transpose{{(W^l(t))}} d\tildeh^{l}_t \\
    &= - \eta \chi_0 \frac{\transpose{{(d\tildeh^l_0)}} d\tildeh^l_t}{m} \tildex^{l-1}_0 -\eta \sum_{s=1}^{t-1} \chi_s \frac{\transpose{{(d\tildeh^l_s)}} d\tildeh^l_t}{m} x^{l-1}_s
\end{align*}
where the second equality stems from \autoref{th:weights-hpz-t}. 
\end{proof}

\begin{lemma}[Z for the forward pass of HPZ at time $t=1$]\label{th:z-forward-hpz-1}
Consider the HPZ parameterization with an activation function $\sigma$ satisfying Assumption~\ref{ass:smooth-act} and no bias terms except at the first layer. Let $\xi \in \mathbb{R}^d$ be an input to the network. Then, for any $l \in [1, L]$, $h^l_1(\xi), x^l_1(\xi), d\tildex^l_1, d\tildeh^l_1$ are vectors in the program, $f_1(\xi)$ is a scalar in the program, and $\chi_1$ is a valid initial scalar in the program. Additionally, dropping the dependency of the forward pass at time $t=1$ on $\xi$, and of the first forward and backward passes on $\xi_0$, one has:
\begin{enumerate}[(i)]
    \item $Z^{h^1_1} = Z^{W^1(1) \xi + B^1(1)} = Z^{U^1\xi} + Z^{v^1} - \eta \scalarlim{\chi}_0 (\transpose{\xi_0} \xi + 1) Z^{d\tildeh^1_0}$,
    
    \item $Z^{h^l_1} = Z^{W^l(1)x^{l-1}_1} = - \eta \scalarlim{\chi}_0 \mathbb{E}[Z^{\tildex^{l-1}_0} Z^{x^{l-1}_1}]  Z^{d\tildeh^l_0}$, \qquad $l \in [2, L]$,

    \item $f_1(\xi) = \transpose{{(W^{L+1}(1))}} x^L_1 \xrightarrow[m \rightarrow \infty]{a.s.} \mathbb{E}[Z^{U^{L+1}} Z^{x^L_1}] - \eta \scalarlim{\chi}_0 \mathbb{E}[Z^{\tildex^L_0} Z^{x^L_1}]$.
\end{enumerate}
\end{lemma}

\begin{proof}
By \autoref{th:weights-hpz-t}, with $t=1$, one has that $h^1_1 = U^1 \xi + v^1 - \eta \chi_0 (\transpose{\xi}_0 \xi + 1) d\tildeh^1_0$. By Lemma~\ref{th:ipllr-forward-backward-0}, $d\tildeh^1_0$ is a vector in the Tensor Program (recall that the tilde variables at initialization do not depend on the choice of parameterization) and by Lemma~\ref{th:muP-forward-backward-0} $\chi_0$ is a valid initial scalar in the program  which has an almost sure limit $\scalarlim{\chi}_0 := \partial_2 \ell(y_0, 0)$ as $m \rightarrow \infty$ (see Remark~\ref{remark:muP-hpz-0}). In addition, $U^1 \xi$ and $v^1$ are initial vectors in the program, which thus shows that $h^1_1$ is a vector in the program by the \NonLin\ operation. This also gives that $x^1_1 = \sigma(h^1_1)$ is a vector in the program since $\sigma$ is pseudo-Lipschitz by assumption. Moreover, by \ZNonLin, we have $Z^{h^1_1} = Z^{U^1 \xi} + Z^{v^1} - \eta \scalarlim{\chi}_0 (\transpose{\xi}_0 \xi + 1) Z^{d\tildeh^1_0}$. 
Let $l \in [2, L]$ and assume that $h^{l-1}_1, x^{l-1}_1$ are vectors in the program. Then, by \autoref{th:weights-hpz-t} with $t=1$, we get
\begin{align*}
    h^l_1 = - \eta \chi_0 \frac{\transpose{{(\tildex^{l-1}_0)}} x^{l-1}_1}{m} d\tildeh^l_0.
\end{align*}
$\transpose{{(\tildex^{l-1}_0)}} x^{l-1}_1 / m$ is a scalar in the program by the \Moment\ operation, and thus by the \MatMul\ and \NonLin\ operations, $h^l_1$ is a vector in the program and thus so is $x^l_1 = \sigma(h^l_1)$, which proves by induction that this is the case for any $l \in [2,L]$. By \ZNonLin\ we thus have
\begin{align*}
    Z^{h^l_1} = - \eta \scalarlim{\chi}_0 \mathbb{E}[Z^{\tildex^{l-1}_0} Z^{x^{l-1}_1}] Z^{d\tildeh^l_0}.
\end{align*}
We then have by \autoref{th:weights-hpz-t} with $t=1$, 
\begin{align*}
    f_1(\xi) = m^{-1} \transpose{{(U^{L+1})}} x^L_1 - \eta \chi_0 \frac{\transpose{{(\tildex^L_0)}} x^L_1}{m}
\end{align*}
$U^{L+1} - \eta \chi_0 \tildex^L_0$ is a vector in the program by the \NonLin\ operation, and the quantity $m^{-1} \transpose{{(U^{L+1} - \eta \chi_0 \tildex^{L}_0)}} x^L_1$ is thus a scalar in the program by the \Moment\ operation, and by the master theorem, we get $f_1(\xi) \rightarrow \mathbb{E}[Z^{U^{L+1}} Z^{x^L_1}] - \eta \scalarlim{\chi}_0 \mathbb{E}[Z^{\tildex^L_0} Z^{x^L_1}]$ almost surely, since both expectations are finite by Lemma~\ref{th:Z0-dist-moments}. Since we did the previous reasoning with an arbitrary $\xi$, we also get that $h^l_1(\xi_1), x^l_1(\xi_1)$ are vectors in the program for any $l \in [1, L]$ and that the formulas in $(i)$, $(ii)$, and $(iii)$ hold when the input is $\xi_1$. In particular, $f_1(\xi_1)$ converges to a finite almost sure limit $\scalarlim{f_1}(\xi_1)$, and thus the continuity of $\partial_2 \ell(y_1, \cdot)$ ensures the almost sure convergence of $\chi_1$ towards $\scalarlim{\chi}_1 := \partial_2 \ell(y_1, \scalarlim{f_1}(\xi_1))$, which means $\chi_1$ is a valid initial scalar in the Tensor Program. Then, dropping the dependency of the second forward pass (at $t=1$) on $\xi_1$, we get by \autoref{th:backward-hpz-t} with $t=1$:
\begin{align*}
    d\tildex^L_1 = U^{L+1} - \eta \chi_0 \tildex^L_0
\end{align*}
which is a vector in the program by \NonLin. Then $d\tildeh^L_1 = d\tildex^L_1 \odot \sigma'(h^L_1)$ is also a vector in the program by \NonLin\ since $\sigma'$ is pseudo-Lipschitz. Let $l \in [2, L-1]$ and assume that $d\tildex^{l+1}_1$ and $d\tildeh^{l+1}_1$ are vectors in the program. Then by \autoref{th:backward-hpz-t} with $t=1$, we have
\begin{align*}
    d\tildex^{l}_1 =  - \eta \chi_0 \frac{\transpose{{(d\tildeh^{l+1}_0)}} d\tildeh^{l+1}_1}{m} \tildex^l_0
\end{align*}
$\transpose{{(d\tildeh^{l+1}_0)}} d\tildeh^{l+1}_1 / m$ is a scalar in the program by the \Moment\ operation and by \MatMul\ and \NonLin\ we thus get that $d\tildex^l_1$ is a vector in the program. Then $d\tildeh^l_1 = d\tildex^l_1 \odot \sigma'(h^l_1)$ is also a vector in the program since $\sigma'$ is pseudo-Lipschitz, which concludes the induction and with it the proof.
\end{proof}

\begin{lemma}[$Z$s of HPZ and IP-LLR are equal at $t=1$]\label{th:hpz-ipllr-1}
Consider the HPZ and IP-LLR parameterization with an activation function $\sigma$ satisfying Assumption~\ref{ass:smooth-homogeneous-act}, and no bias terms except at the first layer, and let us sub/super-script the variables of each models with HPZ and IP respectively. Let $\xi \in \mathbb{R}^d$ be an input to the networks, and assume that HPZ and IP-LLR share the same training samples $(\xi_0, y_0)$ and $(\xi_1, y_1)$ at $t=0$ and $t=1$, the same loss function $\ell$ satisfying Assumption~\ref{ass:loss}, and the same base learning rate $\eta$. Then dropping the dependency of the first forward and backward passes on $\xi_0$ and that of the second forward passes on $\xi$, we have:
\begin{enumerate}[(i)]
    \item $Z^{h^l_{1,\HPZ}} = Z^{h^l_{1,\IP}}$, \quad $Z^{x^l_{1,\HPZ}} = Z^{x^l_{1,\IP}}$, \qquad $l \in [1, L]$,
    
    \item $\lim_{m \rightarrow \infty} f_1^{\HPZ}(\xi) = \lim_{m \rightarrow \infty} f_1^{\IP}(\xi)$,
    
    \item $\scalarlim{\chi}^{\HPZ}_1 = \scalarlim{\chi}^{\IP}_1$,
    
    \item $Z^{d\tildex^l_{1,\HPZ}} = Z^{d\tildex^l_{1,\IP}}$, \quad $Z^{d\tildeh^l_{1,\HPZ}} = Z^{d\tildeh^l_{1,\IP}}$, \qquad $l \in [1, L]$.
\end{enumerate}
\end{lemma}

\begin{proof}
By Lemmas~\ref{th:ipllr-forward-backward-0} and ~\ref{th:ipllr-forward-backward-0} we have $\scalarlim{\chi}_0^{\IP} = \scalarlim{\chi}_0^{\muP} = \scalarlim{\chi}_0^{\HPZ} = \partial_2 \ell(y_0, 0)$, which we simply call $\scalarlim{\chi}_0$ in the remainder of this proof for simplicity. By Corollary~\ref{th:z-forward-ipllr-1} and Lemma~\ref{th:z-forward-hpz-1} we have
\begin{align*}
    Z^{h^1_{1, \IP}} = Z^{U^1 \xi} + Z^{v^1} - \eta \scalarlim{\chi}_0 (\transpose{{\xi_0}} \xi + 1) Z^{d\tildeh^1_0},
\end{align*}
and 
\begin{align*}
    Z^{h^1_{1, \HPZ}} = Z^{U^1 \xi} + Z^{v^1} - \eta \scalarlim{\chi}_0 (\transpose{{\xi_0}} \xi + 1) Z^{d\tildeh^1_0},
\end{align*}
and since the tilde variables are computed independently of any parameterization, we have $Z^{h^1_{1, \HPZ}} = Z^{h^1_{1, \IP}}$. Because IP and HPZ share the same activation function we also get $Z^{x^1_{1, \HPZ}} = Z^{x^1_{1, \IP}}$. Now let $l \in [2, L]$ and assume $Z^{h^{l-1}_{1, \HPZ}} = Z^{h^{l-1}_{1, \IP}}$ as well as $Z^{x^{l-1}_{1, \HPZ}} = Z^{x^{l-1}_{1, \IP}}$. By Corollary~\ref{th:z-forward-ipllr-1} and Lemma~\ref{th:z-forward-hpz-1} we have
\begin{align*}
    Z^{h^l_{1, \IP}} = -\eta \scalarlim{\chi}_0 \mathbb{E}[Z^{\tildex^{l-1}_0} Z^{x^{l-1}_{1, \IP}}] Z^{d\tildeh^l_0},
\end{align*}
and 
\begin{align*}
    Z^{h^l_{1, \HPZ}} = -\eta \scalarlim{\chi}_0 \mathbb{E}[Z^{\tildex^{l-1}_0} Z^{x^{l-1}_{1, \HPZ}}] Z^{d\tildeh^l_0},
\end{align*}
which shows $Z^{h^l_{1, \IP}} = Z^{h^l_{1, \HPZ}}$ since the tilde variables are independent of any choice of parameterization. Since the activation function $\sigma$ is the same for both models we also get $Z^{x^l_{1, \IP}} = Z^{x^l_{1, \HPZ}}$ which concludes the induction. For the output of the networks, we have by Corollary~\ref{th:z-forward-ipllr-1} and Lemma~\ref{th:z-forward-hpz-1}
\begin{align*}
    f_1^{\IP}(\xi) \xrightarrow[m \rightarrow \infty]{a.s.} \mathbb{E}[Z^{U^{L+1}} Z^{x^L_{1, \IP}}] - \eta \scalarlim{\chi}_0 \mathbb{E}[Z^{\tildex^L_0} Z^{x^L_{1, \IP}}],
\end{align*}
and 
\begin{align*}
    f_1^{\HPZ}(\xi) \xrightarrow[m \rightarrow \infty]{a.s.} \mathbb{E}[Z^{U^{L+1}} Z^{x^L_{HPZ}}] - \eta \scalarlim{\chi}_0 \mathbb{E}[Z^{\tildex^L_0} Z^{x^L_{1, \HPZ}}],
\end{align*}
and since $Z^{x^L_{1, \IP}} = Z^{x^L_{1, \HPZ}}$ by the previous induction and the tilde variables are independent of the parameterization, we get $\lim_{m \rightarrow \infty} f_1^{\HPZ}(\xi) = \lim_{m \rightarrow \infty} f_1^{\IP}(\xi) =: \scalarlim{f_1}(\xi)$. Since $\chi_1^\IP = \partial_2(y_1, f_1^\IP(\xi))$ and $\chi_1^\HPZ = \partial_2(y_1, f_1^\HPZ(\xi))$, by continuity of $\partial_2 \ell(y_1, \cdot)$ we get that $\scalarlim{\chi}_1^{\HPZ} = \partial_2 \ell(y_1, \scalarlim{f_1}(\xi)) = \scalarlim{\chi}_1^{\IP}$.
\\ \\
For the backward pass, we have by Lemma~\ref{th:backward-hpz-t} that $d\tildex^L_{1, \HPZ} = U^{L+1} - \eta \chi_{0, \HPZ} \tildex^L_0$ which gives by \NonLin\ $Z^{d\tildex^L_{1, \HPZ}} = Z{U^{L+1}} - \eta \scalarlim{\chi}_0 Z^{\tildex^L_0}$ which is also equal to $Z^{d\tildex^L_{1, \IP}}$ by Lemma~\ref{th:backward-ipllr-t} since the tilde variables are independent of the choice of parameterization. Then, we also get $Z^{d\tildeh^L_{1, \HPZ}} = Z^{d\tildex^L_{1, \HPZ}} \sigma'(Z^{h^L_{1, \HPZ}})$ and $Z^{d\tildeh^L_{1, \IP}} = Z^{d\tildex^L_{1, \HPZ}} \sigma'(Z^{h^L_{1, \IP}})$ which shows $Z^{d\tildeh^L_{1, \HPZ}} = Z^{d\tildeh^L_{1, \IP}}$. Let $l \in [1, L-1]$ and assume $Z^{d\tildex^{l+1}_{1, \HPZ}} = Z^{d\tildex^{l+1}_{1, \IP}}$ as well as $Z^{d\tildeh^{l+1}_{1, \HPZ}} = Z^{d\tildeh^{l+1}_{1, \IP}}$. By Lemma~\ref{th:backward-hpz-t}, we have
\begin{align*}
    d\tildex^l_{1, \HPZ} = -\eta \chi_0 \frac{\transpose{{(d\tildeh^{l+1}_0)}} d\tildeh^{l+1}_{1, \HPZ}}{m} \tildex^l_{0}
\end{align*}
which gives by the master theorem and the \ZNonLin
\begin{align*}
    Z^{d\tildex^l_{1, \HPZ}} = -\eta \scalarlim{\chi}_0 \mathbb{E}[Z^{d\tildeh^{l+1}_0} Z^{d\tildeh^{l+1}_{1, \HPZ}}] Z^{\tildex^l_{0}}
\end{align*}
which is the same expression as $Z^{d\tildex^l_{1, \IP}}$ by Lemma~\ref{th:backward-ipllr-t}. It then  follows that $Z^{d\tildeh^l_{1, \HPZ}} =  Z^{d\tildeh^l_{1, \IP}}$, which concludes the induction and with it the proof.
\end{proof}

\begin{theorem}[Z for the forward pass of HPZ at time $t$]\label{th:z-forward-hpz-t}
Consider the HPZ parameterization with an activation function $\sigma$ satisfying Assumption~\ref{ass:smooth-act} and no bias terms except at the first layer. Let $\xi \in \mathbb{R}^d$ be an input to the network. Then, for any $l \in [1, L]$, $h^l_s(\xi), x^l_s(\xi), d\tildex^l_s, d\tildeh^l_s$ are vectors in the program, $f_s(\xi)$ is a scalar in the program, and $\chi_s$ is a valid initial scalar in the program. Additionally, dropping the dependency of the forward pass at time $t$ on $\xi$, and of the previous forward and backward passes on the corresponding $\xi_s$, one has:
\begin{enumerate}[(i)]
    \item $Z^{h^1_t} = Z^{W^1(t) \xi + B^1(t)} = Z^{U^1\xi} + Z^{v^1} - \eta \scalarlim{\chi}_0 (\transpose{\xi_0} \xi + 1) Z^{d\tildeh^1_0} - \eta \left(\sum_{s=1}^{t-1} \scalarlim{\chi}_s (\transpose{\xi_s} \xi + 1) Z^{d\tildeh^1_s}  \right)$,
    
    \item $Z^{h^l_t} = Z^{W^l(t)x^{l-1}_t} = - \eta \scalarlim{\chi}_0 \mathbb{E}[Z^{\tildex^{l-1}_0} Z^{x^{l-1}_t}]  Z^{d\tildeh^l_0}  - \eta \left(\sum_{s=1}^{t-1} \scalarlim{\chi}_s \mathbb{E}[Z^{x^{l-1}_s} Z^{x^{l-1}_t}]  Z^{d\tildeh^l_s } \right), \quad l \in [2, L],$

    \item $f_t(\xi) = \transpose{{(W^{L+1}(t))}} x^L_t \xrightarrow[m \rightarrow \infty]{a.s.} \mathbb{E}[Z^{U^{L+1}} Z^{x^L_t}] - \eta \scalarlim{\chi}_0 \mathbb{E}[Z^{\tildex^L_0} Z^{x^L_t}] -
    \eta \left(\sum_{s=1}^{t-1} \scalarlim{\chi}_s \mathbb{E}[Z^{x^L_s} Z^{x^L_t}]  \right)$.
\end{enumerate}
\end{theorem}

\begin{proof}
The proof is exactly the same as for \autoref{th:z-forward-ipllr-t} except that whenever a multiplication by $W^l(0)$ appears  with $l \in [2, L]$, it is now replaced by $0$, but the reasoning and all the arguments are the same, which in summary uses an induction over $t$ as well as the master theorem and the \ZNonLin\ rule from the Tensor Program.
\end{proof}

\begin{theorem}[$Z$s of backward pass of HPZ at time $t$]\label{th:z-backward-hpz-t}
Consider the HPZ parameterization with an activation function $\sigma$ satisfying Assumption~\ref{ass:smooth-act} and no bias terms except at the first layer. Then, for any $t \geq 1$, dropping the dependency of the forward pass at time $t$ on $\xi_t$, and of the previous forward and backward passes on the corresponding $\xi_s$, one has:
\begin{enumerate}[(i)]
    \item $Z^{d\tildex^L_t} = Z^{w^{L+1}(t)} = Z^{U^{L+1}} - \eta \scalarlim{\chi}_0 Z^{\tildex^L_0} - \eta \sum_{s=1}^{t-1} \scalarlim{\chi}_s Z^{x^L_s}$,
    \item $Z^{d\tildex^{l-1}_t} = - \eta \scalarlim{\chi}_0 \mathbb{E}[Z^{d\tildeh^l_0} Z^{d\tildeh^l_t}] Z^{\tildex^{l-1}_0} -\eta \sum_{s=1}^{t-1} \scalarlim{\chi}_s \mathbb{E}[Z^{d\tildeh^l_s} Z^{d\tildeh^l_t}] Z^{x^{l-1}_s}$, \quad $l \in [2, L]$.
\end{enumerate}
\end{theorem}

\begin{proof}
As for \autoref{th:z-forward-hpz-t}, the proof follows exactly the same pattern as for \autoref{th:z-backward-hpz-t} except that whenever a multiplication by $W^l(0)$ appears  with $l \in [2, L]$, it is now replaced by $0$.
\end{proof}

\subsubsection{Induction on $t$}

\begin{lemma}[Induction step on the $Z$s of the forward pass]\label{th:hpz-ipllr-induct-z-forward}
Consider the IP-LLR and HPZ parameterizations with an activation function $\sigma$ satisfying Assumption~\ref{ass:smooth-homogeneous-act} and no bias terms except at the first layer, and let us sub/super-script the variables of each models with IP and HP respectively. Let $s \geq 1$, $\xi \in \mathbb{R}^d$ be an input to the networks, and assume that the training routine (see Definition~\ref{def:training-routine}) is the same for both models with a loss satisfying Assumption~\ref{ass:loss}. Assume further that, dropping the dependency of the forward and backward passes at time $t=r$ on $\xi_r$, for all $r \in [1, s]$, we have:
\begin{enumerate}[(i)]
    \item $Z^{h^l_{\HPZ, r}} = Z^{h^l_{\IP, r}},  \quad Z^{x^l_{\HPZ, r}} = Z^{x^l_{\IP, r}}, \qquad l \in [1, L]$,
    
    \item $\lim_{m \rightarrow \infty} f_r^{\HPZ}(\xi) = \lim_{m \rightarrow \infty} f_r^{\IP}(\xi)$,
    
    \item $\scalarlim{\chi}^{\HPZ}_r = \scalarlim{\chi}^{\IP}_r$,
    
    \item $ Z^{d\tildeh^l_{\HPZ, r}} = Z^{d\tildeh^l_{\IP, r}},  \quad Z^{d\tildex^l_{\HPZ, r}} = Z^{d\tildex^l_{\IP, r}}, \qquad l \in [1, L]$.
\end{enumerate}
Then, dropping the dependency of the forward pass at time $t=s+1$ on $\xi$, one has:
\begin{enumerate}[(i)]
    \setcounter{enumi}{4}
    \item $Z^{h^l_{\HPZ, s+1}} = Z^{h^l_{\IP, s+1}},  \quad Z^{x^l_{\HPZ, s+1}} = Z^{x^l_{\IP, s+1}}, \qquad l \in [1, L]$,
    
    \item $\lim_{m \rightarrow \infty} f_{s+1}^{\HPZ}(\xi) = \lim_{m \rightarrow \infty} f_{s+1}^{\IP}(\xi)$,
    
    \item $\scalarlim{\chi}^{\HPZ}_{s+1} = \scalarlim{\chi}^{\IP}_{s+1}$,
    
    \item $ Z^{d\tildeh^l_{\HPZ, r}} = Z^{d\tildeh^l_{\IP, s+1}},  \quad Z^{d\tildex^l_{\HPZ, s+1}} = Z^{d\tildex^l_{\IP, s+1}}, \qquad l \in [1, L]$.
\end{enumerate}
\end{lemma}

\begin{proof}
Since by assumption, for any $r \in [1, s]$, the $Z$s of the forward and backward passes are equal for both parameterizations, we drop the dependency of those quantities on the model, and for $z \in \{h^l_r, x^l_r, d\tildeh^l_r, d\tildex^l_r \}$, we simply call $Z^{z_\HPZ} = Z^{z_\IP} = Z^{z}$. Similarly we simply call $\scalarlim{\chi}_r^\HPZ = \scalarlim{\chi}_r^\IP = \scalarlim{\chi}_r$. We have ~\ref{th:z-forward-ipllr-t}
\begin{align*}
    Z^{h^1_{\HPZ, s+1}} &= Z^{U^1 \xi_{s+1}} - \eta \scalarlim{\chi}_0 (\transpose{\xi_0} \xi_{s+1} + 1) Z^{d\tildeh^1_0} - \eta \sum_{r=1}^s \scalarlim{\chi}_r (\transpose{\xi_r} \xi_{s+1} +1) Z^{d\tildeh^1_r} \\
    &= Z^{h^1_{\IP, s+1}}
\end{align*}
where the first equality stems from \autoref{th:z-forward-hpz-t} and the second one from \autoref{th:z-forward-ipllr-t}. Since both parameterizations use the same linearity $\sigma$, we  get $Z^{x^1_{\HPZ, s+1}} = \sigma(Z^{h^1_{\HPZ, s+1}}) = \sigma(Z^{h^1_{\IP, s+1}}) = Z^{x^1_{\IP, s+1}}$. 
\\\\
Let $l \in [2, L]$ and assume $Z^{h^{l-1}_{\HPZ, s+1}} = Z^{h^{l-1}_{\IP, s+1}}$, $Z^{x^{l-1}_{\HPZ, s+1}} = Z^{x^{l-1}_{\IP, s+1}}$. By \autoref{th:z-forward-hpz-t}, we have 
\begin{align*}
    Z^{h^{l}_{\HPZ, s+1}} &= - \eta \scalarlim{\chi}_0 \mathbb{E}[Z^{\tildex^{l-1}_0} Z^{x^{l-1}_{\HPZ, s+1}}] Z^{d\tildeh^{l}_0} - \eta \sum_{r=1}^s \scalarlim{\chi}_r \mathbb{E}[Z^{x^{l-1}_r} Z^{x^{l-1}_{\HPZ, s+1}}] Z^{d\tildeh^{l}_r} \\
    &= - \eta \scalarlim{\chi}_0 \mathbb{E}[Z^{\tildex^{l-1}_0} Z^{x^{l-1}_{\IP, s+1}}] Z^{d\tildeh^{l}_0} - \eta \sum_{r=1}^s \scalarlim{\chi}_r \mathbb{E}[Z^{x^{l-1}_r} Z^{x^{l-1}_{\IP, s+1}}] Z^{d\tildeh^{l}_r} \\
    &= Z^{h^{l}_{\HPZ, s+1}}
\end{align*}
where the last equality stems from \autoref{th:z-forward-ipllr-t}. Since both parameterizations used the same non-linearity $\sigma$, we  get $Z^{x^{l+1}_{\HPZ, s+1}} = Z^{x^{l+1}_{\IP, s+1}}$.
\\\\
By induction, we thus get that for any $l \in [1, L]$, $Z^{h^l_{\HPZ, s+1}} = Z^{h^l_{\IP, s+1}}$, and $Z^{x^l_{\HPZ, s+1}} = Z^{x^l_{\IP, s+1}}$, which proves $(v)$. We can thus drop the dependency of $h^l_{s+1}$ and $x^l_{s+1}$ on the model HPZ or IP. Now, we thus have by \autoref{th:z-forward-hpz-t}
\begin{align*}
    \lim_{m \rightarrow \infty} f_{s+1}^{\HPZ}(\xi) &= \mathbb{E}[Z^{U^{L+1}} Z^{x^L_{s+1}}] - \eta \scalarlim{\chi}_0 \mathbb{E}[Z^{\tildex^L_0} Z^{x^L_{s+1}}] -
    \eta \left(\sum_{r=1}^{s} \scalarlim{\chi}_r \mathbb{E}[Z^{x^L_r} Z^{x^L_{s+1}}]  \right) \\
    &= \lim_{m \rightarrow \infty} f_{s+1}^{\IP}(\xi)
\end{align*}
where the last equality stems from \autoref{th:z-forward-ipllr-t}, which proves $(vi)$. Then $(vi)$ combined with the continuity of $\partial_2 \ell(y_{s+1}, \cdot)$ proves $(vii)$, and we can thus imply denote $\scalarlim{\chi}^{\IP}_{s+1} = \scalarlim{\chi}^{\HPZ}_{s+1} = \scalarlim{\chi}_{s+1}$. By Theorems~\ref{th:z-backward-hpz-t} and~\ref{th:z-backward-ipllr-t} we  get $Z^{d\tildex^L_{s+1, \HPZ}} = Z^{d\tildex^L_{s+1, \IP}}$, from which it  follows that $Z^{d\tildeh^L_{s+1, \HPZ}} = Z^{d\tildeh^L_{s+1, \IP}}$ by $(v)$ and since both models share the same activation function. Finally, given the previous result, with $(i), (iii), (iv), (v)$ and $(vii)$, an easy induction gives $(viii)$ with the formulas of Theorems~\ref{th:z-backward-hpz-t} and~\ref{th:z-backward-ipllr-t}, which concludes the proof. 
\end{proof}

\subsubsection{Proof of \autoref{th:hpz-ipllr-equivalence}}
\begin{proof}
The claim has already been proved at $t=0$ by Lemmas~\ref{th:ipllr-forward-backward-0} and~\ref{th:muP-forward-backward-0}, and at $t=1$ by Lemma~\ref{th:hpz-ipllr-1}. Then, by Lemma~\ref{th:hpz-ipllr-induct-z-forward}, we get the result at any time step $t \geq 1$ by induction.
\end{proof}

\section{Formal versions of the results for the alternative methods to escape the initial stationary point}

\subsection{Formal version of \autoref{th:ip-bias-no-input-dependence-informal}}\label{app:ip-bias-formal}

\begin{theorem}[Formal]\label{th:ipllr-bias-no-input-dependence-formal}
Consider the IP-bias parameterization as in Equations~\eqref{eq:ip-bias}, with the initial learning rates as in Equations~\eqref{eq:ip-bias-initial-lr-weights} and~\eqref{eq:ip-bias-initial-lr-bias}. Assume the activation function $\sigma$ satisfies Assumption~\ref{ass:smooth-act} and the loss $\ell$ satisfies Assumption~\ref{ass:loss}. Then, for any input $\xi \in \mathbb{R}^d$ to the network, $Z^{h^l_0(\xi)}, Z^{x^l_0(\xi)}$ for $l \geq 2$, and $\lim_{m \rightarrow \infty} f_0(\xi)$ do not depend on $\xi$. In addition, for any vector $x$ in the program such that $Z^x$ does not depend on on the first training input $\xi_0$, $Z^{\Delta W^l(1) x}$ for $l \in [3, L]$, and $\lim_{m \rightarrow \infty} \transpose{{(\Delta W^{L+1}(1))}} x$ do not depend on $\xi_0$.
\end{theorem}

\begin{proof}
We have $h^1_0 = U^1 \xi + v^1$ so that $Z^{h^1_0} = \hatZ^{U^1 \xi} + \hatZ^{v^1} \sim \mathcal{N}(0, ||\xi||^2+1)$, and $Z^{x^1_0} = \sigma(Z^{h^1_0})$. At the second layer $l=2$, we have $h^2_0 = m^{-1/2} \hatW^2 x^1_0 + v^2$ so that by \ZNonLin\ $Z^{h^2_0} = 0 \times \hatZ^{\hatW^2 x^1_0} + Z^{v^2}$, and $\hatZ^{\hatW^2 x^1_0} \sim \mathcal{N}(0, \mathbb{E}[(Z^{x^1_0})^2])$. Because $\sigma$ is pseudo-Lipschitz, it is also polynomially bounded, and the variance of the Gaussian is finite by Lemma~\ref{th:pos-finite-moment}, so that $Z^{h^2_0} = Z^{v^2} \sim \mathcal{N}(0,1)$ which does not depend on $\xi$. Therefore, $Z^{x^2_0} = \sigma(Z^{h^2_0})$ also does not depend on $\xi$. Let $l \in [3, L]$ and assume that $Z^{h^{l-1}_0} = Z^{v^{l-1}}$ and $Z^{x^{l-1}_0} = \sigma(Z^{v^{l-1}})$. Then, we have $Z^{h^l_0} = 0 \times \hatZ^{\hatW^{l} x^{l-1}_0} + Z^{v^l}$, and $\hatZ^{\hatW^{l} x^{l-1}_0} \sim 
\mathcal{N}(0, \mathbb{E}[(Z^{x^{l-1}_0})^2])$, and the variance is again finite by the same arguments as for $l=2$. We thus get $Z^{h^l_0} = Z^{v^l}$ and $Z^{x^l_0} = \sigma(Z^{h^l_0}) = \sigma(Z^{v^l})$ which concludes the induction and shows that $Z^{h^l_0}$ and $Z^{x^l_0}$ do not depend on $\xi$ for all intermediate layers $l$. 
\\ \\
For the output of the network, we have by the master theorem that $m^{-1} \transpose{{(U^{L+1})}} x^L_0$ converges almost surely to $\mathbb{E}[Z^{U^{L+1}} \sigma(Z^{v^L})] = 0$ since $Z^{U^{L+1}}$ has mean $0$ and is independent of $Z^{v^L}$. Since $f_0(\xi) = m^{-1} \transpose{{(U^{L+1})}} x^L_0 + v^{L+1}$ where $v^{L+1} \sim \mathcal{N}(0, 1)$, we have that $f_0(\xi)$ converges almost surely to the Gaussian variable $v^{L+1}$ which does not depend on $\xi$. For the backward pass, recall that we call $d\tildex^l_0 := m^{-1} m^{-(L-l)/2} \nabla_{x^l} f_0(\xi_0)$ and $d\tildeh^l_0 := m^{-1} m^{-(L-l)/2} \nabla_{h^l} f_0(\xi_0)$. Then, we have $d\tildex^L_0 = U^{L+1}$, $d\tildeh^L_0 = U^{L+1} \odot \sigma'(h^L_0)$, and a simple induction shows that for any $l \in [1, L-1]$, $d\tildex^l_0 = \transpose{{(\hatW^{l+1})}} d\tildeh^{l+1}_0$, $d\tildeh^l_0 = d\tildex^l_0 \odot \sigma'(h^l_0)$. We thus have $Z^{d\tildex^L_0} = Z^{U^{L+1}}$ and $Z^{d\tildeh^L_0} = Z^{U^{L+1}} \sigma'(Z^{v^L})$ which does not depend on the first training input $\xi_0$. With the recursive formulas above, and since $Z^{h^l_0} = Z^{v^l}$ for $l \in [2, L]$, it is clear that $Z^{d\tildex^l_0}$ and $Z^{d\tildeh^l_0}$ do not depend on $\xi_0$ for $l \in [2, L]$.
\\ \\
Finally, let $x$ be a vector in the program for which $Z^x$ does not depend on $\xi$, and let $l \in [3, L]$. Then, by design, with the initial learning rates of Equation~\eqref{eq:ip-bias-initial-lr-weights} for the weights with IP-bias, we have
\begin{align*}
    \Delta W^l(1) x = -\eta \chi_0 {\transpose{{(x^{l-1}_0)}} x\over m} d\tildeh^l_0,
\end{align*}
so that by \ZNonLin
\begin{align*}
    Z^{\Delta W^l(1) x} = -\eta \scalarlim{\chi}_0 \mathbb{E}[Z^{x^{l-1}_1} Z^x] Z^{d\tildeh^l_0},
\end{align*}
where $\scalarlim{\chi}_0 := \partial_2 \ell (y_0, v^{L+1})$. Since $v^{L+1}$, $Z^{x^{l-1}_0}$, $Z^x$ and $Z^{d\tildeh^l_0}$ do not depend on $\xi_0$ ($l-1$ and  $l$ are both in $[2, L]$), $Z^{\Delta W^l(1) x}$ also does not depend on the first training input $\xi_0$. To conclude, we have by the master theorem that $\transpose{{(\Delta W^{L+1}(1))}} x$ converges almost surely towards $ - \eta \scalarlim{\chi}_0 \mathbb{E}[Z^{x^L_0} Z^x]$ which is does not depend on $\xi_0$ since this is the case for $v^{L+1}$, $Z^{x^{L}_0}$ and $Z^x$, which concludes the proof.

\end{proof}

\subsection{Formal version of \autoref{th:ip-non-centered-deterministic-informal}}\label{app:ip-non-centered-formal}

\begin{theorem}[Formal]\label{th:ip-non-centered-deterministic-formal}
Consider IP-non-centered with the Naive-IP learning rates at every time step. Assume the activation function $\sigma$ satisfies Assumption~\ref{ass:smooth-act} and the loss $\ell$ satisfies Assumption~\ref{ass:loss}, and let $t \geq 0$ and $\xi \in \mathbb{R}^d$ be an input to the network. Then, calling $d\tildex^l_s := m \nabla_{x^l} f_s(\xi_s)$ and $d\tildeh^l_s := m \nabla_{h^l} f_s(\xi_s)$, one has that:
\begin{enumerate}[(i)]
    \item for any $l \in [2, L-1]$, $Z^{h^l_t}$ and $Z^{x^l_t}$ are deterministic constants,
    \item for any $l \in [2, L-1]$, $Z^{d\tildex^l_t}$ and $Z^{d\tildeh^l_t}$ deterministic constants,
    \item for any $l \in [3, L-1]$, and for any vector $x$ in the program, we have that \\ $Z^{(W^l(t+1) - W^l(0))x} = \left(- \eta \sum_{s=0}^{t} \scalarlim{\chi}_s Z^{d\tildeh^l_s} Z^{x^{l-1}_s} \right) \mathbb{E}[Z^x]$.
\end{enumerate}
\end{theorem}

\begin{remark}\label{remark:ip-non-centered-deterministic-formal}
Point $(iii)$ highlights the fact that in the infinite-width limit the (random) matrix operator $(w^l(t) - w^l(0))$ acts on a vector $x$ as if all the entries of the matrix operator were equal to a single deterministic constant $\left(- \eta \sum_{s=0}^{t-1} \scalarlim{\chi}_s Z^{d\tildeh^l_s} Z^{x^{l-1}_s} \right)$, because then the averages over the coordinates of $x$ involved in $(W^l(t) - W^l(0))x$ would simply yield $\mathbb{E}[Z^x]$ by the master theorem of the Tensor Program.
\end{remark}

The proof \autoref{th:ip-non-centered-deterministic-formal} can be found in Appendix~\ref{sec:ip-non-centered-proof}. The proof is done by inducting over $t$, and we present the case $t=0$ and the induction step first in Appendix~\ref{sec:ip-non-centered-prelim}.

\subsubsection{Preliminaries}\label{sec:ip-non-centered-prelim}

\begin{lemma}[First forward-backward pass and weight updates]\label{th:ip-non-centered-forward-backward-0}
Claims $(i)$, $(ii)$ and $(iii)$ of \autoref{th:ip-non-centered-deterministic-formal} hold at $t=0$.
\end{lemma}

\begin{proof}
$h^1_0$ and $x^1_0 = \sigma(h^1_0)$ are vectors in the program by the \MatMul\ and \NonLin\ rules since $\sigma$ is pseudo-Lipschitz by assumption, and $Z^{h^1_0} = Z^{U^1 \xi} + Z^{v^1} \sim \mathcal{N}(0, ||\xi||^2+1)$, and finally $Z^{x^1_0} = \sigma(Z^{h^1_0})$. Now, we have (recall that as defined in Section~\ref{sec:ip-non-centered} $J$ is the matrix full of ones)
\begin{align*}
    h^2_0 &= m^{-1/2}\hatW^2 x^1_0 + m^{-1} v^2 + u_2 m^{-1} J x^1_0.
\end{align*}
$m^{-1/2} \hatW^2 x^1_0 + m^{-1} v^2$ is a valid vector in the program by \MatMul\ and \NonLin\ because the initial scalars $m^{-1/2}$ and $m^{-1}$ converge to 0 almost surely, and $Z^{m^{-1/2}\hatW^2 x^1_0 + m^{-1}v^2} = 0 \times \hatZ^{\hatW^2 x^1_0} + 0 \times \hatZ^{v^2}$. By the \ZHat\ rule we get that $\hatZ^{\hatW^2 x^1_0}\sim \mathcal{N}(0, \mathbb{E}[(Z^{x^1_0})^2])$, with finite variance by Lemma~\ref{th:pos-finite-moment} since $\sigma$ is pseudo-Lipschitz and thus polynomially bounded, and $\hatZ^{v^2} \sim \mathcal{N}(0, 1)$. We thus get $Z^{m^{-1/2} (\hatW^2 x^1_0 + v^2)} = 0$. On the other hand, $\theta := (1/m) \sum_{q=1}^m x^{1}_{0,q}$ is a valid scalar in the program by the \Moment\ rule and it converges almost surely to $\scalarlim{\theta} = \mathbb{E}[Z^{x^1_0}]$ by the master theorem. The coordinates of $u_2 m^{-1} Jx^1_0$ are thus all equal to $u_2 \theta$, and the vector $u_2 m^{-1} Jx^1_0$ is thus equal to $\psi(x^1_0;\theta)$ coordinate-wise where the function $\psi(\cdot;\cdot) : \mathbb{R} \times \mathbb{R} \rightarrow \mathbb{R}$ is pseudo-Lipschitz and depends \textbf{only} on the second variable with $\psi(x;\alpha) = u_2 \alpha$. By the \NonLin\ rule $u_2 m^{-1} Jx^1_0$ is thus a vector in the program and by \ZNonLin\ we thus get $Z^{u_2 m^{-1} Jx^1_0} = \psi(Z^{x^1_0}; \scalarlim{\theta}) = u_2 \mathbb{E}[Z^{x^1_0}]$. We thus finally get
\begin{align*}
    Z^{h^2_0} &= u_2 \mathbb{E}[Z^{x^1_0}],
\end{align*}
which is a (finite) deterministic constant. Then the same statement holds for $Z^{x^2_0} = \sigma(u_2 \mathbb{E}[Z^{x^1_0}])$. Let $l \in [3, L]$ and assume that $h^{l-1}_0$ and $x^{l-1}_0$ are vectors in the program and that $Z^{h^{l-1}_0}$ and $Z^{x^{l-1}_0}$ are deterministic constants. Then, we have
\begin{align*}
    h^l_0 = m^{-1/2} \hatW^l x^{l-1}_0 + m^{-1}v^l + u_l m^{-1} J x^{l-1}_0
\end{align*}
As for the case $l=2$, we get that $m^{-1/2} \hatW^l x^{l-1}_0 + m^{-1}v^l$ is a vector in the program with $Z^{m^{-1/2} \hatW^l x^{l-1}_0 + m^{-1}v^l} = 0$, and $u_l m^{-1} J x^{l-1}_0 = \psi(x^{l-1}_0; \theta)$ is a vector in the program with $\psi(z; \alpha) = u_l \alpha$ (recall that $\psi$ is taken coordinate-wise) depending only on the second variable and $\theta := (1/m) \sum_{q=1}^m x^{l-1}_{0, q}$ is a valid scalar in the program by the \Moment\ rule, which, by the master theorem, converges almost surely towards $\scalarlim{\theta} = \mathbb{E}[Z^{x^{l-1}_0}] = Z^{x^{l-1}_0}$ since the latter is a deterministic constant by the induction hypothesis. By \NonLin\ $h^l_0$ is a vector in the program and by \ZNonLin\ $Z^{h^l_0} = \psi(Z^{x^{l-1}_0}; \scalarlim{\theta}) = u_l Z^{x^{l-1}_0}$ which is a deterministic constant. The same claim holds for $Z^{x^l_0} = \sigma(u_l Z^{x^{l-1}_0})$, which concludes the induction for the forward pass. 
For the backward pass we get $d\tildex^L_0 = w^{L+1}(0) = U^{L+1} + u_{L+1} \mathbf{1}$ so that by \ZNonLin\ $Z^{d\tildex^L_0} = Z^{U^{L+1}} + u_{L+1} \sim \mathcal{N}(u_{L+1}, 1)$ since $u_{L+1}$ is a valid initial scalar in the program as it converges almost surely to $u_{L+1}$. We then have $Z^{d\tildeh^L_0} = Z^{d\tildex^L_0} \sigma'(Z^{h^L_0})$. Note that both $Z^{d\tildex^L_0}$ and $Z^{d\tildeh^L_0}$ are not deterministic constants because $U^{L+1}$ is Gaussian with variance $1$. We then have:
\begin{align*}
    d\tildex^{L-1}_0 &= m^{-1/2} \transpose{{(\hatW^L)}}  d\tildeh^L_0 + u_L m^{-1} \transpose{{J}} d\tildeh^L_0
\end{align*}
As usual the first term $m^{-1/2} \transpose{{(\hatW^L)}}  d\tildeh^L_0$ is a vector in the program by \MatMul\ and \NonLin\, and $Z^{m^{-1/2} \transpose{{(\hatW^L)}}  d\tildeh^L_0} = 0$. For the second term, since $\transpose{{J}} = J$, $m^{-1} \transpose{{J}} d\tildeh^L_0$ is also a vector in the program and $Z^{m^{-1} \transpose{{J}} d\tildeh^L_0} = u_L \mathbb{E}[Z^{d\tildeh^L_0}]$. We thus get that $d\tildex^{L-1}_0$ is a vector in the program with $Z^{d\tildex^{L-1}_0} = u_L \mathbb{E}[Z^{d\tildeh^L_0}]$ which is a deterministic constant. Then, $d\tildeh^{l-1}_0$ is also a vector in the program and by \ZNonLin\ $Z^{d\tildeh^{L-1}_0} = Z^{d\tildex^{L-1}_0} \sigma'(Z^{h^L_0})$ is a deterministic constant. Repeating the reasoning above at any layer $l \in [2, L-1]$, an easy induction (as in the forward pass) shows that $d\tildex^l_0$ and $d\tildeh^l_0$ are vectors in the program and that $Z^{d\tildex^l_0}$ and $Z^{d\tildeh^l_0}$ are deterministic constants. Note that $Z^{d\tildex^1_0} = u_2 \mathbb{E}[Z^{d\tildeh^2_0}] = u_2 Z^{d\tildeh^2_0}$ is also a deterministic constant but that $Z^{d\tildeh^1_0} = Z^{d\tildex^1_0} \sigma'(Z^{h^1_0})$ is not because $Z^{h^1_0} \sim \mathcal{N}(0, ||\xi||^2 +1)$. Let $l \in [3, L-1]$, and let $x$ be a vector in the program. With the Naive-IP learning rates, we have 
\begin{align*}
    \Delta W^l(1) &= -\eta \chi_0 {d\tildeh^l_0 \transpose{{(x^{l-1}_0)}} \over m}
\end{align*}
Since $l \in [3, L-1]$, $Z^{d\tildeh^l_0}$ is a deterministic constant, and since $l-1 \in [2, L-2]$, $Z^{x^{l-1}_0}$ is also a deterministic constant. By \ZNonLin\ and \ZMoment we get
\begin{align*}
    Z^{\Delta W^l(1) x} &= -\eta \scalarlim{\chi}_0 \mathbb{E}[Z^{x^{l-1}_0} Z^x] Z^{d\tildeh^l_0} \\
    &= -\eta \scalarlim{\chi}_0 Z^{d\tildeh^l_0}  Z^{x^{l-1}_0} \mathbb{E}[ Z^x]
\end{align*}
which concludes the proof. Note that $\chi_0$ is a valid initial scalar in the program because $f_0(\xi_0) = m^{-1} \transpose{{(U^{L+1})}} x^L_0 + u_{L+1} m^{-1} \transpose{\mathbf{1}} x^L_0$ converges almost surely, by the master theorem, to $\mathbb{E}[Z^{U^{L+1}} Z^{x^L_0}] + u_{L+1} \mathbb{E}[Z^{x^L_0}] = u_{L+1} Z^{x^L_0}$ since $Z^{x^L_0}$ is a deterministic constant and $Z^{U^{L+1}} \sim \mathcal{N}(0, 1)$ has mean zero. Since $\partial_2 \ell(y_0, \cdot)$ is continuous by assumption, $\chi_0$ converges almost surely towards $\scalarlim{\chi}_0 := \partial_2 \ell(y_0, u_{L+1} Z^{x^L_0})$. 
\end{proof}

\begin{lemma}[Induction step at time $t \geq 1$]\label{th:ip-non-centered-induct}
Let $t \geq 1$ and assume claims $(i)$, $(ii)$ and $(iii)$ of \autoref{th:ip-non-centered-deterministic-formal} hold at all time steps $s \in [0, t-1]$. Then claims $(i)$, $(ii)$ and $(iii)$ also hold at time step $t$. 
\end{lemma}
\begin{proof}
With the Naive-IP learning rate exponents, we get that for any $t \geq 1$, 
\\
\begin{align*}
    W^1(t) &= U^1 - \eta \sum_{s=0}^{t-1} \chi_s d\tildeh^1_s \transpose{\xi_s}, \\
    B^1(t) &= v^1 - \eta \sum_{s=0}^{t-1} \chi_s d\tildeh^1_s, \\
    W^l(t) &= m^{-1}(U^l + u_l J)  - \eta \sum_{s=0}^{t-1} \chi_s {d\tildeh^l_s \transpose{{(x^{l-1}_s)}} \over m}, && l \in [2, L], \\
    B^l(t) &= m^{-1}v^l  - \eta m^{-1} \sum_{s=0}^{t-1} \chi_s d\tildeh^l_s, && l \in [2, L], \\
    W^{L+1}(t) &= m^{-1}(U^{L+1} + u_{L+1})  - \eta \sum_{s=0}^{t-1} \chi_s {x^L_s \over m}, \\
    B^{L+1}(t) &= m^{-1} v^{L+1}  - \eta m^{-1} \sum_{s=0}^{t-1} \chi_s.
\end{align*}
By a simple induction, all the $h^l_s, x^l_s$ and $d\tildex^l_s, d\tildeh^l_s$ are part of and the scalars $\chi_s$ are valid scalars in the program which have a constant almost sure limit, and by \ZNonLin\ we get:
\begin{align*}
    Z^{h^1_t} = Z^{U^1 \xi} + Z^{v^1} - \eta \sum_{s=0}^{t-1} \scalarlim{\chi}_s (\transpose{\xi_s} \xi + 1) Z^{d\tildeh^1_s}
\end{align*}
and $Z^{x^1_t} = \sigma(Z^{h^1_t})$ is not a deterministic constant because $Z^{U^1 \xi} + Z^{v^1} \sim \mathcal{N}(0, ||\xi||^2+1)$. Let $l \in [2, L-1]$. We have
\begin{align*}
    Z^{h^l_t} &= 0 \times Z^{\hatW^l x^{l-1}_t} + 0 \times Z^{v^l} + u_l \mathbb{E}[Z^{x^{l-1}_t}] - \eta \sum_{s=0}^{t-1} \scalarlim{\chi}_s \mathbb{E}[Z^{x^{l-1}_s} Z^{x^{l-1}_t}] Z^{d\tildeh^l_s} \\
    &= u_l \mathbb{E}[Z^{x^{l-1}_t}] - \eta \sum_{s=0}^{t-1} \scalarlim{\chi}_s \mathbb{E}[Z^{x^{l-1}_s} Z^{x^{l-1}_t}] Z^{d\tildeh^l_s},
\end{align*}
which is a deterministic constant with the assumption on the $Z^{d\tildeh^l_s}$ since $l \in [2, L-1]$. Note that if $l \in [3, L-1]$, we even have that the expectations simplify and we get $Z^{h^l_t} = (u_l - \eta \sum_{s=0}^{t-1} \scalarlim{\chi}_s Z^{x^{l-1}_s}  Z^{d\tildeh^l_s}) Z^{x^{l-1}_t}$. In any case, $Z^{x^l_t} = \sigma(Z^{h^l_t})$ is also a deterministic constant. For the output of the network, we have
\begin{align*}
    f_t(\xi) = { \transpose{{(U^{L+1})}} x^L_t \over m} + u_{L+1} {\transpose{{\mathbf{1}}} x^L_t \over m} + m^{-1}(v^{L+1} - \eta \sum_{s=0}^{t-1} \chi_s) - \eta \sum_{s=0}^{t-1} \chi_s {\transpose{{(x^L_s)}} x^L_t \over m}
\end{align*}
so that even if the $x^L_s$ are not deterministic, $f_t(\xi)$ still converges almost surely, by the master theorem, to $\mathbb{E}[(Z^{U^{L+1}} + u_{L+1}) Z^{x^L_t}] - \eta \sum_{s=0}^{t-1} \scalarlim{\chi}_s \mathbb{E}[Z^{x^L_s} Z^{x^L_t}]$, and since $\partial_2 \ell(y_t, \cdot)$ is continuous by assumption, $\chi_t$ converges almost surely towards the constant $\partial_2 \ell(y_t, \mathbb{E}[(Z^{U^{L+1}} + u_{L+1}) Z^{x^L_t}] - \eta \sum_{s=0}^{t-1} \scalarlim{\chi}_s \mathbb{E}[Z^{x^L_s} Z^{x^L_t}])$. For the backward pass, we get:
\begin{align*}
    Z^{d\tildex^L_t} = Z^{w^{L+1}(t)} = Z^{U^{L+1}} + u_{L+1} - \eta \sum_{s=0}^{t-1} \scalarlim{\chi}_s Z^{x^L_s}
\end{align*}
and $Z^{d\tildeh^L_t} = Z^{d\tildex^L_t} \sigma'(Z^{h^L_t})$. Let $l \in [2, L-1]$, we have
\begin{align*}
    d\tildex^l_t &= \transpose{{(W^{l+1}(t))}} d\tildeh^{l+1}_t \\
    &= m^{-1/2} \transpose{{(\hatW^{l+1})}} d\tildeh^{l+1}_t + m^{-1} u_{l+1} J d\tildeh^{l+1}_t - \eta \sum_{s=0}^{t-1} \chi_s {\transpose{{(d\tildeh^{l+1}_s)}} d\tildeh^{l+1}_t \over m} x^{l}_s,
\end{align*}
so that by \ZNonLin\ we get
\begin{align*}
    Z^{d\tildex^l_t} = u_{l+1} \mathbb{E}[Z^{d\tildeh^{l+1}_t}] - \eta \sum_{s=0}^{t-1} \scalarlim{\chi}_s \mathbb{E}[Z^{d\tildeh^{l+1}_s} Z^{d\tildeh^{l+1}_t}] Z^{x^{l}_s},
\end{align*}
and since $l \in [2, L-1]$, $Z^{x^l_s}$ is a deterministic constant and thus so is $Z^{d\tildex^l_t}$. Then, $Z^{d\tildeh^l_t} = Z^{d\tildex^l_t} \sigma'(Z^{h^l_t})$ and since $l \in [2, L-1]$, $Z^{h^l_t}$ is a deterministic constant. Finally, let $l \in [3, L-1]$, and let $x$ be a vector in the program. We have
\begin{align*}
    (W^{l}(t+1) - W^l(0))x &= - \eta \sum_{s=0}^{t} \chi_s {\transpose{{(x^{l-1}_s)}} x \over m} d\tildeh^l_s,
\end{align*}
and by \ZNonLin
\begin{align*}
    Z^{(W^{l}(t+1) - W^l(0))x} &= - \eta \sum_{s=0}^{t} \scalarlim{\chi}_s \mathbb{E}[Z^{x^{l-1}_s} Z^x] Z^{d\tildeh^l_s} \\
    &= \left(- \eta \sum_{s=0}^{t} \scalarlim{\chi}_s Z^{d\tildeh^l_s} Z^{x^{l-1}_s} \right) \mathbb{E}[Z^x],
\end{align*}
where the last equality stems from the fact that since $l \in [3, L-1]$, $l-1 \in [2, L-2]$ and $Z^{x^{l-1}_s}$ is a deterministic constant for any $s \in [0, t]$. Since $l \in [2, L-1]$, $Z^{d\tildeh^l_s}$ is also a deterministic constant, so that $- \eta \sum_{s=0}^{t} \scalarlim{\chi}_s Z^{d\tildeh^l_s} Z^{x^{l-1}_s}$ is a deterministic constant, which concludes the proof. 
\end{proof}

\subsubsection{Proof of \autoref{th:ip-non-centered-deterministic-formal}}\label{sec:ip-non-centered-proof}

\begin{proof}
The result simply comes by induction over $t$ using Lemmas~\ref{th:ip-non-centered-forward-backward-0} and~\ref{th:ip-non-centered-induct}.
\end{proof}

\section{The variables associated with the initial weights vanish in IP-LLR}\label{sec:ipllr-induct-dynamics}

In this section we wish to study more precisely the evolution and the expression of the variables $Z$ in the dynamics of IP-LLR at any time step $t$. To this end, we will show that the $Z$s of all the forward and backward variables in IP-LLR are functions only of the $\hatZ^{\hatW^l \tildex^{l-1}_0}$ and $\hatZ^{\transpose{{(\hatW^l)}}d\tildeh^{l}_0}$, as well as the initial vectors $U^1 \xi_0, \ldots, U^1 \xi_t$, $v^1$, $U^{L+1}$. We will thus write 
\begin{align*}
    Z^{z} = \psi \left( \left(\hatZ^{\hatW^l \tildex^{l-1}_0} \right)_l, \left(\hatZ^{\transpose{{(\hatW^k)}}d\tildeh^{k}_0} \right)_k, \left(U^1 \xi_s \right)_s, v^1, U^{L+1} \right)
\end{align*}
to \textbf{generically} denote that the variable $Z^{z}$ is a function  \textbf{only} of the variables which appear in the arguments: $\hatZ^{\hatW^l\tildex^{l-1}_0}$, $\hatZ^{\transpose{{(\hatW^k)}}d\tildeh^{k}_0}$, $U^1 \xi_s$, $v^1$, and $U^{L+1}$, (where multiple values of $l$, $k$ and $s$ might actually appear in the argument). This function $\psi$ (we will sometimes also use $\phi$) will of course depend on the $z$ under consideration, and we might denote it by $\psi^z$ (or $\phi^z$,) but most of the time we will omit this dependency and simply use the symbol $\psi$ for different variables to express that the variable $Z^z$ is a function of the arguments of $\psi$ only. 
\\\\
We will see that the function $\psi$ appearing will always be polynomially bounded by some form of composition or product of polynomially bounded functions, which will allow us to prove that the corresponding $Z^z$ is finite almost surely since its arguments, considered as a vector, follow a Gaussian distribution with finite variance (and thus finite moments of any order). Note that in the proofs, we will use extensively (without explicitly saying so) that if $\phi$ and $\psi$ are polynomially bounded then $\phi \times \psi$ is also polynomially bounded, and if $\varphi$ is a polynomially bounded function of a single variable then $\varphi \circ \psi$ is also polynomially bounded. We introduce the following definition and lemma which we will use extensively in the proof by induction:

\begin{definition}[Vector of initial vectors and first forward-backward]\label{def:Z0}
Let $t \geq 1$. Then, dropping the dependency on $t$, we define the random vector:
\begin{align*}
    Z_0 = Z_{0,t} := &\left(\hatZ^{U^1\xi_0}, \ldots, \hatZ^{U^1\xi_t}, Z^{v^1}, \hatZ^{U^{L+1}}, \right. \\
    & \ \ \hatZ^{\hatW^2 \tildex^1_0}, \ldots, \hatZ^{\hatW^L \tildex^{L-1}_0}, \\ & \ \left. \hatZ^{\transpose{{(\hatW^2)}} d\tildeh^2_0}, \ldots, \hatZ^{\transpose{{(\hatW^L)}} d\tildeh^L_0} \right)    
\end{align*}
\end{definition}

\begin{remark}\
\begin{enumerate}
    \item Note that any function of $Z_{0, s}$ will also be a function of $Z_{0, t}$ for $t \geq s$, which is also why we suppress the dependency on $t$ as we can always take the largest possible $t$ when we make a specific claim which involves $Z_0$.
    
    \item Also note that by the \ZDot\ rule of the Tensor Program, for any vector $z$ in the Tensor Program such that $Z^z$ is a function only of $Z_0$, then for any $l \in [2, L]$:
    \begin{align*}
        \begin{cases}
            \dotZ^{\hatW^l z} = \mathbb{E} \left[\frac{\partial Z^z}{\partial \hatZ^{\transpose{{(\hatW^l)}} d\tildeh^{l}_0}} \right] Z^{d\tildeh^l_0} \\
             \\
            \dotZ^{\transpose{{(\hatW^l)}} z} = \mathbb{E} \left[\frac{\partial Z^z}{\partial \hatZ^{\hatW^l \tildex^{l-1}_0}} \right] Z^{\tildex^{l-1}_0}
        \end{cases}
    \end{align*}
\end{enumerate}
\end{remark}

\begin{lemma}[Distribution of $Z_0$ and moments]\label{th:Z0-dist-moments}
One has
\begin{enumerate}[(i)]
    \item $Z_0 \sim \mathcal{N} \left(0, 
    \begin{pmatrix}
        S & 0 & 0 \\
        0 & D_f & 0 \\
        0 & 0 & D_b
    \end{pmatrix}
    \right)$ with \\
    $S :=  
    \begin{pmatrix}
        \Sigma & 0 & 0\\
        0 & 1 & 0\\
        0 & 0 & 1
    \end{pmatrix} 
    \in \mathbb{R}^{(t+3) \times (t+3)}
    $, \qquad $\Sigma_{rs} = \transpose{\xi_r} \xi_s$, \\
    $D_f :=  
    \begin{bmatrix}
    \mathbb{E}[(Z^{\tildex^1_0})^2] & & \\
    & \ddots & \\
    & & \mathbb{E}[(Z^{\tildex^{L-1}_0})^2]
  \end{bmatrix}
  \in \mathbb{R}^{(L-1) \times (L-1)}$, \\
  $D_b :=  
    \begin{bmatrix}
    \mathbb{E}[(Z^{d\tildeh^2_0})^2] & & \\
    & \ddots & \\
    & & \mathbb{E}[(Z^{d\tildeh^{L}_0})^2]
  \end{bmatrix}
  \in \mathbb{R}^{(L-1) \times (L-1)}.$ 
    
    \item $|\mathbb{E}[\psi(Z_0)]| < \infty$, and $|\psi(Z_0)| < \infty$ almost surely for any polynomially bounded function $\psi : \mathbb{R}^{t + 2L} \longrightarrow \mathbb{R}$.
\end{enumerate}
\end{lemma}

\begin{remark}
Note that the lemma stays valid even if $\psi$ does not depend on the whole list of variables inside $Z_0$ but only on a couple of them, which will be the case in the Tensor Program. Point $(ii)$ will be used repeatedly in different proofs to show that the expectations appearing in the forward and backward passes are finite.
\end{remark}

\begin{proof}
Claim $(i)$ simply comes from the definition of the initial vectors $U^1\xi_0$, $\ldots$, $U^1\xi_t$, $U^{L+1}$ and from the ZHat rule in a Tensor Program. Claim $(ii)$ then  follows because all entries in the covariance matrix are finite by Lemma~\ref{th:tilde-pos-variance}, and since $\psi$ is polynomially bounded and the moments of a Gaussian with finite variance are finite, $|\mathbb{E}[\psi(Z_0)]| \leq \mathbb{E}[|\psi(Z_0)|] < \infty$ and thus $|\psi(Z_0)| < \infty$ almost surely. 
\end{proof}

\noindent
Note that by Lemmas~\ref{th:homog-first-forward} and~\ref{th:homog-first-backward}, the first forward and backward passes of IP-LLR easily express in function of the entries of $Z_0$. Let us now take care of the forward and backward passes at $t=1$. As the dynamics evolve with time, the expression of the forward and backward passes of IP-LLR in function of $Z_0$ (or rather of some of the entries of $Z_0$) get more intricate. They are still easy to develop explicitly for $t=1$ but we choose to simply express what variables appear in the expression of the forward and backward passes instead of giving the expression explicitly.

\begin{lemma}[Multiplications by initial weight matrices vanish with polynomially bounded variables]\label{th:initial-W-vanish}
Consider the IP-LLR parameterization and let $z$ be a vector in the program such that $Z^z = \psi(Z_0)$ with $\psi$ polynomially bounded. Then, one has that for any $l \in [2, L]$:
\begin{enumerate}[(i)]
    \item if $\frac{\partial Z^z}{\partial \hatZ^{\transpose{{(\hatW^l)}} d\tildeh^{l}_0}} = \phi(Z_0)$ with $\phi$ polynomially bounded, then $Z^{W^l(0)z} = 0$.
    
    \item if $\frac{\partial Z^z}{\partial \hatZ^{\hatW^l \tildex^{l-1}_0}} = \phi(Z_0)$ with $\phi$ polynomially bounded, then $Z^{\transpose{{(W^l(0))}} z} = 0$.
\end{enumerate}
\end{lemma}

\begin{proof}
Let $l \in [2, L]$. We simply write
\begin{align*}
    Z^{W^l(0) z} = \scalarlim{\omega}_l \hatZ^{\hatW^l z} + \scalarlim{\omega}_l \dotZ^{\hatW^l z}
\end{align*}
where $\hatZ^{W^l(0) z} \sim \mathcal{N}(0, \mathbb{E}[(Z^z)^2])$ and the variance is finite by Lemma~\ref{th:Z0-dist-moments} because $(Z^z)^2$ is a polynomially bounded function of $Z_0$ since $Z^z$ is. This shows that $|\hatZ^{W^l(0) z}| < \infty$ almost surely and thus that $\scalarlim{\omega}_l \hatZ^{\hatW^l z} = 0$ since $\scalarlim{\omega}_l = 0$ in IPs. On the other hand,
\begin{align*}
    \dotZ^{\hatW^l z} = \mathbb{E} \left[\frac{\partial Z^z}{\partial \hatZ^{\transpose{{(\hatW^l)}} d\tildeh^{l}_0}} \right] Z^{d\tildeh^l_0}
\end{align*}
and the expectation is finite by Lemma~\ref{th:Z0-dist-moments} since $\partial Z^z /  \partial \hatZ^{\transpose{{(\hatW^l)}}} = \phi(Z_0)$ with $\phi$ polynomially bounded, and 
\begin{align*}
Z^{d\tildeh^l_0} =
    \begin{cases}
        \hatZ^{\transpose{{(\hatW^{l+1})}} d\tildeh{l+1}_0} \sigma'(\hatZ^{\hatW^l \tildex^{l-1}_0}) \text{ if } \ l \in [2, L-1] \\
         \hatZ^{U^{L+1}} \sigma'(\hatZ^{\hatW^l \tildex^{L-1}_0}) \text{ if } \ l=L
    \end{cases}
\end{align*}
In any case, $Z^{d\tildeh^l_0}$ is a polynomially bounded function of $Z_0$ and is thus finite almost surely, which entails $\scalarlim{\omega}_l \hatZ^{\hatW^l z} = 0$, and therefore $\hatZ^{W^l(0) z} = 0$ which gives $(i)$. 
\\\\
The same reasoning with $\transpose{{W^l(0)}}$ gives $(ii)$ if $\partial Z^z / \partial \hatZ^{\hatW^l \tildex^{l-1}_0} = \phi(Z_0)$ with $\phi$ polynomially bounded. 

\end{proof}

\subsubsection{The case $t=1$}

\begin{lemma}[$Z_0$ in the forward pass of IP-LLR at $t=1$]\label{th:ipllr-forward-func-Z0-1}
Consider the IP-LLR parameterization with a positively $p$-homogeneous activation function, and $p \geq 2$. Then, dropping the dependency of the forward pass on $\xi_1$, one has:
\begin{enumerate}[(i)]
    \item $Z^{h^1_1} = \psi \left(\hatZ^{U^1\xi_0}, \ \hatZ^{U^1\xi_1}, \ \hatZ^{\transpose{(\hatW^2) } d\tildeh^2_0} \right)$
    
    \item $Z^{h^l_1} = \psi \left( \hatZ^{\hatW^l \tildex^{l-1}_0}, \ \hatZ^{\transpose{(\hatW^{l+1}) } d\tildeh^{l+1}_0} \right)$, \ $l \in [2, L-1]$
    
    \item $Z^{h^L_1} = \psi \left(\hatZ^{U^{L+1}}, \ \hatZ^{\hatW^{L} \tildex^{L-1}_0} \right)$
\end{enumerate}
and 
\begin{enumerate}[(i)]
    \setcounter{enumi}{3}
    \item $\frac{\partial Z^{h^{l-1}_1}}{\hatZ^{\transpose{{(\hatW^l)}} d\tildeh^l_0}}  = \psi \left(Z_0 \right)$, $l \in [2, L]$
    
    \item $\frac{\partial Z^{h^{l}_1}}{\hatZ^{\hatW^l \tildex^{l-1}_0}}  = \psi \left(Z_0 \right)$, $l \in [2, L]$
\end{enumerate}
and \textbf{all} the different $\psi$ that appear are polynomially bounded. 
\end{lemma}

\begin{remark}\
\begin{enumerate}
    \item Recall that we simply use $\psi$ or $\phi$ to mean that the variable \textbf{is a function of} the arguments of $\psi$ (or $\phi$) only, and that the different $\psi$ and $\phi$ which appear in the different claims $(i)$ to $(v)$ are \textbf{not} actually the same. 
    
    \item For the partial derivatives we chose not to make a precise statement on which variables exactly appear in the expression as this will not matter and would only over-complicate things for close to none added-value. 
    
    \item Note that with the claims above, one can  prove that $\scalarlim{\omega}_l Z^{\hatW^l x^{l-1}_1} = 0$ because both of the terms $\hatZ$ and $\dotZ$ defining $Z^{\hatW^l x^{l-1}_1}$ are polynomially bounded functions of Gaussians which has finite covariance matrices, and $\scalarlim{\omega}_l = 0$ in IPs. 
\end{enumerate}
\end{remark}

\begin{proof}
Using \autoref{th:z-forward-ipllr-t} with $\xi = \xi_1$ and $t=1$ we have claim $(i)$ because, first $\hatZ^{d\tildex^1_0} = \hatZ^{\transpose{{(\hatW^2)}} d\tildeh^2_0}$, and second $\sigma'$ is polynomially bounded (see Appendix~\ref{sec:pos-homog}). Claim $(ii)$ also stems from \autoref{th:z-forward-ipllr-t} since $\hatZ^{d\tildex^l_0} = \hatZ^{\transpose{{(\hatW^{l+1})}} d\tildeh^{l+1}_0}$, $\hatZ^{\tildeh^l_0} = \hatZ^{\hatW^l \tildex^{l-1}_0}$, and $\sigma'$ is polynomially bounded. Finally, claim $(iii)$ also stems from \autoref{th:z-forward-ipllr-t} since $\hatZ^{d\tildex^L_0} = Z^{U^{L+1}}$, $\hatZ^{\tildeh^L_0} = \hatZ^{\hatW^l \tildex^{L-1}_0}$, and $\sigma'$ is polynomially bounded.
\\\\
From \autoref{th:z-forward-ipllr-t}, we get:
\begin{align*}
    \frac{\partial Z^{h^1_1}}{\partial \hatZ^{\transpose{{(\hatW^2)}} d\tildeh^{2}_0} } = - \eta \scalarlim{\chi}_0 (\transpose{\xi_0} \xi_1) \sigma'(\hatZ^{U^1 \xi_0})
\end{align*}
For $l \in [3, L]$, from \autoref{th:z-forward-ipllr-t}, we get
\begin{align*}
        \frac{\partial Z^{h^{l-1}_1}}{Z^{\transpose{{(\hatW^l)}} d\tildeh^{l}_0}} = - \eta \scalarlim{\chi}_0 \mathbb{E}[Z^{\tildex^{l-2}_0} Z^{\tildex^{l-2}_1}] \sigma'(\hatZ^{\hatW^l \tildex^{l-1}_0} )
\end{align*}
which immediately gives claim $(iv)$ since $\sigma'$ is polynomially bounded and with claim $(ii)$ and Lemma~\ref{th:Z0-dist-moments}, we also have $|\mathbb{E}[Z^{\tildex^{l-2}_0} Z^{\tildex^{l-2}_1}]| < \infty$.
Similarly, for $l \in [2, L]$, from \autoref{th:z-forward-ipllr-t}, we get
\begin{align*}
        \frac{\partial Z^{h^{l}_1}}{\partial \hatZ^{\hatW^l \tildex^{l}_0} } = - \eta \scalarlim{\chi}_0 \mathbb{E}[Z^{\tildex^{l-1}_0} Z^{\tildex^{l-1}_1}] Z^{d\tildex^l_0} \sigma''(\hatZ^{\hatW^l \tildex^{l-1}_0} )
\end{align*}
and $Z^{d\tildex^{l-1}_0} = \hatZ^{\transpose{{(\hatW^{l+1})}} d\tildeh^{l+1}_0}$ if $l \in [2, L-1]$, and $Z^{d\tildex^{l-1}_0} = \hatZ^{U^{L+1}}$ if $l=L$. Since the expectation is finite by claim $(ii)$ and Lemma~\ref{th:Z0-dist-moments}, and since $\sigma''$ is polynomially bounded, we get claim $(v)$.
\end{proof}

\begin{lemma}[$Z_0$ in the backward pass of IP-LLR at $t=1$]\label{th:ipllr-backward-func-Z0-1}
Consider the IP-LLR parameterization with a positively $p$-homogeneous activation function, and $p \geq 2$. Then, dropping the dependency of the forward and backward passes on $\xi_1$, one has:
\begin{enumerate}[(i)]
    \item $Z^{d\tildex^L_1} = \psi \left( \hatZ^{U^{L+1}}, \ \hatZ^{\hatW^{L} \tildex^{L-1}_0} \right)$, \\
    $Z^{d\tildeh^L_1} = \psi \left( \hatZ^{U^{L+1}}, \ \hatZ^{\hatW^{L} \tildex^{L-1}_0} \right)$
    
    \item $Z^{d\tildex^{l-1}_1} = \psi \left( \hatZ^{\hatW^{l-1} \tildex^{l-2}} \right)$, \\
    $Z^{d\tildeh^{l-1}_1} = \psi \left( \hatZ^{\hatW^{l-1} \tildex^{l-2}}, \  \hatZ^{\transpose{{(\hatW^l)}} d\tildeh^l_0} \right)$, $l \in [3, L]$
    
    \item $Z^{d\tildex^1_1} = \psi \left( \hatZ^{U^1 \xi_0} \right)$, \\
    $Z^{d\tildeh^1_1} = \psi \left( \hatZ^{U^1 \xi_0}, \ \hatZ^{U^1 \xi_1}, \ \hatZ^{\transpose{(\hatW^2) } d\tildeh^2_0} \right)$
\end{enumerate}
and 
\begin{enumerate}[(i)]
    \setcounter{enumi}{3}
    \item $\frac{\partial Z^{d\tildeh^l_1}}{\partial \hatZ^{\hatW^l\tildex^{l-1}_0}} = \psi(Z_0)$, $l \in [1, L]$ 
    
    \item $\frac{\partial Z^{d\tildeh^l_1}}{\partial \hatZ^{\transpose{{(\hatW^{l+1})}} d\tildeh^{l+1}_0}} = \psi(Z_0)$, $l \in [1, L-1]$ 
\end{enumerate}
and \textbf{all} the different $\psi$ that appear are polynomially bounded. 
\end{lemma}

\begin{proof}
For the backward pass, we have by definition of the tilde variables for $t \geq 1$, $d\tildex^L_1 = w^{L+1}(1) = U^{L+1} - \eta \chi_0 \tildex^L_1$ by Lemma~\ref{th:ipllr-first-weight-updates}, and thus 
\begin{align*}
    Z^{d\tildex^L_1} = \hatZ^{U^{L+1}} - \eta \scalarlim{\chi}_0 \sigma(\hatZ^{\hatW^l \tildex^{L-1}_0})
\end{align*}
Then, 
\begin{align*}
    Z^{d\tildeh^L_1} = Z^{d\tildex^L_1} \sigma'(Z^{h^L_1})
\end{align*}
which gives claim $(i)$ since $\sigma$ and $\sigma'$ are polynomially bounded, and $Z^{h^L_1} = \psi (\hatZ^{U^{L+1}}, \ \hatZ^{\hatW^{L} \tildex^{L-1}_0} )$ and $\psi$ is polynomially bounded by Lemma~\ref{th:ipllr-forward-func-Z0-1}.  \\\\
For $l=L-1$, we have 
\begin{align*}
    d\tildex^{L-1}_1 &= \transpose{{(W^L(1))}} d\tildeh^L_1 \\
    &= \omega_L \transpose{{(\hatW^l)}} d\tildeh^L_1 - \eta \chi_0 \frac{\transpose{{(d\tildeh^L_0)}} d\tildeh^L_1}{m} \tildex^{L-1}_0 
\end{align*}
which gives 
\begin{align*}
    Z^{d\tildex^{L-1}_1} = \scalarlim{\omega}_L Z^{\transpose{{(\hatW^l)}} d\tildeh^L_1} - \eta \scalarlim{\chi}_0 \mathbb{E}[Z^{d\tildeh^L_0} Z^{d\tildeh^L_1}] Z^{\tildex^{L-1}_0 }
\end{align*}
Now, by the previous expression of $Z^{d\tildeh^L_1}$ and by Lemma~\ref{th:ipllr-forward-func-Z0-1}, we get 
\begin{align*}
    \frac{\partial Z^{d\tildeh^L_1}}{\partial \hatW^l \tildex^{L-1}_0} =& -\eta \scalarlim{\chi}_0 \sigma'(\hatZ^{\hatW^l \tildex^{L-1}_0}) \sigma'(Z^{h^L_1}) +  Z^{d\tildex^L_1} \, \frac{\partial Z^{h^L_1}}{\hatW^l \tildex^{L-1}_0} \sigma''(Z^{h^L_1})
\end{align*}
and by claim $(i)$ and Lemma~\ref{th:ipllr-forward-func-Z0-1} we get 
\begin{align*}
    \frac{\partial Z^{d\tildeh^L_1}}{\partial \hatW^l \tildex^{L-1}_0} = \psi(Z_0)
\end{align*}
with $\psi$ polynomially bounded since $\sigma'$ and $\sigma''$ are polynomially bounded. Therefore, by Lemma~\ref{th:initial-W-vanish}, we get $\scalarlim{\omega}_L Z^{\transpose{{(\hatW^l)}} d\tildeh^L_1} = 0$.
\\\\
We thus simply get 
\begin{align*}
    Z^{d\tildex^{L-1}_1} = - \eta \scalarlim{\chi}_0 \mathbb{E}[Z^{d\tildeh^L_0} Z^{d\tildeh^L_1}] Z^{\tildex^{L-1}_0 }
\end{align*}
$Z^{d\tildeh^L_0}$ and $Z^{d\tildeh^L_1}$ are polynomially bounded functions of $Z_0$ and thus so is $Z^{d\tildeh^L_0} Z^{d\tildeh^L_1}$, and by Lemma~\ref{th:Z0-dist-moments}, $|\mathbb{E}[Z^{d\tildeh^L_0} Z^{d\tildeh^L_1}]| < \infty$. Since $Z^{d\tildex^{L-1}_1} = \psi(Z^{\tildex^{L-1}_0 })$ with $\psi$ polynomially bounded, we thus get $Z^{d\tildex^{L-1}_1} = \psi(\hatZ^{\hatW^{l-1} \tildex^{L-2}_0})$ and $\psi$ is polynomially bounded (indeed: $\psi(z) = - \eta \scalarlim{\chi}_0 \mathbb{E}[Z^{d\tildeh^L_0} Z^{d\tildeh^L_1}] \sigma(z)$). 
\\\\
We have
\begin{align*}
    Z^{d\tildeh^{L-1}_1} = Z^{d\tildex^{L-1}_1} \sigma'(Z^{h^{L-1}_1})
\end{align*}
and since by Lemma~\ref{th:ipllr-forward-func-Z0-1}, $Z^{h^{L-1}_1} = \psi(\hatZ^{\hatW^{l-1} \tildex^{L-2}_0}, \hatZ^{\transpose{{(\hatW^l)}} d\tildeh^L_0})$ with $\psi$ polynomially bounded, by the previous result for $Z^{d\tildex^{L-1}_1}$ and since $\sigma'$ is polynomially bounded we  get
\begin{align*}
    Z^{d\tildeh^{L-1}_1} = \psi(\hatZ^{\hatW^{l-1} \tildex^{L-2}_0}, \hatZ^{\transpose{{(\hatW^l)}} d\tildeh^L_0})
\end{align*}
with $\psi$ polynomially bounded. 
\\\\
We have
\begin{align*}
    \frac{\partial Z^{d\tildeh^{L-1}_1}}{\partial \hatZ^{\hatW^{l-1} \tildex^{L-2}_0}} =& \frac{\partial Z^{d\tildex^{L-1}_1}}{\partial \hatZ^{\hatW^{l-1} \tildex^{L-2}_0}} \sigma'(Z^{h^{L-1}_1}) + Z^{d\tildex^{L-1}_1} \frac{\partial Z^{h^{L-1}_1}}{\partial \hatZ^{\hatW^{l-1} \tildex^{L-2}_0}} \sigma''(Z^{h^{L-1}_1}) \\
    =& - \eta \scalarlim{\chi}_0 \mathbb{E}[Z^{d\tildeh^L_0} Z^{d\tildeh^L_1}] \sigma'(\hatZ^{\hatW^{l-1} \tildex^{L-2}_0}) \sigma'(Z^{h^{L-1}_1}) \ + \\
    &Z^{d\tildex^{L-1}_1} \frac{\partial Z^{h^{L-1}_1}}{\partial \hatZ^{\hatW^{l-1} \tildex^{L-2}_0}} \sigma''(Z^{h^{L-1}_1})
\end{align*}
By Lemma~\ref{th:ipllr-forward-func-Z0-1} and since $\sigma'$ and $\sigma''$ are polynomially bounded, and we have already proven that $Z^{d\tildex^{L-1}_1} = \psi(Z_0)$ with $\psi$ polynomially bounded, as well as $|\mathbb{E}[Z^{d\tildeh^L_0} Z^{d\tildeh^L_1}]| < \infty$, we get 
\begin{align*}
    \frac{\partial Z^{d\tildeh^{L-1}_1}}{\partial \hatZ^{\hatW^{l-1} \tildex^{L-2}_0}} = \psi(Z_0)
\end{align*}
with $\psi$ polynomially bounded. 
\\\\
Similarly, we have 
\begin{align*}
   \frac{\partial Z^{d\tildeh^{L-1}_1}}{\partial \hatZ^{\transpose{{(\hatW^{L})}} d\tildeh^{L}_0}} =& \frac{\partial Z^{d\tildex^{L-1}_1}}{\partial Z^{\transpose{{(\hatW^{L})}} d\tildeh^{L}_0}} \sigma'(Z^{h^{L-1}_1}) + Z^{d\tildex^{L-1}_1} \frac{\partial Z^{h^{L-1}_1}}{\partial Z^{\transpose{{(\hatW^{L})}} d\tildeh^{L}_0}} \sigma''(Z^{h^{L-1}_1}) \\
    =& Z^{d\tildex^{L-1}_1} \frac{\partial Z^{h^{L-1}_1}}{\partial Z^{\transpose{{(\hatW^{L})}} d\tildeh^{L}_0}} \sigma''(Z^{h^{L-1}_1})
\end{align*}
Now, we have shown above that $Z^{d\tildex^{L-1}_1} = \psi(Z_0)$ with $\psi$ polynomially bounded, and by Lemma~\ref{th:ipllr-forward-func-Z0-1} we have that both $\partial Z^{h^{L-1}_1} / \partial Z^{\transpose{{(\hatW^{L})}} d\tildeh^{L}_0}$ and $Z^{h^{L-1}_1}$ are polynomially bounded functions of $Z_0$, which gives 
\begin{align*}
    \frac{\partial Z^{d\tildeh^{L-1}_1}}{\partial \hatZ^{\transpose{{(\hatW^{L})}} d\tildeh^{L}_0}} = \psi(Z_0)
\end{align*}
with $\psi$ polynomially bounded. 
\\\\
Let $l \in [2, L-1]$ and assume claims $(ii)$, $(iv)$ and $(v)$ are true for layer $l$. We have 
\begin{align*}
    Z^{d\tildex^{l-1}_1} = \scalarlim{\omega}_l Z^{\transpose{{(\hatW^l)}} d\tildeh^l_1} - \eta \scalarlim{\chi}_0 \mathbb{E}[Z^{d\tildeh^l_0} Z^{d\tildeh^l_1}] Z^{\tildex^{l-1}_0}
\end{align*}
Since by the induction hypothesis $\partial Z^{d\tildeh^l_1} / \partial \hatZ^{\hatW^l\tildex^{l-1}_0} = \psi(Z_0)$ with $\psi$ polynomially bounded, and $Z^{d\tildeh^l_1}$ is a polynomially bounded function of $Z_0$, by Lemma~\ref{th:initial-W-vanish} we get $\scalarlim{\omega}_l Z^{\transpose{{(\hatW^l)}} d\tildeh^l_1} = 0$. Then, we simply get
\begin{align*}
     Z^{d\tildex^{l-1}_1} = - \eta \scalarlim{\chi}_0 \mathbb{E}[Z^{d\tildeh^l_0} Z^{d\tildeh^l_1}] Z^{\tildex^{l-1}_0}
\end{align*}
Again here, since both $Z^{d\tildeh^l_0}$ and $Z^{d\tildeh^l_1}$ are polynomially bounded functions of $Z_0$, then so is $Z^{d\tildeh^l_0} Z^{d\tildeh^l_1}$, which shows by Lemma~\ref{th:Z0-dist-moments} that $|\mathbb{E}[Z^{d\tildeh^l_0} Z^{d\tildeh^l_1}]| < \infty$. If $l \geq 3$, since $Z^{\tildex^{l-1}_0} = \sigma(\hatZ^{\hatW^{l-1} \tildex^{l-2}_0})$ and $\sigma$ is polynomially bounded, we get that 
\begin{align*}
    Z^{d\tildex^{l-1}_1} = \psi(\hatZ^{\hatW^{l-1} \tildex^{l-2}_0})
\end{align*}
If $l=2$, since $Z^{\tildex^{1}_0} = \sigma(Z^{U^1 \xi_0})$ and $\sigma$ is polynomially bounded we get:
\begin{align*}
    Z^{d\tildex^{l-1}_1} = \psi(Z^{U^1 \xi_0})
\end{align*}
with $\psi$ polynomially bounded. 
\\\\
We then have
\begin{align*}
    Z^{d\tildeh^{l-1}_1} = Z^{d\tildex^{l-1}_1} \sigma'(Z^{h^{l-1}_1})
\end{align*}
and thus
\begin{align*}
    \frac{\partial Z^{d\tildeh^{l-1}_1}}{\partial \hatZ^{\hatW^{l-1} \tildex^{l-2}_0}} =& \frac{\partial Z^{d\tildex^{l-1}_1}}{\partial \hatZ^{\hatW{^l-1}(0) \tildex^{l-2}_0}} \sigma'(Z^{h^{l-1}_1}) + Z^{d\tildex^{l-1}_1} \frac{\partial Z^{h^{l-1}_1}}{\partial \hatZ^{\hatW^{l-1} \tildex^{l-2}_0}} \sigma''(Z^{h^{l-1}_1}) \\
    =& - \eta \scalarlim{\chi}_0 \mathbb{E}[Z^{d\tildeh^l_0} Z^{d\tildeh^l_1}] \sigma'(\hatZ^{\hatW^{l-1} \tildex^{l-2}_0}) \sigma'(Z^{h^{l-1}_1}) \ + \\
    &Z^{d\tildex^{l-1}_1} \frac{\partial Z^{h^{l-1}_1}}{\partial \hatZ^{\hatW^{l-1} \tildex^{l-2}_0}} \sigma''(Z^{h^{l-1}_1}) \\
\end{align*}
By Lemma~\ref{th:ipllr-forward-func-Z0-1} as well as the previous result on $Z^{d\tildex^{l-1}_1}$, and since $|\mathbb{E}[Z^{d\tildeh^l_0} Z^{d\tildeh^l_1}]| < \infty$, and $\sigma'$ and $\sigma''$ are polynomially bounded, we get that 
\begin{align*}
    \frac{\partial Z^{d\tildeh^{l-1}_1}}{\partial \hatZ^{\hatW^{l-1} \tildex^{l-2}_0}} = \psi(Z_0)
\end{align*}
with $\psi$ polynomially bounded.
\\\\
Similarly
\begin{align*}
    \frac{\partial Z^{d\tildeh^{l-1}_1}}{\partial \hatZ^{\transpose{{(\hatW^{L})}} d\tildeh^{l}_0}} =& \frac{\partial Z^{d\tildex^{l-1}_1}}{\partial \hatZ^{\transpose{{(\hatW^{l+1})}} d\tildeh^{l}_0}} \sigma'(Z^{h^{l-1}_1}) + Z^{d\tildex^{l-1}_1} \frac{\partial Z^{h^{l-1}_1}}{\partial \hatZ^{\transpose{{(\hatW^{L})}} d\tildeh^{l}_0}} \sigma''(Z^{h^{l-1}_1}) \\
    =& Z^{d\tildex^{l-1}_1} \frac{\partial Z^{h^{l-1}_1}}{\partial \hatZ^{\transpose{{(\hatW^{L})}} d\tildeh^{l}_0}} \sigma''(Z^{h^{l-1}_1}) \\
\end{align*}
and the three quantities in the product are polynomially bounded functions of $Z_0$ (shown above for the first term and by Lemma~\ref{th:ipllr-forward-func-Z0-1} for the two other terms). We thus get
\begin{align*}
    \frac{\partial Z^{d\tildeh^{l-1}_1}}{\partial \hatZ^{\transpose{{(\hatW^{L})}} d\tildeh^{l}_0}} = \psi(Z_0)
\end{align*}
with $\psi$ polynomially bounded. This concludes the induction and thus proves claims $(ii)$, $(iii)$, $(iv)$ and $(v)$ by induction.\\
\end{proof}

\begin{corollary}[Multiplications by the initial weight matrices vanish in IP-LLR at $t=1$]\label{th:z-mult-w-0-vanish}
Consider the IP-LLR parameterization with a positively $p$-homogeneous activation function, and $p \geq 2$. Then for any $l \in [2, L]$, one has:
\begin{align*}
    \begin{cases}
        Z^{W^l(0)x^{l-1}_1} = \scalarlim{\omega}_l Z^{\hatW^lx^{l-1}_1} = 0 \\
        Z^{\transpose{{(W^l(0))}} d\tildeh^{l}_1} = \scalarlim{\omega}_l  Z^{\transpose{{(\hatW^l)}} d\tildeh^{l}_1} = 0
    \end{cases}
\end{align*}
\end{corollary}

\begin{proof}
Those results are actually hidden in the proof of Lemma~\ref{th:ipllr-backward-func-Z0-1} and  come from Lemma~\ref{th:initial-W-vanish}. 
\end{proof}

\subsubsection{The case $t=2$}\label{sec:ipllr-dyn-2}

\begin{lemma}[$Z_0$ in the forward pass of IP-LLR at $t=2$]\label{th:ipllr-forward-func-Z0-2}
Consider the IP-LLR parameterization with a positively $p$-homogeneous activation function, and $p \geq 2$. Then, dropping the dependency of the forward pass on $\xi_2$, one has:
\begin{enumerate}[(i)]
    \item $Z^{h^1_2} = \psi \left(\hatZ^{U^1\xi_0}, \ \hatZ^{U^1\xi_1}, \ \hatZ^{U^1\xi_2}, \ \hatZ^{\transpose{(\hatW^2) } d\tildeh^2_0} \right)$
    
    \item $Z^{h^l_2} = \psi \left( \hatZ^{\hatW^l \tildex^{l-1}_0}, \ \hatZ^{\transpose{(\hatW^{l+1}) } d\tildeh^{l+1}_0} \right)$, \ $l \in [2, L-1]$
    
    \item $Z^{h^L_2} = \psi \left(\hatZ^{U^{L+1}}, \ \hatZ^{\hatW^{L} \tildex^{L-1}_0} \right)$
\end{enumerate}
and 
\begin{enumerate}[(i)]
    \setcounter{enumi}{3}
    \item $\frac{\partial Z^{h^{l-1}_2}}{\hatZ^{\transpose{{(\hatW^l)}} d\tildeh^l_0}}  = \psi \left(Z_0 \right)$, $l \in [2, L]$
    
    \item $\frac{\partial Z^{h^{l}_2}}{\hatZ^{\hatW^l \tildex^{l-1}_0}}  = \psi \left(Z_0 \right)$, $l \in [2, L]$
\end{enumerate}
and 
\begin{enumerate}[(i)]
    \setcounter{enumi}{5}
    \item $Z^{W^l(0)x^{l-1}_2} = 0$, $l \in [2, L]$
\end{enumerate}
and \textbf{all} the different $\psi$ that appear are polynomially bounded. 
\end{lemma}

\begin{proof}
We have 
\begin{align*}
    h^1_2 = U^1 \xi_2 - \eta \chi_0 (\transpose{\xi_0} \xi_2) d\tildeh^1_0 - \eta \chi_1 (\transpose{\xi_1} \xi_2) d\tildeh^1_1
\end{align*}
which gives
\begin{align*}
    Z^{h^1_2} = Z^{U^1 \xi_2} - \eta \scalarlim{\chi}_0 (\transpose{\xi_0} \xi_2) Z^{d\tildeh^1_0} - \eta \scalarlim{\chi}_1 (\transpose{\xi_1} \xi_2) Z^{d\tildeh^1_1}
\end{align*}
By Lemma~\ref{th:ipllr-backward-func-Z0-1} $Z^{d\tildeh^1_1} = \psi ( \hatZ^{U^1 \xi_0}, \ \hatZ^{U^1 \xi_1}, \ \hatZ^{\transpose{(\hatW^2) } d\tildeh^2_0} )$ and we also have $Z^{d\tildeh^1_0} = \psi(Z^{U^1 \xi_0}, \hatZ^{\transpose{(\hatW^2) } d\tildeh^2_0})$ where the different $\psi$ are polynomially bounded, which gives claim $(i)$.
\\\\
We have 
\begin{align*}
    \frac{\partial Z^{h^1_2}}{\partial \hatZ^{\transpose{{(\hatW^2)}} d\tildeh^2_0}} = - \eta \scalarlim{\chi}_0 (\transpose{\xi_0} \xi_2) \frac{\partial Z^{d\tildeh^1_0}}{\partial \hatZ^{\transpose{{(\hatW^2)}} d\tildeh^2_0}}  - \eta \scalarlim{\chi}_1 (\transpose{\xi_1} \xi_2) \frac{\partial Z^{d\tildeh^1_1}}{\partial \hatZ^{\transpose{{(\hatW^2)}} d\tildeh^2_0}} 
\end{align*}
with 
\begin{align*}
    \frac{\partial Z^{d\tildeh^1_0}}{\partial \hatZ^{\transpose{{(\hatW^2)}} d\tildeh^2_0}} = \sigma'(Z^{U^1 \xi_0})
\end{align*}
which is a polynomially bounded function of $Z_0$ and so is $\partial Z^{d\tildeh^1_1} / \partial \hatZ^{\transpose{{(\hatW^2)}} d\tildeh^2_0}$ by Lemma~\ref{th:ipllr-backward-func-Z0-1}. We thus get claim $(iv)$ for $l=2$. 
\\\\
We have
\begin{align*}
    Z^{h^2_2} = \scalarlim{\omega_2} Z^{\hatW^2x^1_2} - \eta \scalarlim{\chi}_0 \mathbb{E}[Z^{\tildex^1_0} Z^{x^1_2}] Z^{d\tildeh^2_0} - \eta \scalarlim{\chi}_1 \mathbb{E}[Z^{x^1_1} Z^{x^1_2}] Z^{d\tildeh^2_1}
\end{align*}
Now $Z^{x^1_2} = \sigma(Z^{h^1_2})$ is a polynomially bounded function of $Z_0$ because $Z^{h^1_2}$ is and $\sigma$ is polynomially bounded. Secondly, we have
\begin{align*}
     \frac{\partial Z^{x^1_2}}{\partial \hatZ^{\transpose{{(\hatW^2)}} d\tildeh^2_0}} =  \frac{\partial Z^{h^1_2}}{\partial \hatZ^{\transpose{{(\hatW^2)}} d\tildeh^2_0}} \sigma'(Z^{h^1_2})
\end{align*}
which is a polynomially bounded function of $Z_0$ by the previous results. By Lemma~\ref{th:initial-W-vanish} we get that $\scalarlim{\omega_2} Z^{\hatW^2x^1_2} = 0$ which gives claim $(vi)$ for $l=2$. In addition, this yields 
\begin{align*}
    Z^{h^2_2} = - \eta \scalarlim{\chi}_0 \mathbb{E}[Z^{\tildex^1_0} Z^{x^1_2}] Z^{d\tildeh^2_0} - \eta \scalarlim{\chi}_1 \mathbb{E}[Z^{x^1_1} Z^{x^1_2}] Z^{d\tildeh^2_1}
\end{align*}
which gives claim $(ii)$ for $l=2$ by the results for the backward passes at time $t=0$ and $t=1$ and because the expectations are finite since the integrands are polynomially bounded functions of $Z_0$, as they are products of such variables by the induction hypothesis. Additionally, we have
\begin{align*}
        \frac{\partial Z^{h^2_2}}{\partial \hatZ^{\transpose{{(\hatW^3(0))}} d\tildeh^3_0}} = - \eta \scalarlim{\chi}_0  \mathbb{E}[Z^{\tildex^1_0} Z^{x^1_2}] \frac{\partial Z^{d\tildeh^2_0}}{\partial \hatZ^{\transpose{{(\hatW^3(0))}} d\tildeh^3_0}}  - \eta \scalarlim{\chi}_1 \mathbb{E}[Z^{x^1_1} Z^{x^1_2}] \frac{\partial Z^{d\tildeh^2_1}}{\partial \hatZ^{\transpose{{(\hatW^3(0))}} d\tildeh^3_0}} 
\end{align*}
and we have 
\begin{align*}
    \frac{\partial Z^{d\tildeh^2_0}}{\partial \hatZ^{\transpose{{(\hatW^3(0))}} d\tildeh^3_0}} = \sigma'(\hatZ^{\hatW^2 \tildex^1_0})
\end{align*}
and $\partial Z^{d\tildeh^2_1} / \partial \hatZ^{\transpose{{(\hatW^3(0))}} d\tildeh^3_0}$ is a polynomially bounded function of $Z_0$ by Lemma~\ref{th:ipllr-backward-func-Z0-1}. Once again, since the expectations are finite, we thus get that
\begin{align*}
    \frac{\partial Z^{h^2_2}}{\partial \hatZ^{\transpose{{(\hatW^3(0))}} d\tildeh^3_0}} = \psi(Z_0)
\end{align*}
with $\psi$ polynomially bounded. A similar reasoning would prove that 
\begin{align*}
    \frac{\partial Z^{h^2_2}}{\partial \hatZ^{\hatW^3(0) \tildex^2_0}} = \psi(Z_0)
\end{align*}
with $\psi$ polynomially bounded because 
\begin{align*}
    \frac{\partial Z^{d\tildeh^1_0}}{\partial \hatZ^{\hatW^3(0) \tildex^2_0}} = \hatZ^{\transpose{{(\hatW^3(0))}} d\tildeh^3_0} \sigma''(\hatZ^{\hatW^2 \tildex^1_0})
\end{align*}
and $\partial Z^{d\tildeh^1_1} / \partial \hatZ^{\hatW^3(0) \tildex^2_0} = \psi(Z_0)$ with $\psi$ polynomially bounded by Lemma~\ref{th:ipllr-backward-func-Z0-1}. 
\\\\
Let $l \in [2, L-1]$ and assume claims $(ii)$, $(iv)$, $(v)$, and $(vi)$ for layer $l$. Then, we have:
\begin{align*}
    Z^{h^{l+1}_2} = \scalarlim{\omega}_{l+1} Z^{\hatW^{l+1} x^l_2} - \eta \scalarlim{\chi}_0 \mathbb{E}[Z^{\tildex^l_0} Z^{x^l_2}] Z^{d\tildeh^{l+1}_0} - \eta \scalarlim{\chi}_1 \mathbb{E}[Z^{x^l_1} Z^{x^l_2}] Z^{d\tildeh^{l+1}_1}
\end{align*}
Now $Z^{x^l_2} = \sigma(Z^{h^l_2})$ is a polynomially bounded function of $Z_0$ because $Z^{h^l_2}$ is by the induction hypothesis and $\sigma$ is polynomially bounded. Secondly, we have
\begin{align*}
     \frac{\partial Z^{x^l_2}}{\partial \hatZ^{\transpose{{(\hatW^{l+1})}} d\tildeh^{l+1}_0}} =  \frac{\partial Z^{h^l_2}}{\partial \hatZ^{\transpose{{(\hatW^{l+1})}} d\tildeh^{l+1}_0}} \sigma'(Z^{h^l_2})
\end{align*}
which is a polynomially bounded function of $Z_0$ by the induction hypothesis. By Lemma~\ref{th:initial-W-vanish} we get that $\scalarlim{\omega_{l+1}} Z^{\hatW^{l+1}x^l_2} = 0$ which gives claim $(vi)$ for layer $l+1$. In addition, this yields 
\begin{align*}
    Z^{h^{l+1}_2} = - \eta \scalarlim{\chi}_0 \mathbb{E}[Z^{\tildex^l_0} Z^{x^l_2}] Z^{d\tildeh^{l+1}_0} - \eta \scalarlim{\chi}_1 \mathbb{E}[Z^{x^1_1} Z^{x^1_2}] Z^{d\tildeh^{l+1}_1}
\end{align*}
which gives claim $(ii)$ for layer $l+1$ by the results for the backward passes at time $t=0$ and $t=1$ and because the expectations are finite since the integrands are polynomially bounded functions of $Z_0$, as they are products of such variables. The only thing that one has to be careful with is that if $l+1 = L$, then $Z^{\tildeh^{l+1}_0} = \psi(Z^{U^{L+1}}, \hatZ^{\hatW^l \tildex^{L-1}_0})$ and $Z^{\tildeh^{l+1}_1} = \psi(Z^{U^{L+1}}, \hatZ^{\hatW^l \tildex^{L-1}_0})$ with both $\psi$ polynomially bounded, which gives claim $(iii)$. Otherwise, if $l+1 \leq L-1$, $Z^{\tildeh^{l+1}_0} = \psi(\hatZ^{\hatW^{l+1} \tildex^{l}_0}, \hatZ^{\transpose{{(\hatW^{l+2}(0))}} d\tildeh^{l+2)}}$ and $Z^{\tildeh^{l+1}_1} = \psi(\hatZ^{\hatW^{l+1} \tildex^{l}_0}, \hatZ^{\transpose{{(\hatW^{l+2}(0))}} d\tildeh^{l+2)}}$ with both $\psi$ polynomially bounded, which gives claim $(ii)$ for layer $l+1$.
\\\\
Now, if $l+1 \leq L-1$, 
\begin{align*}
        \frac{\partial Z^{h^{l+1}_2}}{\partial \hatZ^{\transpose{{(\hatW^{l+2}(0))}} d\tildeh^{l+2}_0}} = - \eta \scalarlim{\chi}_0  \mathbb{E}[Z^{\tildex^l_0} Z^{x^l_2}] \frac{\partial Z^{d\tildeh^{l+1}_0}}{\partial \hatZ^{\transpose{{(\hatW^{l+2}(0))}} d\tildeh^{l+2}_0}}  - \eta \scalarlim{\chi}_1 \mathbb{E}[Z^{x^l_1} Z^{x^l_2}] \frac{\partial Z^{d\tildeh^{l+1}_1}}{\partial \hatZ^{\transpose{{(\hatW^{l+2}(0))}} d\tildeh^{l+2}_0}} 
\end{align*}
and we have 
\begin{align*}
    \frac{\partial Z^{d\tildeh^{l+1}_0}}{\partial \hatZ^{\transpose{{(\hatW^{l+2}(0))}} d\tildeh^{l+2}_0}} = \sigma'(\hatZ^{\hatW^{l+1} \tildex^l_0})
\end{align*}
and $\partial Z^{d\tildeh^{l+1}_1} / \partial \hatZ^{\transpose{{(\hatW^{l+2}(0))}} d\tildeh^{l+2}_0}$ is a polynomially bounded function of $Z_0$ by Lemma~\ref{th:ipllr-backward-func-Z0-1}. Once again, since the expectations are finite, we thus get that
\begin{align*}
    \frac{\partial Z^{h^{l+1}_2}}{\partial \hatZ^{\transpose{{(\hatW^{l+2}(0))}} d\tildeh^{l+2}_0}} = \psi(Z_0)
\end{align*}
with $\psi$ polynomially bounded, which proves claim $(iv)$ for layer $l+1$. A similar reasoning would prove that 
\begin{align*}
    \frac{\partial Z^{h^{l+1}_2}}{\partial \hatZ^{\hatW^{l+1} \tildex^l_0}} = \psi(Z_0)
\end{align*}
with $\psi$ polynomially bounded because 
\begin{align*}
    \frac{\partial Z^{d\tildeh^{l+1}_0}}{\partial \hatZ^{\hatW^{l+2}(0) \tildex^l_0}} =
    \begin{cases}
        \hatZ^{\transpose{{(\hatW^{l+2}(0))}} d\tildeh^{{l+2}}_0} \sigma''(\hatZ^{\hatW^{l+1} \tildex^l_0}) \text{ if } \ l + 1 \leq L-1 \\
        \hatZ^{U^{L+1}} \sigma''(\hatZ^{\hatW^{L} \tildex^{L-1}_0}) \text{ if } \ l + 1 = L
    \end{cases}
\end{align*}
and $\partial Z^{d\tildeh^{l+1}_1} / \partial \hatZ^{\hatW^{l+1} \tildex^l_0} = \psi(Z_0)$ with $\psi$ polynomially bounded by Lemma~\ref{th:ipllr-backward-func-Z0-1}. This proves claim $(v)$ and thus concludes the induction and with it the proof.
\end{proof}

\begin{lemma}[$Z_0$ in the backward pass of IP-LLR at $t=2$]\label{th:ipllr-backward-func-Z0-2}
Consider the IP-LLR parameterization with a positively $p$-homogeneous activation function, and $p \geq 2$. Then, dropping the dependency of the forward and backward passes on $\xi_2$, one has:
\begin{enumerate}[(i)]
    \item $Z^{d\tildex^L_2} = \psi \left( \hatZ^{U^{L+1}}, \ \hatZ^{\hatW^{L} \tildex^{L-1}_0} \right)$, \\ 
    $Z^{d\tildeh^L_2} = \psi \left( \hatZ^{U^{L+1}}, \ \hatZ^{\hatW^{L} \tildex^{L-1}_0} \right)$
    
    \item $Z^{d\tildex^{l-1}_2} = \psi \left( \hatZ^{\hatW^{l-1} \tildex^{l-2}},  \  \hatZ^{\transpose{{(\hatW^l)}} d\tildeh^l_0} \right)$, \\ 
    $Z^{d\tildeh^{l-1}_2} = \psi \left( \hatZ^{\hatW^{l-1} \tildex^{l-2}}, \  \hatZ^{\transpose{{(\hatW^l)}} d\tildeh^l_0} \right)$, $l \in [3, L]$
    
    \item $Z^{d\tildex^1_2} = \psi \left( \hatZ^{U^1 \xi_0}, \ \hatZ^{U^1 \xi_1}, \ \hatZ^{\transpose{(\hatW^2) } d\tildeh^2_0} \right)$, \\
    $Z^{d\tildeh^1_2} = \psi \left( \hatZ^{U^1 \xi_0}, \ \hatZ^{U^1 \xi_1}, \ \hatZ^{U^1 \xi_2}, \ \hatZ^{\transpose{(\hatW^2) } d\tildeh^2_0} \right)$
\end{enumerate}
and 
\begin{enumerate}[(i)]
    \setcounter{enumi}{3}
    \item $\frac{\partial Z^{d\tildeh^l_2}}{\partial \hatZ^{\hatW^l\tildex^{l-1}_0}} = \psi(Z_0)$, $l \in [2, L]$ 
    
    \item $\frac{\partial Z^{d\tildeh^{l-1}_2}}{\partial \hatZ^{\transpose{{(\hatW^{L})}} d\tildeh^{l}_0}} = \psi(Z_0)$, $l \in [2, L]$ 
\end{enumerate}
and
\begin{enumerate}[(i)]
    \setcounter{enumi}{5}
    \item $Z^{\transpose{{(W^l(0))}} d\tildeh^{l}_2} = 0$, $l \in [2, L]$
\end{enumerate}
and \textbf{all} the different $\psi$ that appear are polynomially bounded. 
\end{lemma}

\begin{proof}
We have:
\begin{align*}
    Z^{d\tildex^L_2} = \hatZ^{U^{L+1}} - \eta\scalarlim{\chi}_0 Z^{\tildex^L_0} - \eta\scalarlim{\chi}_1 Z^{x^L_1}
\end{align*}
where $Z^{\tildex^L_0} = \sigma(\hatZ^{\hatW^l \tildex^{L-1}_0})$ and $Z^{x^L_1} = \psi(\hatZ^{U^{L+1}}, \hatZ^{\hatW^l \tildex^{L-1}_0})$ with $\psi$ polynomially bounded by Lemma~\ref{th:ipllr-forward-func-Z0-1}. Combining all this gives
\begin{align*}
    Z^{d\tildex^L_2} = \psi(\hatZ^{U^{L+1}}, \hatZ^{\hatW^l \tildex^{L-1}_0})
\end{align*}
with $\psi$ polynomially bounded since $\sigma$ is also polynomially bounded. Then
\begin{align*}
    Z^{d\tildeh^L_2} = Z^{d\tildex^L_2} \sigma'(Z^{h^L_2})
\end{align*}
and since $Z^{h^L_2} = \psi(\hatZ^{U^{L+1}}, \hatZ^{\hatW^l \tildex^{L-1}_0})$ with $\psi$ polynomially bounded by Lemma \ref{th:ipllr-forward-func-Z0-2}, we get 
\begin{align*}
    Z^{d\tildeh^L_2} = \psi(\hatZ^{U^{L+1}}, \hatZ^{\hatW^l \tildex^{L-1}_0})
\end{align*}
with $\psi$ polynomially bounded since $\sigma'$ is also polynomially bounded. This thus proves claim $(i)$. Now, we have 
\begin{align*}
    \frac{\partial Z^{d\tildeh^L_2}}{\partial \hatZ^{\hatW^l \tildex^{L-1}_0}} =& \frac{\partial Z^{d\tildex^L_2}}{\partial \hatZ^{\hatW^l \tildex^{L-1}_0}} \sigma'(Z^{h^L_2}) + Z^{d\tildex^L_2} \frac{\partial Z^{h^L_2}}{\partial \hatZ^{\hatW^l \tildex^{L-1}_0}} \sigma''(Z^{h^L_2})
\end{align*}
where $Z^{h^L_2}$, $\partial Z^{h^L_2} / \partial \hatZ^{\hatW^l \tildex^{L-1}_0}$, and $Z^{d\tildex^L_2}$ are polynomially bounded functions of $Z_0$ by the previous result and by Lemma~\ref{th:ipllr-forward-func-Z0-2}. We have
\begin{align*}
    \frac{\partial Z^{d\tildex^L_2}}{\partial \hatZ^{\hatW^l \tildex^{L-1}_0}} =& - \eta \scalarlim{\chi}_0 \sigma'(\hatZ^{\hatW^l \tildex^{L-1}_0})  - \eta \scalarlim{\chi}_1 \frac{\partial Z^{h^L_1}}{\partial \hatZ^{\hatW^l \tildex^{L-1}_0}} \sigma'(Z^{h^L_1})
\end{align*}
which is a polynomially bounded function of $Z_0$ since $\sigma'$ is polynomially bounded and by Lemma~\ref{th:ipllr-forward-func-Z0-1}. We thus get 
\begin{align*}
    \frac{\partial Z^{d\tildeh^L_2}}{\partial \hatZ^{\hatW^l \tildex^{L-1}_0}} = \psi(Z_0)
\end{align*}
with $\psi$ polynomially bounded since $\sigma'$ and $\sigma''$ are polynomially bounded. This proves $(iv)$ for $l=L$. 
\\\\
We have:
\begin{align*}
    Z^{d\tildex^{L-1}_2} = \scalarlim{\omega}_L \hatZ^{\transpose{{(\hatW^l)}} d\tildeh^L_2} - \eta\scalarlim{\chi}_0 \mathbb{E}[Z^{d\tildeh^L_0} Z^{d\tildeh^L_2}] Z^{\tildex^{L-1}_0} - \eta\scalarlim{\chi}_1 \mathbb{E}[Z^{d\tildeh^L_1} Z^{d\tildeh^L_2}] Z^{x^{L-1}_1}
\end{align*}
From the previous step we have that both $Z^{d\tildeh^L_2}$ and $\partial Z^{d\tildeh^L_2} / \partial \hatZ^{\hatW^l \tildex^{L-1}_0}$ are polynomially bounded functions of $Z_0$. By Lemma~\ref{th:initial-W-vanish}, this first shows that $\scalarlim{\omega}_L \hatZ^{\transpose{{(\hatW^l)}} d\tildeh^L_2} = 0$, and thus gives $(vi)$ for $l=L$, leading to:
\begin{align*}
    Z^{d\tildex^{L-1}_2} = - \eta\scalarlim{\chi}_0 \mathbb{E}[Z^{d\tildeh^L_0} Z^{d\tildeh^L_2}] Z^{\tildex^{L-1}_0} - \eta\scalarlim{\chi}_1 \mathbb{E}[Z^{d\tildeh^L_1} Z^{d\tildeh^L_2}] Z^{x^{L-1}_1}
\end{align*}
Now $Z^{\tildex^{L-1}_0} = \sigma(\hatZ^{\hatW^{l-1} \tildex^{L-2}_0})$ and by Lemma~\ref{th:ipllr-forward-func-Z0-1}, we also have that $Z^{x^{L-1}_1} = \psi(\hatZ^{\hatW^{l-1} \tildex^{L-2}_0}, \hatZ^{\transpose{{(\hatW^{L})}} d\tildeh^{L}_0})$ with $\psi$ polynomially bounded. As always the expectations are finite by Lemma~\ref{th:Z0-dist-moments} because the integrands are polynomially bounded functions of $Z_0$ as products of such variables. Since $\sigma$ is also polynomially bounded, this gives
\begin{align*}
    Z^{d\tildex^{L-1}_2} = \psi(\hatZ^{\hatW^{l-1} \tildex^{L-1}_0}, \hatZ^{\transpose{{(\hatW^{L})}} d\tildeh^{L}_0})
\end{align*}
Then, we have
\begin{align*}
    Z^{d\tildeh^{L-1}_2} = Z^{d\tildex^{L-1}_2} \sigma'(Z^{h^{L-1}_2})
\end{align*}
and since $Z^{h^{L-1}_2} = \psi(\hatZ^{\hatW^{l-1} \tildex^{L-1}_0}, \hatZ^{\transpose{{(\hatW^{L})}} d\tildeh^{L}_0})$ with $\psi$ polynomially bounded by Lemma \ref{th:ipllr-forward-func-Z0-2}, we get 
\begin{align*}
    Z^{d\tildeh^{L-1}_2} = \psi(\hatZ^{\hatW^{l-1} \tildex^{L-1}_0}, \hatZ^{\transpose{{(\hatW^{L})}} d\tildeh^{L}_0})
\end{align*}
with $\psi$ polynomially bounded since $\sigma'$ is also polynomially bounded. This thus proves claim $(ii)$ for $l=L-1$. Now, let $Z \in \{\hatZ^{\hatW^{l-1} \tildex^{L-1}_0}, \hatZ^{\transpose{{(\hatW^{L})}} d\tildeh^{L}_0}\}$. We have 
\begin{align*}
    \frac{\partial Z^{d\tildeh^{L-1}_2}}{\partial Z} =& \frac{\partial Z^{d\tildex^{L-1}_2}}{\partial Z} \sigma'(Z^{h^{L-1}_2}) + Z^{d\tildex^{L-1}_2} \frac{\partial Z^{h^{L-1}_2}}{\partial Z} \sigma''(Z^{h^{L-1}_2})
\end{align*}
where $Z^{h^{L-1}_2}$, $\partial Z^{h^{L-1}_2} / \partial Z$, and $Z^{d\tildex^{L-1}_2}$ are polynomially bounded functions of $Z_0$ by the previous result and by Lemma~\ref{th:ipllr-forward-func-Z0-2}. We have
\begin{align*}
    \frac{\partial Z^{d\tildex^{L-1}_2}}{\partial Z} =& - \eta \scalarlim{\chi}_0 \mathbb{E}[Z^{d\tildeh^L_0} Z^{d\tildeh^L_2}] \frac{\partial Z^{\tildeh^{L-1}_0}}{\partial Z} \sigma'(Z^{\tildeh^{L-1}_0}) - \eta \scalarlim{\chi}_1 \mathbb{E}[Z^{d\tildeh^L_1} Z^{d\tildeh^L_2}] \frac{\partial Z^{h^{L-1}_1}}{\partial Z} \sigma'(Z^{h^{L-1}_1})
\end{align*}
which is a polynomially bounded function of $Z_0$ since $\sigma'$ is polynomially bounded and by Lemma~\ref{th:ipllr-forward-func-Z0-1}.\\\\
For both possible values of $Z$, the expression of $\partial Z^{\tildeh^{L-1}_0} / \partial Z$ is easy to obtain and is a polynomially bounded function of $Z_0$ (this has actually already been shown for the proofs at time $t=1$), and $Z^{h^{L-1}_1} / \partial Z = \psi(Z_0)$  with $\psi$ polynomially bounded by Lemma~\ref{th:ipllr-forward-func-Z0-1}. Since the expectations are finite and $\sigma'$ is polynomially bounded, we get 
\begin{align*}
    \frac{\partial Z^{d\tildex^{L-1}_2}}{\partial Z} = \psi(Z_0)
\end{align*}
with $\psi$ polynomially bounded and thus 
\begin{align*}
    \frac{\partial Z^{d\tildeh^{L-1}_2}}{\partial Z} = \psi(Z_0)
\end{align*}
with $\psi$ polynomially bounded. This proves $(iv)$ and $(v)$ for $l=L-1$. 
\\\\
Let $l \in [2 , L-1]$, and assume claims $(ii)$, $(iv)$, $(v)$, are true at layer $l$ and claim $(vi)$ is true at layer $l+1$. We have:
\begin{align*}
    Z^{d\tildex^{l-1}_2} = \scalarlim{\omega}_l \hatZ^{\transpose{{(\hatW^l)}} d\tildeh^l_2} - \eta\scalarlim{\chi}_0 \mathbb{E}[Z^{d\tildeh^l_0} Z^{d\tildeh^l_2}] Z^{\tildex^{l-1}_0} - \eta\scalarlim{\chi}_1 \mathbb{E}[Z^{d\tildeh^l_1} Z^{d\tildeh^l_2}] Z^{x^{l-1}_1}
\end{align*}
From the induction hypothesis we have that both $Z^{d\tildeh^l_2}$ and $\partial Z^{d\tildeh^l_2} / \partial \hatZ^{\hatW^l \tildex^{l-1}_0}$ are polynomially bounded functions of $Z_0$. By Lemma~\ref{th:initial-W-vanish}, this first shows that $\scalarlim{\omega}_l \hatZ^{\transpose{{(\hatW^l)}} d\tildeh^l_2} = 0$, and thus gives $(vi)$ for layer $l$, leading to:
\begin{align*}
    Z^{d\tildex^{l-1}_2} = - \eta\scalarlim{\chi}_0 \mathbb{E}[Z^{d\tildeh^l_0} Z^{d\tildeh^l_2}] Z^{\tildex^{l-1}_0} - \eta\scalarlim{\chi}_1 \mathbb{E}[Z^{d\tildeh^l_1} Z^{d\tildeh^l_2}] Z^{x^{l-1}_1}
\end{align*}
Now, if $l-1 \geq 2$, $Z^{\tildex^{l-1}_0} = \sigma(\hatZ^{\hatW^{l-1} \tildex^{l-2}_0})$ and by Lemma~\ref{th:ipllr-forward-func-Z0-1}, we also have that $Z^{x^{l-1}_1} = \psi(\hatZ^{\hatW^{l-1} \tildex^{l-2}_0}, \hatZ^{\transpose{{(\hatW^{L})}} d\tildeh^{l}_0})$ with $\psi$ polynomially bounded. On the other hand, if $l-1 = 1$, we have $Z^{\tildex^{l-1}_0} = \sigma(\hatZ^{U^1 \xi_0})$ and we also have that $Z^{x^{l-1}_1} = \psi(\hatZ^{U^1 \xi_0}, \hatZ^{U^1 \xi_1}, \hatZ^{\transpose{{(\hatW^{2})}} d\tildeh^{2}_0})$ by Lemma~\ref{th:ipllr-forward-func-Z0-1}. As always the expectations are finite by Lemma~\ref{th:Z0-dist-moments} because the integrands are polynomially bounded functions of $Z_0$ as products of such variables. Since $\sigma$ is also polynomially bounded, this gives
\begin{align*}
    Z^{d\tildex^{l-1}_2} = 
    \begin{cases}
        \psi(\hatZ^{\hatW^{l-1} \tildex^{l-2}_0}, \hatZ^{\transpose{{(\hatW^{L})}} d\tildeh^{l}_0}) \text{ if } l-1 \geq 2 \\
        \psi(\hatZ^{U^1 \xi_0}, \hatZ^{U^1 \xi_1}, \hatZ^{\transpose{{(\hatW^{2})}} d\tildeh^{2}_0}) \text{ if } l-1 = 1
    \end{cases}
\end{align*}
Since $Z^{d\tildeh^{l-1}_2} = Z^{d\tildex^{l-1}_2} \sigma'(Z^{h^{l-1}_2})$, by Lemma~\ref{th:ipllr-forward-func-Z0-1} we get
\begin{align*}
    Z^{d\tildeh^{l-1}_2} = 
    \begin{cases}
        \psi(\hatZ^{\hatW^{l-1} \tildex^{l-2}_0}, \hatZ^{\transpose{{(\hatW^{L})}} d\tildeh^{l}_0}) \text{ if } l-1 \geq 2 \\
        \psi(\hatZ^{U^1 \xi_0}, \hatZ^{U^1 \xi_1}, \hatZ^{U^1 \xi_2}, \hatZ^{\transpose{{(\hatW^{2})}} d\tildeh^{2}_0}) \text{ if } l-1 = 1
    \end{cases}
\end{align*}
This gives claim $(ii)$ for layer $l-1$ and claim $(iii)$ for the case when $l-1=1$. Now, let $Z \in \{\hatZ^{\hatW^{l-1} \tildex^{l-2}_0}, \hatZ^{\transpose{{(\hatW^{L})}} d\tildeh^{l}_0}\}$. We have 
\begin{align*}
    \frac{\partial Z^{d\tildeh^{l-1}_2}}{\partial Z} =& \frac{\partial Z^{d\tildex^{l-1}_2}}{\partial Z} \sigma'(Z^{h^{l-1}_2}) + Z^{d\tildex^{l-1}_2} \frac{\partial Z^{h^{l-1}_2}}{\partial Z} \sigma''(Z^{h^{l-1}_2})
\end{align*}
where $Z^{h^{l-1}_2}$ and $Z^{d\tildex^{l-1}_2}$ are polynomially bounded functions of $Z_0$ by the previous result and by Lemma~\ref{th:ipllr-forward-func-Z0-2}. Also by Lemma~\ref{th:ipllr-forward-func-Z0-2}, we have
\begin{align*}
    \frac{\partial Z^{h^{l-1}_2}}{\partial Z} = 
    \begin{cases}
        0 \ \text{ if } \ l-1=1 \text{ and } Z=\hatZ^{\hatW^2 \tildex^1_0} \\
        \psi(Z_0) \ \text{ otherwise}
    \end{cases}
\end{align*}
with $\psi$ polynomially bounded. In any case, $\partial Z^{h^{l-1}_2} / \partial Z$ is a polynomially bounded function of $Z_0$. On the other hand, we have
\begin{align*}
    \frac{\partial Z^{d\tildex^{l-1}_2}}{\partial Z} =& - \eta \scalarlim{\chi}_0 \mathbb{E}[Z^{d\tildeh^l_0} Z^{d\tildeh^l_2}] \frac{\partial Z^{\tildeh^{l-1}_0}}{\partial Z} \sigma'(Z^{\tildeh^{l-1}_0}) - \eta \scalarlim{\chi}_1 \mathbb{E}[Z^{d\tildeh^l_1} Z^{d\tildeh^l_2}] \frac{\partial Z^{h^{l-1}_1}}{\partial Z} \sigma'(Z^{h^{l-1}_1})
\end{align*}
For both possible values of $Z$, $\partial Z^{\tildeh^{l-1}_0} / \partial Z$ has an easy expression and is a polynomially bounded function of $Z_0$ (essentially because $\sigma$ and its derivatives are polynomially bounded). On the other hand, $\partial Z^{\tildeh^{l-1}_1} / \partial Z$ is a polynomially bounded function of $Z_0$ by Lemma~\ref{th:ipllr-backward-func-Z0-1}. $\sigma'$ is polynomially bounded, and the expectations are finite by Lemma~\ref{th:Z0-dist-moments} since the integrands are polynomially bounded functions of $Z_0$ as they are products of such functions by Lemma~\ref{th:ipllr-backward-func-Z0-1} and by the induction hypothesis. We thus get that 
\begin{align*}
    \frac{\partial Z^{d\tildex^{l-1}_2}}{\partial Z} =  \psi(Z_0)
\end{align*}
with $\psi$ polynomially bounded. We thus have that:
\begin{align*}
    \frac{\partial Z^{d\tildeh^{l-1}_2}}{\partial Z} =  \psi(Z_0)
\end{align*}
with $\psi$ polynomially bounded, which proves claims $(iv)$ and $(v)$ at layer $l-1$. This thus concludes the induction, and with it the proof.
\end{proof}

\subsubsection{The case $t \geq 2$}
We have now treated the base case $t=2$ and are thus equipped to do the induction for $t \geq 2$. To make things easier we first introduce some equations. Let $t \geq 2$, we define the following assertions, where the different $\psi$ appearing are assumed to be polynomially bounded:\\
\textbf{Forward pass at time $t$}:
\begin{align}\label{eq:ipllr-Z0-forward-1-t}
    (i) \ \ Z^{h^1_t} = \psi \left(\hatZ^{U^1\xi_0}, \ \ldots, \ \hatZ^{U^1 \xi_t}, \ \hatZ^{\transpose{(\hatW^2) } d\tildeh^2_0} \right)
\end{align}
For $l \in [2, L]$,
\begin{align}\label{eq:ipllr-Z0-forward-l-t}
    (i) \ \ Z^{h^l_t} = \psi \left( \hatZ^{\hatW^l \tildex^{l-1}_0}, \ \hatZ^{\transpose{(\hatW^{l+1}) } d\tildeh^{l+1}_0} \right)
\end{align}
\begin{align}\label{eq:ipllr-Z0-forward-L-t}
   (iii) \ \ Z^{h^L_t} = \psi \left(\hatZ^{U^{L+1}}, \ \hatZ^{\hatW^{L} \tildex^{L-1}_0} \right)
\end{align}
For $l \in [2, L]$,
\begin{align}\label{eq:ipllr-Z0-forward-partial-t}
   (iv) \ \ \frac{\partial Z^{h^{l-1}_t}}{\hatZ^{\transpose{{(\hatW^l)}} d\tildeh^l_0}}  = \psi \left(Z_0 \right)
\end{align}
\begin{align}\label{eq:ipllr-Z0-forward-partial-transpose-t}
   (v) \ \ \frac{\partial Z^{h^{l}_t}}{\hatZ^{\hatW^l \tildex^{l-1}_0}}  = \psi \left(Z_0 \right)
\end{align}
\begin{align}\label{eq:ipllr-Z0-forward-mutl-W-equals-0}
   (vi) \ \ Z^{W^l(0)x^{l-1}_t} = 0
\end{align}
\\\\
\textbf{Backward pass at time $t$}:
\begin{align}\label{eq:ipllr-Z0-backward-L-t}
    (i1) \ \ Z^{d\tildex^L_t} = \psi \left( \hatZ^{U^{L+1}}, \ \hatZ^{\hatW^{L} \tildex^{L-1}_0} \right) \\
    (i2) \ \ Z^{d\tildeh^L_t} = \psi \left( \hatZ^{U^{L+1}}, \ \hatZ^{\hatW^{L} \tildex^{L-1}_0} \right)
\end{align}
For $l \in [3, L]$,
\begin{align}\label{eq:ipllr-Z0-backward-l-t}
    (ii1) \ \ Z^{d\tildex^{l-1}_t} = \psi \left( \hatZ^{\hatW^{l-1} \tildex^{l-2}},  \  \hatZ^{\transpose{{(\hatW^l)}} d\tildeh^l_0} \right) \\
    (ii2) \ \ Z^{d\tildeh^{l-1}_t} = \psi \left( \hatZ^{\hatW^{l-1} \tildex^{l-2}}, \  \hatZ^{\transpose{{(\hatW^l)}} d\tildeh^l_0} \right)
\end{align}
\begin{align}\label{eq:ipllr-Z0-backward-1-t}
    (iii1) \ \ Z^{d\tildex^1_t} = \psi \left( \hatZ^{U^1 \xi_0}, \ \ldots, \ \hatZ^{U^1 \xi_{t-1}}, \ \hatZ^{\transpose{(\hatW^2) } d\tildeh^2_0} \right) \\
    (iii2) \ \ Z^{d\tildeh^1_t} = \psi \left( \hatZ^{U^1 \xi_0}, \ \ldots, \, \hatZ^{U^1 \xi_t}, \ \hatZ^{\transpose{(\hatW^2) } d\tildeh^2_0} \right)
\end{align}
For $l \in [2, L]$,
\begin{align}\label{eq:ipllr-Z0-backward-partial-t}
   (iv) \ \ \frac{\partial Z^{d\tildeh^l_t}}{\partial \hatZ^{\hatW^l\tildex^{l-1}_0}} = \psi(Z_0)
\end{align}
\begin{align}\label{eq:ipllr-Z0-backward-partial-transpose-t}
    (v) \ \ \frac{\partial Z^{d\tildeh^{l-1}_t}}{\partial \hatZ^{\transpose{{(\hatW^{L})}} d\tildeh^{l}_0}} = \psi(Z_0)
\end{align}
\begin{align}\label{eq:ipllr-Z0-backward-mutl-W-equals-0}
   (vi) \ \ Z^{\transpose{{(W^l(0))}} d\tildeh^{l}_t} = 0
\end{align}
Note that we have proved in Appendix~\ref{sec:ipllr-dyn-2} that all the assertions above hold for $t=2$. Our goal is now to show by induction that they hold for any $t \geq 2$. For this we prove the following two lemmas. The proofs will essentially follow exactly the same pattern as for $t=2$, the only difference is that the formulas will involve more terms, but since any finite sum of polynomially bounded functions is polynomially bounded, we will get the same results. Before proving the lemmas, we introduce the following quantities for $0 \leq s < t$:\\
For $l \in [2, L]$
\begin{align}
    \gamma^f_{s,t,l} :=
    \begin{cases}
        \mathbb{E}[Z^{\tildex^{l-1}_0} Z^{x^{l-1}_t}] \ \text{ if } \ s=0 \\
        \mathbb{E}[Z^{x^{l-1}_s} Z^{x^{l-1}_t}] \ \text{ otherwise}
    \end{cases}
\end{align}
For $l \in [1, L-1]$
\begin{align}
        \gamma^b_{s,t,l} := \mathbb{E}[Z^{d\tildeh^{l+1}_s} Z^{d\tildeh^{l+1}_t}]
\end{align}
$\gamma^f_{s,t,l}$ (\textit{resp.} $\gamma^b_{s,t,l}$) will appear when expressing the variables of the $l$-th layer at time $t$ in the forward (\textit{resp.} backward) pass. We will show in the proofs that as for $t=1$ and $t=2$, those expectations are finite by Lemma~\ref{th:Z0-dist-moments} because the integrands are polynomially bounded functions of $Z_0$ as they are products of such variables. 

\begin{lemma}[Induction step in IP-LLR, forward pass]\label{th:ipllr-induct-vansih-forward}
Consider the IP-LLR parameterization with a positively $p$-homogeneous activation function, and $p \geq 2$. Let $t \geq 2$, and assume that all of the assertions of Equation~\eqref{eq:ipllr-Z0-forward-1-t} up until Equation~\eqref{eq:ipllr-Z0-backward-mutl-W-equals-0} hold for every time step $s \in [2, t]$. Then, the assertions of the forward pass, \ie from Equation~\eqref{eq:ipllr-Z0-forward-1-t} up until Equation~\eqref{eq:ipllr-Z0-forward-mutl-W-equals-0}, hold at time $t+1$.
\end{lemma}

\begin{proof}
We follow the proof of Lemma~\ref{th:ipllr-forward-func-Z0-2}. By \autoref{th:z-forward-ipllr-t}, we have
\begin{align*}
    Z^{h^1_{t+1}} = Z^{U^1 \xi_{t+1}} - \eta \sum_{s=0}^t \scalarlim{\chi}_s (\transpose{\xi_s} \xi_{t+1}) Z^{d\tildeh^1_s}
\end{align*}
By Lemma~\ref{th:ipllr-backward-func-Z0-1} $Z^{d\tildeh^1_1} = \psi ( \hatZ^{U^1 \xi_0}, \ \hatZ^{U^1 \xi_1}, \ \hatZ^{\transpose{(\hatW^2) } d\tildeh^2_0} )$ and and by assumption we also have $Z^{d\tildeh^1_s} = \psi(Z^{U^1 \xi_0}, \ldots, Z^{U^1 \xi_s}, \hatZ^{\transpose{(\hatW^2) } d\tildeh^2_0})$ where the different $\psi$ are polynomially bounded, which gives claim $(i)$ at time $t+1$.
\\\\
We have 
\begin{align*}
    \frac{\partial Z^{h^1_{t+1}}}{\partial \hatZ^{\transpose{{(\hatW^2)}} d\tildeh^2_0}} = - \eta \sum_{s=0}^t \scalarlim{\chi}_s (\transpose{\xi_s} \xi_{t+1}) \frac{\partial Z^{d\tildeh^1_s}}{\partial \hatZ^{\transpose{{(\hatW^2)}} d\tildeh^2_0}}
\end{align*}
with 
\begin{align*}
    \frac{\partial Z^{d\tildeh^1_0}}{\partial \hatZ^{\transpose{{(\hatW^2)}} d\tildeh^2_0}} = \sigma'(Z^{U^1 \xi_0})
\end{align*}
which is a polynomially bounded function of $Z_0$ and so is $\partial Z^{d\tildeh^1_1} / \partial \hatZ^{\transpose{{(\hatW^2)}} d\tildeh^2_0}$ by Lemma~\ref{th:ipllr-backward-func-Z0-1}. In addition, by assumption, for $s \in [2, t]$, $\partial Z^{d\tildeh^1_s} / \partial \hatZ^{\transpose{{(\hatW^2)}} d\tildeh^2_0} = \psi(Z_0)$ with $\psi$ polynomially bounded. We thus get claim $(iv)$ for $l=2$ at time $t+1$. 
\\\\
We have by \autoref{th:z-forward-ipllr-t}
\begin{align*}
    Z^{h^2_{t+1}} = \scalarlim{\omega_2} Z^{\hatW^2x^1_{t+1}} - \eta \sum_{s=0}^t \scalarlim{\chi}_s \gamma^f_{s, t+1, 2} Z^{d\tildeh^2_s}
\end{align*}
Now $Z^{x^1_{t+1}} = \sigma(Z^{h^1_{t+1)}}$ is a polynomially bounded function of $Z_0$ because $Z^{h^1_{t+1}}$ is and $\sigma$ is polynomially bounded. Secondly, we have
\begin{align*}
     \frac{\partial Z^{x^1_{t+1}}}{\partial \hatZ^{\transpose{{(\hatW^2)}} d\tildeh^2_0}} =  \frac{\partial Z^{h^1_{t+1}}}{\partial \hatZ^{\transpose{{(\hatW^2)}} d\tildeh^2_0}} \sigma'(Z^{h^1_{t+1)}}
\end{align*}
which is a polynomially bounded function of $Z_0$ by the previous results and because $\sigma'$ is polynomially bounded. By Lemma~\ref{th:initial-W-vanish} we get that $\scalarlim{\omega_2} Z^{\hatW^2x^1_{t+1}} = 0$ which gives claim $(vi)$ for $l=2$ at time $t+1$. In addition, this yields 
\begin{align*}
    Z^{h^2_{t+1}} = - \eta \scalarlim{\chi}_0 \gamma^f_{0, t+1, 2} Z^{d\tildeh^2_0} - \eta \sum_{s=0}^t \scalarlim{\chi}_s \gamma^f_{s, t+1, 2} Z^{d\tildeh^2_s}
\end{align*}
The expectations defining the $\gamma^f$ are finite by Lemma~\ref{th:Z0-dist-moments} since the integrands are polynomially bounded functions of $Z_0$, as they are products of such variables by the previous result on $Z^{x^1_{t+1}}$ and by the assumption. 
This gives claim $(ii)$ for $l=2$ by the results for the backward passes at time $t=0$ and $t=1$ and by the assumptions. Let $Z \in \{\hatZ^{\hatW^2 \tildex^1_0}, \hatZ^{\transpose{{(\hatW^3)}} d\tildeh^3_0} \}$. We have
\begin{align*}
        \frac{\partial Z^{h^2_{t+1}}}{\partial Z} =- \eta \sum_{s=0}^t \scalarlim{\chi}_s  \gamma^f_{s, t+1, 2} \frac{\partial Z^{d\tildeh^2_{s}}}{\partial Z} 
\end{align*}
$\partial Z^{d\tildeh^2_0} / \partial Z$ has a simple expression and is a polynomially bounded function of $Z_0$. Additionally, by the results of the backward pass for $t=1$, and by assumption, for $s \in [1, t]$, $\partial Z^{d\tildeh^2_s} / \partial Z = \psi(Z_0)$ with $\psi$ polynomially bounded. Since the $\gamma^f$ are finite, we thus get
\begin{align*}
    \frac{\partial Z^{h^2_{t+1}}}{\partial Z} = \psi(Z_0)
\end{align*}
with $\psi$ polynomially bounded. This gives claims $(iv)$ and $(v)$ at time $t+1$. 
\\\\
Let $l \in [2, L-1]$ and assume claims $(ii)$, $(iv)$, $(v)$, and $(vi)$ for layer $l$ at time $t+1$. Then, by \autoref{th:z-forward-ipllr-t} we have:
\begin{align*}
    Z^{h^{l+1}_{t+1}} = \scalarlim{\omega}_{l+1} Z^{\hatW^{l+1} x^l_{t+1}} - \eta \sum_{s=0}^t \scalarlim{\chi}_s \gamma^f_{s, t+1, l+1} Z^{d\tildeh^{l+1}_s}
\end{align*}
Now $Z^{x^l_{t+1}} = \sigma(Z^{h^l_{t+1}})$ is a polynomially bounded function of $Z_0$ because $Z^{h^l_{t+1}}$ is by the induction hypothesis and $\sigma$ is polynomially bounded. Secondly, we have
\begin{align*}
     \frac{\partial Z^{x^l_{t+1}}}{\partial \hatZ^{\transpose{{(\hatW^{l+1})}} d\tildeh^{l+1}_0}} =  \frac{\partial Z^{h^l_{t+1}}}{\partial \hatZ^{\transpose{{(\hatW^{l+1})}} d\tildeh^{l+1}_0}} \sigma'(Z^{h^l_{t+1}})
\end{align*}
which is a polynomially bounded function of $Z_0$ by the induction hypothesis. By Lemma~\ref{th:initial-W-vanish} we get that $\scalarlim{\omega_{l+1}} Z^{\hatW^2x^l_{t+1}} = 0$ which gives claim $(vi)$ for layer $l+1$ at time $t+1$. In addition, this yields 
\begin{align*}
    Z^{h^{l+1}_{t+1}} = - \eta \sum_{s=0}^t \scalarlim{\chi}_s \gamma^f_{s, t+1, l+1} Z^{d\tildeh^{l+1}_s}
\end{align*}
The expectations defining the $\gamma^f$ are finite by Lemma~\ref{th:Z0-dist-moments} since the integrands are polynomially bounded functions of $Z_0$, as they are products of such variables by the assumption and by the induction hypothesis. If $l+1 = L$, we have, for any $s \in [0, s]$, $Z^{d\tildeh^{l+1}_s} = \psi(\hatZ^{U^{L+1}}, \hatZ^{\hatW^l \tildex^{L-1}_0})$ with $\psi$ polynomially bounded, which shows 
\begin{align*}
    Z^{h^{l+1}_{t+1}} = \psi(\hatZ^{U^{L+1}}, \hatZ^{\hatW^l \tildex^{L-1}_0})
\end{align*}
with $\psi$ polynomially bounded, which gives claim $(iii)$. If $l+1 \leq L-1$, for any $s \in [0, s]$, $Z^{d\tildeh^{l+1}_s} = \psi(\hatZ^{\hatW^{l+1} \tildex^{l}_0}, \hatZ^{\transpose{{(\hatW^{l+2}()}} d\tildeh^{l+2}_0})$ with $\psi$ polynomially bounded, which shows 
\begin{align*}
    Z^{h^{l+1}_{t+1}} = \psi(\hatZ^{\hatW^{l+1} \tildex^{l}_0}, \hatZ^{\transpose{{(\hatW^{l+2})}} d\tildeh^{l+2}_0})
\end{align*}
$\psi$ polynomially bounded, which shows claim $(ii)$ at layer $l+1$ for time $t+1$. Let $Z \in \{\hatZ^{\hatW^{l+1} \tildex^l_0}, \hatZ^{\transpose{{(\hatW^{l+2})}} d\tildeh^{l+2}_0} \}$. Note that the second value is only valid if $l+1 \leq L-1$. Whenever $Z$ is well-defined, we have
\begin{align*}
    \frac{\partial Z^{h^{l+1}_{t+1}}}{\partial Z} = -\eta \sum_{s=0}^t \scalarlim{\chi}_s \gamma^f_{s, t+1, l+1} \frac{\partial Z^{d\tildeh^{l+1}_s}}{\partial Z}
\end{align*}
For both possible values of $Z$, $\partial Z^{d\tildeh^{l+1}_0} / \partial Z$ has a simple expression and is a polynomially bounded function of $Z_0$. $Z^{d\tildeh^{l+1}_1} / \partial Z$ is a polynomially bounded function of $Z_0$ by the results of the backward pass at time $t=1$ (Lemma~\ref{th:ipllr-backward-func-Z0-1}), and finally for  $s \in [2, t]$, $Z^{d\tildeh^{l+1}_1} / \partial Z = \psi(Z_0)$ with $\psi$ polynomially bounded by assumption. Since the $\gamma^f$ are finite, this gives
\begin{align*}
    \frac{\partial Z^{h^{l+1}_{t+1}}}{\partial Z} = \psi(Z_0)
\end{align*}
with $\psi$ polynomially bounded. This proves claim $(iv)$ and $(v)$ for layer $l+1$ at time $t+1$, and thus concludes the induction on $l$ and with it the proof.
\end{proof}

\begin{lemma}[Induction step in IP-LLR, backward pass]\label{th:ipllr-induct-vansih-backward}
Consider the IP-LLR parameterization with a positively $p$-homogeneous activation function, and $p \geq 2$. Let $t \geq 2$, and assume that all of the assertions of Equation~\eqref{eq:ipllr-Z0-forward-1-t} up until Equation~\eqref{eq:ipllr-Z0-backward-mutl-W-equals-0} for every time step $s \in [2, t]$. Additionally assume that the assertions of the forward pass, \ie from Equation~\eqref{eq:ipllr-Z0-forward-1-t} up until Equation~\eqref{eq:ipllr-Z0-forward-mutl-W-equals-0}, hold at time $t+1$. Then, the assertions of the backward pass, \ie from Equation~\eqref{eq:ipllr-Z0-backward-L-t} up until Equation~\eqref{eq:ipllr-Z0-backward-mutl-W-equals-0}, hold at time $t+1$.
\end{lemma}

\begin{proof}
We follow the proof of Lemma~\ref{th:ipllr-backward-func-Z0-2}. We have:
\begin{align*}
    Z^{d\tildex^L_{t+1}} = \hatZ^{U^{L+1}} - \eta\scalarlim{\chi}_0 Z^{\tildex^L_0} - \eta \sum_{s=1}^t \scalarlim{\chi}_s Z^{x^L_t}
\end{align*}
where $Z^{\tildex^L_0} = \sigma(\hatZ^{\hatW^l \tildex^{L-1}_0})$, $Z^{x^L_1} = \psi(\hatZ^{U^{L+1}}, \hatZ^{\hatW^l \tildex^{L-1}_0})$ with $\psi$ polynomially bounded by Lemma~\ref{th:ipllr-forward-func-Z0-1} and for $s \in [2, t]$, $Z^{x^L_s} = \psi(\hatZ^{U^{L+1}}, \hatZ^{\hatW^l \tildex^{L-1}_0})$ with $\psi$ polynomially bounded by assumption. Combining all this gives
\begin{align*}
    Z^{d\tildex^L_{t+1}} = \psi(\hatZ^{U^{L+1}}, \hatZ^{\hatW^l \tildex^{L-1}_0})
\end{align*}
with $\psi$ polynomially bounded since $\sigma$ is also polynomially bounded. Then
\begin{align*}
    Z^{d\tildeh^L_{t+1}} = Z^{d\tildex^L_{t+1}} \sigma'(Z^{h^L_{t+1}})
\end{align*}
and since $Z^{h^L_{t+1}} = \psi(\hatZ^{U^{L+1}}, \hatZ^{\hatW^l \tildex^{L-1}_0})$ with $\psi$ polynomially bounded by assumption, we get 
\begin{align*}
    Z^{d\tildeh^L_{t+1}} = \psi(\hatZ^{U^{L+1}}, \hatZ^{\hatW^l \tildex^{L-1}_0})
\end{align*}
with $\psi$ polynomially bounded since $\sigma'$ is also polynomially bounded. This thus proves claim $(i)$. Now, we have 
\begin{align*}
    \frac{\partial Z^{d\tildeh^L_{t+1}}}{\partial \hatZ^{\hatW^l \tildex^{L-1}_0}} =& \frac{\partial Z^{d\tildex^L_{t+1}}}{\partial \hatZ^{\hatW^l \tildex^{L-1}_0}} \sigma'(Z^{h^L_{t+1}}) + Z^{d\tildex^L_{t+1}} \frac{\partial Z^{h^L_{t+1}}}{\partial \hatZ^{\hatW^l \tildex^{L-1}_0}} \sigma''(Z^{h^L_{t+1}})
\end{align*}
where $Z^{h^L_{t+1}}$, $\partial Z^{h^L_{t+1}} / \partial \hatZ^{\hatW^l \tildex^{L-1}_0}$, and $Z^{d\tildex^L_{t+1}}$ are polynomially bounded functions of $Z_0$ by assumption and by the previous result on $Z^{d\tildex^L_{t+1}}$. Additionally, we have
\begin{align*}
    \frac{\partial Z^{d\tildex^L_{t+1}}}{\partial \hatZ^{\hatW^l \tildex^{L-1}_0}} =& - \eta \scalarlim{\chi}_0 \sigma'(\hatZ^{\hatW^l \tildex^{L-1}_0})  - \eta \sum_{s=1}^t \scalarlim{\chi}_s \frac{\partial Z^{h^L_s}}{\partial \hatZ^{\hatW^l \tildex^{L-1}_0}} \sigma'(Z^{h^L_s})
\end{align*}
$\sigma'$ is polynomially bounded and by the results of the forward pass at $t=1$ (Lemma~\ref{th:ipllr-forward-func-Z0-1}) $\partial Z^{h^L_1} / \partial \hatZ^{\hatW^l \tildex^{L-1}_0} = \psi(Z_0)$ with $\psi$ polynomially bounded. In addition, by assumption, for any $s \in [2, t]$, $\partial Z^{h^L_s} / \partial \hatZ^{\hatW^l \tildex^{L-1}_0} = \psi(Z_0)$ with $\psi$ polynomially bounded. This thus gives
\begin{align*}
    \frac{\partial Z^{d\tildex^L_{t+1}}}{\partial \hatZ^{\hatW^l \tildex^{L-1}_0}} = \psi(Z_0)
\end{align*}
with $\psi$ polynomially bounded, and thus
\begin{align*}
    \frac{\partial Z^{d\tildeh^L_{t+1}}}{\partial \hatZ^{\hatW^l \tildex^{L-1}_0}} = \psi(Z_0)
\end{align*}
with $\psi$ polynomially bounded since $\sigma'$ and $\sigma''$ are polynomially bounded. This proves $(iv)$ for $l=L$ at time $t+1$. 
\\\\
We have:
\begin{align*}
    Z^{d\tildex^{L-1}_{t+1}} = \scalarlim{\omega}_L \hatZ^{\transpose{{(\hatW^l)}} d\tildeh^L_{t+1}} - \eta\scalarlim{\chi}_0 \gamma^b_{0, t+1, L-1} Z^{\tildex^{L-1}_0} - \eta \sum_{s=1}^t \scalarlim{\chi}_s  \gamma^b_{s, t+1, L-1} Z^{x^{L-1}_s}
\end{align*}
From the previous step we have that both $Z^{d\tildeh^L_{t+1}}$ and $\partial Z^{d\tildeh^L_{t+1}} / \partial \hatZ^{\hatW^l \tildex^{L-1}_0}$ are polynomially bounded functions of $Z_0$. By Lemma~\ref{th:initial-W-vanish}, this first shows that $\scalarlim{\omega}_L \hatZ^{\transpose{{(\hatW^l)}} d\tildeh^L_{t+1}} = 0$, and thus gives $(vi)$ for $l=L$, leading to:
\begin{align*}
    Z^{d\tildex^{L-1}_{t+1}} = - \eta\scalarlim{\chi}_0 \gamma^b_{0, t+1, L} Z^{\tildex^{L-1}_0} - \eta \sum_{s=1}^t \scalarlim{\chi}_s  \gamma^b_{s, t+1, L} Z^{x^{L-1}_s}
\end{align*}
Now $Z^{\tildex^{L-1}_0} = \sigma(\hatZ^{\hatW^{l-1} \tildex^{L-2}_0})$ and we also have that $Z^{x^{L-1}_1} = \psi(\hatZ^{\hatW^{l-1} \tildex^{L-2}_0}, \hatZ^{\transpose{{(\hatW^{L})}} d\tildeh^{L}_0})$ with $\psi$ polynomially bounded by Lemma~\ref{th:ipllr-forward-func-Z0-1}. In addition, we have or any $s \in [2, t]$, we get $Z^{x^{L-1}_s} = \psi(\hatZ^{\hatW^{l-1} \tildex^{L-2}_0}, \hatZ^{\transpose{{(\hatW^{L})}} d\tildeh^{L}_0})$ with $\psi$ polynomially bounded by assumption since it is the case for $Z^{h^{L-1}_s}$ and $\sigma$ is polynomially bounded. As always the expectations defining the $\gamma^b$ are finite by Lemma~\ref{th:Z0-dist-moments} because the integrands are polynomially bounded functions of $Z_0$ as products of such variables by the results for the backward pass at times $t=0$ and $t=1$, by the assumptions and by the previous result on $Z^{d\tildeh^L_{t+1}}$. Since $\sigma$ is also polynomially bounded, this gives
\begin{align*}
    Z^{d\tildex^{L-1}_{t+1}} = \psi(\hatZ^{\hatW^{l-1} \tildex^{L-1}_0}, \hatZ^{\transpose{{(\hatW^{L})}} d\tildeh^{L}_0})
\end{align*}
Then, we have
\begin{align*}
    Z^{d\tildeh^{L-1}_{t+1}} = Z^{d\tildex^{L-1}_{t+1}} \sigma'(Z^{h^{L-1}_{t+1)}}
\end{align*}
and since $Z^{h^{L-1}_{t+1}} = \psi(\hatZ^{\hatW^{l-1} \tildex^{L-1}_0}, \hatZ^{\transpose{{(\hatW^{L})}} d\tildeh^{L}_0})$ with $\psi$ polynomially bounded by assumption
\begin{align*}
    Z^{d\tildeh^{L-1}_2} = \psi(\hatZ^{\hatW^{l-1} \tildex^{L-1}_0}, \hatZ^{\transpose{{(\hatW^{L})}} d\tildeh^{L}_0})
\end{align*}
with $\psi$ polynomially bounded since $\sigma'$ is also polynomially bounded. This thus proves claim $(ii)$ for $l=L-1$. Now, let $Z \in \{\hatZ^{\hatW^{l-1} \tildex^{L-1}_0}, \hatZ^{\transpose{{(\hatW^{L})}} d\tildeh^{L}_0}\}$. We have 
\begin{align*}
    \frac{\partial Z^{d\tildeh^{L-1}_{t+1}}}{\partial Z} =& \frac{\partial Z^{d\tildex^{L-1}_{t+1}}}{\partial Z} \sigma'(Z^{h^{L-1}_{t+1)}} + Z^{d\tildex^{L-1}_{t+1}} \frac{\partial Z^{h^{L-1}_{t+1}}}{\partial Z} \sigma''(Z^{h^{L-1}_{t+1)}}
\end{align*}
where $Z^{h^{L-1}_{t+1}}$, $\partial Z^{h^{L-1}_{t+1}} / \partial Z$, and $Z^{d\tildex^{L-1}_{t+1}}$ are polynomially bounded functions of $Z_0$ by assumption and by the previous result. We have
\begin{align*}
    \frac{\partial Z^{d\tildex^{L-1}_{t+1}}}{\partial Z} =& - \eta \scalarlim{\chi}_0 \gamma^b_{0, t+1, L-1} \frac{\partial Z^{\tildeh^{L-1}_0}}{\partial Z} \sigma'(Z^{\tildeh^{L-1}_0}) - \eta \sum_{s=1}^t \scalarlim{\chi}_s \gamma^b_{s,t+1,L-1} \frac{\partial Z^{h^{L-1}_s}}{\partial Z} \sigma'(Z^{h^{L-1}_s})
\end{align*}
For both possible values of $Z$, $\partial Z^{\tildeh^{L-1}_0} / \partial Z$ has a simple expression and is a polynomially bounded function of $Z_0$, as is $\tildeh^{L-1}_0$. In addition, $Z^{h^{L-1}_1}$ and $\partial Z^{h^{L-1}_1} / \partial Z$ are polynomially bounded functions of $Z_0$ by the results of the forward pass at $t=1$, and finally, for $s \in [2, t]$, $Z^{h^{L-1}_1}$ and $\partial Z^{h^{L-1}_1} / \partial Z$ are polynomially bounded functions of $Z_0$ by assumption. Since the $\gamma^b$ are finite and $\sigma'$ is polynomially bounded, we get 
\begin{align*}
    \frac{\partial Z^{d\tildex^{L-1}_{t+1}}}{\partial Z} = \psi(Z_0)
\end{align*}
with $\psi$ polynomially bounded and thus 
\begin{align*}
    \frac{\partial Z^{d\tildeh^{L-1}_{t+1}}}{\partial Z} = \psi(Z_0)
\end{align*}
with $\psi$ polynomially bounded since $\sigma'$ and $\sigma''$ are polynomially bounded. This proves $(iv)$ and $(v)$ for $l=L-1$. 
\\\\
Let $l \in [2 , L-1]$, and assume claims $(ii)$, $(iv)$, $(v)$, are true at layer $l$ and claim $(vi)$ is true at layer $l+1$. We have:
\begin{align*}
    Z^{d\tildex^{l-1}_{t+1}} = \scalarlim{\omega}_l \hatZ^{\transpose{{(\hatW^l)}} d\tildeh^l_{t+1}} - \eta\scalarlim{\chi}_0 \gamma^b_{0, t+1, l-1} Z^{\tildex^{l-1}_0} - \eta \sum_{s=1}^t \scalarlim{\chi}_s \gamma^b_{s, t+1, l-1} Z^{x^{l-1}_s}
\end{align*}
From the induction hypothesis we have that both $Z^{d\tildeh^l_{t+1}}$ and $\partial Z^{d\tildeh^l_{t+1}} / \partial \hatZ^{\hatW^l \tildex^{l-1}_0}$ are polynomially bounded functions of $Z_0$. By Lemma~\ref{th:initial-W-vanish}, this first shows that $\scalarlim{\omega}_l \hatZ^{\transpose{{(\hatW^l)}} d\tildeh^l_{t+1}} = 0$, and thus gives $(vi)$ for layer $l$, leading to:
\begin{align*}
    Z^{d\tildex^{l-1}_{t+1}} = - \eta\scalarlim{\chi}_0 \gamma^b_{0, t+1, l-1} Z^{\tildex^{l-1}_0} - \eta \sum_{s=1}^t \scalarlim{\chi}_s \gamma^b_{s, t+1, l-1} Z^{x^{l-1}_s}
\end{align*}
Now, if $l-1 \geq 2$, $Z^{\tildex^{l-1}_0} = \sigma(\hatZ^{\hatW^{l-1} \tildex^{l-2}_0})$ and by Lemma~\ref{th:ipllr-forward-func-Z0-1}, we also have that $Z^{x^{l-1}_1} = \psi(\hatZ^{\hatW^{l-1} \tildex^{l-2}_0}, \hatZ^{\transpose{{(\hatW^{L})}} d\tildeh^{l}_0})$ with $\psi$ polynomially bounded because it is the case for $Z^{h^{l-1}_1}$ and $\sigma$ is polynomially bounded. In addition, by assumption, we have for any $s \in [2, t]$, $Z^{x^{l-1}_s} = \psi(\hatZ^{\hatW^{l-1} \tildex^{l-2}_0}, \hatZ^{\transpose{{(\hatW^{L})}} d\tildeh^{l}_0})$ with $\psi$ polynomially bounded since it is the case for $Z^{h^{l-1}_s}$ and $\sigma$ is polynomially bounded. As always the expectations defining the $\gamma^b$ are finite by Lemma~\ref{th:Z0-dist-moments} because the integrands are polynomially bounded functions of $Z_0$ as products of such variables by the results of the backward passes at times $t=0$ and $t=1$ and by the induction hypothesis. We thus get 
\begin{align*}
    Z^{d\tildex^{l-1}_{t+1}} = \psi(\hatZ^{\hatW^{l-1} \tildex^{l-2}_0}, \hatZ^{\transpose{{(\hatW^{L})}} d\tildeh^{l}_0})
\end{align*}
with $\psi$ polynomially bounded. On the other hand, if $l-1 = 1$, we have $Z^{\tildex^{1}_0} = \sigma(\hatZ^{U^1 \xi_0})$ and by Lemma~\ref{th:ipllr-forward-func-Z0-1}, we have $Z^{x^{1}_1} = \psi(\hatZ^{U^1 \xi_0}, \hatZ^{U^1 \xi_1}, \hatZ^{\transpose{{(\hatW^{2})}} d\tildeh^{2}_0})$ with $\psi$ polynomially bounded. In addition, by assumption we have for $s \in [2, t]$, $Z^{x^{1}_s} = \psi(\hatZ^{U^1 \xi_0}, \ldots, \hatZ^{U^1 \xi_{s}}, \hatZ^{\transpose{{(\hatW^{2})}} d\tildeh^{2}_0})$ with $\psi$ polynomially bounded.  Since $\sigma$ is also polynomially bounded, this gives
\begin{align*}
    Z^{d\tildex^{1}_{t+1}} = 
        \psi(\hatZ^{U^1 \xi_0}, \ldots, \hatZ^{U^1 \xi_t}, \hatZ^{\transpose{{(\hatW^{2})}} d\tildeh^{2}_0})
\end{align*}
$\psi$ polynomially bounded. Since $Z^{d\tildeh^{l-1}_{t+1}} = Z^{d\tildex^{l-1}_{t+1}} \sigma'(Z^{h^{l-1}_{t+1}})$, and by assumption $Z^{h^{l-1}_{t+1}} = \psi(\hatZ^{\hatW^{l-1} \tildex^{l-2}_0}, \hatZ^{\transpose{{(\hatW^{L})}} d\tildeh^{l}_0})$ if $l-1 \geq 2$, and otherwise $Z^{h^{1}_{t+1}} = \psi(\hatZ^{U^1 \xi_0}, \ldots, \hatZ^{U^1 \xi_{t+1}}, \hatZ^{\transpose{{(\hatW^{2})}} d\tildeh^{2}_0})$, we get
\begin{align*}
    Z^{d\tildeh^{l-1}_{t+1}} = 
    \begin{cases}
        \psi(\hatZ^{\hatW^{l-1} \tildex^{l-2}_0}, \hatZ^{\transpose{{(\hatW^{L})}} d\tildeh^{l}_0}) \text{ if } l-1 \geq 2 \\
        \psi(\hatZ^{U^1 \xi_0}, \hatZ^{U^1 \xi_1}, \ldots, \hatZ^{U^1 \xi_{t+1}}, \hatZ^{\transpose{{(\hatW^{2})}} d\tildeh^{2}_0}) \text{ if } l-1 = 1
    \end{cases}
\end{align*}
This gives claim $(ii)$ for layer $l-1$ and claim $(iii)$ for the case when $l-1=1$. Now, let $Z \in \{\hatZ^{\hatW^{l-1} \tildex^{l-2}_0}, \hatZ^{\transpose{{(\hatW^{L})}} d\tildeh^{l}_0}\}$. We have 
\begin{align*}
    \frac{\partial Z^{d\tildeh^{l-1}_{t+1}}}{\partial Z} =& \frac{\partial Z^{d\tildex^{l-1}_{t+1}}}{\partial Z} \sigma'(Z^{h^{l-1}_{t+1}}) + Z^{d\tildex^{l-1}_{t+1}} \frac{\partial Z^{h^{l-1}_{t+1}}}{\partial Z} \sigma''(Z^{h^{l-1}_{t+1}})
\end{align*}
where $Z^{h^{l-1}_{t+1}}$ and $Z^{d\tildex^{l-1}_{t+1}}$ are polynomially bounded functions of $Z_0$ by assumption and by the previous result on $Z^{d\tildex^{l-1}_{t+1}}$. Also by assumption, we have
\begin{align*}
    \frac{\partial Z^{h^{l-1}_{t+1}}}{\partial Z} = 
    \begin{cases}
        0 \ \text{ if } \ l-1=1 \text{ and } Z=\hatZ^{\hatW^2 \tildex^1_0} \\
        \psi(Z_0) \ \text{ otherwise}
    \end{cases}
\end{align*}
with $\psi$ polynomially bounded. In any case, $\partial Z^{h^{l-1}_{t+1}} / \partial Z$ is a polynomially bounded function of $Z_0$. On the other hand, we have
\begin{align*}
    \frac{\partial Z^{d\tildex^{l-1}_2}}{\partial Z} =& - \eta \scalarlim{\chi}_0 \gamma^b_{0, t+1, l-1} \frac{\partial Z^{\tildeh^{l-1}_0}}{\partial Z} \sigma'(Z^{\tildeh^{l-1}_0}) - \eta \sum_{s=1}^t \scalarlim{\chi}_t \gamma^b_{s, t+1, l-1} \frac{\partial Z^{h^{l-1}_s}}{\partial Z} \sigma'(Z^{h^{l-1}_s})
\end{align*}
For both possible values of $Z$, $\partial Z^{\tildeh^{l-1}_0} / \partial Z$ has an easy expression and is a polynomially bounded function of $Z_0$ (essentially because $\sigma$ and its derivatives are polynomially bounded). On the other hand, $\partial Z^{\tildeh^{l-1}_s} / \partial Z$ is a polynomially bounded function of $Z_0$ by assumption. $\sigma'$ is polynomially bounded, and the $\gamma^b$ are finite. We thus get that 
\begin{align*}
    \frac{\partial Z^{d\tildex^{l-1}_{t+1}}}{\partial Z} =  \psi(Z_0)
\end{align*}
with $\psi$ polynomially bounded, and thus:
\begin{align*}
    \frac{\partial Z^{d\tildeh^{l-1}_{t+1}}}{\partial Z} =  \psi(Z_0)
\end{align*}
with $\psi$ polynomially bounded, which proves claims $(iv)$ and $(v)$ at layer $l-1$ for time $t+1$. This thus concludes the induction on $l$, and with it the proof.
\end{proof}

\subsection{Main result}\label{sec:th-ipllr-initial-weight-vanish}

\begin{theorem}[Multiplications by the initial weight matrices vanish in IP-LLR for $t \geq 1$]\label{th:ipllr-initial-weight-vanish-t-geq-1}
Consider the IP-LLR parameterization with a positively $p$-homogeneous activation function, and $p \geq 2$. Then, for any $t \geq 1$, and for any $l \in [2, L]$, one has:
\begin{align*}
    \begin{cases}
        Z^{W^l(0)x^{l-1}_t} = \scalarlim{\omega}_l Z^{\hatW^lx^{l-1}_t} = 0 \\
        Z^{\transpose{{(W^l(0))}} d\tildeh^{l}_t} = \scalarlim{\omega}_l Z^{\transpose{{(\hatW^l)}} d\tildeh^{l}_t} = 0
    \end{cases}
\end{align*}
\end{theorem}

\begin{proof}
The result for $t=1$ has essentially been proved already early on in Corollary~\ref{th:z-mult-w-0-vanish} (which stems from Lemmas~\ref{th:ipllr-forward-func-Z0-1} and~\ref{th:ipllr-backward-func-Z0-1}). For $t=2$, the result has been proved in Lemmas~\ref{th:ipllr-forward-func-Z0-2} and~\ref{th:ipllr-backward-func-Z0-2}. Then we can prove the result for any $t \geq 2$ by induction using Lemmas~\ref{th:ipllr-induct-vansih-forward} and~\ref{th:ipllr-induct-vansih-backward}.
\end{proof}

\section{Expectations with ReLU}\label{app:exp-relu}
In all this section, we consider $Z \sim \mathcal{N}(0, \sigma^2)$, so that $Z = \sigma U$ where $U \sim \mathcal{N}(0,1)$.

\subsection{First moment}\label{app:exp-relu-first}
For $\phi(z) = \max(0,z)$ and $Z \sim \mathcal{N}(0, \sigma^2)$, we have 
\begin{align*}
    \mathbb{E}[\phi(Z)] &= \mathbb{E}[\phi(\sigma U)]
    = \frac{\sigma}{\sqrt{2 \pi}} \int_{0}^{\infty} u e^{-u^2/2} \mathrm{d}u 
    = \frac{\sigma}{\sqrt{2 \pi}}.
\end{align*}

\subsection{Second moment}\label{app:exp-relu-second}
For $\phi(z) = \max(0,z)$ and $Z \sim \mathcal{N}(0, \sigma^2)$, we have
\begin{align*}
    \mathbb{E}[\phi(Z)^2]
    = {1\over 2} \mathbb{E}[ Z^2] 
    = \frac{\sigma^2}{2}.
\end{align*}

\subsection{First forward pass moments}\label{app:relu-tilde-forward}
We have, for any $l \in [1,L]$, with $\sigma_0 := \sqrt{||\xi_0||^2 +1}$,
\begin{align*}
    &\mathbb{E}[\hatZ^{\tildeh^l_0}] = 0, \ \ \ \ \mathbb{E}[(\hatZ^{\tildeh^l_0})^2] = \frac{\sigma_0^2}{2^{l-1}} \\
    &\mathbb{E}[Z^{\tildex^l_0}] = \frac{\sigma_0}{\sqrt{2^l \pi}}, \ \ \ \   \mathbb{E}[(Z^{\tildex^l_0})^2] = \frac{\sigma_0^2}{2^{l}} 
\end{align*}

\subsection{First derivative moments}\label{app:exp-relu-derivative}
For $\phi(z) = \max(0,z)$, we have $\phi'(z) = \mathds{1}_{z \geq 0}$ almost everywhere, so for $Z \sim \mathcal{N}(0, \sigma^2)$, we have
\begin{align*}
    \mathbb{E}[\phi'(Z)] &= \mathbb{P}(Z\geq 0)=1/2.
\end{align*}
Note that since $\phi'(z)^p = \phi'(z)$ for any $p >0$, all the moments of $\phi'(Z)$ are equal to the first moment. 

\subsection{First backward pass moments}\label{app:relu-tilde-backward}
We have, for any $l \in [1,L]$, with,
\begin{align*}
    &\mathbb{E}[\hatZ^{d\tildex^l_0}] = 0, \ \ \ \ \mathbb{E}[(\hatZ^{d\tildex^l_0})^2] = \frac{1}{2^{L-l}} \\
    &\mathbb{E}[Z^{d\tildeh^l_0}] = 0, \ \ \ \   \mathbb{E}[(Z^{d\tildeh^l_0})^2] = \frac{1}{2^{L-l+1}}
\end{align*}

\bibliography{biblio}

\end{document}